%% file: main.tex
\documentclass[accepted]{uai2026} % after acceptance, for a revised version; 
\usepackage[american]{babel}
\usepackage{natbib} % has a nice set of citation styles and commands
\usepackage{mathtools} % amsmath with fixes and additions
\usepackage{booktabs} % commands to create good-looking tables
\usepackage{tikz} % nice language for creating drawings and diagrams

\usepackage{algorithm}
\usepackage[noend]{algpseudocode}
\algrenewcommand\algorithmicindent{0.7em} 

\usepackage{amsmath}

\usepackage{booktabs} % for professional tables
\usepackage{hyperref}

\usepackage{hyperref}
\hypersetup{
    colorlinks=true,
    linkcolor=blue,    % Internal links (sections, equations)
    citecolor=blue,    % Citations
    filecolor=magenta, 
    urlcolor=blue
}

\usepackage{amsmath}
\usepackage{amssymb}
\usepackage{mathtools}
\usepackage{amsthm}

\usepackage{tikz-cd}

\usepackage{etoc}
\usepackage[capitalize,noabbrev]{cleveref}

\theoremstyle{plain}
\usepackage[textsize=small]{todonotes}
\setuptodonotes{inline}

\input{header.tex}

\input{def.tex}

\title{Gaussian Approximation and Multiplier Bootstrap for Federated Linear Stochastic Approximation}

\author[1]{{Ilya Levin\thanks{These authors contributed equally to this work.}}}

\author[*1]{%
Maksim Shuklin
}

\author[2]{%
Eric Moulines
}

\author[3]{%
Paul Mangold
}

\author[1]{%
Sergey Samsonov
}

\affil[1]{%
    HSE University, Russia
}

\affil[2]{%
    Mohamed bin Zayed University of Artificial Intelligence, UAE
}

\affil[3]{%
    CMAP, CNRS, École polytechnique, Institut Polytechnique de Paris, Palaiseau, France % CMAP, CNRS, École Polytechnique, Institut Polytechnique de Paris, 91120 Palaiseau, France
}

\begin{document}
\maketitle

\begin{abstract}
 In this paper, we establish Berry-Esseen-type bounds for federated linear stochastic approximation (LSA). Our results provide the first federated Gaussian approximations for LSA that explicitly capture communication-computation trade-offs and heterogeneity-aware error terms, quantifying the effects of local step size, number of local updates, and heterogeneity on convergence rates. We present results for both (i) constant step size regime and (ii) decreasing step size with an increasing number of local iterations, recovering the recent rates of~\cite{bonnerjee2025sharp} as a special case. As a primary application of our results, we develop an online multiplier bootstrap procedure for inference on the last iterate, which avoids explicit estimation of the asymptotic covariance matrix, and obtain non-asymptotic validity guarantees for this procedure. % Our analysis relies on recent Gaussian approximation techniques for nonlinear statistics of independent random variables.
  % In this paper, we establish the Berry-Esseen type bounds for federated linear stochastic approximation (FedLSA) algorithm. Our results provide the first federated Gaussian approximations for LSA with communication–computation trade‑offs and heterogeneity‑aware error terms. We quantify the effects of local step size, number of local steps, and agent heterogeneity on the convergence rates. In particular, in the regime of constant number of local iterations our finding recover the recent rates obtained in [Bonnerjee et al, 2025]. Our analysis builds on the Gaussian approximation results for nonlinear statistics of independent random variables.
\end{abstract}

\section{Introduction}
\label{sec:intro}
\input{introduction}

\section{Related Work}
\label{sec:related_work}

\input{related_work}

\section{Federated LSA Setting}
\label{sec:sgd_decreasing_step_size}
\input{fedlsa_setting}

\section{Gaussian Approximation}%  for FedLSA}
\label{sec:normal_approximation_fedlsa}
\input{normal_approx}

\section{Numerical experiments}
\label{sec:numeric_experiments}

\input{numerical_section}

\section{Conclusion}
\label{sec:conslusion}
\input{conclusion}

% \clearpage
% \newpage
% References
\bibliography{uai2026-template}

\newpage
\clearpage
\appendix
\onecolumn
\section*{Appendix}
\etocsettocstyle{\subsection*{Table of Contents}}{}
\etocsetnexttocdepth{subsubsection}
\tableofcontents

% \title{Supplementary Material}
% \maketitle

\section{Proof for moment bounds}
\label{sec:sgd_decreasing_step_size_proofs}
\input{appendix/fedlsa_decreasing}
\input{appendix/fedlsa_boot}

\section{Experimental Results}
\label{sec:experiments}
\input{appendix/experiments}

\section{Technical results}
\label{sec:add_lemmas}
\input{appendix/add_lemmas}
\label{sec:technical-lemmas}
\input{appendix/technical_lemmas}

\end{document}

%% file: header.tex
\usepackage{graphicx}	% Including figure files
\usepackage[lofdepth,lotdepth]{subfig}
\usepackage{amsmath,mathtools,amsthm}	% Advanced maths commands
\usepackage{amssymb,dsfont,bbm}	% Extra maths symbols
\usepackage{xcolor}  % colors
\usepackage{comment}
\usepackage{bm}

\usepackage{xargs}
\usepackage{fancyhdr}
\usepackage{natbib}

\usepackage{setspace}
\usepackage{lastpage}
\usepackage{upgreek}

\usepackage{mathtools}	% Advanced maths commands
\usepackage{enumitem}
\usepackage{mathrsfs}
\usepackage{aliascnt}
\usepackage{cleveref}

\usepackage{booktabs}
\usepackage{subcaption}

\usepackage{tabularx}
\newcolumntype{Y}{>{\centering\arraybackslash}X}

%% file: def.tex
\newcommandx{\norm}[2][2=]{\Vert#1 \Vert_{{#2}}}
\newcommandx{\bnorm}[2][2=]{\Big\Vert#1 \Big\Vert_{{#2}}}

\newcommandx{\thetalim}[1][1=]{\theta^{#1}_\star}
\newcommandx{\noisecov}[1][1=]{\Sigma^{#1}_{\varepsilon}}
\newcommandx{\noisecovA}[1][1=]{\Sigma^{#1}_{\funcAw}}
\newcommandx{\noisecovst}[1][1=]{\Sigma^{#1}_{*}}

\def\step{\eta}
\def\initstep{\eta}
\def\msz{\mathsf{Z}}
\def\mcz{\mathcal{Z}}
\def\Var{\mathrm{Var}}

\newcommand{\supnorm}[1]{\norm{ #1 }[\infty]}
\def\supconsteps{\supnorm{\funnoisew}}

\def\varhet{\bar{\sigma}_{\textrm{het}}}
\def\varbound{\bar{\sigma}_{\varepsilon}}
\def\varboundA{\bar{\sigma}_{\funcAw}}
\def\varnoise{\bar{\sigma}_{\textrm{noise}}}
\newcommand{\sumHT}[1]{{S}_{#1}}

\def\noisecovavgst{\Sigma_{*}^{\sf avg}}

\def\firsthgty{\zeta_1}
\def\secondhgty{\zeta_2}

\def\funcAw{\mathbf{A}}

\def\funcbw{\mathbf{b}}
\newcommand{\funcb}[1]{\funcbw(#1)}
\newcommand{\zfuncA}[2][1=]{\funcAw^{#1}(#2)}
\newcommand{\zfuncb}[2][1=]{\funcbw^{#1}(#2)}
\newcommandx{\zmfuncA}[2][1=]{\widetilde{\funcAw}^{#1}(#2)}
\newcommandx{\zmfuncAw}[1][1=]{\widetilde{\funcAw}_{#1}}
\newcommandx{\zmfuncbw}[1][1=]{\widetilde{\funcbw}_{#1}}
\newcommandx{\zmfuncb}[2][1=]{\widetilde{\funcbw}^{#1}(#2)}
\def\funnoisew{\varepsilon}
\newcommandx{\funcnoise}[2][1=]{\funnoisew^{#1}(#2)}
\newcommandx{\funcnoiset}[1][1=]{\tilde{\funnoisew}^{#1}}
\newcommandx{\bfuncnoise}[1]{\bar{\funnoisew}(#1)}
\newcommandx{\funcnoiseth}[3][1=]{\funnoisew^{#1}(#2,#3)}
\newcommandx{\bfuncnoiseth}[2]{\bar{\funnoisew}(#1,#2)}

\newcommandx{\funcct}[2][1=]{\funcctilde^{#1}(#2)}
\def\qcond{\kappa_{Q}}
\def\Auxconst{\operatorname{c}}

\newcommandx{\bA}[1][1=]{\bar{\mathbf{A}}^{#1}}
\newcommandx{\barb}[1][1=]{\bar{\mathbf{b}}^{#1}}

\def\nagent{N}

\def\nlupdates{H}

\def\initnlupdates{H}

\def\varphiavg{\varphi^{\sf avg}}
\def\tauavg{\tau^{\sf avg}}
\def\rhoavg{\rho^{\sf avg}}
\def\Gammaavg{\Gamma^{\sf avg}}
\def\Gammaavgboot{\Gamma^{\boot, \sf avg}}
\newcommandx{\Gammaavgpert}[1]{\Gamma^{\sf{avg}, (#1)}}

\def\Deltaavg{\Delta^{\sf avg}}

\newcommandx{\Deltaavgpert}[1]{\Delta^{\sf{avg}, (#1)}}

\def\Epsavg{\mathcal{E}^{\sf avg}}

\def\Epsavgboot{\mathcal{E}^{\boot, \sf avg}}
\def\Gavg{G^{\sf avg}}

\def\boot{\mathsf{b}}
\newcommand{\PPb}{\mathbb{P}^\boot}
\newcommand{\PEb}{\mathbb{E}^\boot}

\def\Conv{\mathsf{C}}
\def\ProdB{\Gamma}

\newcommandx{\ProdBatc}[2]{\ProdB_{#1}^{#2}}
\newcommandx{\avgProdB}[2]{\bar{\Gamma}_{#1}^{#2}}
\newcommandx{\avgProdDetB}[2]{\bar{\Gamma}^{#2}_{#1}}
\newcommandx{\RavgProdDetB}[2]{\bar{\Theta}^{#2}_{#1}}
\newcommandx{\avgbias}[1]{\bar{\rho}_{#1}}
\newcommandx{\avgflbias}[1]{\bar{\tau}_{#1}}
\newcommandx{\bbias}[2]{\rho_{#2}}
\newcommandx{\vbias}[1]{\tilde{v}_{#1}}
\newcommandx{\avgfl}[1]{\bar{\varphi}_{#1}}
\def\kolmogorov{\mathsf{d}_{\Conv}}
\def\kolmogorovboot{\mathsf{d}_{\Conv}^\boot}

\newcommand{\bConst}[1]{\operatorname{C}_{{\bf #1}}}

\newcommand{\PP}{\mathbb{P}}

\def\Id{\mathrm{I}}
\def\rmd{\mathrm{d}}
\def\rme{\mathrm{e}}
\def\rset{\mathbb{R}}

\newcommand{\continuation}{??}

\def\eqsp{\,}

\def\PE{\mathbb{E}}

\newcommandx{\CPE}[3][1=]{\mathsf{E}_{#1}\bigl[  #2 | #3 \bigr]}

\def\nset{\mathbb{N}}
\def\rset{\mathbb{R}}

\DeclareMathAlphabet\mathbfcal{OMS}{cmsy}{b}{n} % Bold Mathcal font

\crefname{section}{Section}{Sections}
\crefname{appendix}{Appendix}{Appendices}

\newtheorem{theorem}{Theorem}
\crefname{theorem}{theorem}{theorems}
\Crefname{theorem}{Theorem}{Theorems}

\newtheorem{lemma}{Lemma}
\crefname{lemma}{lemma}{lemmas}
\Crefname{lemma}{Lemma}{Lemmas}

\newtheorem{corollary}{Corollary}
\crefname{corollary}{corollary}{corollaries}
\Crefname{corollary}{Corollary}{Corollaries}

\newtheorem{proposition}{Proposition}
\crefname{proposition}{proposition}{propositions}
\Crefname{proposition}{Proposition}{Propositions}

\newaliascnt{definition}{theorem}

\aliascntresetthe{definition}
\crefname{definition}{definition}{definitions}
\Crefname{Definition}{Definition}{Definitions}

\newaliascnt{definitionProposition}{theorem}

\aliascntresetthe{definitionProposition}
\crefname{Proposition and Definition}{Proposition and Definition}{Proposition and Definition}
\Crefname{Proposition and Definition}{Proposition and Definition}{Proposition and Definition}

\newaliascnt{remark}{theorem}

\aliascntresetthe{remark}
\crefname{remark}{remark}{remarks}
\Crefname{Remark}{Remark}{Remarks}

\crefname{example}{example}{examples}
\Crefname{Example}{Example}{Examples}

\crefname{figure}{figure}{figures}
\Crefname{Figure}{Figure}{Figures}

\newtheorem{assumption}{\textbf{A}\hspace{-3pt}}
\Crefname{assumption}{\textbf{A}\hspace{-3pt}}{\textbf{A}\hspace{-3pt}}
\crefname{assumption}{\textbf{A}}{\textbf{A}}

\newtheorem{assumprime}{\textbf{A'}\hspace{-1pt}}
\Crefname{assumprime}{\textbf{A'}\hspace{-1pt}}{\textbf{A'}\hspace{-1pt}}
\crefname{assumprime}{\textbf{A'}}{\textbf{A'}}

\newtheorem{stepsize}{\textbf{S}\hspace{-3pt}}
\Crefname{stepsize}{\textbf{S}\hspace{-3pt}}{\textbf{S}\hspace{-3pt}}
\crefname{stepsize}{\textbf{S}}{\textbf{S}}

\Crefname{assumptionC}{\textbf{C}\hspace{-3pt}}{\textbf{C}\hspace{-3pt}}
\crefname{assumptionC}{\textbf{C}}{\textbf{C}}

\def\Zset{\mathsf{Z}}

\newcommandx{\metricd}[1][1=]{\mathsf{d}_{#1}}

\def\mrl\mathrm{L}

\def\iid{i.i.d.}
\def\gauss{\mathrm{N}}

\newcommand{\term}[1]{\boldsymbol{\mathrm{#1}}}

%% file: introduction.tex
% Something about the great importance of our bounds.
Federated learning \citep{mcmahan2017communication, kairouz2021advances} enables collaborative training of machine learning models across distributed data sources. In this context, the Federated Averaging (FedAvg) algorithm \citep{mcmahan2017communication} is considered the most fundamental algorithm in federated learning. It reduces communication complexity by using local updates with periodic aggregation \citep{stich2018local, khaled2020tighter, woodworth2020minibatch}. Most of the federated learning literature adopts an optimization perspective, establishing convergence rates and communication complexity for variants of the FedAvg optimization scheme \citep{khaled2020tighter, karimireddy2020scaffold, li2020federated, mishchenko2022proxskip, ogier2022flamby}. At the same time, an important question concerns inference procedures that aim to provide confidence intervals for parameters of interest or their functionals. The most challenging aspect of theoretical grounding in such algorithms is the Gaussian approximation (GAR) for the underlying model parameters. While there is a growing number of results studying the rates of Gaussian approximation for SGD \citep{shao2022berry, sheshukova2025gaussian} and other stochastic approximation algorithms \citep{samsonov2024gaussian, wu2024statistical, wu2025uncertainty, butyrin2025improved}, the number of related contributions to the analysis of federated learning algorithms remains limited.

%In the centralized setting, significant effort has been devoted to understanding the limiting behavior of SGD, establishing central limit theorems for the algorithm's iterates \citep{polyak1992acceleration, kushner2003stochastic, benveniste2012adaptive}, and high-probability convergence rates \citep{rakhlin2011making, sadiev2023high}. More recently, non-asymptotic convergence rates to a Gaussian limit, known as \emph{Berry-Esseen} type bounds, have been derived for SGD \citep{samsonov2024gaussian, sheshukova2025gaussian}. These results are particularly important for deriving valid statistical inference and obtaining confidence intervals through multiplier bootstrap methods \citep{fang2018online, fang2019scalable, zhong2023online}.

To our knowledge, the few existing works on statistical inference for FedAvg \citep{li2022statistical,gu2024statistical} establish only \emph{asymptotic} convergence to a limiting Gaussian distribution. A notable exception is the recent work of \citet{bonnerjee2025sharp}, which establishes GAR rates for federated averaging with a \emph{constant} number of local updates. Notably, they apply their rates to an inference procedure based on the multiplier bootstrap approach \citep{fang2018online} and conjecture, without proof, a rate of approximation for this procedure for the Polyak-Ruppert averaged iterates. This remains an important gap in the federated literature, which lacks non-asymptotic methods with formal guarantees for quantifying uncertainty.

In this paper, we aim to close this gap and develop GAR results for federated averaging with an \emph{increasing number of local iterations}, along with an inference procedure based on the multiplier bootstrap with non-asymptotic validity. Our main contributions are as follows:

%In this paper, we develop a new statistical inference result for FedAvg for the federated linear stochastic approximation problem (see detailed setting in \Cref{sec:sgd_decreasing_step_size}). We provide convergence rates, in terms of convex distance, of FedAvg toward a limiting normal distribution, whose variance we characterize. Importantly, we derive these rates for decreasing step size and an \emph{increasing number of local iterations}. Based on these results, we construct a multiplier bootstrap procedure and establish its validity, confirming the possibility of statistical inference via bootstrap conjectured by \citet{bonnerjee2025sharp}. 

\begin{itemize}[leftmargin=*]
    \item We present new high-order moment bounds for the error of the FedLSA algorithm with decreasing step size and increasing local iterations, capturing the effects of heterogeneity bias and noisy updates. To our knowledge, these are the first moment bounds of order $p$ with $p \geq 2$ available in the literature.
    \item We present the first Gaussian approximation bound for the last iterate of FedLSA (see \Cref{thm:gauss_approx_last_iter}). Our analysis establishes a rate of order up to $t^{-1/2}$ in terms of convex distance between probability distributions (see \Cref{sec:normal_approximation_fedlsa} for precise definitions). Our bounds are \emph{fully non-asymptotic} and capture the interplay between the number of local iterations and the choice of step size. 
    \item  We present an inference procedure to approximate the error of the last iterate of FedLSA using the multiplier bootstrap \citep{fang2018online}.  {Specifically, our bounds show that this distribution can be approximated at a rate of up to $t^{-1/2}$ in convex distance, bypassing the Gaussian approximation step involving the asymptotic covariance matrix of the last iterate, $\Sigma_{\infty}$; see \Cref{sec:bootstrap}.} This underscores the distinction between our proposed procedure and various plug-in or batch-mean techniques for constructing confidence intervals, which seek to directly approximate $\Sigma_{\infty}$. Therefore, we provide \emph{the first theoretically grounded method for assessing uncertainty in federated learning}, confirming the hypothesis of \cite{bonnerjee2025sharp}.
\end{itemize}
\paragraph{Notations and definitions.} For a matrix $A \in \rset^{d \times d}$, we denote its operator norm by $\norm{A}$. Let $\nagent$ be the number of agents; we use the notation $\smash{\PE_c[a_c] = \nagent^{-1}\sum_{c=1}^\nagent a_c}$ for the average over different agents. For a symmetric positive-definite matrix $Q = Q^\top \succ 0$, $Q \in \rset^{d \times d}$, and $x \in \rset^{d}$, we define the corresponding norm $\|x\|_Q = \sqrt{x^\top Q x}$, and the respective matrix $Q$-norm of a matrix $B \in \rset^{d \times d}$ by $\norm{B}[Q] = \sup_{x \neq 0} \norm{Bx}[Q]/\norm{x}[Q]$. For sequences $a_n$ and $b_n$, we write $a_n \lesssim_{\log_n} b_n$ if there exist $c, \alpha > 0$ (not depending on $n$) such that $a_n \leq c (1+\log n)^{\alpha} b_n$. In this text, the following abbreviations are used: "w.r.t." stands for "with respect to," "\iid" for "independent and identically distributed," and "GAR" for "Gaussian Approximation."

%% file: related_work.tex
% \paragraph{Gaussian approximation.}

% \paragraph{Statistical inference.}

% \begin{itemize}
%     \item Inference algorithms: asymptotic covariance matrix estimation (iid + Markov noise), and bootstrap.
% \end{itemize}

\paragraph{Federated stochastic approximation.}
Most of the federated learning literature has aimed to establish convergence rates for federated averaging under various assumptions \citep{mcmahan2017communication, khaled2020tighter, wang2024unreasonable, mangold2025refined, glasgow2022sharp, patel2023still, stich2018local, pmlr-v130-gorbunov21a}, along with mechanisms to mitigate the impact of heterogeneity using control variates \citep{karimireddy2020scaffold, mishchenko2022proxskip, mangold2025scaffold, luo2025revisiting}, local proximal updates \citep{li2020federated, grudzien2023can}, or modified objectives \citep{wang2021novel, mitra2021linear}.
% A substantial part of the federated learning literature aimed to characterize the communication-computation trade-off of federated averaging with homogeneous clients \citep{}. 
% under data heterogeneity, as well as algorithmic mechanisms to mitigate the client drift (e.g., control variates, proximal corrections, and related ideas) \citep{khaled2020tighter,karimireddy2020scaffold,li2020federated,mishchenko2022proxskip}.

Finite-sample properties of federated SA were first studied by \citet{doan2020local,wai2020convergence,liu2023distributed} without local training. \citet{khodadadian2022federated,wang2024federated} subsequently examined federated SA with local updates. More recently, \citet{mangold2024scafflsa,zhu2025achieving} analyzed federated linear SA with local training and bias correction. In this paper, we extend this line of work by addressing the decreasing-step-size regime alongside an increasing number of local iterations, and by obtaining $p$-moment bounds.

\paragraph{Gaussian approximation for SGD and SA.}
Classical asymptotic normality results for stochastic approximation and SGD algorithms date back to foundational papers \citep{polyak1992acceleration, kushner2003stochastic, benveniste2012adaptive}. Gaussian approximation and Berry-Esseen type bounds for nonlinear statistics of i.i.d. random variables were extensively studied in \cite{shorack2000probability, shao2022berry}. These methods were applied to Polyak–Ruppert averaged linear stochastic approximation in \cite{samsonov2024gaussian, wu2024statistical}, and to SGD in \cite{sheshukova2025gaussian}, along with non-asymptotic multiplier bootstrap guarantees. The case of stochastic approximation with Markovian noise was recently studied in \cite{wu2025uncertainty, samsonov2025statistical, liu2025central}. In the decentralized setting, Gaussian approximation results are much more limited: \citet{bonnerjee2025sharp} established Gaussian approximation results for local SGD, including a Berry-Esseen bound for Polyak–Ruppert averaged SA and outlined the corresponding results for the last iterate, both in the regime of a \emph{constant} number of local training steps. In contrast, our results track the heterogeneity-aware error terms that arise in the more general regime of an increasing number of local training steps.

\paragraph{Statistical inference.} Classical methods for constructing confidence sets for optimal parameters are based on direct estimation of the asymptotic covariance matrix using plug-in or batch-mean methods \cite{chen2020aos,chen2021statistical,chen2022online}. For federated learning, \citet{li2022statistical} studied asymptotic versions of the functional CLT and their applications to statistical estimation via plug-in and random scaling methods. \citet{gu2024statistical} developed similar inference based on batch-mean methods. Both papers remain asymptotic in their analysis. Alternative approaches use bootstrap and resampling schemes \cite{fang2018online,zhong2023online}, with non-asymptotic analysis provided in \cite{samsonov2024gaussian,samsonov2025statistical}. I this paper, we aim to generalize the fully non-asymptotic results obtained in these papers for the setting of Federated LSA.

%% file: fedlsa_setting.tex
In this paper, we consider federated linear stochastic approximation (LSA) with $\nagent$ agents, who aim to collaboratively solve the following linear system:
\begin{equation}
 \label{eq:lsa_eq_def}
 \bA \thetalim = \barb \eqsp, ~\text{with} \eqsp~ \bA = \frac{1}{N} \sum_{c = 1}^N \bA[c], \eqsp \barb = \frac{1}{N} \sum \limits_{c = 1}^N \barb[c],
\end{equation}
where, for $c \in \{1, \dots, \nagent\}$, $\bA[c] \in \rset^{d \times d}$ and $ \barb[c] \in \rset^d$ are distributed among multiple agents.
Furthermore, we assume that the system matrices $\bA[c]$ and vectors $\barb[c]$ are unknown, and agent $c$ only has access to estimates of $\bA[c]$ and $\barb[c]$.

\paragraph{Stochastic estimation.}
To estimate $\bA[c]$ and $\barb[c]$, agent $c$ obtains independent samples $Z_1^c, \dots, Z_T^c$ from a distribution $\pi_c$, and computes, for $t \in \{1, \dots, T\}$, the estimates $\zfuncA[c]{Z_t^c}$ and $\zfuncb[c]{Z_t^c}$ such that
\begin{equation}
\PE_{Z \sim \pi_c}[\zfuncA[c]{Z}] = \bA[c]
\eqsp,
\quad
\PE_{Z \sim \pi_c}[\zfuncb[c]{Z}] = \barb[c]
\eqsp.
\end{equation}
We introduce the notation for the errors of observations
\begin{equation}
\label{eq:def-tilde-Ab-epsilon}
\begin{aligned}
    \zmfuncA[c]{z} &= \zfuncA[c]{z} - \bA[c] \eqsp, \quad \zmfuncb[c]{z} = \zfuncb[c]{z} - \barb[c] \eqsp, \\
    &\quad \funcnoise[c]{\theta, z} = \zmfuncA[c]{z} \theta - \zmfuncb[c]{z} \eqsp.
\end{aligned}
\end{equation}
We assume that the $\zfuncA[c]{\cdot}$ satisfy the following assumption, which is classical in the LSA literature \citep{samsonov2024improved, mangold2024scafflsa}.
\begin{assumption}[p]
\label{ass: linear-decay}
For each agent $c \in [\nagent]$, the matrix $-\bA[c]$ is Hurwitz, and there exist constants $a > 0$ and $\eta_{\infty,p} > 0$ such that 
$\eta_{\infty,p} a \leq \frac{1}{2}$ and for all $0 < \eta < \eta_{\infty,p}$ and $u \in \mathbb{R}^d$,
\[
\PE^{1/p}
\bigl[
\|(\Id - \eta \zfuncA[c]{Z})u\|^p
\bigr]
\leq
(1 - 2\eta a)\|u\| \eqsp,
\]
for $Z \sim \pi_c$ with a distribution $\pi_c$ over $\mathsf{Z}$.
\end{assumption}
% \begin{assumption}[p]
% \label{ass: linear-decay}
%     There exists some constant $a > 0$ and $\eta_{\infty,p} > 0$, such that $\eta_{\infty,p} a \leq \frac{1}{2}$ and for any $0 < \eta < \eta_{\infty,p}$, $c = 1 \ldots N$ and $u \in \mathbb{R}^d$ it holds 
%     \[
%     \PE^{1/p} [ \| (\Id - \eta A^c (Z)) u \|^p ] \leq (1 - \eta a) \| u \|\eqsp, 
%     \]
%     for $Z \sim \pi_c$ with a distribution $\pi_c$ over $\mathsf{Z}$.
% \end{assumption}
% We note that Assumption~\Cref{ass: linear-decay} is standard in the analysis of linear stochastic approximation and has been adopted in [..., ..., and ...], as it is satisfied, for instance, in the case of TD learning.
This assumption ensures the existence and uniqueness of $\thetalim$ and $\thetalim[c]$, respectively defined as the solutions of $\bA \thetalim = \barb$ and $\bA[c] \thetalim[c] = \barb[c]$.
Given the existence of $\thetalim[c]$, we will sometimes omit the dependence on $\theta$ in $\funcnoise[c]{\thetalim[c], z}$, denoting
\begin{equation}
\funcnoise[c]{z}
=
\funcnoise[c]{\thetalim[c], z}
=
\zmfuncA[c]{z} \thetalim[c] - \zmfuncb[c]{z}
\eqsp.
\end{equation}
To handle stochasticity in the estimates of $\bA[c]$ and $\barb[c]$, we introduce the following quantities:
\begin{align*}
    \supconsteps = \max_{c \in [\nagent]} \sup_{z \in \Zset} \norm{\funcnoise[c]{z}} \eqsp, \quad \bConst{A} = \max_{c \in [\nagent]} \sup_{z \in \Zset} \norm{\zfuncA[c]{z}} \eqsp,
\end{align*}
which measure observation noise and the stochastic part of the LSA error. Now, we can introduce assumptions about the underlying random process and the stability of the system.
\begin{assumption}
\label{assum:noise-level-flsa}
We assume that $(Z_{k}^c)_{k \in \nset, c\in [\nagent]}$ are \iid\ random variables with values in $(\msz,\mcz)$ and distribution $\pi_c$ satisfying $\PE_{\pi_c}[\zfuncA[c]{Z_{k}^c}] = \bA[c]$ and $\PE_{\pi_c}[ \funcb{Z_{k}^c} ] = \barb[c]$. Moreover, the uniform boundness condition holds
\begin{align*}
    \supconsteps &= \max_{c \in [\nagent]} \sup_{z \in \Zset} \norm{\funcnoise[c]{z}} < \infty \eqsp,\\
    \bConst{A} &= \max_{c\in [\nagent]}\sup_{z\in\Zset}\norm{\zfuncA[c]{z}} \vee \norm{\zmfuncA[c]{z}} < \infty \eqsp.
\end{align*}
Moreover, we assume that the smallest eigenvalue of the average noise covariance matrix $\noisecovavgst =  \nagent^{-1}\sum_{c=1}^\nagent \PE[\funcnoiseth[c]{\thetalim}{Z_{c}} (\funcnoiseth[c]{\thetalim}{Z_{c}})^{\top}]$ is bounded away from 0, that is,
\begin{align*}
    \lambda_{\min}(\noisecovavgst) > 0 \eqsp.
\end{align*}
\end{assumption}
\paragraph{Heterogeneity.}
For federated LSA, a crucial challenge is to ensure that we obtain a solution to the averaged system \eqref{eq:lsa_eq_def}, even though $\bA[c]$ and $\barb[c]$ differ across agents. To assess the impact of heterogeneity, we measure the differences between $\bA[c]$ and $\barb[c]$ using the following two quantities.
\begin{equation}
\label{eq:def-heterogeneity}
\firsthgty^2 \!=\! \frac{1}{\nagent} \sum_{c=1}^\nagent \norm{ \bA[c](\thetalim[c] \!- \thetalim)}^2
\eqsp,
~
\secondhgty^2 \!=\! \frac{1}{\nagent} \sum_{c=1}^\nagent \norm{ \bA[c] \!- \bA }^2
\eqsp,
\end{equation}
where, $\firsthgty$ measures the heterogeneity of the solutions, while $\secondhgty$ measures the heterogeneity of the matrices $\bA[c]$. Under the standard assumptions \Cref{ass: linear-decay} and \Cref{assum:noise-level-flsa}, these two quantities are always defined and bounded. They play a key role in analyzing the impact of heterogeneity.

\paragraph{Federated LSA algorithm.} 
To solve the federated LSA problem~\eqref{eq:lsa_eq_def}, we use the \textsc{FedLSA} algorithm~\cite{doan2020local,mangold2024scafflsa}. This method leverages local training on each agent to reduce communication overhead. At each communication round $t$, the server broadcasts the current global model $\theta_t$ to all agents. Each agent $c$ then performs $H_t$ local stochastic updates using its own data, starting from $\theta_{t,0}^c = \theta_{t-1}$ and, for $h = 1, \dots, H_t$, performing the local update
\begin{equation}
\label{eq:fedlsa_iter_def_main}
\theta^c_{t,h} = \theta^c_{t,h-1} - \eta_t \big( \mathbf{A}^c(Z^c_{t,h}) \, \theta^c_{t,h-1} - \mathbf{b}^c(Z^c_{t,h}) \big) \eqsp, % \quad h = 1, \dots, H_t,
\end{equation}
where $\eta_t$ is the learning rate at round $t$. After completing the local updates, the server aggregates the agent models by averaging: $\theta_{t+1} = \frac{1}{N}\sum_{c=1}^N \theta^c_{t,H_t}$. The full procedure is detailed in Algorithm~\ref{alg:fedlsa}. In contrast to previous work, we consider a decreasing sequence of step sizes along with an increasing sequence of local step numbers. 
% \begin{algorithm}[t]
% \caption{FedSGD}
% \label{alg:fedlsa}
% \begin{algorithmic}[1]

% % \State \textbf{Input:} stepsizes $\{\eta_t\}_{t=1}^T$, 
% % initial model $\theta_0 \in \mathbb{R}^d$, 
% % number of rounds $T$, 
% % number of agents $N$, 
% % local steps $H_t$
% \For{$t = 1$ to $T$}
%     \State
%     $\theta_{t} \leftarrow \theta_{t - 1} - \eta_t \nabla f(\theta_t)$
    
% \EndFor
% \State \textbf{Return:} final parameter $\theta_{T}$ 

% \end{algorithmic}
% \end{algorithm}
\begin{algorithm}[t]
\caption{FedLSA}
\label{alg:fedlsa}
\begin{algorithmic}[1]

\State \textbf{Input:} stepsizes $\{\eta_t\}_{t=1}^T$, 
initial model $\theta_0 \in \mathbb{R}^d$, 
number of rounds $T$, 
number of agents $N$, 
local steps $H_t$
\For{$t = 1$ to $T$}
   % \State $\theta_{t,0} \gets \theta_{t-1}$
    
    \For{$c = 1$ to $N$}
        \State Initialize current round $\theta^c_{t,0} \gets \theta_{t-1}$
        
        \For{$h = 1$ to $H_t$}
            \State Receive data noise $Z^c_{t,h}$
            \State Compute: \mbox{$g_{t,h}^c \leftarrow \mathbf{A}^c(Z^c_{t,h})\theta^c_{t,h-1}-\mathbf{b}^c(Z^c_{t,h})$}
            \State Update local parameter: $\theta_{t,h}^c \leftarrow \theta_{t,h-1}^c - \eta_t g_{t,h}^c$
            % \vspace{-1.5em}
            % % \State Update $\theta^c_{t,h}$ using
            % %\State
            % { \small
            % \begin{equation*}
            % \qquad \quad {\theta^c_{t,h}\gets\theta^c_{t,h-1}-\eta_t()}
            % \end{equation*}
            % \vspace{-1.5em}
            % }
            % \State $\theta^c_{t,h} \gets 
            % \theta^c_{t,h - 1}
            % - \eta_t\big(
            % \mathbf{A}^c(Z^c_{t,h})
            % \theta^c_{t,h-1}
            % - \mathbf{b}^c(Z^c_{t,h})
            % \big)$
        \EndFor
    \EndFor
    
    \State Update global parameter $
    \theta_t \gets
    \frac{1}{N}
    \sum_{c=1}^{N}
    \theta^c_{t,H_t}$
    
\EndFor
\State \textbf{Return:} final parameter $\theta_{T}$ 

\end{algorithmic}
\end{algorithm}

\section{Moment bounds for FedLSA}
We first derive bounds on the moments of the error in FedLSA. First, we bound the second moment of the FedLSA error when using decreasing step sizes and increasing numbers of local iterations. This result will be used in \Cref{sec:normal_approximation_fedlsa} to obtain rates for Gaussian approximation in FedLSA. Next, we extend the analysis to high-order moment bounds that will be used in the analysis of the federated multiplier bootstrap procedure presented in \Cref{sec:bootstrap}.

\subsection{Error Decomposition for FedLSA}
\label{sec:decomp-error-fedlsa}
To derive our moment bounds, we extend the error decomposition framework from \cite{mangold2024scafflsa} to accommodate non-constant step sizes and numbers of local iterations. We define the following quantities:
\begin{align}
%\label{eq:def-gamma-prod}
\nonumber
    \Gamma^c_{t, l:r} = \prod \limits_{h = l}^r (\Id - &\eta_{t, h} \zfuncA[c]{Z_{t,h}^c})
    \eqsp, \eqsp\eqsp
%     \qquad
%     \Gamma^c_{t} = \Gamma^c_{t,1:\nlupdates_t}
%     \eqsp,\\
%     &\qquad
% \textstyle
    \Gammaavg_t = \frac{1}{\nagent} \sum_{c=1}^\nagent \Gamma^c_{t,1:\nlupdates_t}
    \eqsp.
\end{align}
For brevity, we will denote $\Gamma^c_{t} = \Gamma^c_{t,1:\nlupdates_t}$.
We also define the noisy part of these matrices as
\begin{align*} \textstyle
    \Delta_{s, t} = 
    \prod_{i = s + 1}^t \Gammaavg_i 
    - \mathbb{E} \big[\prod_{i = s + 1}^t \Gammaavg_i \big] \eqsp,
\end{align*}
which captures the fluctuations of the matrix product $\Gammaavg_i$. Using these notations, we obtain the following error decomposition for the last iterate:
\begin{align}
\label{eq:decomp-theta-t-one-iter-main}
\theta_t - \thetalim = \tilde{\theta}_t^{\mathrm{(tr)}} + \tilde{\theta}_t^{\mathrm{(bi, bi)}} + \tilde{\theta}_t^{\mathrm{(fl, bi)}} + \tilde{\theta}_t^{\mathrm{(fl)}} \eqsp,
\end{align}
where we introduce the following notations
\begin{align}
    \label{eq:def-theta-tr}
    \tilde{\theta}_t^{\mathrm{(tr)}}
    & = \textstyle
    \prod_{s = 1}^t \Gammaavg_s \left\{ \theta_0 - \thetalim \right\}
    \eqsp, 
    \\
\label{eq:def-theta-bibi}
    \tilde{\theta}_t^{\mathrm{(bi, bi)}} 
    & = \textstyle
    \sum_{s = 1}^t \mathbb{E} \Big[\prod_{i = s + 1}^t \Gammaavg_i \Big] \rhoavg_{s} 
    \eqsp,
    \\
\label{eq:def-theta-flbi}
    \tilde{\theta}_t^{\mathrm{(fl, bi)}}
    & = \textstyle
    \sum_{s = 1}^t \prod_{i = s + 1}^t \Gammaavg_i \tauavg_{s} + \Delta_{s, t} \rhoavg_{s} 
    \eqsp,
    \\
\label{eq:def-theta-fl}
    \tilde{\theta}_t^{\mathrm{(fl)}} 
    & = \textstyle
    \sum_{s = 1}^t \big\{  \prod_{i = s + 1}^t \Gammaavg_i \cdot \varphiavg_{s} \big\}
    \eqsp.
\end{align}
The first term, $\smash{\tilde{\theta}_t^{\mathrm{(tr)}}}$, characterizes the rate at which the initial error decays. The terms $\smash{\tilde{\theta}_t^{\mathrm{(bi, bi)}}}$ and $\smash{\tilde{\theta}_t^{\mathrm{(fl, bi)}}}$ capture the bias and fluctuations arising from statistical heterogeneity. In the special case of homogeneous agents (i.e., $\smash{\bA[c] = \bA}$ for all agents $c \in [\nagent]$), these two terms vanish. Finally, $\smash{\tilde{\theta}_t^{\text{(fl)}}}$ accounts for the fluctuations of $\theta_t$ around the solution $\thetalim$.

% and also $\Delta_{s, t} = \left \{ \prod_{i = s + 1}^t \Gammaavg_i \right \} - \mathbb{E} \Big[\prod_{i = s + 1}^t \Gammaavg_i \Big]$, the fluctuations of the matrix product $\Gammaavg_i$. Then, we establish error decomposition for the last iterate
% \begin{align}
% \label{eq:decomp-theta-t-one-iter}
%     \theta_t - \thetalim = \Gammaavg_t (\theta_{t - 1} - \thetalim) + \rhoavg_{t} + \tauavg_{t} + \varphiavg_{t}
%     \eqsp.
% \end{align}

\subsection{Moment Bounds}
\label{sec:moment bounds}
% Assumptions (on $\gamma$ on $H$).

We now establish bounds on the moments of the error.
We emphasize that, although similar to \citet{mangold2024scafflsa}'s decomposition, the varying step size and number of local updates change the nature of the terms from decomposition \Cref{eq:decomp-theta-t-one-iter-main}. Thus we emphasize that the bounds derived in this section do not follow from the ones obtained in \citet{mangold2024scafflsa}. We consider two choices of step size and local update schedules: 
\begin{stepsize}[Constant]
\label{ass: lr-constant}
    For any $t > 0$, we set $\step_t = \step$ and $\nlupdates_t = \nlupdates$ for some $0 < \step \le \step_{\infty,p}$ and $\nlupdates \geq 1$.
\end{stepsize}
\begin{stepsize}[Decreasing]
\label{ass: lr-polynomial-decay}
    For any $t > 0$, we set $\step_t = \step(1+t)^{-\gamma_\step}$ and $\nlupdates_t = \lceil \nlupdates(1+t)^{\gamma_\nlupdates} \rceil$, for some $\step \in (0, 1)$, $\nlupdates \geq 1$ and  $\gamma_\step \in [{1\over 2}, 1)$, $\gamma_\nlupdates \geq 0$, such that $\gamma_\step \geq \gamma_\nlupdates$. Additionally, we assume that $\step \nlupdates < a^{-1}$ and we denote $\gamma = \gamma_\step - \gamma_\nlupdates$.
\end{stepsize}
The scheme \Cref{ass: lr-polynomial-decay} extends the classical polynomially decaying step size regime by allowing the number of local updates to increase over time. 
In federated settings, such scaling is natural: as the step size decreases, bias reduces, and performing more local updates per communication round can improve communication efficiency. 
The condition $\gamma_\step \ge \gamma_\nlupdates$ ensures that the overall contraction remains stable, while $\smash{\step \nlupdates < a^{-1}}$ guarantees uniform stability of the local dynamics. To our knowledge, this joint regime of decaying step sizes and increasing number of local updates has not been previously analyzed for federated learning.

\paragraph{Second moment bounds.}
We first derive a general theorem on the squared error, applicable to any sequence of decreasing step sizes and increasing numbers of local iterations.

\begin{theorem}
\label{theorem: mse-estimation-for-general-eta-main}
    Assume \Cref{ass: linear-decay}(2), \Cref{assum:noise-level-flsa}. Let $\step_s \le \step_{\infty, 2}$, be decreasing, and $\nlupdates_s$ be increasing, for all $s \ge 0$. Then, it holds   
    \begin{align*}
        & \textstyle \PE^{1/2} [\| \theta_t %- \tilde{\theta}_t^{(\sf{bi, bi})}
        - \thetalim \|^2]
        \leq
        \prod_{s = 1}^t (1 - \eta_{s} a)^{H_s} \| \theta_0 - \thetalim\|
        \\
        & \textstyle
        +
        \zeta_1 \zeta_2\sum_{s=1}^t \step_s^2 \nlupdates_s^2 \exp{\Big( -a\sum_{i=s+1}^t \step_i \nlupdates_i \Big)} \\
        & \textstyle
        +  
        \frac{\varbound + 2 \varhet}{\sqrt{N}} %\sqrt{\frac{\varbound^2}{N}}
        \sqrt{\sum_{s=1}^t  \eta_s^2 H_s  
        \exp(- 2a \sum_{r=1}^t \eta_r \nlupdates_r))}
        % \\
        % & \quad \textstyle
        % +
        % \sqrt{\frac{4 \varhet^2}{N} }\sqrt{\sum \limits_{s = 1}^t
        %  \eta_s^2 H_s \exp(-2a \sum_{r = s }^t \eta_r H_r )
        %   }
        \\
        & \textstyle
        +
         \sqrt{ \frac{4 \varboundA^2 \firsthgty^2 \secondhgty^2 }{\nagent a e} }
         \sqrt{
    \sum_{s=1}^t \step_s^5 H_s^4  \exp \big( - a \sum_{r = s + 1}^t \eta_r H_r \big) } 
    \eqsp,
    \end{align*}
    where $\varbound, \varhet$ and $\varboundA$ are defined in \eqref{eq:def-heter-and-noise}.
\end{theorem}
The proof of \Cref{theorem: mse-estimation-for-general-eta-main} relies on bounding each term in the error decomposition \eqref{eq:decomp-theta-t-one-iter-main}. We provide a  proof with explicit constants in \Cref{appendix: mse-bound-estimation}.
When the step size and the number of local updates are constant over time, the above result simplifies and recovers an improved variant of Theorem A.6 in \citet{mangold2024scafflsa}.
\begin{corollary}
\label{coro:second-moment-cst}
    Assume \Cref{ass: linear-decay}(2), \Cref{assum:noise-level-flsa} and \Cref{ass: lr-constant}. 
    Then, it holds    
    \begin{align*}
        & \PE^{1/2} [\norm{\theta_t - 
        \thetalim}^2]
        \leq \textstyle 
        (1 - \step a)^{H t} \| \theta_0 - \thetalim\|
        +
        \frac{\step \nlupdates \zeta_1 \zeta_2 }{a}
        \\ 
        & \qquad\qquad\qquad \textstyle +
        \sqrt{\frac{\step \varbound^2}{N a}}
        +
        \sqrt{\frac{\step \varhet^2}{N a} } 
        +
         \sqrt{ \frac{4 \step^4 H^3 \varboundA^2 \firsthgty^2 \secondhgty^2 }{\nagent a^2 e} }
    \eqsp.
    \end{align*}
\end{corollary}
This bound shows that in order to obtain MSE of order $\PE[\norm{\theta_t - \thetalim}^2] \leq \epsilon^2$, one needs to choose the step size $\eta \asymp \epsilon^2$, and the product $\step \nlupdates \asymp \epsilon$.
% To remedy this, we introduce a polynomial decay schedule for the step size together with a (potentially) increasing number of local updates, which will be used in subsequent results.
This choice is due to non-vanishing variance terms. 
In contrast, one can make the variance vanish by choosing appropriately decreasing step size.
% Decaying step sizes aim to remedy this problem. \som{Why it is a problem?}
We thus establish the result on the last iterate moment bounds under the scheme \Cref{ass: lr-polynomial-decay}.
\begin{corollary}
\label{thm:moment_bounds_main}
Assume \Cref{ass: linear-decay}(2), \Cref{assum:noise-level-flsa} and \Cref{ass: lr-polynomial-decay}. Then, for all $t \geq 1$, the following bound holds
\begin{align*}
& \textstyle 
\PE^{1/2} \big[\|\theta_t
- \thetalim\|^2\big] 
\lesssim
\exp\big(- \frac{a \initstep \initnlupdates}{1 - \gamma} (t + 1)^{1 - \gamma} \big) \|\theta_0 - \thetalim\|
\\
& \textstyle 
+ \frac{8\zeta_1 \zeta_2 \step\nlupdates}{a (t+1)^{\gamma}}
+ \textstyle
\sqrt{\frac{\initstep}{N a}} 
\frac{\varbound + \varhet}{(t+1)^{\gamma_\step / 2}} 
+ \sqrt{\tfrac{\initstep^4 \initnlupdates^3}{\nagent a e}} \frac{\varboundA \firsthgty \secondhgty}{(t+1)^{(4 \gamma_\step - 3 \gamma_\nlupdates)/2}}
\eqsp.
\end{align*}
\end{corollary}
This corollary shows that, when using decaying step size, one can obtain convergence to arbitrary precision without starting with tiny step sizes.
Specifically, this is the case as long as $\gamma > 0$, implying that step sizes must decrease faster than the number of local iteration increases. 

\paragraph{Higher-order moment bounds.}
Similarly, we establish the result for higher-order moment bounds, which are crucial for further multiplier bootstrap analysis in \Cref{sec:bootstrap}. 
\begin{theorem}
\label{theorem: last-iterate-moment-bound-estimation}
    Assume \Cref{ass: linear-decay}(p), \Cref{assum:noise-level-flsa}. Let $\step_s \le \step_{\infty, 2}$, be decreasing, and $\nlupdates_s$ be increasing, for all $s \ge 0$. Then, it holds   
    \begin{align*}
            & \textstyle \PE^{1/p}[\norm{\theta_t - \thetalim}^p]
            \leq \exp{\left( -a \sum _{s = 1}^{t} \step_s \nlupdates_s \right)} \, \norm{\theta_0 - \thetalim}\\
            & \textstyle+ \sqrt{ \frac{p^3 \supconsteps^2}{\nagent^2} \sum _{s = 1}^{t} \step_s^2 \nlupdates_s \exp{\left( -2 a \sum \limits_{i = s + 1}^t \step_i \nlupdates_i \right)} }\\
            & \textstyle + \zeta_1 \zeta_2\sum_{s=1}^t \step_s^2 \nlupdates_s^2 \exp{\left( -a\sum_{i=s+1}^t \step_i \nlupdates_i \right)}\\
            & \textstyle +  
            2 \bConst{A} p^2 \zeta_*
            %\frac{2 \bConst{A} p^2\sum _{c = 1}^\nagent \norm{\thetalim[c] - \thetalim}}{\nagent}
            \sqrt{ \sum _{s = 1}^{t} \step_s^2 \nlupdates_s \exp{\Big( -2a \sum _{i = s}^{t} \step_i \nlupdates_i \Big)} } \\
            & \textstyle + \frac{\bConst{A} p^3 \zeta_1 \zeta_2}{\sqrt{\nagent a \rme}} \sqrt{ \sum _{s = 1}^{t} \step_s^5 \nlupdates_s^4 \exp{\Big( -a \sum _{i = s + 1}^{t} \step_i \nlupdates_i \Big)}}
            \eqsp,
        \end{align*}
        where $\zeta_* = 1/\nagent \sum _{c = 1}^\nagent \norm{\thetalim[c] - \thetalim}$.
\end{theorem}
%\som{dimensional factor?}
This theorem is very similar to \Cref{theorem: mse-estimation-for-general-eta-main}, up to an additional factor depending on $p$.
Following the second-order case, we obtain the counterpart of \Cref{coro:second-moment-cst,thm:moment_bounds_main}, which we state in \Cref{sec:moment-bound-appendix}, together with a proof of \Cref{theorem: last-iterate-moment-bound-estimation}.
% to a and exactly matches it in the special case $p=2$.
% We observe that, in the special case $p = 2$, the bound in  \Cref{theorem: last-iterate-moment-bound-estimation} coincides with that of \Cref{theorem: mse-estimation-for-general-eta}, 
% up to a dimensional factor depending on $p$ (which reduces to a constant when $p=2$) 
% and the homogeneity-related constants.

%% file: normal_approx.tex
\subsection{\mbox{Gaussian approximation framework}}
\label{sec:gaussian-approx-framework}
In this section, we present our main result, establishing non-asymptotic Gaussian approximation rates for the last iterate of FedLSA. Before presenting particular results, we outline a general scheme for proving the normal approximation in \Cref{sec:gaussian-approx-framework}. 
\par 
The classical object studied in the literature on Gaussian approximation \citep{chen2007, shao2022berry} is the general vector-valued nonlinear statistic $T(X_1, \ldots, X_n) \in \rset^{d}$, which can be represented as
\begin{equation}
\label{eq:linear_non_linear_t_decomposition}
\textstyle
T = W + D\eqsp,
\end{equation}
where $W$ is a linear statistic of $X_1, \ldots, X_n$, and $D$ has small moments. We consider the decomposition \eqref{eq:linear_non_linear_t_decomposition} and assume, without loss of generality, that $\PE[W W^\top] = \Id_{d}$. To measure the approximation quality, we focus on the convex distance $\kolmogorov$, defined for probability measures $\mu, \nu$ on $\rset^{d}$ as 
\begin{equation}
\label{eq:convex_distance_def} 
\kolmogorov(\mu, \nu) = \sup_{B \in \Conv(\rset^{d})} |\mu(B) - \nu(B)|\eqsp. 
\end{equation}
Estimating $\kolmogorov(T, \mathcal{N}(0,\Id_{d}))$ can be reduced to the following two steps. First, one needs to estimate the convex distance $\kolmogorov(W, \mathcal{N}(0, \Id_{d}))$ using the tools for normal approximation for linear statistics, such as \cite{bentkus2004,shao2022berry}. In the second step, one proceeds with one of the following approaches:
\begin{enumerate} %[leftmargin=*]
\item Estimate moments $\PE[\norm{D}^p]$ for some $p \geq 1$ and then apply \Cref{prop:nonlinearapprox} (see also \cite[Proposition~1]{sheshukova2025gaussian});
\item Estimation of the perturbations of the term $D$ within the randomized concentration approach \cite{shao2022berry}, requiring a bound of the fluctuations of the term $D$ when one of the variables $X_i$, for some $i \in [n]$, is modified.  
\end{enumerate}
The first approach is technically simpler, yet it requires one to estimate moments of increasing order $p$, which produces some subtle logarithmic factors in the final convergence rates. The second one is technically more involved, yet it requires only estimating the second moment of the term $D$ and provides the bounds that are sharper in terms of extra logarithmic factors.

\subsection{Linear statistic of FedLSA} % for Gaussian Approximation}
\label{sec:lin-stat-fedlsa}

\paragraph{Self-normalized statistic.}
To apply Gaussian approximation results for nonlinear statistics of
independent random variables, we extract the leading linear component
of the noise, a linear statistic of $\{ Z_k \}_{k \geq 0}$, and apply \citet[Theorem 2.1]{shao2022berry}. We write the error as follows.
\begin{align}
\label{eq:linearization_uai}
    \step_t^{-1/2}(\theta_t - \thetalim)
    =
    M_t
    +
    D_t^{\mathrm{rm}},
\end{align}
where $M_t$ collects the linear statistics of $\{Z_{s,h}^c\}$ and
$D_t^{\mathrm{rm}}$ is a higher-order remainder term.
Following this decomposition, we define the covariance matrix
\begin{align}
\label{eq:sigma_t_def_main}
    \Sigma_t
    :=
    \PE[M_t M_t^\top]
    \eqsp,
\end{align}
which characterizes the effective variance of the
last iterate. Under \Cref{ass: lr-polynomial-decay},
$\Sigma_t$ converges to the solution of the associated discrete-time Lyapunov equation $\Sigma_\infty$, which determines the asymptotic covariance of FedLSA. 
We then analyze the self-normalized statistic
\begin{align}
\label{eq:self_normalized_stat_uai}
    T_t
    :=
    \step_t^{-1/2}
    \Sigma_t^{-1/2}
    (\theta_t - \thetalim),
\end{align}
where $\Sigma_t$ is a covariance matrix representing the main linear part of non-linear statistics $T_{t}$, and is defined in \eqref{eq:sigma_t_def_main}.
Under the stability
conditions imposed in the previous section, $T_t$ weakly converges to 
$\mathcal{N}(0,\Id_d)$ as $t \to \infty$. Our objective is to quantify this
convergence at finite time using the convex distance \eqref{eq:convex_distance_def}.

% \subsection{Linearization of the Error}
\paragraph{Linear statistic of FedLSA's error.}
We start from the decomposition
\begin{align}
\label{eq:error_decomp_last_iter_main_uai}
\theta_t - \thetalim
&=
\textstyle
\prod_{i=1}^{t}\Gammaavg_i(\theta_0-\thetalim)
\\
&\quad+
\nonumber
\textstyle
\sum_{s=1}^{t}
\prod_{i=s+1}^{t}\Gammaavg_i
\big(
    \Deltaavg_{s,\nlupdates_s}
    -
    \step_s \Epsavg_{s,\nlupdates_s}
\big),
\end{align}
where we recall that
\begin{align*}
\textstyle
\Deltaavg_{t,h}
&=
\textstyle
\nagent^{-1}
\sum_{c=1}^{\nagent}
(\Id-\Gamma_{t,h}^c)
(\thetalim[c]-\thetalim),
\\
\Epsavg_{t,h}
&=
\textstyle
\nagent^{-1}
\sum_{c=1}^{\nagent}
\sum_{s=1}^{h}
\Gamma_{t,s+1:h}^c
\funcnoise[c]{Z_{t,s,c}}.
\end{align*}
The transient term $\prod_{i = 1}^t \Gammaavg_i (\theta_0 - \thetalim)$ in \eqref{eq:error_decomp_last_iter_main_uai}
is negligible under \Cref{ass: linear-decay}. The main stochastic contribution arises from the second term, which aggregates agent noise and heterogeneity-induced fluctuations. 
By
\Cref{thm:moment_bounds_main}, both components contribute at the same
order up to heterogeneity constants ($\varbound^2$ and $\varhet^2$ respectively) and must be handled jointly.
To obtain the linear statistic, we center all stochastic contraction matrices around their expectation in \eqref{eq:error_decomp_last_iter_main_uai}, using the following matrices, defined for $t, l, r \ge 0$,
\begin{align}
G_{t,l:r}^{c} = \PE[ \Gamma_{t,l:r}^c ]
\eqsp,
\text{   and   }
\Gavg_i = \PE[ \Gammaavg_i ]
\eqsp.
\end{align}
Unrolling the resulting decomposition and identifying the leading term at each step gives the linear statistic % $M_t$ is given by
\begin{align}
% \label{eq:M_t_def_main}
    \nonumber M_t =
    -\nagent^{-1}
    \step_t^{-1/2}
    \sum_{s=1}^{t}
    \step_s
    \prod_{i=s+1}^{t}\Gavg_i
    \, (\mathrm{U}_{s}^{\mathrm{noise}} + \mathrm{U}_{s}^{\mathrm{het}}) \eqsp,
\end{align}
where $\mathrm{U}_{s}^{\mathrm{noise}}$ accounts for the noise directly due to noisy updates, and $\mathrm{U}_{s}^{\mathrm{het}}$ is a noise term which arises due to the heterogeneity of the local solutions,
\begin{align*}
    &\textstyle \mathrm{U}_{s}^{\mathrm{noise}} := \sum_{c=1}^{\nagent}
    \sum_{h=1}^{\nlupdates_s}
    G_{s,h+1:\nlupdates_s}^{c}
    \funcnoise[c]{Z_{s,h,c}}
    \eqsp, \\
    &\textstyle \mathrm{U}_{s}^{\mathrm{het}} := 
    \sum_{c=1}^{\nagent}
    \sum_{h=1}^{\nlupdates_s}
    G_{s,1:h-1}^{c}
    \zmfuncA[c]{Z_{s,h,c}}
    (\thetalim[c]-\thetalim)
    \eqsp.
\end{align*}
The remainder term is then given by
\begin{align*}
    D_t^{\text{rm}} = \step_t^{-1/2} \sum_{s=1}^t \prod_{i=s+1}^t \Gammaavg_i (\Deltaavg_{s,\nlupdates_s} - \step_s \Epsavg_{s,\nlupdates_s}) - M_t \eqsp.
\end{align*}
Then, we can rewrite our self-normalized statistic $T_{t}$ as
\begin{align}
\label{eq:self-normalized-statistics-decomposition}
    \step_t^{-1/2} \Sigma_t^{-1/2} (\theta_t - \thetalim) = W_t + D_t \eqsp,
\end{align}
where
%\begin{align*}
   % &
    $W_t := \Sigma_t^{-1/2} M_t$ and %\eqsp, \quad
    $D_t = \Sigma_t^{-1/2} D_{t}^{\mathrm{rm}}$.

\subsection{Last Iterate Gaussian Approximation}
\label{sec: gaussian-approximation-for-the-last-iterate}

We now present the main result of this section. Applying \citet[Theorem~2.1]{shao2022berry} to the decomposition~\eqref{eq:self-normalized-statistics-decomposition} yields a non-asymptotic Gaussian approximation for 
$\smash{\step_t^{-1/2} \Sigma_t^{-1/2} (\theta_t - \thetalim)}$, leading to the following result. % which forms the basis of \Cref{thm:gauss_approx_last_iter}/ % presented in \Cref{sec: gaussian-approximation-for-the-last-iterate}.

\begin{theorem}
\label{thm:gauss_approx_last_iter}
    Under assumptions \Cref{ass: linear-decay}(2), \Cref{assum:noise-level-flsa} and \Cref{ass: lr-polynomial-decay}, setting $Y \sim \mathcal{N}(0, \Id_d)$ and choosing $\gamma_\nlupdates < \gamma_\step/2$, we obtain
    \begin{align*}
        &\kolmogorov(\step_t^{-1/2}\Sigma_t^{-1/2}(\theta_t - \thetalim), Y) \\
        &\lesssim \nagent^{-1/2} (1+t)^{-\gamma_\step/2} + \nagent^{1/2}(1+t)^{-(2\gamma - \gamma_\step)/2} \eqsp,
    \end{align*}
    where $\Sigma_t$ is defined in \eqref{eq:sigma_t_def_main}.
\end{theorem}
\Cref{thm:gauss_approx_last_iter} extends \cite[Theorem 2.2]{bonnerjee2025sharp} - which assume a constant number of local steps ($\gamma_{H} = 0$). Note that \Cref{thm:gauss_approx_last_iter} implies, since convex distance is invariant under non-degenerate linear mappings, approximation rates for
\[
\kolmogorov(\step_t^{-1/2}(\theta_t - \thetalim), \mathcal{N}(0,\Sigma_t))\eqsp,
\]
and not for 
\[
\kolmogorov(\step_t^{-1/2}(\theta_t - \thetalim), \mathcal{N}(0,\Sigma_\infty))\eqsp,
\]
where $\Sigma_\infty$ is the limiting covariance matrix of the last iterate of federated LSA defined as the unique solution of the Lyapunov equation (see e.g. \cite{fort2015central}):
\begin{equation}
\label{eq:lyapunov}
\bA \Sigma_\infty + \Sigma_\infty \bA[\top] = \nagent^{-1}\noisecovavgst \eqsp.
\end{equation}
Here $\noisecovavgst$ is the covariance of noise at $\thetalim$
\begin{align*}
\textstyle
    \noisecovavgst =  \nagent^{-1}\sum_{c=1}^\nagent \PE[\funcnoiseth[c]{\thetalim}{Z_{c}} (\funcnoiseth[c]{\thetalim}{Z_{c}})^{\top}]\eqsp.
\end{align*}
Although the empirical covariance matrix $\Sigma_t$ converges to $\Sigma_\infty$ as $t \to \infty$, the rate of convergence may be slow and can therefore dominate the overall Gaussian approximation error. Similar effects were previously observed for Polyak-Ruppert averaged SGD in the single-agent setting in \cite{shao2022berry} and \cite{sheshukova2025gaussian}. Interestingly, the same phenomenon persists for the last iterate bounds. The next proposition quantifies this effect.
\begin{proposition}
\label{proposition: covariance-matrix-comparison-main}
Under assumptions \Cref{ass: linear-decay}(2), \Cref{assum:noise-level-flsa}, and \Cref{ass: lr-polynomial-decay}, and choosing $\step\nlupdates \leq \beta_\infty$ for some $\beta_\infty > 0$, we have
\begin{align*}
\norm{\Sigma_t - \Sigma_\infty} \leq C_{\infty, 1} (1+t)^{\gamma - 1} + C_{\infty, 2} (1+t)^{-\gamma} \eqsp,
\end{align*}
where the constants $C_{\infty, 1}$ and $C_{\infty, 2}$ are given in \eqref{eq:constant_sigma_infty_1} and \eqref{eq:constant_sigma_infty_2}.
\end{proposition}
\Cref{proposition: covariance-matrix-comparison-main} 
reveals an intrinsic trade-off in the choice of the parameter 
$\gamma = \gamma_\step - \gamma_\nlupdates$. 
On the one hand, $\gamma$ cannot be too small, as sufficiently fast decay is required to ensure moment control and Gaussian approximation rates. 
On the other hand, taking $\gamma$ close to $1$ deteriorates the covariance stabilization rate through the term 
\begin{equation}
\label{eq:gaus_approximation_tight}
C_{\infty,1}(1+t)^{\gamma - 1}\eqsp.
\end{equation}
Interestingly, we can show the lower bound demonstrating that both terms in the bound of \Cref{proposition: covariance-matrix-comparison-main} are essential. 

\textbf{Lower bound}. Consider a $1$-dimensional linear stochastic approximation with single agent ($\nagent = 1$):
\[
\bA \thetalim = \barb, \qquad \bA = 1, \quad \barb = 0 \eqsp,
\]
where the unbiased stochastic estimates are given by
\[
\mathbf{A}(Z_j) = 1 + \xi_j, \qquad \mathbf{b}(Z_j) = \barb = 0,
\]
with $\xi_j \sim \mathcal{N}(0,1)$. Note that in this setting there are no synchronization steps across agents (since $\nagent = 1$), so $\nlupdates_t~=~1$, and the algorithm reduces to the usual $1$-dimensional LSA. Then, the \textsc{FedLSA} algorithm with starting point $\theta_0 = 0$ can be written as the recurrence
\begin{equation}
    \theta_t = \theta_{t-1} - \step_t (1 + \xi_i) \theta_{t-1}, \quad \theta_0 = 0.
\end{equation}
We set the step sizes as $\eta_t = 1/(1 + t)^{\gamma_\step}$, where $\gamma_\step \in (0,1)$. Note that in the single-agent case we have $\gamma_\nlupdates = 0$, so that $\gamma = \gamma_\step$. In this case, it is possible to compute the exact value of the limiting variance $\sigma_{\infty}^2$ and to derive a lower bound for the convex distance:
\begin{proposition}
\label{proposition: limiting-variance-lower-bound}
    Assume that $\step_1 = 0$ and $\step_{\ell} = (1+\ell)^{-\gamma}$ with $0 < \gamma < 1$. For $1$-dimensional single agent \textsc{FedLSA} with $\bA \thetalim = \barb, \bA = 1, \barb = 0$ and $\mathbf{A}(Z_j) = 1 + \xi_j, \mathbf{b}(Z_j) = \barb$ with $\xi_j \sim \mathcal{N}(0,1)$ it holds that
    \begin{align*}
        \sigma_\infty^2 := \lim_{t \to \infty} \Var[\step_t^{-1/2} (\theta_t - \thetalim)] = \frac{1}{2} \eqsp.
    \end{align*}
    Moreover, for any $t \geq 4$, it can be shown that
    \begin{align}
    \label{eq:limiting_cov_lower_bound}
        \Big| v_t - \frac{1}{2} \Big| \geq \bConst{v} (t^{\gamma - 1} + t^{-\gamma}) \eqsp,
    \end{align}
    for some positive constant $\bConst{v} > 0$ that depends only on $\gamma$.
\end{proposition}

% In this case, the limiting variance is $\sigma_\infty^2 = \tfrac{1}{2}$ (one can easily deduce it from the Lyapunov equation \eqref{eq:lyapunov}), and the convex distance can be lower bounded as
% \begin{equation}
% \label{eq:normal_approx_sigma_infty}
% \kolmogorov\big(\step_t^{-1/2} (\theta_t - \thetalim), \mathcal{N}(0, \sigma_\infty^2)\big) \gtrsim t^{-\gamma} + t^{\gamma - 1} \eqsp.
% \end{equation}
We provide the proof of this lower bound in Section~\ref{appendix: lower-bound-for-gaussian-approximation}. To the best of our knowledge, this is the first result that establishes a non-asymptotic lower bound for the last iterate of LSA in terms of the convex distance, thereby extending the lower bound results for Polyak--Ruppert averaging due to~\citet[Theorem 5]{sheshukova2025gaussian}.

% \Cref{proposition: covariance-matrix-comparison-main} shows up appearing trade-off for the parameter $\gamma$: from one side it shouldn't be too small as otherwise we will not have convergence in terms of moment bounds and Gaussian approximation. From the other side, \Cref{proposition: covariance-matrix-comparison-main} does not allow to make $\gamma$ large enough as $\gamma \to 1$ will make the first term $C_{\infty, 1} (1 + t)^{\gamma - 1}$ dominating. Moreover, it can be show that this trade-off on parameter $\gamma$ is essential, i.e. Gaussian approximation of $\step_t^{-1/2} (\theta_t - \thetalim)$ to $\mathcal{N}(0, \Sigma_\infty)$ cannot be to small. More formally, we can provide the lower bound on $\kolmogorov(\step_t^{-1/2} (\theta_t - \thetalim), \mathcal{0, \Sigma_\infty})$:

% \textbf{Lower bound}. Todo.
\textbf{Discussion.} \Cref{thm:gauss_approx_last_iter} demonstrates that increasing the number of local updates over time 
(i.e., $\gamma_\nlupdates > 0$) reduces $\gamma$ and therefore 
deteriorates the rate of normal approximation. While enlarging the number of local iterations may improve 
optimization accuracy or communication efficiency, it adversely affects distributional convergence. From the perspective of Gaussian approximation, 
this suggests that a constant number of local updates ($\gamma_\nlupdates = 0$) is preferable. However, this choice can not always be implementable due to communication constraints.

\section{\!\!\!\mbox{\textls[-25]{Multiplier bootstrap for FedLSA}}}
\label{sec:bootstrap}
We now aim to perform statistical inference for the last-iterate estimator $\theta_t$, and to construct confidence intervals for $\thetalim$. Instead of explicitly estimating the asymptotic covariance matrix $\Sigma_\infty$ and relying on plug-in Gaussian approximations \citep{li2022statistical}, we use an online multiplier bootstrap that directly approximates the sampling distribution of $\theta_t$ in a data-driven manner.
To this end, we define a bootstrapped sequence, starting with the same initialization as the original FedLSA dynamics, but using randomly weighted local updates.

% \label{eq:bootstrap_update_intro}

{
Bootstrap iterates are defined by
\begin{align*}
\theta_{t,h}^{\boot,c}
=\theta_{t,h-1}^{\boot,c} -\step_t w_{t,h}\big(
\zfuncA[c]{Z_{t,h,c}}\theta_{t,h-1}^{\boot,c}
-\zfuncb[c]{Z_{t,h,c}}\big) \eqsp,
\end{align*}
where the weights $w_{s,h}$ are i.i.d.\ and independent of the observations, with $\PE[w_{1,1}]=1$ and $\Var[w_{1,1}]=1$. For each $(s,h)$, all clients use the same weight $w_{s,h}$. We denote their collection by
\begin{align*}
    \mathcal{W}^t=\{w_{s,h}:s\in[t],\ h\in[\nlupdates_s]\} \eqsp.
\end{align*}
\par}
After each block of local updates, the global bootstrap iterate is updated as $\theta_t^{\boot} = 1/\nagent \sum_{c=1}^\nagent \theta_{t,H_t}^{\boot, c}$.

\paragraph{Confidence intervals.}
Conditionally on the observed data, the distribution of $\smash{\step_t^{-1/2}(\theta_t^\boot - \theta_t)}$ approximates the sampling distribution of $\smash{\step_t^{-1/2}(\theta_t - \thetalim)}$. More precisely, denoting conditional probability and expectation given~the~data~by
\[
    \mathbb{P}^b(\cdot) = \mathbb{P}(\cdot \mid \mathcal{Z}^t),
    \qquad
    \mathbb{E}^b(\cdot) = \mathbb{E}(\cdot \mid \mathcal{Z}^t) \eqsp.
\]
Our objective, which is the main result of this section, is to 
show that the bootstrap validity holds, that is, the quantity
\begin{align}
\label{eq:bootstrap_objective}
\nonumber
    \sup_{B \in \text{Conv}(\mathbb{R}^d)} \Big| \PPb(\step_t^{-1/2}&(\theta_t^\boot - \theta_t) \in B) \\
    &- \PP(\step_t^{-1/2}(\theta_t - \thetalim) \in B) \Big|
\end{align}
converges to zero as $t \to \infty$, and to establish a rate for this convergence. 
This approximation result is essential, as it allows us to construct confidence intervals for the \emph{true solution} $\thetalim$ based only on the distribution of $\smash{\step_t^{-1/2}(\theta_t^\boot - \theta_t)}$. For example, for a coordinate $j \in [d]$ and confidence level $(1-\alpha)$, with $\alpha \in (0, 1)$, 
let $q_{\alpha/2}^b$ and $q_{1-\alpha/2}^b$ denote the empirical conditional quantiles of 
$\smash{\step_t^{-1/2}(\theta_{t,j}^\boot - \theta_{t,j})}$. A $(1-\alpha)$ confidence interval for $\thetalim$'s $j$-th coordinate is then% given by
\[
\big[
\theta_{t,j} - \step_t^{1/2} q_{1-\alpha/2}^b,
\;
\theta_{t,j} - \step_t^{1/2} q_{\alpha/2}^b
\big].
\]
Similarly, simultaneous confidence intervals can be constructed using bootstrap quantiles of $\smash{\|\step_t^{-1/2}(\theta_t^\boot - \theta_t)\|}$. 
This approach avoids explicit covariance estimation and remains fully online.
To establish its validity, we first specify the linear statistics of the bootstrap iterates in \Cref{sec:boot-decomp}, and then establish our main theorem in \Cref{sec:boot-main-result}.

\subsection{Bootstrap decomposition}
\label{sec:boot-decomp}

We now aim to extract the linear statistic from $\theta_t^\boot - \theta_t$. 
To this end, we introduce the bootstrap counterparts of matrix product, for $t, \ell, r \ge 0$, as

% \\
%&\Gammaavgboot_t = \nagent^{-1}\sum_{c=1}^\nagent \Gamma_t^{\boot, c} \eqsp

{
\begin{align*}
    \Gamma_{t,\ell:r}^{\boot,c}
    &=\prod_{h=\ell}^r(\Id-\step_t w_{t,h}\zfuncA[c]{Z_{t,h,c}}) \eqsp,\\
    \Gamma_t^{\boot,c}&=\Gamma_{t,1:\nlupdates_t}^{\boot,c} \eqsp,\qquad
    \Gammaavgboot_t=\nagent^{-1}\sum_{c=1}^{\nagent}\Gamma_t^{\boot,c} \eqsp.
\end{align*}
These products incorporate the common weights $w_{t,h}$ into the local dynamics.
\par}
Arguing as in the original process, we obtain the following decomposition of $\theta_t^\boot - \theta_t$ using these bootstrap matrices
\begin{align*}
    \textstyle
\theta_t^\boot - \theta_t
&
    \textstyle=
\prod_{i=1}^t \Gammaavgboot_i(\theta_0^\boot - \theta_0) \\
&
    \textstyle \quad - \sum_{s=1}^t \step_s \prod_{i=s+1}^t \Gammaavgboot_i \Epsavgboot_s \eqsp,
\end{align*}
where $\Epsavgboot_s$ is defined as

% \\
% \textstyle
% \funcnoiset[c]_{t,h} &
% \textstyle= \funcnoise[c]{Z_{t,h,c}} + \zfuncA[c]{Z_{t,h,c}}(\theta_{t,h-1}^c - \thetalim[c])\eqsp.

{
\begin{align*}
    \Epsavgboot_s
    &=\nagent^{-1}\sum_{h=1}^{\nlupdates_s}(w_{s,h}-1)
    \sum_{c=1}^{\nagent}\Gamma_{s,h+1:\nlupdates_s}^{\boot,c}\funcnoiset[c]_{s,h} \eqsp,\\
    \funcnoiset[c]_{s,h}
    &=\funcnoise[c]{Z_{s,h,c}}
    +\zfuncA[c]{Z_{s,h,c}}(\theta_{s,h-1}^c-\thetalim[c]) \eqsp.
\end{align*}
\par}

{
Note that the recursion contains no separate deterministic heterogeneity term. Using independent bootstrap weights across clients would generally introduce an additional leading term due to heterogeneity. Extracting the leading statistic $M_t^\boot$, which is linear in the centered weights $w_{s,h}-1$, yields the decomposition
\par}
\begin{align}
\label{eq:error_decomp_boot_main}
\step_t^{-1/2}
(\Sigma_t^\boot)^{-1/2}
(\theta_t^\boot - \theta_t)
=
W_t^\boot
+
D_t^\boot,
\end{align}
where we define the bootstrap covariance matrix as
\begin{align}
    &\Sigma_t^\boot =\mathbb{E}^b\!\left[ M_t^\boot (M_t^\boot)^{\top}
\right] \eqsp, \label{eq:Sigma_t_boot_def_main}
% \\
    % \nonumber &W_{t}^\boot := (\Sigma_t^\boot)^{-1/2} M_t^\boot \eqsp, \\
    % \nonumber &D_{t}^\boot = \step_t^{-1/2} (\Sigma_t^\boot)^{-1/2} (\theta_t^\boot - \theta_t) - W_t^\boot \eqsp
\end{align}
together with
$D_{t}^\boot = \step_t^{-1/2} (\Sigma_t^\boot)^{-1/2} (\theta_t^\boot - \theta_t) - W_t^\boot$, and $W_{t}^\boot := (\Sigma_t^\boot)^{-1/2} M_t^\boot$  % \eqsp, \\
    % \nonumber &
     where the linear term $M_t^\boot$ given by

{
\begin{align*}
M_t^\boot &= -\nagent^{-1}\step_t^{-1/2}\sum_{s=1}^t
\step_s\prod_{i=s+1}^t\Gammaavg_i\,\mathrm{U}_s^{\mathrm{noise},\boot} \eqsp,\\
\mathrm{U}_s^{\mathrm{noise},\boot}
&= \sum_{h=1}^{\nlupdates_s}(w_{s,h}-1)
\sum_{c=1}^{\nagent}\Gamma_{s,h+1:\nlupdates_s}^c\\
& \quad\quad\quad\times\big\{\funcnoiseth[c]{\thetalim}{Z_{s,h,c}}
+\bA[c](\thetalim-\thetalim[c])\big\} \eqsp.
\end{align*}
The common weights cancel the unweighted sum $\sum_{c=1}^{\nagent}\bA[c](\thetalim-\thetalim[c])=0$.
\par}

\subsection{Bootstrap validity}
\label{sec:boot-main-result}

We now establish the validity of the proposed multiplier bootstrap procedure 
for the last-iterate estimator in FedLSA. 
To conduct our analysis, we impose one additional technical assumption on the bootstrap weights $\mathcal{W}^t$, together with a condition on the minimal number of iterations.
\begin{assumption}[{p}]
\label{ass: boot_weights}

The bootstrap weights $\mathcal{W}^t = \{w_{s,h} : s \in [t], h \in [\nlupdates_s]\}$ are i.i.d. and independent of the observations. For each $(s,h)$, all clients use the same weight $w_{s,h}$. These weights satisfy $0 < W_{\min} < w_{1,1} < W_{\max} < +\infty$ almost surely for some $W_{\min}, W_{\max} > 0$. Additionally, we assume that $\step < \step_{\infty,p}/W_{\max}$.
\end{assumption}
\begin{assumption}
\label{assum:sample_size}

{
The number of communication rounds satisfies $t\geq t_0$, where the threshold $t_0$ is specified in \Cref{assum:sample_size_prime}.
\par}
\end{assumption}
These assumptions are mostly technical.
Examples of random variables that satisfy the assumption \Cref{ass: boot_weights} are provided in, e.g., \cite{sheshukova2025gaussian};   
% Here we state an additional assumption on the minimal number of iterations.

{\Cref{assum:sample_size} controls both $\norm{\Sigma_t-\Sigma_\infty}$ and $\norm{\Sigma_t^\boot-\Sigma_t}$.\par}
We now show that the conditional distribution of the bootstrap statistic provides an accurate Gaussian approximation to the sampling distribution of the normalized estimation error.
\begin{theorem}[Bootstrap validity]
\label{thm:bootstrap_validity}

{
Set
\begin{align*}
    q_t=\log\big[4d\nagent(\nlupdates+1)^2(1+t)^{3+2\gamma_\nlupdates}\big] \eqsp.
\end{align*}
Assume \Cref{ass: linear-decay}($q_t$), \Cref{assum:noise-level-flsa}, \Cref{ass: lr-polynomial-decay},  \Cref{ass: boot_weights}($q_t$), and \Cref{assum:sample_size}. Suppose that $\gamma_\nlupdates < \gamma_\step/2$ and
\begin{align*}
    0 < \step < \min\left\{ {\step_{\infty,q_t} \over W_{\max}}, \frac{a}{q_t\bConst{A}^2}, {\beta_\infty \over H}\right\}\eqsp.
\end{align*}
Then there exists an event $\Omega_0$ with $\PP(\Omega_0)\geq1-4/t$ such that, on this event,
\begin{align*}
&\sup_{B\in\mathrm{Conv}(\mathbb R^d)}
\Big|\PPb\big(\step_t^{-1/2}(\theta_t^\boot-\theta_t)\in B\big)\\
&\quad-\PP\big(\step_t^{-1/2}(\theta_t-\thetalim)\in B\big)\Big|\\
&\lesssim_{\log_t}(1+t)^{-\gamma_\step/2}
+\nagent^{1/2}(1+t)^{-(2\gamma-\gamma_\step)/2} \eqsp.
\end{align*}
\par}
\end{theorem}
\paragraph{Proof sketch.}
The proof relies on Gaussian approximations for both the original and bootstrap iterates, linked via a Gaussian comparison inequality. The main idea of the proof is illustrated with the diagram in \Cref{fig:diagram-proof}, following the general pipeline outlined in \cite{spokoiny2015,samsonov2024gaussian}:
\begin{figure}[H]
\centering
\begin{tikzcd}[column sep=40pt, row sep=30pt]
\label{eq:diagram_appendix}
    \text{Real world:} \;\; \step_t^{-1/2} (\theta_{t} - \thetalim) 
        \arrow[<->]{r}{\text{Theorem~\ref{thm:gauss_approx_last_iter}}}  
    & \mathcal{N}(0, \Sigma_t)  
        \arrow[<->]{d}[swap]{\text{Gaussian comparison}~}{~\text{\Cref{lemma:sigma_t_boot_norm_bound}}} 
    \\
    \text{Bootstrap:} \;\; \step_t^{-1/2} (\theta_{t}^\boot - \theta_t) 
        \arrow[<->]{r}{\text{Theorem~\ref{thm:norm_approx_boot_Yb}}} 
    & \mathcal{N}(0, \Sigma_t^\boot) \eqsp
\end{tikzcd}
\caption{Illustration of the multiplier bootstrap bounds: real-world estimates are close to $\mathcal{N}(0, \Sigma_t)$, which is itself close to $\mathcal{N}(0, \Sigma_t^{\boot})$.} %
\label{fig:diagram-proof}
\end{figure}
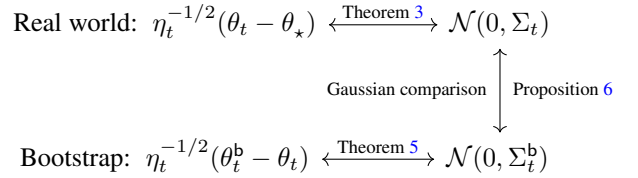

{Let $Y^\boot\mid\mathcal Z^t\sim\mathcal N(0,\Sigma_t^\boot)$ and $Y\sim\mathcal N(0,\Sigma_t)$.}
We define the convex distance under the conditional measure $\mathbb{P}^b$ as, for random vectors $X, Y \in \mathbb{R}^d$,
\begin{align*}
\textstyle 
    \kolmogorovboot(X, Y) = \sup_{B \in \text{Conv}(\mathbb{R}^d)} |\PP^\boot(X \in B) - \PP^\boot(Y \in B)| \eqsp.
\end{align*}
Then the convex distance is split using the triangle inequality
\begin{align*}
\sup_{B \in \text{Conv}(\mathbb{R}^d)} &\Big| \PPb(\step_t^{-1/2}(\theta_t^\boot - \theta_t) \in B) \\
&\quad\quad\quad\quad - \PP(\step_t^{-1/2}(\theta_t - \thetalim) \in B) \Big|\\
&\le 
T_1 + T_2 + T_3 ,
\end{align*}
where we have defined
\begin{align*}
T_1 &= 
\kolmogorovboot
\big(
\step_t^{-1/2}(\theta_t^\boot - \theta_t),
Y^\boot
\big), \\
T_2 &= 
\kolmogorov
\big(
\step_t^{-1/2}(\theta_t - \thetalim),
Y
\big), \quad\\% \\
T_3 &= 
\sup_{B \in \text{Conv}(\mathbb{R}^d)}\left| \PP(Y \in B) - \PPb(Y^\boot \in B) \right|\eqsp.
\end{align*}

{
The term $T_1$ relates to the Gaussian approximation in the bootstrap world from \Cref{thm:norm_approx_boot_Yb}. We estimate it using the decomposition \eqref{eq:error_decomp_boot_main}, 
moment bounds for $D_t^\boot$, and the result of 
\citet[Theorem 2.1]{shao2022berry}. The term $T_2$ is controlled by the Gaussian approximation result 
for the original iterate (\Cref{thm:gauss_approx_last_iter}). 
The term $T_3$ is handled using matrix concentration bounds 
for $\|\Sigma_t^\boot - \Sigma_t\|$ presented in ~\Cref{lemma:sigma_t_boot_norm_bound}. 
Combining the bounds implies the result.
\par}

\paragraph{Discussion.}

{A counterpart of \Cref{thm:bootstrap_validity} is stated without the proof for Polyak-Ruppert averaged federated SGD in \cite{bonnerjee2025sharp}. For heterogeneous clients, this bound does not hold in general when the bootstrap multipliers are sampled independently across clients. Under client heterogeneity, independent multipliers need not yield $\Sigma_t^\boot-\Sigma_t\to0$, so bootstrap consistency can fail. As our result shows, common multipliers remove this obstruction. Importantly, normal approximation with $\mathcal{N}(0,\Sigma_{\infty})$ is bypassed even in the proof of \Cref{thm:bootstrap_validity}, and is expected to yield slower approximation rates due to the lower bound \eqref{eq:limiting_cov_lower_bound}.\par}

%% file: numerical_section.tex
\begin{table}[t]
\centering
\caption{Coverage for nominal confidence level $\alpha=0.95$ with $\gamma_\nlupdates = 0$. Standard deviations are shown as subscripts. }%Additional experiments are reported in Appendix~\ref{sec:experiments}.}
\label{tab:numerical_main_gamma0}
\small
\setlength{\tabcolsep}{4pt}
% \begin{tabular}{c|ccc}
\begin{tabularx}{\columnwidth}{c|YYY}
\toprule
$T$ & SDB (ours) & EQ (ours) & PE \\
\midrule
2000
& $0.773_{\scriptscriptstyle \pm 0.013}$
& $0.780_{\scriptscriptstyle \pm 0.013}$
& $0.754_{\scriptscriptstyle \pm 0.013}$ \\
6000
& $0.936_{\scriptscriptstyle \pm 0.008}$
& $0.941_{\scriptscriptstyle \pm 0.007}$
& $0.931_{\scriptscriptstyle \pm 0.008}$ \\
10000
& $0.938_{\scriptscriptstyle \pm 0.008}$
& $0.947_{\scriptscriptstyle \pm 0.007}$
& $0.938_{\scriptscriptstyle \pm 0.008}$ \\
14000
& $0.941_{\scriptscriptstyle \pm 0.007}$
& $0.946_{\scriptscriptstyle \pm 0.007}$
& $0.941_{\scriptscriptstyle \pm 0.007}$ \\
\bottomrule
\end{tabularx}
\end{table}

\paragraph{Numerical Illustration.}
We now numerically illustrate our normal approximation and bootstrap confidence intervals for FedLSA. To this end, we instantiate a Garnet environment \cite{archibald1995generation,geist2014off} with $n=30$ states, feature dimension $d=5$, $a=2$ actions, and branching factor $b=2$. We generate $\nagent=5$ heterogeneous agents by perturbing this common Garnet MDP, following \citet{mangold2024scafflsa}. We run $R = 1024$ trajectories and construct confidence intervals for a random one-dimensional projection $u^\top \theta_t$. For each trajectory, we then check whether $\thetalim$ lies within the interval, average this indicator over the $R$ trajectories, and report- the result as coverage, with the standard error shown as a subscript. For each run of the bootstrap methods, we compute $256$ bootstrap trajectories.
 
Based on our results, we propose two methods (i) empirical bootstrap quantiles (EQ), and (ii) bootstrap standard deviation intervals (SDB). We compare these methods against a plug-in Gaussian approximation baseline, based on an estimated asymptotic covariance (PE). We provide full experimental details, additional confidence levels, % standard deviations,
and extended tables in Appendix~\ref{sec:experiments}.

% Table~\ref{tab:numerical_main} reports representative results for nominal confidence level $\alpha=0.95$. We consider two regimes: a constant number of local updates ($\gamma_\nlupdates=0$) and an increasing number of local updates ($\gamma_\nlupdates=0.2$). When the number of local updates is fixed, bootstrap-based methods achieve better finite-sample coverage than the plug-in estimator for moderate iteration budgets. In contrast, when the number of local updates increases, all methods become comparable for sufficiently large $T$, indicating that the plug-in covariance estimate improves as more local samples are collected.

We report results for nominal confidence level $\alpha=0.95$ using step size $\gamma_\eta=0.6$ and two different regimes of local updates: in Table~\ref{tab:numerical_main_gamma0} with a constant number of local updates ($\gamma_\nlupdates=0$), and in Table
~\ref{tab:numerical_main_gamma02} with an increasing number of local updates ($\gamma_\nlupdates=0.2$).
In both cases, our methods (SDB and EQ) achieve better coverage during early iterations ($T=2000$ and $T=6000$), highlighting the superiority of these methods in regimes that are far from the asymptotic behavior. 
% When the number of local updates is fixed, bootstrap-based methods achieve slightly better finite-sample coverage than the plug-in estimator for moderate iteration budgets. 
% In contrast, when the number of local updates increases, all methods become comparable for sufficiently large $T$, indicating that the plug-in covariance estimate improves as more local samples are collected.
In contrast, as the number of total samples grow, all methods reach comparable coverage. This indicates that the plug-in covariance estimate improves as more samples are collected, and the algorithm reaches its asymptotic behavior.
Finally, we note that using a constant number of local step sizes may allow for identifying confidence intervals with low precision earlier, while using an increasing number of local step sizes allows to reach good coverage with fewer communications.

\begin{table}[t]
\centering
\caption{Coverage for nominal confidence level $\alpha=0.95$ with $\gamma_\nlupdates = 0.2$. Standard deviations are shown as subscripts.}% Additional experiments are reported in Appendix~\ref{sec:experiments}.}
\label{tab:numerical_main_gamma02}
\small
\setlength{\tabcolsep}{4pt}
\begin{tabularx}{\columnwidth}{c|YYY}
\toprule
$T$ & SDB (ours) & EQ (ours) & PE \\
\midrule
2000
& $0.398_{\scriptscriptstyle \pm 0.015}$
& $0.404_{\scriptscriptstyle \pm 0.015}$
& $0.376_{\scriptscriptstyle \pm 0.015}$ \\
6000
& $0.959_{\scriptscriptstyle \pm 0.006}$
& $0.961_{\scriptscriptstyle \pm 0.006}$
& $0.959_{\scriptscriptstyle \pm 0.006}$ \\
10000
& $0.946_{\scriptscriptstyle \pm 0.007}$
& $0.949_{\scriptscriptstyle \pm 0.007}$
& $0.946_{\scriptscriptstyle \pm 0.007}$ \\
14000
& $0.950_{\scriptscriptstyle \pm 0.007}$
& $0.950_{\scriptscriptstyle \pm 0.007}$
& $0.952_{\scriptscriptstyle \pm 0.007}$ \\
\bottomrule
\end{tabularx}
\end{table}

% \begin{table}[ht]
% \centering
% \caption{Representative coverage results for nominal confidence level $\alpha=0.95$. Standard deviations are shown as subscripts. Additional experiments are reported in Appendix~\ref{sec:experiments}.}
% \label{tab:numerical_main}
% \small
% \setlength{\tabcolsep}{4pt}
% \begin{tabular}{c|c|ccc}
% \hline
% $\gamma_\nlupdates$ & $T$ & SDB & EQ & PE \\
% \hline
% 0 & 2000
% & $0.773_{\scriptscriptstyle \pm 0.013}$
% & $0.780_{\scriptscriptstyle \pm 0.013}$
% & $0.754_{\scriptscriptstyle \pm 0.013}$ \\
% 0 & 6000
% & $0.936_{\scriptscriptstyle \pm 0.008}$
% & $0.941_{\scriptscriptstyle \pm 0.007}$
% & $0.931_{\scriptscriptstyle \pm 0.008}$ \\
% \\
% \hline
% 0.2 & 2000
% & $0.398_{\scriptscriptstyle \pm 0.015}$
% & $0.404_{\scriptscriptstyle \pm 0.015}$
% & $0.376_{\scriptscriptstyle \pm 0.015}$ \\
% 0.2 & 6000
% & $0.959_{\scriptscriptstyle \pm 0.006}$
% & $0.961_{\scriptscriptstyle \pm 0.006}$
% & $0.959_{\scriptscriptstyle \pm 0.006}$ \\
% \hline
% \end{tabular}
% \end{table}

%% file: conclusion.tex
We derived non-asymptotic Berry-Esseen–type bounds for the last iterate of FedLSA, providing finite-sample Gaussian approximation guarantees under heterogeneous data and communication constraints. Our results quantify the joint impact of the number of agents, step-size decay, and the schedule of local updates on the normal approximation rate.
To derive these results, we established a new analysis of the convergence of FedLSA with decreasing step size and increasing number of local iterations, providing the first bounds on higher-order moments for this method. Based on this result, we established the validity of an online multiplier bootstrap procedure for FedLSA, which approximates the distribution of the normalized last iterate without requiring estimation of the asymptotic covariance matrix.
Our results enable practical construction of confidence intervals and uncertainty quantification in federated settings.

\begin{acknowledgements}
The work of Ilya Levin, Maksim Shuklin, and Sergey Samsonov is an output of a research project HSE-BR-2025-019 implemented as part of the Basic Research Program at HSE University. This research was supported in part through computational resources of HPC facilities at HSE University \cite{kostenetskiy2021hpc}.
The work of Paul Mangold is supported by Hi! PARIS and ANR/France 2030 program (ANR-23-IACL-0005).
\end{acknowledgements}

%% file: appendix/fedlsa_decreasing.tex
We are trying to solve the following  linear equation system across $N$ clients $c = 1 \ldots N$:
\begin{align*}
    \bA \thetalim = \barb
\end{align*}
where $\bA = \frac1N \sum \limits_{c = 1}^N \bA[c]$ and $\barb = \frac1N \sum \limits_{c = 1}^N \barb[c]$.

Let $\thetalim$ be the unique solution of $\bA \thetalim = \barb$ and $\thetalim[c]$ the unique solution of $\bA[c] \thetalim[c] = \barb[c]$. Both $\bA[c]$ and $\barb[c]$ are not observed directly in our setting. Instead, we could observe $A^c(Z_t^c)$ and $b^c(Z_t^c)$ during $t$-th iteration at $c$-th client. For $z \in \Zset$ we denote the stochastic part of the matrices and vectors $A^c(z)$ and $b^c(z)$
\begin{align}
    \zmfuncA[c]{z} = \zfuncA[c]{z} - \bA[c]
    \eqsp,
    \quad
    \zmfuncb[c]{z} = \zfuncb[c]{z} - \barb[c]
    \eqsp,
\end{align}
    as well as the noise representations
    \begin{align*}
    \funcnoise[c]{z} 
    =
    \zmfuncA[c]{z} \thetalim[c] - \zmfuncb[c]{z}
    \eqsp,
    \qquad \quad
    &\omega^c(z) = \zmfuncA[c]{z} \thetalim - \zmfuncb[c]{z} 
    \eqsp.
\end{align*}
The first corresponds to the intrinsic stochastic noise around the local solution,
while the second measures the stochastic perturbation evaluated at the global solution. To control stochastic fluctuations, we define uniform bounds
\begin{align*}
    \supconsteps = \max_{c \in [\nagent]} \sup_{z \in \Zset} \norm{\funcnoise[c]{z}} \eqsp, \quad \bConst{A} = \max_{c \in [\nagent]} \sup_{z\in \Zset} \norm{\zfuncA[c]{z}} \eqsp.
\end{align*}
and the two following quantities, which measure variance
\begin{align}
\label{eq:def-heter-and-noise}
    \varhet^2 = \mathbb{E}_{c}\!\left[ \left\| \Sigma_{\mathbf{A}}^{c} \right\| \, \|\theta_{\star}^{c} - \theta_{\star}\|^{2} \right]
    \eqsp, 
    \qquad 
    \varbound^2 = \mathbb{E}_{c}\!\left[ \mathrm{Tr}\!\left(\Sigma_{\varepsilon}^{c}\right) \right]
    \eqsp,
    \qquad
    \varboundA^2 = \frac{1}{\nagent} \sum \limits_{c = 1}^\nagent \PE[\norm{\Sigma_{\mathbf{A}}^c}^2]
\end{align}
Finally, we define the two following quantities, that define first and second order heterogeneity
\begin{align}
\label{eq:def-heterogeneity-measure}
    \firsthgty^2
    = 
    \frac{1}{\nagent} \sum_{c=1}^\nagent \norm{ \bA[c] (\thetalim[c] - \thetalim) }^2
    \eqsp,
    \qquad 
    \secondhgty^2
    = 
    \frac{1}{\nagent} \sum_{c=1}^\nagent \norm{ \bA[c] - \bA }^2
    \eqsp.
\end{align}

% During all this work, we make the following assumption:

% \begin{assumption}
% \label{ass: linear-decay}
%     There exist $a > 0, \eta_\infty > 0$, such that $\eta_\infty a \leq \frac{1}{2}$ and for any $0 < \eta < \eta_\infty$, $c = 1 \ldots N$ and $u \in \mathbb{R}^d$ it holds $\PE^{1/2} [ \| (\Id - \eta A^c (Z)) u \|^2 ] \leq (1 - \eta a) \| u \|$, for $Z \sim \pi_c$ with a distribution $\pi_c$ over $\mathsf{Z}$.
% \end{assumption}

\subsection{Expansion of the error of FedLSA}
\label{appendix: mse-bound-estimation}
We first express a single local update of \textsc{FedLSA} in terms of the local error 
$\theta_{t,h}^c - \thetalim[c]$. Using $\bA[c]\thetalim[c] = \barb[c]$, we obtain
% Let us write a single local iteration of FedLSA in terms of $\theta_{t, h}^c - \thetalim[c]$ error:
\begin{align}
\nonumber
    \theta_{t, h}^c - \thetalim[c] 
    & = (\Id - \eta_{t, h} \zfuncA[c]{Z_{t,h}^c}) (\theta_{t, h - 1}^c - \thetalim[c]) - \eta_{t, h} \left( \zfuncA[c]{Z_{t,h}^c} \thetalim[c] - \zfuncb[c]{Z_{t,h}^c} \right)
    \\
\label{eq:decomp_last_agent_rec}
    & = (\Id - \eta_{t, h} \zfuncA[c]{Z_{t,h}^c}) (\theta_{t, h - 1}^c - \thetalim[c]) - \eta_{t, h} \funcnoise[c]{Z_{t,h}^c}
    \eqsp.
\end{align}
Iterating the recursion over $h = 1,\dots,H_t$ yields
\begin{align}
\label{eq:decomp_last_iter_agent_t}
    \theta_{t, H_t}^c - \thetalim[c] 
    & = 
    \Gamma^c_{t, 1:H_t} (\theta_{t, 0}^c - \thetalim[c]) 
    - \sum \limits_{h = 1}^{H_t} \eta_{t, h} \Gamma^c_{t, h+1:H_t} \funcnoise[c]{Z_{t,h}^c}
    \eqsp,
\end{align}
where we defined the following matrix products
\begin{align}
\label{eq:def-gamma-prod}
    \Gamma^c_{t, l:r} = \prod \limits_{h = l}^r (\Id - \eta_{t, h} \zfuncA[c]{Z_{t,h}^c})
    \eqsp,
    \qquad
    \Gamma^c_{t} = \Gamma^c_{t,1:\nlupdates_t}
    \eqsp,
    \qquad
    \Gammaavg_t = \frac{1}{\nagent} \sum_{c=1}^\nagent \Gamma^c_t
    \eqsp.
\end{align}
We now derive the global recursion for the error
$\theta_t - \thetalim$. Since
$\theta_t = \frac{1}{\nagent}\sum_{c=1}^{\nagent}\theta_{t,H_t}^c$,
we write
\begin{align*}
\theta_t - \thetalim
&=
\frac{1}{\nagent}
\sum \limits_{c = 1}^{\nagent} \theta_{t, \nlupdates_t}^c - \thetalim
= 
\frac{1}{\nagent}
\sum_{c=1}^{\nagent}
(\theta_{t,H_t}^c - \thetalim[c])
+
\frac{1}{\nagent}
\sum_{c=1}^{\nagent}
(\thetalim[c] - \thetalim).
\end{align*}
Substituting \eqref{eq:decomp_last_iter_agent_t} and using
$\theta_{t,0}^c = \theta_{t-1}$ gives
\begin{align*}
\theta_t - \thetalim
&=
\frac{1}{\nagent}
\sum_{c=1}^{\nagent}
\Gamma_t^c
(\theta_{t-1} - \thetalim[c])
-
\frac{1}{\nagent}
\sum_{c=1}^{\nagent}
\sum_{h=1}^{H_t}
\eta_{t,h}
\Gamma^c_{t,h+1:H_t}
\funcnoise[c]{Z_{t,h}^c} +
\frac{1}{\nagent}
\sum_{c=1}^{\nagent}
(\thetalim[c] - \thetalim).
\end{align*}

Rearranging terms, we obtain the compact global error recursion
\begin{align*}
\theta_t - \thetalim
&=
\Gammaavg_t(\theta_{t-1} - \thetalim)
+
\frac{1}{\nagent}
\sum_{c=1}^{\nagent}
(\Id - \Gamma_t^c)
(\thetalim[c] - \thetalim)
\\
&\quad
-
\frac{1}{\nagent}
\sum_{c=1}^{\nagent}
\sum_{h=1}^{H_t}
\eta_{t,h}
\Gamma^c_{t,h+1:H_t}
\funcnoise[c]{Z_{t,h}^c}.
\end{align*}
We may decompose $\theta_t - \thetalim$ as
\begin{align}
\label{eq:decomp-theta-t-one-iter}
    \theta_t - \thetalim = \Gammaavg_t (\theta_{t - 1} - \thetalim) + \rhoavg_{t} + \tauavg_{t} + \varphiavg_{t}
    \eqsp,
\end{align}
where we introduced following notations for \textsf{FedLSA}'s deterministic bias $\rhoavg_t$ and its fluctuations $\tauavg_t$, as well as the fluctuations of the updates $\varphiavg_t$,
\begin{gather}
\label{eq:def-rhoavg-tauavg}
    \rhoavg_{t} = \frac1N \sum \limits_{c = 1}^N (\Id - G_t^c) (\thetalim[c] - \thetalim)
    \eqsp,
    \quad
    \tauavg_{t} = \frac1N \sum \limits_{c = 1}^N (G_t^c - \Gamma^c_{t, 1:H_t}) (\thetalim[c] - \thetalim)
    \eqsp,
    \\
\label{eq:def-varphiavg}
    \varphiavg_{t} = -\frac1N \sum \limits_{c = 1}^N \sum \limits_{h = 1}^{H_t} \eta_{t, h} \Gamma_{t, h + 1:H_t}^c \varepsilon^c (Z^c_{t, h})
    \eqsp, 
\end{gather}
as well as the following matrix notations, representing the expected value of the $\Gamma_{t,l:r}^c$, $\Gamma_t^c$, and $\Gamma_t^{\textrm{avg}}$ matrices
\begin{align}
\label{eq:def-G-deterministic-contract}
G_{t,l:r}^c = \prod_{h = l}^{r} (\Id - \eta_{t, h} \bA[c])
\eqsp,
\qquad
G_t^c = G_{t,1:\nlupdates_t}
\eqsp,
\qquad
\Gavg_t = \frac{1}{\nagent} \sum_{c=1}^\nagent G_t^c
\eqsp.
\end{align}
Note that
\begin{align*}
    \mathbb{E}[\Gamma_{t,l:r}^c] = G_{t,l:r}^c,
    \qquad
    \mathbb{E}[\Gamma_t^c] = G_t^c \eqsp.
\end{align*}
Unrolling \eqref{eq:decomp-theta-t-one-iter}, we obtain the following decomposition of the global iterate's error at step $t$,
\begin{align}
\nonumber
    \theta_t - \thetalim
    & =
    \prod_{s = 1}^t \Gammaavg_s (\theta_0 - \thetalim) 
    + \sum \limits_{s = 1}^t \prod_{j = s + 1}^{t} \Gammaavg_j (\rhoavg_{s} + \tauavg_{s} + \varphiavg_{s}) 
    \\
\label{eq: decomp_last_iter_t}
    & = 
    \tilde{\theta}_t^{\text{(tr)}} + \tilde{\theta}_t^{\text{(bi, bi)}} + \tilde{\theta}_t^{\text{(fl, bi)}} + \tilde{\theta}_t^{\text{(fl)}}
    \eqsp.
\end{align}
Where we introduced the notations
\begin{align}
% \label{eq:def-theta-tr}
    \tilde{\theta}_t^{(\mathrm{tr})}
    & = 
    \prod_{s = 1}^t \Gammaavg_s \left\{ \theta_0 - \thetalim \right\} 
    \\
% \label{eq:def-theta-bibi}
    \tilde{\theta}_t^{\text{(bi, bi)}} 
    & =
    \sum \limits_{s = 1}^t \mathbb{E} \Big[\prod_{i = s + 1}^t \Gammaavg_i \Big] \rhoavg_{s} 
    \\
% \label{eq:def-theta-flbi}
    \tilde{\theta}_t^{\text{(fl, bi)}}
    & = \sum \limits_{s = 1}^t \prod_{i = s + 1}^t \Gammaavg_i \tauavg_{s} + \Delta_{s, t} \rhoavg_{s} 
    \\
% \label{eq:def-theta-fl}
    \tilde{\theta}_t^{\text{(fl)}} 
    & = \sum \limits_{s = 1}^t \Big\{  \prod \limits_{i = s + 1}^t \Gammaavg_i \cdot \varphiavg_{s} \Big\}
    \eqsp,
\end{align}
where we also defined $\Delta_{s, t} = \left \{ \prod_{i = s + 1}^t \Gammaavg_i \right \} - \mathbb{E} \Big[\prod_{i = s + 1}^t \Gammaavg_i \Big]$, the fluctuations of the matrix product $\Gammaavg_i$.

Finally, we define the past and future filtrations
\begin{align}
\mathcal{F}_{s,h}^-
&=
\sigma\!\left(
Z_{t,k}^c:
t<s
\text{ or }
(t=s,\ k\le h)
\right),
\\
\mathcal{F}_{s,h}^+
&=
\sigma\!\left(
Z_{t,k}^c:
t>s
\text{ or }
(t=s,\ k\ge h)
\right).
\end{align}
Here, $\mathcal{F}_{s,h}^-$ represents the noise up to iteration $(s,h)$, 
while $\mathcal{F}_{s,h}^+$ represents the future noise starting from $(s,h)$.

% Finally, we define the effective sum of step size and local update, for $t \ge s$,
% \begin{align*}
% \sumHT{s:t} = \sum_{i=s}^t \step_{i} \nlupdates_i
% \eqsp.
% \end{align*}

\subsection{Proofs for the second moment bound $\PE^{1/2}[\norm{ \theta_t - \thetalim }^2]$}

\paragraph{Transient term.}
First, we give a bound on the transient term $\tilde{\theta}^{\sf{(tr)}}_t$, which quantifies the impact of the initialization.
\begin{lemma}
\label{lemma: theta-tr-estimation}
    Assume \Cref{ass: linear-decay}(2), \Cref{assum:noise-level-flsa}, and let $\step_s \le \step_{\infty,2}$. Then, it holds   
    \begin{align*}
        \PE^{1/2}[\|\tilde{\theta}^{\sf{(tr)}}_t\|^2] \leq \prod_{s = 1}^t (1 - \eta_{s} a)^{H_s} \| \theta_0 - \thetalim\|
        \eqsp.
    \end{align*}
\end{lemma}
\begin{proof}
Expanding $\tilde{\theta}_t^{\mathrm{(tr)}}$ using \eqref{eq:def-theta-tr} gives
    \begin{align*}
        \PE^{1/2}[\| \tilde{\theta}_t^{\mathrm{(tr)}} \|^2]
        &=
        \PE^{1/2} \Big[ \Big \| \prod_{s = 1}^t \Gammaavg_s \cdot (\theta_0 - \thetalim) \Big \|^2 \Big]
        =
        \PE^{1/2} \Big[ \Big \| \frac{1}{N} \sum \limits_{c = 1}^N \Gamma_{t, 1:H_t}^c \prod_{s = 1}^{t - 1} \Gammaavg_s \cdot (\theta_0 - \thetalim) \Big \|^2 \Big] 
        \eqsp.
    \end{align*}
    Using \eqref{ass: linear-decay} recursively then allows to bound
    \begin{align*}
    \PE^{1/2}[\| \tilde{\theta}_t^{\mathrm{(tr)}} \|^2]
        & \leq
        \frac{1}{N} \sum \limits_{c = 1}^N \PE^{1/2}\Big[\bnorm{ \Gamma_{t, 1:H_t} \prod_{s = 1}^{t - 1} \Gammaavg_{s} \cdot (\theta_0 - \thetalim) }^2 \Big]
        % \leq
        % (1 - \eta_{t} a)^{H_t} \PE^{1/2} \Big[ \Big \| \prod_{s = 1}^{t - 1} \Gammaavg_s \cdot (\theta_0 - \thetalim) \Big \|^2 \Big]
        \leq 
        \prod_{s = 1}^t (1 - \eta_{s} a)^{H_s} \|\theta_0 - \thetalim\|
        \eqsp,
    \end{align*}
    which is the result of the lemma.
\end{proof}

\paragraph{Bound on heterogeneity.}
To assess the impact of heterogeneity, we first give a bound on the accumulated heterogeneity between two communications.
\begin{lemma}
\label{lemma: rho-t-estimation}
    Assume \Cref{assum:noise-level-flsa}, \Cref{ass: linear-decay}, and let $\step_s \le \step_\infty$ for all $s \ge 0$, then
    \begin{align*}
        \| \rhoavg_{t} \| \leq \frac{\eta_t^2 H_t^2}{2} \firsthgty \secondhgty
        \eqsp.
    \end{align*}
\end{lemma}
\begin{proof}
We start by recalling the definition \eqref{eq:def-rhoavg-tauavg} of $\rhoavg_t$, and expanding the difference of matrices using \Cref{lemma: difference-of-product}
    \begin{align*}
        \rhoavg_{t}
        & = \frac1N \sum \limits_{c = 1}^N (\Id- G_t^c ) (\thetalim[c] - \thetalim)
          = \frac1N \sum \limits_{c = 1}^N 
        \sum_{h=0}^{\nlupdates_t-1} \step_t (\Id - \step_t \bA[c] )^h \bA[c] (\thetalim[c] - \thetalim)
        \eqsp.
    \end{align*}
    where we used the fact that $G_t^c  = (\Id- \eta_t \bA[c])^{H_t}$ is defined in \eqref{eq:def-G-deterministic-contract}.
    Moreover, since $\sum_{c=1}^\nagent \bA[c](\thetalim[c] - \thetalim) = 0$, we can further write
    \begin{align*}
        \rhoavg_{t}
        & = \frac1N \sum \limits_{c = 1}^N 
        \sum_{h=0}^{\nlupdates_t-1} \step_t 
        \Big( (\Id - \step_t \bA[c] )^h 
        - (\Id - \step_t \bA )^h \Big)
        \bA[c] (\thetalim[c] - \thetalim)
        \\
        & = \frac{\step_t^2}{\nagent} \sum \limits_{c = 1}^N 
        \sum_{h=0}^{\nlupdates_t-1}
        \sum_{\ell=1}^{h-1} 
         (\Id - \step_t \bA[c] )^{\ell-1} ( \bA[c] - \bA )
     (\Id - \step_t \bA )^{h - \ell - 1}
        \bA[c] (\thetalim[c] - \thetalim)
        \eqsp.
    \end{align*}
    Taking the norm of this equality, using the triangle inequality, bounding matrices using $\norm{ \Id - \step_t \bA[c] } \le 1$, and computing the sum, we obtain
    \begin{align*}
        \norm{ \rhoavg_{t} }
        & \le \frac{\step_t^2 \nlupdates_t (\nlupdates_t - 1)}{2 \nagent} \sum \limits_{c = 1}^N \norm{ \bA[c] - \bA }
        \norm{ \bA[c] (\thetalim[c] - \thetalim) }
        \\
        & \le \frac{\step_t^2 \nlupdates_t (\nlupdates_t - 1)}{2} \Big( \frac{1}{\nagent} \sum \limits_{c = 1}^N \norm{ \bA[c] - \bA }^2 \Big)^{1/2}
        \Big( \frac{1}{\nagent} \sum_{c=1}^\nagent 
        \norm{ \bA[c] (\thetalim[c] - \thetalim) }^2 \Big)^{1/2}
        \eqsp,
    \end{align*}
    where we used the Cauchy-Schwarz inequality in the second inequality.
    The result follows from the definition of heterogeneity \eqref{eq:def-heterogeneity-measure}
\end{proof}
And we can track the propagation of this bias through successive communications.

\paragraph{Bound on the bias.} Next, we provide the bound on $\tilde{\theta}_t^{\sf{(bi, bi)}}$ second moment bounds.
\begin{lemma}
\label{lemma: theta-bi-bi-estimation}
    Under assumptions \Cref{ass: linear-decay}(2), \Cref{assum:noise-level-flsa} and \Cref{ass: lr-polynomial-decay} it holds that:
    % \begin{align*}
    %     \norm{\tilde{\theta}_t^{\sf{(bi, bi)}}} \leq \frac{4 \step^2 H^2}{1 - \mathrm{e}^{-a \step H}} \firsthgty \secondhgty \eqsp. %\frac{1}{1 - \mathrm{e}^{-a L_0}}
    % \end{align*}
    \begin{align*}
        \norm{\tilde{\theta}_t^{\sf{(bi, bi)}}}
        \leq
        \zeta_1 \zeta_2\sum_{s=1}^t \step_s^2 \nlupdates_s^2 \exp{\left( -a\sum_{i=s+1}^t \step_i \nlupdates_i \right)}
        %8a^{-1}\zeta_1 \zeta_2 \step\nlupdates (1+t)^{-\gamma}
        \eqsp.
    \end{align*}
\end{lemma}
\begin{proof}
    By the definition of $\tilde{\theta}_t^{\text{(bi, bi)}}$ in \eqref{eq:def-theta-bibi}
    \begin{align*}
        \PE^{1/2}[\norm{ \tilde{\theta}_t^{\text{(bi, bi)}} }^2] &= \bnorm{ \sum \limits_{s = 1}^t \PE \Big[ \prod_{i = s + 1}^t \Gammaavg_i \Big] \rhoavg_s }% \\
       % &
       = \bnorm{ \sum \limits_{s = 1}^t \PE \Big[ \prod_{i = s + 1}^t \Gammaavg_i \, \rhoavg_s \Big]  }
       \eqsp.
    \end{align*}
    By applying \Cref{ass: linear-decay} recursively along with \Cref{lemma: rho-t-estimation}, we have % and \Cref{lemma: legendary-sum}, we have
    \begin{align*}
        \PE^{1/2}[\norm{ \tilde{\theta}_t^{\text{(bi, bi)}} }^2] &\leq \sum \limits_{s = 1}^t \prod_{i = s + 1}^t (1 - \step_i a)^{H_i} \, \frac{\step_s^2 H_s^2}{2} \firsthgty \secondhgty 
    %\\
        % &\leq 4\step^2 \nlupdates ^2 \firsthgty \secondhgty \sum \limits_{s = 1}^t \mathrm{e}^{-a \step \nlupdates (t - s)}
        %&
        \leq \zeta_1 \zeta_2\sum_{s=1}^t \step_s^2 \nlupdates_s^2 \exp{\left( -a\sum_{i=s+1}^t \step_i \nlupdates_i \right)} %\\
       % &\leq 8a^{-1}\zeta_1 \zeta_2 \step\nlupdates (1+t)^{-\gamma}
       \eqsp,
    \end{align*}
    which is the result of the lemma.
\end{proof}

\paragraph{Bound on fluctuations.}
We then bound the fluctuations of the algorithm, which decompose in a term $\tilde{\theta}_t^{\sf{(fl)}}$, corresponding to the fluctuations of the updates, and a term $\tilde{\theta}_t^{\text{(fl, bi)}}$ which quantifies the fluctuations of the bias. 
We start with $\tilde{\theta}_t^{\sf{(fl)}}$, showing that the algorithm's variance scales in $1/\nagent$.
\begin{lemma}
\label{lemma: theta-fl-estimation}
    Assume \Cref{ass: linear-decay}(2), \Cref{assum:noise-level-flsa} and let $\step_s \le \step_{\infty,2}$ for all $s \ge 0$. Then, it holds   
    \begin{align*}
        \PE [ \| \tilde{\theta}_t^{\sf{(fl)}} \|^2 ] \leq \sum \limits_{s = 1}^t 
        % \prod_{j = s + 1}^t (1 - \eta_j a)^{2H_j} 
        \exp(- 2a \sumHT{s+1:t})
        \cdot
        (H_s \eta_s^2 \wedge \frac{\eta_s }{a})
        \cdot 
        \frac{\varbound^2}{N}
        \eqsp.
    \end{align*}
\end{lemma}
\begin{proof}
Defining $X_s = \prod_{i = s + 1}^t \Gammaavg_i \cdot \varphiavg_{s}$ for $s \ge 0$, we can expand \eqref{eq:def-theta-fl} as
    \begin{align}
    \label{eq:expansion-theta-fl-XsXr}
        \PE [ \| \tilde{\theta}_t^{\text{(fl)}} \|^2 ] = \PE \Big[ \Big \| \sum \limits_{s = 1}^t \prod \limits_{i = s + 1}^t \Gammaavg_i \cdot \varphiavg_{s} \Big \|^2 \Big] =  \PE \Big[ \Big \| \sum \limits_{s = 1}^t X_s \Big \|^2  \Big] = \PE \Big[ \sum \limits_{s, r = 1}^t X_s^\top X_r \Big]
        \eqsp.
    \end{align}
    Assuming $r > s$ we get the following:
    \begin{align*}
        \PE \Big[ X_s^\top X_r \Big] 
        &=
        \PE \Big[ \PE \Big[ X_s^\top X_r \Big| \mathcal{F}_{s + 1, 1}^+ \Big] \Big]
        = 
        \PE \left[ \PE \Big[ \left( \varphiavg_{r} \right)^\top \cdot \prod_{i = r + 1}^t \Gammaavg_i \times \prod_{j = s + 1}^t \Gammaavg_j \cdot \varphiavg_{s} \Big| \mathcal{F}_{s + 1, 1}^+ \Big] \right] 
        \\
        & =
        \PE \left[ \left( \varphiavg_{r} \right)^\top \cdot \prod_{i = r + 1}^t \Gammaavg_i \times \prod_{j = s + 1}^t \Gammaavg_j \cdot \underbrace{\PE \Big[ \varphiavg_{s} \Big| \mathcal{F}_{s + 1, 1}^+ \Big]}_\text{0} \right] 
        = 0
        \eqsp.
    \end{align*}
    Plugging this in \eqref{eq:expansion-theta-fl-XsXr} gives
    \begin{align*}
        \PE [ \| \tilde{\theta}_t^{\text{(fl)}} \|^2 ] 
        =
        \sum \limits_{s = 1}^t 
        \PE \Big[ \bnorm{ \prod_{i = s + 1}^t \Gammaavg_i \cdot \varphiavg_{s} }^2 \Big]
        \eqsp.
    \end{align*}
    We now bound each term of this sum using \Cref{ass: linear-decay} conditionally, which gives
    \begin{align}
    \nonumber
    \PE^{1/2} \Big[ \bnorm{ \prod_{i = s + 1}^t \Gammaavg_i \cdot \varphiavg_{s} }^2 \Big]
    & \leq 
    \frac1N \sum \limits_{c = 1}^N 
    \PE^{1/2} \Big[ \bnorm{ \Gamma_{t, 1:H_t}^c \prod_{i = s + 1}^{t - 1} \Gammaavg_i \cdot \varphiavg_{s} }^2 \Big] 
    \\
    \nonumber
    & \leq
    (1 - \eta_{t} a)^{H_t} \PE^{1/2} \Big[\bnorm{ \prod_{i = s + 1}^{t - 1} \bar{\Gamma}_i \cdot \varphiavg_{s} }^2 \Big]  
    \\
    \label{eq:proof-thetafl-bound-prod-phi}
    & \leq
    \prod_{j = s + 1}^t (1 - \eta_{j} a)^{H_j} \PE^{1/2} [\norm{ \varphiavg_{s} }^2 ]
    \eqsp.
    \end{align}
    Finally, we bound $\PE^{1/2} [\| \varphiavg_{s} \|^2]$. As $\PE[\Gamma_{s, h + 1:H_s}^c \funcnoise[c]{Z_{s, h}^c} | \mathcal{F}_{s, h + 1}^+] = 0$, 
    \begin{align*}
        \PE [\| \varphiavg_{s} \|^2] 
        & =
        \frac{1}{N^2} \sum \limits_{c = 1}^N \sum \limits_{h = 1}^{H_s} \PE[\| \eta_{s} \Gamma_{s, h+1:H_s}^c \funcnoise[c]{Z_{s, h}^c} \|^2] 
        % \leq \{ \text{as in the previous case} \} \leq
        \\
        & \leq
        \frac{1}{N^2} \sum \limits_{c = 1}^N \sum \limits_{h = 1}^{H_s} \eta_{s}^2 (1 - \eta_{s} a)^{2 (H_s - h)} \PE[\| \funcnoise[c]{Z_{s, h}^c} \|^2] \\
        & =
        \frac{\eta_s^2}{N^2} \sum \limits_{c = 1}^N \PE[\| \funcnoise[c]{Z_{s, h}^c}\|^2]\cdot \sum \limits_{h = 1}^{H_s} (1 - \eta_s)^{2(H_s - h)}
        \eqsp.
    \end{align*}
    Since $\sum \limits_{h = 0}^{H_s - 1} (1 - \eta_s a)^{2h} \leq H_s \wedge \frac{1}{\eta_s a}$ and $\frac{1}{N} \sum \limits_{c = 1}^N \PE[\| \funcnoise[c]{Z_{s, h}^c} \|^2] = \PE_c[\mathrm{Tr}(\Sigma_\varepsilon^c)] = \varbound^2$, we have
    \begin{align}
    \label{eq:proof-thetafl-bound-phi}
        \PE[\| \varphiavg_{s} \|^2 ] \leq \frac{\varbound^2}{N} \cdot (H_s \eta_s^2 \wedge \frac{\eta_s }{a})
        \eqsp.
    \end{align}
    Plugging \eqref{eq:proof-thetafl-bound-phi} in \eqref{eq:proof-thetafl-bound-prod-phi} and using $\prod_{j=s+1}^t (1 - \step_j a)^{\nlupdates_j} \le \exp(-a \sumHT{s+1:t})$ gives the result.
    % Putting everything together we get
    % \begin{align*}
    %     \PE [ \| \tilde{\theta}_t^{\text{(fl)}} \|^2 ] \leq \sum \limits_{s = 1}^t \prod_{j = s + 1}^t (1 - \eta_j a)^{2H_j} \cdot \frac{\varbound^2}{N} \cdot (H_s \eta_s^2 \wedge \frac{\eta_s }{a})
    %     \eqsp,
    % \end{align*}
\end{proof}

\paragraph{Fluctuations of the bias}.  We now bound $\tilde{\theta}_t^{\text{(fl, bi)}}$, the fluctuations of the bias. 
We show that it also scales in $1/\nagent$.
\begin{lemma}
\label{lemma: lemma-theta-fl-bi}
    Assume \Cref{ass: linear-decay}(2), \Cref{assum:noise-level-flsa} and let $\step_s \le \step_{\infty,2}$, for all $s \ge 0$, be a decreasing sequence. Then, it holds   
    \begin{align*}
        \PE^{1/2} [\| \tilde{\theta}_t^{\sf{(fl, bi)}} \|^2]
        & \leq 
        \sqrt{\frac{4 \varhet^2}{N} }\sqrt{\sum \limits_{s = 1}^t
         \eta_s^2 H_s 
        \exp(-2a \sum_{r = s }^t \eta_r H_r)
         } 
         +
         \sqrt{ \frac{4 \varboundA^2 \firsthgty^2 \secondhgty^2 }{\nagent a e} }
         \sqrt{
    \sum_{s=1}^t \step_s^5 H_s^4  \exp \big( - a \sum_{r = s + 1}^t \eta_r H_r \big) }
        \eqsp.
    \end{align*}
    % \begin{align*}
    %     \PE^{1/2} [\| \tilde{\theta}_t^{\sf{(fl, bi)}} \|^2]
    %     & \leq 
    %     % \sqrt{ \frac{2 \tilde{v}_{\text{heter}} }{\mathrm{e} \cdot aN} } \sum \limits_{s = 1}^t \mathrm{e}^{-a \sqrt{\eta_s} \sum \limits_{i = s + 1}^t H_i \eta_i}
    %     \sqrt{\frac{2 \varhet^2}{N a \mathrm{e}}} \sum \limits_{s = 1}^t \sqrt{\eta_s} \exp{\Big( -a \sum \limits_{i = s + 1}^t H_i \eta_i \Big)}
    %     \\
    %     & \quad  
    %     + \sqrt{\frac{\PE_c[\| \noisecovA[c] \|]}{\nagent}}
    %     \sum \limits_{s = 1}^t \left( \| \rhoavg_s\| \mathrm{e}^{-a \sum \limits_{i = s + 1}^t H_i \eta_i} \sqrt{\sum \limits_{i = s + 1}^t 4 H_i \eta_i^2 } \right)
    %     \eqsp.
    %     % &+ \frac{\sqrt{\PE_c[\| \noisecovA[c] \|]}}{N^\frac{1}{2}} \sum \limits_{s = 1}^t \left( \| \rhoavg_s\| \mathrm{e}^{-a \sum \limits_{i = s + 1}^t H_i \eta_i} \sqrt{\sum \limits_{i = s + 1}^t \frac{H_i \eta_i^2}{(1 - \eta_i a)^2}} \right)
    % \end{align*}
\end{lemma}
\begin{proof}
    Starting from \eqref{eq:def-theta-flbi}, we have
    \begin{align}
    \label{eq:decomp-flbi-T1-T2}
        \tilde{\theta}_t^{\text{(fl, bi)}} 
        =
        \underbrace{\sum \limits_{s = 1}^t \prod_{i = s + 1}^t \Gammaavg_i \tauavg_{s}}_{\term{T_1}} 
        + 
        \underbrace{\sum \limits_{s = 1}^t \left \{ \prod_{i = s + 1}^t \Gammaavg_i - \mathbb{E} \Big[\prod_{i = s + 1}^t \Gamma_i^{\textrm{avg}} \Big]  \right \} \rhoavg_{s} }_{\term{T_2}}
        \eqsp.
    \end{align}

    \textbf{Bound on $\term{T_1}$.}
    We start by deriving a bound on $\tauavg_{s} = \frac{1}{N} \sum \limits_{c = 1}^N (G_s^c - \Gamma_{s}^c) (\thetalim[c] - \thetalim)$,
    \begin{align}
    \label{eq:bound-flbi-decomp-in-Sc}
        \PE[\| \tauavg_{s} \|^2] = \frac{1}{N^2} \sum \limits_{c = 1}^N \underbrace{\PE[\| (G_s^c - \Gamma_{s}^c) (\thetalim[c] - \thetalim) \|^2]}_{\term{U_c}}
        \eqsp.
    \end{align}
    Using \Cref{lemma: difference-of-product} to rewrite the matrix product difference, we have
    \begin{align*}
        \term{U_c}
        & = 
        \PE\Big[\bnorm{ \Big( \prod_{h = 1}^{H_t} (\Id - \eta_t \bA[c]) - \prod_{h = 1}^{H_t} (\Id - \eta_t \zfuncA[c]{Z_{t,h}^c}) \Big) (\thetalim[c] - \thetalim) }^2 \Big] 
        \\
        & =
        \eta_t^2 \PE\Big[ \bnorm{  \sum \limits_{h = 1}^{H_t} (\Id - \eta_t \bA[c])^{h - 1} (\bA[c] - \zfuncA[c]{Z_{t,h}^c}) \Gamma_{t, h + 1:H_t}^c (\thetalim[c] - \thetalim)  }^2 \Big]
        \eqsp.
    \end{align*}
    Noticing that $\PE[(\Id - \eta_t \bA[c])^{h - 1} (\zfuncA[c]{Z_{t,h}^c} - \bA[c]) \Gamma_{t, h + 1:H_t}^c | \mathcal{F}_{t, h + 1}^+] = 0$, we write
    \begin{align*}
        \term{U_c}
        & =
        \eta_t^2 \sum \limits_{h = 1}^{H_t} \PE[\| (\Id - \eta_t \bA[c])^{h - 1} \zmfuncA[c]{Z_{t, h}^c} \Gamma_{t, h+1:H_t}^c (\thetalim[c] - \thetalim) \|^2] 
        \\
        & \leq 
        \eta_t^2 \sum \limits_{h = 1}^{H_t} (1 - \eta_t a)^{2(h - 1)}  \PE \Big[ (\thetalim[c] - \thetalim)^\top \Gamma_{t, h+1:H_t}^{c \top} \tilde{A}^c(Z_{t, h}^c)^\top \tilde{A}^c(Z_{t, h}^c) \Gamma_{t, h+1:H_t}^{c} (\thetalim[c] - \thetalim) \Big]
        \eqsp.
    \end{align*}
    As $v := \Gamma_{t, h+1:H_t}^c (\thetalim[c] - \thetalim)$ is $\mathcal{F}_{t, h + 1}^+$-measurable, the law of total expectation gives
    \begin{align*}
        \PE[ v^\top \tilde{A}^c(Z_{t, h}^c)^\top \tilde{A}^c(Z_{t, h}^c) v ] %&= \PE \Big[ \PE[ v^\top \tilde{A}^c(Z_{t, h}^c)^\top \tilde{A}^c(Z_{t, h}^c) v | \mathcal{F}_{t, h + 1}^+ ] \Big] % = \\
        &= \PE \Big[ v^\top \cdot  \PE[ \tilde{A}^c(Z_{t, h}^c)^\top \tilde{A}^c(Z_{t, h}^c) | \mathcal{F}_{t, h + 1}^+ ]  \cdot v \Big] %= \\
        % &= \PE \Big[v^\top \cdot  \PE[ \tilde{A}^c(Z_{t, h}^c)^\top \tilde{A}^c(Z_{t, h}^c)]  \cdot v \Big] = \\
        = \PE \Big[ v^\top \noisecovA[c] v \Big]
        \eqsp.
    \end{align*}
    Then we can estimate $\PE[v^\top \noisecovA[c] v] $ by using the fact that $\noisecovA[c]$ is symmetric
    \begin{align*}
        %&\PE[ (\thetalim[c] - \thetalim)^\top \Gamma_{t, h+1:H_t}^{c \top} \tilde{A}^c(Z_{t, h}^c)^\top \tilde{A}^c(Z_{t, h}^c) \underbrace{\Gamma_{t, h+1:H_t}^{c} (\thetalim[c] - \thetalim)}_\text{$v$}] = 
        \PE[v^\top \noisecovA[c] v] 
        & =
        \PE[\| (\Sigma_{\tilde{A}}^{c})^{1/2} v \|^2] 
        \leq \| \noisecovA[c] \| \cdot \PE[\| v \|^2]
        % = \| \noisecovA[c] \| \cdot \PE[ \| \Gamma_{t, h+1:H_t}^{c} (\thetalim[c] - \thetalim) \|^2 ] 
        \leq (1 - \eta_t a)^{2(H_t - h)} \| \noisecovA[c] \| \cdot \| \thetalim[c] - \thetalim \|^2
        \eqsp.
    \end{align*}
    Hence, we obtain the bound
    \begin{align}
    \label{eq:bound-flbi-bound-Sc}
        \!\!\term{U_c}
        & \leq
        \eta_t^2 \sum \limits_{h = 1}^{H_t} (1 - \eta_t a)^{2(H_t - 1)} \| \noisecovA[c] \| \| \thetalim[c] - \thetalim \|^2 
        = H_t \eta_t^2 (1 - \eta_t a)^{2(H_t - 1)} \| \noisecovA[c] \| \| \thetalim[c] - \thetalim \|^2
        \eqsp.
    \end{align}
    Plugging \eqref{eq:bound-flbi-bound-Sc} in \eqref{eq:bound-flbi-decomp-in-Sc}, we obtain the following bound on $\tauavg_{s}$
    \begin{align}
    \label{eq:bound-flbi-tauavg}
        \PE[\| \tauavg_{s} \|^2] 
        &
        \leq
        \frac{H_s \eta_s^2 (1 \!-\! \eta_s a)^{2(H_s \!- 1)}}{N^2}
        \sum \limits_{c = 1}^N \| \noisecovA[c] \| \| \thetalim[c] \!-\! \thetalim \|^2
        \!=\!
        H_s \eta_s^2 (1 \!-\! \eta_s a)^{2(H_s \!- 1)}
        \frac{\varhet^2}{N}
        \eqsp.
    \end{align}
    Here, we recall that $\varhet^2 = \frac{1}{\nagent} \sum \limits_{c = 1}^N \| \noisecovA[c] \| \| \thetalim[c] - \thetalim \|^2$. 
    Now, using the definition of $\sumHT{s+1:t}$, we have the following bound
    \begin{align*}
        \PE[\| \term{T_1} \|^2] 
        & =
        \sum \limits_{s = 1}^t \PE[ \| \prod_{i = s + 1}^t \Gammaavg_i \cdot \tauavg_{s} \|^2 ] 
        \leq
        \sum \limits_{s = 1}^t 
        % \prod_{i = s + 1}^t (1 - \eta_i a)^{2 H_i} 
        \exp(-2a \sumHT{s+1:t})
        \PE[\| \tauavg_{s} \|^2] 
        \eqsp.
        % \\
        % &\leq \sum \limits_{s = 1}^t \frac{\prod_{i = s + 1}^t (1 - \eta_i a)^{2 H_i} \cdot H_s \eta_s^2 (1 - \eta_s a)^{2(H_s - 1)} \cdot \sum \limits_{c = 1}^N \| \noisecovA[c] \| \| \thetalim[c] - \thetalim \|^2}{N^2}
    \end{align*}
    Using \eqref{eq:bound-flbi-tauavg} and $\step_{s} a \le 1/2$, we get
    \begin{align} \nonumber
        \PE[\| \term{T_1} \|^2] 
        & \leq 
        \sum \limits_{s = 1}^t
        % \prod_{i = s + 1}^t (1 - \eta_i a)^{2H_i}          \cdot
        \exp(-2a \sumHT{s+1:t})
         (1 \!-\! \eta_s a)^{2(H_s \!- 1)}
         H_s \eta_s^2
        \frac{\varhet^2}{N}
        \\ 
     \label{eq:bound-flbi-T1}
        & \leq 
        4 \sum \limits_{s = 1}^t
        \exp(-2a \sumHT{s:t})
         H_s \eta_s^2
        \frac{\varhet^2}{N}
        % \\
        % &\leq 
        % \sum \limits_{s = 1}^t 
        % \exp{\Big( -2a \sum \limits_{i = s + 1}^t H_i \eta_i \Big)} 
        % \exp(-2 a (H_s-1) \eta_s)
        % H_s \eta_s^2  \frac{\varhet^2}{N}
        % \\
        % &\leq 4 \tilde{v}_{\text{heter}} \cdot \sum \limits_{s = 1}^t \frac{ \exp{\left( -2a \sum \limits_{i = s + 1}^t H_i \eta_i \right)} \cdot \mathrm{e}^{- 2 a H_s \eta_s} \cdot H_s a \eta_s \cdot \eta_s} {Na}
        \eqsp.
    \end{align}
    % As $x \cdot \mathrm{e}^{-2x} \leq \frac{1}{2\mathrm{e}}$:
    % \begin{align}
    % \label{eq:bound-flbi-T1}
    %     \PE[\|\term{T_1}\|^2]
    %     \leq 
    %     \frac{2 \varhet^2}{N a \mathrm{e}} \cdot \sum \limits_{s = 1}^t \eta_s \exp{\Big( -2a \sum \limits_{i = s + 1}^t H_i \eta_i \Big)}
    %     \eqsp.
    % \end{align}

    \textbf{Bound on $\term{T_2}$.}
    Now we need to bound $\term{T_2} = \sum_{s = 1}^t \Delta_{s,t} \rhoavg_{s}$, where we recall that $\Delta_{s, t} = \{ \prod_{i = s + 1}^t \Gammaavg_i \} - \mathbb{E} [\prod_{i = s + 1}^t \Gammaavg_i ]$. 
    Let's begin with a single summand $\Delta_{s, t} \rhoavg_{s}$, which can be rewritten as
    \begin{align*}
        \Delta_{s, t} \rhoavg_{s} 
        & =
        \Big( \prod_{i = s + 1}^t \Gammaavg_{i} - \prod_{i = s + 1}^t G_i^{\textrm{avg}} \Big) \rhoavg_{s}
        \\
        & = \Big( \sum \limits_{i = s + 1}^t
        \Big\{ \!\!\prod_{r = i + 1}^{t} \!\! \Gammaavg_r \Big\} \cdot \Big\{ \Gammaavg_i - G^{\textrm{avg}}_i \Big\} \cdot \Big\{ \!\! \prod_{r = s + 1}^{i - 1} \!\! G^{\textrm{avg}}_r \Big\} \Big) \rhoavg_{s}
        \eqsp.
    \end{align*}
    We first bound the norm of $(\Gammaavg_i - G^{\textrm{avg}}_i) u$ for an arbitrary vector $u \in \mathbb{R}^d$ and $i \ge 0$.
    By clients' independence, we have
    \begin{align*}
        \PE\big[\| (\Gammaavg_i - G^{\textrm{avg}}_i) u \|^2 \big] 
        & =
        \PE\Big[\bnorm{ \frac{1}{N} \sum \limits_{c = 1}^N (\Gamma_{i, 1:H_i}^c - G_{i, 1:H_i}^c) u }^2 \Big] 
        \\
        & = \frac{1}{N^2} \sum \limits_{c = 1}^n \PE\big[\| (\Gamma_{i, 1:H_i}^c - G_{i, 1:H_i}^c) u \|^2\big] 
        \eqsp.
    \end{align*}
    Then, \Cref{lemma: difference-of-product} gives
    \begin{align*}
    \PE[\| (\Gammaavg_i - G^{\textrm{avg}}_i) u \|^2] 
    %    &= \frac{1}{N^2} \sum \limits_{c = 1}^N \PE[\| \Big( \Gamma_{i, 1:H_i}^c - (\Id- \eta_i \bA[c])^{H_i} \Big) u \|^2] \\
    %    &= \frac{1}{N^2} \sum \limits_{c = 1}^N \PE[\| \sum \limits_{h = 1}^{H_i} (\Id- \eta_i \bA[c])^{h - 1} \cdot \Big( (\Id- \eta_i\bA[c]) - (\Id- \eta_i A^c(Z_{i, h}^c)) \Big) \cdot \Gamma_{i, h+1:H_i}^c u \|^2] \\
        &= 
        \frac{\eta_i^2}{N^2}  \sum \limits_{c = 1}^N \sum \limits_{h = 1}^{H_i} \PE[\| (\Id- \eta_i \bA[c])^{h - 1} \zmfuncA[c]{Z_{i, h}^c} \Gamma_{i, h+1:H_i}^c u \|^2]
        \eqsp.
        %\\
        %&\leq \frac{1}{N^2} \eta_i^2 \sum \limits_{c = 1}^N \sum \limits_{h = 1}^{H_i} (1 - \eta_i a)^{2(h - 1)} \PE[u^\top (\Gamma_{i, h+1:H_i}^c)^\top (\tilde{A}(Z_{i, h}^c))^\top  \tilde{A}(Z_{i, h}^c) \Gamma_{i, h+1:H_i}^c u]
    \end{align*}
    By applying conditional expectation $\PE[\cdot | \mathcal{F}_{i + 1, 1}^+]$, we can rewrite it as:
    \begin{align*}
        \PE[\| (\Gammaavg_i - G^{\textrm{avg}}_i) u \|^2] 
        &=
        \frac{\eta_i^2}{N^2}  \sum \limits_{c = 1}^N \sum \limits_{h = 1}^{H_i} (1 - \eta_i a)^{2(h - 1)} \PE[\| (\noisecovA[c])^\frac{1}{2} \Gamma_{i, h+1:H_i}^c u \|^2] 
        \\
        & \le
        % \\
        % &\leq \frac{1}{N^2} \eta_i^2 \sum \limits_{c = 1}^N \sum \limits_{h = 1}^{H_i} (1 - \eta_i a)^{2(h - 1)} \| \noisecovA[c] \|\cdot (1 - \eta_i a)^{2(H_i - h)} \| u \|^2
        % \\
        % &
        % \frac{\eta_i^2 H_i}{N^2} \sum \limits_{c = 1}^N \| \noisecovA[c] \| \cdot \| u \|^2
        % \\
        % &= 
        (1 - \step_i a)^{2(\nlupdates_i - 1)}
        \frac{\eta_i^2 H_i}{N} \varboundA^2  \| u \|^2
        \eqsp.
    \end{align*}
    Using this knowledge, we can finally estimate $\Delta_{s, t} \rhoavg_s$'s squared norm 
    \begin{align*}
        \PE[ \| \Delta_{s, t} \rhoavg_s \|^2 ]
        & =
        \PE \Big[ \Big \| \sum \limits_{i = s + 1}^t \{ \prod_{r = i + 1}^t \Gammaavg_r \} \cdot (\Gammaavg_i - G^{\textrm{avg}}_i) \cdot \{ \prod_{r = s + 1}^{i - 1} G^{\textrm{avg}}_r \} \rhoavg_{s} \Big \|^2 \Big]
        \\
        & \le
        \sum \limits_{i = s + 1}^t
        \frac{\eta_i^2 H_i}{N} \varboundA^2
        \prod_{r = i + 1}^t (1 - \eta_r a)^{2 H_r} 
        (1 - \eta_i a)^{2H_i - 2}
        \prod_{r = s + 1}^{i - 1} (1 - \eta_r a)^{2H_r} 
        \| \rhoavg_s \|^2
        \\
        & \le
        4 \Big\{ \prod_{r = s + 1}^t (1 - \eta_r a)^{2 H_r} \Big\}
        \sum \limits_{i = s + 1}^t \frac{\eta_i^2 H_i}{N} \varboundA^2
        \| \rhoavg_s \|^2
        % \\
        % &=
        % \PE_c[\| \noisecovA[c] \|]
        % \| \rhoavg_s \|^2 
        % \sum \limits_{i = s + 1}^t
        % \frac{\eta_i^2 H_i}{N} 
        % \frac{1}{(1 - \eta_i a)^2}
        % \prod_{r = s + 1}^t (1 - \eta_r a)^{2H_r} 
        \eqsp.
    \end{align*}
    where we used $1 - \eta_i a \ge 1/2$.
    Using the fact that the step size decreases, and the fact that for all $c > 0$ and $x \ge 0$, $x \exp(-2cx) \le 1/(ec) \exp(-cx)$, we obtain the following bound 
    \begin{align}
    \nonumber
        \PE[ \| \Delta_{s, t} \rhoavg_s \|^2 ]
        & \le
        \frac{4 \step_s \varboundA^2}{\nagent} \| \rhoavg_s \|^2
        \exp \big( - 2 a \sum_{r = s + 1}^t \eta_r H_r \big)
        \sum \limits_{r = s + 1}^t \eta_r H_r
        \\
    \label{eq:interm-bound-Delta-rho}
        & \le
        \frac{4 \step_s \varboundA^2}{\nagent a e} \| \rhoavg_s \|^2
        \exp \big( - a \sum_{r = s + 1}^t \eta_r H_r \big)
        \eqsp.
    \end{align}
    Using the reverse martingale structure, \eqref{eq:interm-bound-Delta-rho}, and \Cref{lemma: rho-t-estimation}, we obtain a bound on $\PE[\|T_2\|^2 ]$
    \begin{align}
    \nonumber
        \PE [\| \term{T_2} \|^2] 
       &  \leq \frac{4 \varboundA^2}{\nagent a e} \sum_{s=1}^t
     \step_s \| \rhoavg_s \|^2 \exp \big( - a \sum_{r = s + 1}^t \eta_r H_r \big)
       \\
   \label{eq:bound-flbi-T2}
    & \le \frac{4 \varboundA^2}{\nagent a e} \firsthgty^2 \secondhgty^2
    \sum_{s=1}^t \step_s^5 H_s^4  \exp \big( - a \sum_{r = s + 1}^t \eta_r H_r \big)
        \eqsp.
    \end{align}
    The result follows from taking the root of the expected squared norm of \eqref{eq:decomp-flbi-T1-T2}, decomposing it two using Minkowski's inequality, and bounding each term using \eqref{eq:bound-flbi-T1} and \eqref{eq:bound-flbi-T2}.
\end{proof}

% \begin{lemma}
%     Under assumption \Cref{ass: linear-decay}:
%     \begin{align*}
%         \| \tilde{\theta}_t^{(\text{bi, bi})} - (\Id- \PE[\Gammaavg_i])^{-1} \rhoavg_{t, H} \| \leq \ldots 
%     \end{align*}
% \end{lemma}
% \begin{proof}
%     \begin{align*}
%         \tilde{\theta}_t^{(\text{bi, bi})} = \sum \limits_{s = 1}^t \PE[\prod_{i = s + 1}^t \Gammaavg_i] \rhoavg_{s, H} = \sum \limits_{s = 1}^t \prod \limits_{i = s + 1}^t \PE[\Gammaavg_i] = \sum \limits_{s = 1}^t \prod_{i = s + 1}^t \left[ \frac{1}{N} \sum \limits_{c = 1}^N (\Id- \eta_i \bA[c])^H \right]
%     \end{align*}
% \end{proof}

\paragraph{Convergence results.}
Combining our results, we can now give an upper bound on the error and prove \Cref{theorem: mse-estimation-for-general-eta-main}.
\subsubsection{Proof of the \Cref{theorem: mse-estimation-for-general-eta-main}}
\begin{proof}
Using the decomposition \eqref{eq: decomp_last_iter_t} and applying Minkowski’s inequality, we obtain
\begin{align*}
\PE^{1/2}\big[\|\theta_t - \thetalim\|^2\big]
&\le
\PE^{1/2}\big[\|\tilde{\theta}_t^{(\mathrm{tr})}\|^2\big]
+
\PE^{1/2}\big[\|\tilde{\theta}_t^{(\mathrm{bi,bi})}\|^2\big]
\\
&\quad
+
\PE^{1/2}\big[\|\tilde{\theta}_t^{(\mathrm{fl,bi})}\|^2\big]
+
\PE^{1/2}\big[\|\tilde{\theta}_t^{(\mathrm{fl})}\|^2\big]
\end{align*}

The four terms are controlled respectively by
\Cref{lemma: theta-tr-estimation},
\Cref{lemma: theta-bi-bi-estimation},
\Cref{lemma: lemma-theta-fl-bi},
and \Cref{lemma: theta-fl-estimation}.
Combining these bounds yields the claimed result.
\end{proof}

When the step size and number of local updates are constant, we recover a slightly improved variant of the Theorem A.6 of \citet{mangold2024scafflsa}.
\begin{corollary}[Constant step size]
    Assume \Cref{ass: linear-decay}(2), \Cref{assum:noise-level-flsa} and \Cref{ass: lr-constant}. Then, it holds    
    \begin{align*}
        \PE^{1/2} [\| \theta_t %- \tilde{\theta}_t^{(\sf{bi, bi})} 
        - \thetalim \|^2]
        &
        \leq
        (1 - \step a)^{H t} \| \theta_0 - \thetalim\|
        +
        \frac{\step \nlupdates \zeta_1 \zeta_2 }{a}
        +        
        \sqrt{\frac{\step \varbound^2}{N a}}
         +
        \sqrt{\frac{\step \varhet^2}{N a} } 
          +
         \sqrt{ \frac{4 \step^4 H^3 \varboundA^2 \firsthgty^2 \secondhgty^2 }{\nagent a^2 e} }
    \eqsp,
    \end{align*}
\end{corollary}
\begin{proof}
By \Cref{theorem: mse-estimation-for-general-eta-main} with fixed step size and number of local updates, we have   
\begin{align*}
        & \PE^{1/2} [\| \theta_t% - \tilde{\theta}_t^{(\sf{bi, bi})} 
        - \thetalim \|^2]
        \\
        &
        \leq
        (1 - \step a)^{H t} \| \theta_0 - \thetalim\|
        +
        \zeta_1 \zeta_2\sum_{s=1}^t \step^2 \nlupdates^2 \exp{\Big( -a \step \nlupdates (t-s) \Big)} 
        +
        \sqrt{\frac{\step^2 \nlupdates \varbound^2}{N}}
        \sqrt{\sum_{s=1}^t \exp(- 2a \eta \nlupdates t)}
        \\
        & \quad +
        \sqrt{\frac{4 \varhet^2}{N} }\sqrt{\sum \limits_{s = 1}^t
        \eta^2 H \exp(-2a \eta^2 \nlupdates (t-s) ) } 
          +
         \sqrt{ \frac{4 \varboundA^2 \firsthgty^2 \secondhgty^2 }{\nagent a e} }
         \sqrt{ 
    \sum_{s=1}^t \step^5 H^4  \exp \big( - a \eta H (t-s) \big) } 
    \eqsp.
    \end{align*}
    Bounding the sums, for $u > 0$, of $\sum_{s=1}^t \exp(- u (t-s)) \le (1 - \exp(-u))^{-1} \le 1/u $ gives the result.
    % \begin{align*}
    %     \PE^{1/2} [\| \theta_t - \tilde{\theta}_t^{(\sf{bi, bi})} - \thetalim \|^2]
    %     &
    %     \leq
    %     (1 - \step a)^{H t} \| \theta_0 - \thetalim\|
    %     +
    %     \frac{\step \nlupdates \zeta_1 \zeta_2 }{a}
    %     +
    %     \sqrt{\frac{\step \varbound^2}{N a}}
    %      +
    %     \sqrt{\frac{\step \varhet^2}{N a} } 
    %       +
    %      \sqrt{ \frac{4 \step^4 H^3 \varboundA^2 \firsthgty^2 \secondhgty^2 }{\nagent a^2 e} }
    % \eqsp,
    % \end{align*}
    % and the result follows.
\end{proof}

Finally, with decreasing step sizes and number of local updates, we obtain
\begin{corollary}
Assume \Cref{ass: linear-decay}(2), \Cref{assum:noise-level-flsa} and \Cref{ass: lr-polynomial-decay}, meaning that the step sizes and local updates satisfy, $\step_t = \frac{\initstep}{(t + 1)^{\gamma_\step}}$, $\nlupdates_t = \lceil {\initnlupdates}{(t+1)^{\gamma_\nlupdates}} \rceil$,
for some $\initstep, \initnlupdates > 0$ and $0 < \gamma_\step < 1$ and $0 < \gamma_\nlupdates < 1$, such that $\gamma_\step \ge \gamma_\nlupdates$ and $\gamma = \gamma_\step - \gamma_\nlupdates \neq 1$.
Then, for all $t \ge 1$, the following bound holds
\begin{align*}
\PE^{1/2} \big[\|\theta_t %- \tilde{\theta}_t^{(\sf{bi, bi})} 
- \thetalim\|^2\big] 
& \le
\exp\Big(- \frac{a \initstep \initnlupdates}{1 - \gamma} (t + 1)^{1 - \gamma} \Big) \|\theta_0 - \thetalim\|
+ \frac{8\zeta_1 \zeta_2 \step\nlupdates}{a} \frac{1}{(t+1)^{\gamma}}
+ \sqrt{\frac{8 \initstep \varbound^2}{N a}} 
\frac{1}{(t+1)^{\gamma_\step / 2}}
\\
& \quad + 
\sqrt{\frac{32 \initstep \varhet^2}{N a}} 
\frac{1}{(t+1)^{\gamma_\step / 2}}
+ \sqrt{\frac{32 \initstep^4 \initnlupdates^3 \varboundA^2 \firsthgty^2 \secondhgty^2}{\nagent a e}} \frac{1}{(t+1)^{(4 \gamma_\step - 3 \gamma_\nlupdates)/2}}
\eqsp.
\end{align*}
\end{corollary}
\begin{proof}
    We apply \Cref{lemma: legendary-sum} to the result of \Cref{theorem: mse-estimation-for-general-eta-main}, thus obtain the bound.
\end{proof}

\subsection{Proofs for high-order moment bounds $\PE^{1/p}[\norm{\theta_t - \thetalim}^p]$.}
\label{sec:moment-bound-appendix}

To prove \Cref{theorem: last-iterate-moment-bound-estimation}, 
we establish moment bounds for each component in the decomposition
$\tilde{\theta}_t^{(\mathrm{tr})}$,
$\tilde{\theta}_t^{(\mathrm{fl})}$,
$\tilde{\theta}_t^{(\mathrm{fl,bi})}$,
and
$\tilde{\theta}_t^{(\mathrm{bi,bi})}$.

\paragraph{Bound on the transient term.}
We begin with a bound on the $p$-th moment of the transient component.

\begin{lemma}
\label{lemma: high-order-moment-bound-theta-tr}
    Assume \Cref{ass: linear-decay}(p), \Cref{assum:noise-level-flsa}, and let $\step_s \le \step_{\infty, p}$. Then, it holds
    \begin{align*}
        \PE^{1/p}[\norm{ \tilde{\theta}_t^{\sf{(tr)}} }^p] \leq \exp{\Big( -a \sum \limits_{s = 1}^{t} \step_s \nlupdates_s \Big)} \, \norm{\theta_0 - \thetalim} \eqsp.
    \end{align*}
\end{lemma}
\begin{proof}
    We start with bounding the transient term $\PE^{1/p}[\norm{\tilde{\theta}_t^{\sf{(tr)}} }^p]$. It holds that
    \begin{align*}
    \PE^{1/p}\Big[\norm{\tilde{\theta}_t^{\sf{(tr)}} }^p\Big] 
    = \PE^{1/p}\Big[\bnorm{ \prod_{s = 1}^t \Gammaavg_s \, \{ \theta_0 - \thetalim \} }^p\Big] 
    = \PE^{1/p}\Big[\bnorm{ \Big( \frac{1}{\nagent} \sum \limits_{c = 1}^{\nagent} \Gamma_t^c \Big) \prod_{s = 1}^{t - 1} \Gammaavg_s \{ \theta_0 - \thetalim \} }^p\Big]
    \eqsp.
\end{align*}
Applying Minkowski's inequality, we have
\begin{align*}
    \PE^{1/p}[\norm{\tilde{\theta}_t^{\sf{(tr)}} }^p]
    \leq
    \frac{1}{\nagent} 
    \sum \limits_{c = 1}^{\nagent} \PE^{1/p}\Big[\bnorm{ \Gamma_t^c \cdot \prod_{s = 1}^{t - 1} \Gammaavg_s \{ \theta_0 - \thetalim \} }^p\Big] \eqsp. 
\end{align*}
By using \Cref{ass: linear-decay} conditionally, we get
\begin{align*}
    \PE^{1/p}[\norm{\tilde{\theta}_t^{\sf{(tr)}} }^p] 
    & \leq 
    (1 - \step_t a)^{\nlupdates_s} \PE^{1/p}\Big[\bnorm{ \prod_{s = 1}^{t - 1} \Gammaavg_s \{ \theta_0 - \thetalim }^p\Big]
   % \\
    %&
    \leq 
    \prod_{s = 1}^{t} (1 - \step_s a)^{\nlupdates_s} \norm{\theta_0 - \thetalim} \eqsp,
\end{align*}
which gives the following bound on the transient term
\begin{align*}
\PE^{1/p}[\norm{\tilde{\theta}_t^{\sf{(tr)}} }^p] 
\leq \exp{\Big( -a \sum \limits_{s = 1}^{t} \step_s \nlupdates_s \Big)} \, \norm{\theta_0 - \thetalim}
\eqsp.
\end{align*}
\end{proof}

\paragraph{Bound on the fluctuation term.}

\begin{lemma}
\label{lemma: higher-order-moment-bounds-theta-fl}
    Assume \Cref{ass: linear-decay}(p), \Cref{assum:noise-level-flsa}, and let $\step_s \le \step_{\infty, p}$. Then, it holds
    \begin{align*}
        \PE^{1/p}[\norm{ \tilde{\theta}_t^{\sf{(fl)}} }^p] 
        &\leq p \left( \sum \limits_{s = 1}^{t} \exp{\Big(-2a \sum \limits_{i = s + 1}^{t} \step_i \nlupdates_i \Big) } \,
        \nagent^{-1}  p^2 \supconsteps^2 (\step_s^2 \nlupdates_s \wedge \frac{\step_s}{a}) \right)^{1/2}
    \eqsp.
    \end{align*}    
\end{lemma}
\begin{proof}
    We now bound the fluctuation term $\PE^{1/p}[\norm{ \tilde{\theta}_t^{\text{(fl)}} }^p]$.
By definition of $\tilde{\theta}_t^{\text{(fl)}}$, we have
\begin{align*}
    \PE^{1/p}[\norm{ \tilde{\theta}_t^{\text{(fl)}} }^p]
    = \PE^{1/p}\Big[\bnorm{ \sum \limits_{s = 1}^t \Big \{ \prod_{i = s + 1}^t \Gammaavg_i \cdot \varphiavg_s \Big \} }^p\Big]
    \eqsp.
\end{align*}
Notice that the $\varphiavg_t$ are of zero expectation, and independent of each other. Thus, $\{ \prod_{i = s + 1}^t \Gammaavg_i \cdot \varphiavg_s \}$ forms a martingale difference w.r.t the filtration $\mathcal{F}_{i + 1, 1}^{+}$.
Using Burkholder's inequality (\citealp{osekowski:2012}, Theorem 8.6), we then obtain the following bound
\begin{align*}
    \PE^{1/p}[\norm{ \tilde{\theta}_t^{\text{(fl)}} }^p]
    \leq p \left( \PE^{2/p} \Big[ \Big( \sum \limits_{s = 1}^t \bnorm{ \prod_{i = s + 1}^{t} \Gammaavg_i \cdot \varphiavg_s }^2 \Big)^{p/2} \Big] \right)^{1/2}
\end{align*}
Using Minkowski's inequality along with \Cref{ass: linear-decay}, we have
\begin{align*}
    \PE^{1/p}[\norm{ \tilde{\theta}_t^{\text{(fl)}} }^p] 
    &
    \leq
    p \left( \sum \limits_{s = 1}^{t} \PE^{2/p}\Big[\bnorm{ \prod_{i = s + 1}^t \Gammaavg_i \cdot \varphiavg_s }^{p}\Big] \right)^{1/2} 
    \leq 
    p \left( \sum \limits_{s = 1}^{t} \exp{\Big( -2a \sum \limits_{i = s + 1}^{t} \step_i \nlupdates_i \Big) \, \PE^{2/p}[\norm{ \varphiavg_s }^p]} \right)^{1/2} 
    \eqsp.
%    \\
    % &\leq p \sum \limits_{s = 1}^{t} \exp{\Big( -a \sum \limits_{i = s + 1}^{t} \step_i \nlupdates_i \Big)} \PE^{1/p}[\norm{ \varphiavg_s }^p] \eqsp .
\end{align*}
Now we have to estimate $\PE^{1/p}[\norm{ \varphiavg_s }^p]$. 
Recall that $\varphiavg_s = \frac{1}{\nagent} \sum_{c = 1}^{\nagent} \sum_{h = 1}^{\nlupdates_s} \step_{s} \Gamma_{s, h + 1:\nlupdates_s}^c \varepsilon^c (Z^c_{s, h})$.
Therefore, applying Burkholder's inequality and Minkowski's inequality again gives
\begin{align*}
    \PE^{1/p}[\norm{ \varphiavg_s }^p] 
   % &=
%    \PE^{1/p}[\norm{ \frac{1}{\nagent} \sum \limits_{c = 1}^{\nagent} \sum \limits_{h = 1}^{\nlupdates_s} \step_{s} \Gamma_{s, h + 1:\nlupdates_s}^c \varepsilon^c (Z^c_{s, h}) }^p] 
    & \leq 
    \frac{p}{\nagent}  \left( \PE^{2/p} \Big[ \Big( \sum \limits_{c = 1}^{\nagent} \sum \limits_{h = 1}^{\nlupdates_s} \step_{s}^2 \norm{ \Gamma_{s, h + 1:\nlupdates_s}^c \funcnoise[c]{Z_{s, h}^c} }^2 \Big)^{p/2} \Big] \right)^{1/2} 
    \\
    &\leq \frac{p}{\nagent} \left( \sum \limits_{c = 1}^{\nagent} \sum \limits_{h = 1}^{\nlupdates_s} \step_{s}^2 \PE^{2/p}[\norm{ \Gamma_{s, h + 1:\nlupdates_s} \funcnoise[c]{Z_{s, h}^c} }^p]  \right)^{1/2} \eqsp.
\end{align*}
Using \Cref{ass: linear-decay} and the definition of $\supconsteps$
\begin{align*}
    \PE^{1/p}[\norm{ \varphiavg_s }^p] \leq \frac{p \step_s}{\nagent} \left( \sum \limits_{c = 1}^{\nagent} \sum \limits_{h = 1}^{\nlupdates_s} (1 - \step_s)^{2 (\nlupdates_s - h)} \supconsteps^2 \right)^{1/2} \eqsp \leq p \nagent^{-1/2}\step_s \supconsteps  \sqrt{(\nlupdates_s \wedge \frac{1}{\step_s a})} %= p \nagent^{-1/2} \supconsteps \sqrt{(\step_s^2 \nlupdates_s \wedge \frac{\step_s}{a})}
    \eqsp,
\end{align*}
where $\sum_{h = 1}^{\nlupdates_s} (1 - \step_s)^{2(\nlupdates_s - h)}$ is bounded by $\nlupdates_s \wedge \frac{1}{\step_s a}$ exactly as in \Cref{lemma: theta-fl-estimation}. 
Finally, we bound
\begin{align*}
    \PE^{1/p}[\norm{ \tilde{\theta}_t^{\text{(fl)}} }^p] 
    &\leq p \left( \sum \limits_{s = 1}^{t} \exp{\Big(-2a \sum \limits_{i = s + 1}^{t} \step_i \nlupdates_i \Big) } \,
    \nagent^{-1}  p^2 \supconsteps^2 (\step_s^2 \nlupdates_s \wedge \frac{\step_s}{a}) \right)^{1/2}
    \eqsp.
    % \\
    % &\leq \sqrt{ \frac{p^4 \supconsteps^2}{\nagent} \sum \limits_{s = 1}^{t} \exp{\left( -2 a S_{s + 1:t} \right) (\step_s^2 \nlupdates_s \wedge \frac{\step_s}{a})} }
\end{align*}
\end{proof}

\paragraph{Bound on the bias.}
The bias term, $\PE^{1/p}[\norm{\tilde{\theta}_t^{\text{(bi, bi)}}}^p]$, is deterministic, it's $p$-th moment coincides with $\norm{\tilde{\theta}_t^{\text{(bi, bi)}}}$, and it follows from \Cref{lemma: theta-bi-bi-estimation}, which states
\begin{align*}
    \PE^{1/p}[\norm{\tilde{\theta}_t^{\text{(bi, bi)}}}^p] = \norm{ \tilde{\theta}_t^{\text{(bi, bi)}} } \leq \zeta_1 \zeta_2\sum_{s=1}^t \step_s^2 \nlupdates_s^2 \exp{\Big( -a\sum_{i=s+1}^t \step_i \nlupdates_i \Big)}
    \eqsp.
\end{align*}

\paragraph{Bound on the fluctuations of the bias.}

\begin{lemma}
\label{lemma: higher-order-moment-bound-for-theta-fl-bi}
    Assume \Cref{ass: linear-decay}(p), \Cref{assum:noise-level-flsa}, and let $\step_s \le \step_{\infty, p}$. Then, it holds
    \begin{align*}
        \PE^{1/p}[\norm{\tilde{\theta}_t^{\sf{(fl, bi)}}}^p] &\leq 2 \bConst{A} \zeta_3 p^2 \nagent^{-1/2} \left( \sum \limits_{s = 1}^{t} \step_s^2 \nlupdates_s \exp{\Big( -2a \sum \limits_{i = s}^{t} \step_i \nlupdates_i \Big)} \right)^{1/2} \\
        &+ \frac{\bConst{A} p^3 \zeta_1 \zeta_2}{\sqrt{\nagent a \rme}} \sqrt{ \sum \limits_{s = 1}^{t} \step_s^5 \nlupdates_s^4 \exp{\Big( -a \sum \limits_{i = s + 1}^{t} \step_i \nlupdates_i \Big)}}
        \eqsp.
    \end{align*}
\end{lemma}
\begin{proof}
Recall the definition of $\tilde{\theta}_t^{\text{(fl, bi)}}$ from \eqref{eq:def-theta-flbi}:
\begin{align*}
    \tilde{\theta}_t^{\text{(fl, bi)}} 
    =
    \underbrace{\sum \limits_{s = 1}^t \prod_{i = s + 1}^t \Gammaavg_i \tauavg_{s}}_{\term{T_1}} 
    + 
    \underbrace{\sum \limits_{s = 1}^t \left \{ \prod_{i = s + 1}^t \Gammaavg_i - \mathbb{E} \Big[\prod_{i = s + 1}^t \Gamma_i^{\textrm{avg}} \Big]  \right \} \rhoavg_{s} }_{\term{T_2}}
    \eqsp.
\end{align*}

\textit{Bound on $\term{T_1}$.}
We start by deriving a $p-$th moment of $\tauavg_{s} = \frac{1}{\nagent} \sum_{c = 1}^N (G_s^c - \Gamma_{s}^c) (\thetalim[c] - \thetalim)$. 
Since the clients have independent noise, we can apply Burkholder's inequality to bound
\begin{align*}
    \PE^{1/p}[\norm{ \tauavg_s }^p] 
    \leq \frac{1}{\nagent} \Big( \sum \limits_{c = 1}^{\nagent} \underbrace{\PE^{2/p}[\norm{(G_s^c - \Gamma_s^c) (\thetalim[c] - \thetalim)}^p]}_{ \term{U_{c, p}^2} } \Big)^{1/2} \eqsp.
\end{align*}
Now we use \Cref{lemma: difference-of-product} to estimate $\term{U_{c, p}}$:
\begin{align*}
    \term{U_{c, p}}
    & = 
    \PE^{1/p} \Big[\bnorm{ \Big( \prod_{h = 1}^{\nlupdates_s} (\Id - \step_s \bA[c]) - \prod_{h = 1}^{\nlupdates_s} (\Id - \step_s \zfuncA[c]{Z_{s,h}^c}) \Big) (\thetalim[c] - \thetalim) }^p \Big] 
    \\
    & =
    \step_s \PE^{1/p}\Big[ \bnorm{  \sum \limits_{h = 1}^{\nlupdates_s} (\Id - \step_s \bA[c])^{h - 1} (\bA[c] - \zfuncA[c]{Z_{s,h}^c}) \Gamma_{s, h + 1:H_s}^c (\thetalim[c] - \thetalim)  }^p \Big]
    \eqsp.
\end{align*}
As $\PE[(\Id - \step_s \bA[c])^{h - 1} (\zfuncA[c]{Z_{s, h}^c} - \bA[c]) \Gamma_{s, h + 1:H_s}^c | \mathcal{F}_{s, h + 1}^+] = 0$, we can apply Burkholder's inequality here
\begin{align*}
    \term{U_{c, p}} &\leq p \step_s \left( \sum \limits_{h = 1}^{\nlupdates_s} \PE^{2/p} \Big[ \bnorm{ (\Id - \step_s \bA[c])^{h - 1} (\bA[c] - \zfuncA[c]{Z_{s,h}^c}) \Gamma_{s, h + 1:H_s}^c (\thetalim[c] - \thetalim) }^p \Big] \right)^{1/2} \\
    &\leq p \step_s \left( \sum \limits_{h = 1}^{\nlupdates_s} (1 - \step_s a)^{2(h - 1)}  \bConst{A}^2 (1 - \step_s a)^{2 (\nlupdates_s - h)} \norm{\thetalim[c] - \thetalim}^2 \right)^{1/2} \\
    &= p \step_s \bConst{A} \left( \sum \limits_{h = 1}^{\nlupdates_s} (1 - \step_s a)^{2 (\nlupdates_s - 1)} \right)^{1/2} \norm{\thetalim[c] - \thetalim} = p \step_s \bConst{A} (1 - \step_s a)^{H_s - 1} \nlupdates_s^{1/2} \norm{\thetalim[c] - \thetalim}
\end{align*}
Therefore, 
\begin{align*}
    \PE^{1/p}[\norm{\tauavg_s}^p] &\leq \bConst{A} p \nagent^{-1/2} \step_s \nlupdates_s^{1/2} (1 - \step_s a)^{\nlupdates_s - 1} \frac{\sum \limits_{c = 1}^{\nagent} \norm{\thetalim[c] - \thetalim}^2}{\nagent}\\
    &= \bConst{A} \zeta_3 p \nagent^{-1/2} \step_s \nlupdates_s^{1/2} (1 - \step_s a)^{\nlupdates_s - 1} \eqsp.
\end{align*}
Now we are ready to estimate $\term{T_1}$. We use here that $\{ \prod_{i = s + 1}^{t} \Gammaavg_i \tau_s \}_{s = 1}^{t}$ forms a martingale difference, therefore we apply Burkholder's inequality
\begin{align*}
    \PE^{1/p}[\norm{\term{T_1}}^p] &= \PE^{1/p}\Big[\bnorm{ \sum \limits_{s = 1}^{t} \prod_{i = s + 1}^{t} \Gammaavg_i \, \tauavg_s }^p \Big] \leq p \left( \sum \limits_{s = 1}^{t} \PE^{2/p}\Big[ \bnorm{ \prod_{i = s + 1}^{t} \Gammaavg_i \tauavg_s }^p \Big] \right)^{1/2} \\
    &\leq p \left( \sum \limits_{s = 1}^{t} \prod_{i = s + 1}^{t} (1 - \step_i a)^{2 \nlupdates_i} \, \PE^{2/p}[\norm{\tauavg_s}^p] \right)^{1/2} = 2 \bConst{A} \zeta_3 p^2 \nagent^{-1/2} \left( \sum \limits_{s = 1}^{t} \prod_{i = s}^{t} (1 - \step_i a)^{2 \nlupdates_i} \, \step_s^2 \nlupdates_s \right)^{1/2}
\end{align*}
We can rewrite this estimate as
\begin{align*}
    \PE^{1/p}[\norm{\term{T_1}}^p] \leq 2 \bConst{A} \zeta_3 p^2 \nagent^{-1/2} %\sum \limits_{c = 1}^\nagent
    \left( \sum \limits_{s = 1}^{t} \step_s^2 \nlupdates_s \exp{\Big( -2a \sum \limits_{i = s}^{t} \step_i \nlupdates_i \Big)} \right)^{1/2}
\end{align*}

\textit{Bound on \term{$T_2$}.}
Now we estimate $p-$th moment of $\term{T_2} = \sum_{s = 1}^t \Delta_{s,t} \rhoavg_{s}$, where we recall that $\Delta_{s, t} = \{ \prod_{i = s + 1}^t \Gammaavg_i \} - \mathbb{E} [\prod_{i = s + 1}^t \Gammaavg_i ]$. We start from single summand $\Delta_{s, t} \rhoavg_s$
    \begin{align*}
        \Delta_{s, t} \rhoavg_{s} 
        & =
        \Big( \prod_{i = s + 1}^t \Gammaavg_{i} - \prod_{i = s + 1}^t G_i^{\textrm{avg}} \Big) \rhoavg_{s}
        = \Big( \sum \limits_{i = s + 1}^t
        \Big\{ \!\!\prod_{r = i + 1}^{t} \!\! \Gammaavg_r \Big\} \cdot \Big\{ \Gammaavg_i - G^{\textrm{avg}}_i \Big\} \cdot \Big\{ \!\! \prod_{r = s + 1}^{i - 1} \!\! G^{\textrm{avg}}_r \Big\} \Big) \rhoavg_{s}
        \eqsp.
    \end{align*}
    We first bound moment of $(\Gammaavg_i - G^{\textrm{avg}}_i) u$ for an arbitrary vector $u \in \mathbb{R}^d$ and $i \ge 0$.
    By applying Burkholder's inequality
    \begin{align*}
        \PE\big[\| (\Gammaavg_i - G^{\textrm{avg}}_i) u \|^p \big] 
        & =
        \PE\Big[\bnorm{ \frac{1}{N} \sum \limits_{c = 1}^N (\Gamma_{i, 1:H_i}^c - G_{i, 1:H_i}^c) u }^p \Big] 
        \\
        &\leq \Big( \frac{p}{N} \Big)^p \left( \PE^{2/p} \Big[ \Big( \sum \limits_{c = 1}^{\nagent} \norm{(\Gamma_{i, 1:\nlupdates_i}^{c} - G_{i, 1:\nlupdates_i}^c) u}^2 \Big)^{p/2} \Big] \right)^{p/2} \\
        &\leq \Big( \frac{p}{\nagent} \Big)^p \left( \sum \limits_{c = 1}^{\nagent} \PE^{2/p}[\norm{ (\Gamma_{i, 1:\nlupdates_i}^c - G_{i, 1:\nlupdates_i}^c) u }^{p}] \right)^{p / 2}
        \eqsp.
    \end{align*}
    Using \Cref{lemma: difference-of-product}, we get
    \begin{align*}
        \PE\big[\| (\Gammaavg_i - G^{\textrm{avg}}_i) u \|^p \big] &\leq \Big( \frac{p \step_i}{\nagent} \Big)^p \left( \sum \limits_{c = 1}^{\nagent} \PE^{2/p} \Big[ \bnorm{\sum \limits_{h = 1}^{\nlupdates_i} (\Id - \step_i \bA[c])^{h - 1} \zmfuncA[c]{Z_{i, h}^c} \Gamma_{i, h+1:\nlupdates_i}^c u}^{p} \Big]  \right)^{p/2} \\
        &\leq \Big( \frac{p \step_i}{\nagent} \Big)^p \left( \sum \limits_{c = 1}^{\nagent} \sum \limits_{h = 1}^{\nlupdates_i} \PE^{2/p} \Big[\bnorm{ (I - \step_i \bA[c])^{h - 1} \zmfuncA[c]{Z_{i, h}^c} \Gamma_{i, h+1:H_i}^c u }^p \Big] \right)^{p/2} \\
        &\leq \Big( \frac{p \step_i}{\nagent} \Big)^p \left( \nagent \nlupdates_i (1 - \step_i a)^{2(h - 1)} \bConst{A}^2 (1 - \step_i a)^{2(\nlupdates_i - h)} \right)^{p/2}
        \eqsp,
    \end{align*}
    or, in other words,
    \begin{align*}
        \PE^{1/p}[\norm{ (\Gammaavg_i - \Gavg_i) u }^p] \leq \frac{2 \bConst{A} p \step_i \sqrt{\nlupdates_i} (1 - \step_i a)^{\nlupdates_i} \norm{u}}{\sqrt{\nagent}}
        \eqsp.
    \end{align*}
    Now we can bound $\Delta_{s, t} \rhoavg_{s}$
    % new approach
    \begin{align*}
        \PE^{1/p}[\norm{ \Delta_{s, t} \rhoavg_s }^p] &= \PE^{1/p}\Big[ \Big \| \sum \limits_{i = s + 1}^t \{ \prod_{r = i + 1}^t \Gammaavg_r \} \cdot (\Gammaavg_i - G^{\textrm{avg}}_i) \cdot \{ \prod_{r = s + 1}^{i - 1} G^{\textrm{avg}}_r \} \rhoavg_{s} \Big \|^p \Big] \\
        &\leq p \left( \sum \limits_{i = s + 1}^{t} \PE^{2/p} \Big[ \bnorm{ \{ \prod_{r = i + 1}^t \Gammaavg_r \} \cdot (\Gammaavg_i - G^{\textrm{avg}}_i) \cdot \{ \prod_{r = s + 1}^{i - 1} G^{\textrm{avg}}_r \} \rhoavg_{s} }^p \Big] \right)^{1/2} \\
        &\leq p \left( \sum \limits_{i = s + 1}^{t} \prod_{r = s + 1, r \neq i}^{t} (1 - \step_r a)^{2 \nlupdates_r} \frac{4 \bConst{A}^2 p^2 \step_i^2 \nlupdates_i (1 - \step_i a)^{2 \nlupdates_i} \norm{ \rhoavg_s }^2}{\nagent} \right)^{1/2} \\
        &\leq \frac{2 \bConst{A} p^2 \norm{\rhoavg_s}}{\sqrt{\nagent}} 
        \left( \prod_{r = s + 1}^{t} (1 - \step_r a)^{2 \nlupdates_r} \right)^{1/2} \left( \sum \limits_{i = s + 1}^{t} \step_i^2 \nlupdates_i \right)^{1/2}
    \end{align*}
    As $\step_i \geq \step_s$ for all $i > s$, we can rewrite the bound as
    \begin{align*}
        \PE^{1/p}[\norm{ \Delta_{s, t} \rhoavg }^p] \leq \frac{2 \bConst{A} p^2 \norm{\rhoavg_s}}{\sqrt{\nagent}} \step_s^{1/2} \exp{\Big( -a \sum \limits_{i = s + 1}^{t} \step_i \nlupdates_i \Big)} \sqrt{\sum \limits_{i = s + 1}^{t} \step_i \nlupdates_i }
    \end{align*}
    Now, finally, we can bound $\term{T_2}$.
    Remark that $\PE^{1/p}[\norm{\term{T_2}}^p] \leq p \left( \sum_{s = 1}^{t} \PE^{2/p}[\norm{ \Delta_{s, t} \rhoavg_s }^p]  \right)^{1/2}$, and thus
    \begin{align*}
        \PE^{1/p}[\norm{\term{T_2}}^p]
        & \leq p \left( \sum \limits_{s = 1}^{t} \frac{4 \bConst{A}^2 p^4 \norm{\rhoavg_s}^2 }{\nagent} \step_s \exp{\Big( -2a \sum \limits_{i = s + 1}^{t} \step_i \nlupdates_i \Big)} \sum \limits_{i = s + 1}^{t} \step_i \nlupdates_i \right)^{1/2} \\
        &\leq \frac{2 \bConst{A} p^3}{\sqrt{\nagent}} \sqrt{ \sum \limits_{s = 1}^{t} \norm{\rhoavg_s}^2 \step_s \exp{\Big( -2a \sum \limits_{i = s + 1}^{t} \step_i \nlupdates_i \Big)} \sum \limits_{i = s + 1}^{t} \step_i \nlupdates_i } \\
        &\leq \frac{\bConst{A} p^3 \zeta_1 \zeta_2}{\sqrt{\nagent}} \sqrt{ \sum \limits_{s = 1}^{t} \step_s^5 \nlupdates_s^4 \exp{\Big( -2a \sum \limits_{i = s + 1}^{t} \step_i \nlupdates_i \Big)} \sum \limits_{i = s + 1}^{t} \step_i \nlupdates_i }
        \eqsp.
    \end{align*}
    Then we use the fact that $x \exp{(-2cx)} \leq 1 / (ec) \exp{(-cx)}$ for $c > 0$ and $x \geq 0$ and get:
    \begin{align*}
        \PE^{1/p}[\norm{\term{T_2}}^p] \leq \frac{\bConst{A} p^3 \zeta_1 \zeta_2}{\sqrt{\nagent a \rme}} \sqrt{ \sum \limits_{s = 1}^{t} \step_s^5 \nlupdates_s^4 \exp{\Big( -a \sum \limits_{i = s + 1}^{t} \step_i \nlupdates_i \Big)}}
        \eqsp.
    \end{align*}
    Combining $\term{T_1}$ and $\term{T_2}$ finishes the proof.
\end{proof}

\paragraph{Convergence result.} Now, we are ready to state the main result for the moment bounds $\PE^{1/p}[\norm{\theta_t - \thetalim}^p]$.

\subsubsection{Proof of the \Cref{theorem: last-iterate-moment-bound-estimation}}
\begin{proof}
    We start from the decomposition of the last iterate (\ref{eq: decomp_last_iter_t}). Then, by applying Minkowski's inequality, we have
    \begin{align*}
        \PE^{1/p}[\norm{ \theta_t - \thetalim }^p] \leq \PE^{1/p}[\norm{\tilde{\theta}_t^{\text{tr}} }^p] +
        \PE^{1/p}[\norm{ \tilde{\theta}_t^{\text{(bi, bi)}} }^p] + \PE^{1/p}[\norm{ \tilde{\theta}_t^{\text{(fl, bi)}} }^p] + \PE^{1/p}[\norm{ \tilde{\theta}_t^{\text{(fl)}} }^p]
    \end{align*}
    By controlling each of the terms by using \Cref{lemma: theta-bi-bi-estimation}, \Cref{lemma: high-order-moment-bound-theta-tr}, \Cref{lemma: higher-order-moment-bounds-theta-fl} and \Cref{lemma: higher-order-moment-bound-for-theta-fl-bi} we get the desired result.
\end{proof}

Note that the bound in \Cref{theorem: last-iterate-moment-bound-estimation} can be rewritten in a slightly simpler form. 
Indeed, the last two terms in the bound can be controlled by $(1+t)^{-\gamma}$ up to a multiplicative constant. 
This simplified representation will be useful in the Gaussian approximation analysis for the multiplier bootstrap.
\begin{corollary}
\label{corr:last_iter_pth_moment}
    Under the assumptions of \Cref{theorem: last-iterate-moment-bound-estimation} and \Cref{ass: lr-polynomial-decay}, for any $t \geq 1$, we have
    \begin{align}
    \label{eq:last_iter_pth_bound}
        \PE^{1/p}[\norm{\theta_t - \thetalim}^p] \leq \exp{\left( -a \sum \limits_{s = 1}^{t} \step_s \nlupdates_s \right)} \, \norm{\theta_0 - \thetalim} + \Auxconst_{\text{last},1} p^3 \nagent^{-1/2} (1+t)^{-\gamma_\step/2} + \Auxconst_{\text{last},2} (1+t)^{-\gamma} \eqsp,
    \end{align}
    where $\Auxconst_{\text{last},1}$ and $\Auxconst_{\text{last},2}$ are defined in \eqref{eq:auxconst_last_iter_bound}.
\end{corollary}
\begin{proof}
    We introduce the constants
    \begin{align}
    \label{eq:auxconst_last_iter_bound}
        \nonumber \Auxconst_{\text{last},1} &= (2\supconsteps + 2\bConst{A}\zeta_3 + 4(a\rme)^{-1/2}\bConst{A} \zeta_1 \zeta_2)((\step\nlupdates)^{1/2} + a^{-1/2}L^{-1/2}(1 + (a \step \nlupdates)^{\gamma + \gamma_\step \over 2(1-\gamma)})) \eqsp,\\
        \Auxconst_{\text{last},2} &= \zeta_1\zeta_2 (\step^{1/2}\nlupdates + a^{-1/2} L^{1/2}(1 + (a\step \nlupdates)^{\gamma \over 1 - \gamma})) \eqsp.
    \end{align}
    Applying \Cref{lemma: legendary-sum}, we obtain the bound  \eqref{eq:last_iter_pth_bound}.
\end{proof}

\section{Normal approximation for FedLSA}

In this section, we provide the proof of 
\Cref{thm:gauss_approx_last_iter}, which establishes a Berry–Esseen type bound 
for the last iterate of \textsc{FedLSA}. 

For notational simplicity, we slightly abuse notation and write
$Z_{i,j,c} := Z_{i,j}^c$ and $\theta_{i,j,c} := \theta_{i,j}^c$.

\subsection{Error decomposition and linear statistics for Gaussian approximation}

Firstly, for any $t \in \{1, \dots, T\}$ and $h \in \{1, \dots, \nlupdates\}$, we unroll the recursion in \eqref{eq:decomp_last_agent_rec}, as follows
\begin{align}
\label{eq:decomp_theta_th}
    \theta_{t,h}^c - \thetalim[c] &= (\Id - \step_{t} \zfuncA[c]{Z_{t,h,c}})(\theta_{t,h-1}^c - \thetalim[c]) - \step_{t}\funcnoise[c]{Z_{t,h,c}}\\
    &= \Gamma_{t,h}^c(\theta_{t,0}^c-\thetalim[c]) - \step_t \sum_{j=1}^h \Gamma_{t,j+1:h}^c \funcnoise[c]{Z_{t,j,c}}\eqsp,
\end{align}
where for any $j \geq 1$ and $\theta \in \rset^d, z \in \Zset$, we recall the definition of the following values
\begin{align*}
    &\funcnoiseth[c]{\theta}{z} = \zmfuncA[c]{z}\theta - \zmfuncb[c]{z} \eqsp,\\
    &G_{t,\ell:r}^c = (\Id - \step_{t}\bA[c])^{r-\ell+1} , \quad \ell \leq r \eqsp,\\
    &G_{t,\ell:r}^c = \Id, \quad \ell > r \eqsp,\\
    &G_{t,j}^c = G_{t,1:j}^c \eqsp.
\end{align*}
Also, we set the following quantities
\begin{align*}
    &\Gammaavg_{t,h} = \nagent^{-1}\sum_{c=1}^{\nagent}\Gamma_{t,1:h}^c \eqsp, \quad \Gavg_{t,h} = \nagent^{-1}\sum\limits_{c=1}^{\nagent} G_{t,h}^c \eqsp. 
\end{align*}
Now, we note that
\begin{align*}
    \nagent^{-1}\sum_{c=1}^{\nagent}\{\theta_{t,h}^c - \thetalim\}
    &= \Gammaavg_{t,h} (\theta_{t-1} - \thetalim) + \Deltaavg_{t,h} - \step_t \Epsavg_{t,h}\\
    &= \Gammaavg_{t,h}\prod_{i=1}^{t-1}\Gammaavg_i (\theta_0 -\thetalim) + \Gammaavg_{t,h}\sum_{s=1}^{t-1}\prod_{i=s+1}^{t-1}\Gammaavg_i (\Deltaavg_{s,\nlupdates_s} - \step_s\Epsavg_{s,\nlupdates_s}) + \Deltaavg_{t,h} - \step_t\Epsavg_{t,h} \eqsp,
\end{align*}
where we have set, for any $h \ge 0$,
\begin{align*}
    \Deltaavg_{t,h} &= \nagent^{-1}\sum_{c=1}^{\nagent} (\Id - \Gamma_{t,h}^c)(\thetalim[c]-\thetalim) \eqsp,
    \quad
    \text{and } \quad 
    \Epsavg_{t,h} = \nagent^{-1} \sum_{c=1}^{\nagent}\sum_{s=1}^h \Gamma_{t,s+1:h}^c \funcnoise[c]{Z_{t,s,c}} \eqsp.
\end{align*}
Particularly, for $h=\nlupdates_t$, we obtain the last iterate decomposition, that is,
\begin{align*}
    \theta_t - \thetalim = \prod_{i=1}^{t}\Gammaavg_i (\theta_0 -\thetalim) + \sum_{s=1}^{t}\prod_{i=s+1}^{t}\Gammaavg_i (\Deltaavg_{s,\nlupdates_s} - \step_s\Epsavg_{s,\nlupdates_s}) \eqsp.
\end{align*}
% Extracting a leading term we have
In this decomposition, the dominant term is the sum of the noise variables $\Epsavg_{s,\nlupdates_s}$.
We isolate it in the decomposition
\begin{align}
\nonumber
    \theta_t - \thetalim &= \underbrace{-N^{-1} \sum \limits_{s = 1}^t \step_s \prod_{i = s + 1}^t \Gavg_i \sum \limits_{c = 1}^N \sum \limits_{h = 1}^{\nlupdates_s} G_{s, h + 1:\nlupdates_s}^{c} \funcnoise[c]{Z_{s, h, c}}}_\text{Linear statistics}%  \\
    %&
    + \prod_{i = 1}^t \Gammaavg_i (\theta_0 - \thetalim) + \sum \limits_{s = 1}^t \prod_{i = s + 1}^t \Gammaavg_i \Deltaavg_{s, H_s}
    \\
    \label{eq:decomposition-theta-leading-term}
    &- N^{-1}\sum \limits_{s = 1}^{t} \step_s \left( \prod_{i = s + 1}^t \Gammaavg_i - \prod_{i = s + 1}^t \Gavg_i \right) \sum \limits_{c = 1}^N \sum \limits_{h = 1}^{\nlupdates_s} \Gamma_{s, h+1:\nlupdates_{s}} \funcnoise[c]{Z_{s, h, c}} \\
\nonumber
    &- N^{-1} \sum \limits_{s = 1}^t \step_s \prod_{i = s + 1}^t \Gavg_i \sum \limits_{c = 1}^N \sum \limits_{h = 1}^{\nlupdates_s} (\Gamma_{s, h+1:\nlupdates_s} - G_{s, h+1:\nlupdates_s}) \funcnoise[c]{Z_{s, h, c}}
\end{align}
However, the underbraced term is not the complete leading linear term. To get the correct leading term we have to extract the hidden from the first glance linear statistics from the summand $\sum \limits_{s = 1}^{t} \prod_{i = s + 1}^{t} \Gammaavg_i \Deltaavg_{s, \nlupdates_s}$. Recalling the definitions from \eqref{eq:def-rhoavg-tauavg}, we have
\begin{align*}
    \Deltaavg_{s,\nlupdates_s} = \rhoavg_s + \tauavg_s \eqsp.
\end{align*}
Here, $\rhoavg_{t} = \frac{1}{\nagent} \sum \limits_{c = 1}^{\nagent} (\Id - G_s^c) (\thetalim[c] - \thetalim)$ represents the deterministic bias and does not contain any random noise, while $\tauavg_{s} = \frac{1}{\nagent} \sum \limits_{c = 1}^N (G_s^c - \Gamma^c_{t, 1:H_s}) (\thetalim[c] - \thetalim)$ contains linear noise from the $(G_s^c - \Gamma_{t, 1:\nlupdates_s}^{c})$. Therefore, to extract full linear statistics from $\theta_t - \thetalim$ we need to linearize $\tauavg_{s}$.

    \paragraph{Linearization of $\boldsymbol{\tauavg_s}$.}
    We need the following linearization to establish the correct leading term. We aim to linearize the $\tauavg_s$ term in order to extract the leading linear statistic. To this end, we linearize terms of the form
    \begin{align}
        \label{eq:linear_tau_s_decomp}
        \nonumber
        G_s^c \!-\! \Gamma_s^c &= \step_s \sum_{i=1}^{\nlupdates_s}G_{s,1:i-1}^c \zmfuncA[c]{Z_{s,i,c}} \Gamma_{s,i+1:\nlupdates_s}^c \\
        &
        = \underbrace{\step_s \sum_{i=1}^{\nlupdates_s}G_{s,1:i-1}^c \zmfuncA[c]{Z_{s,i,c}}}_{U_s^c} + \underbrace{\step_s \sum_{i=1}^{\nlupdates_s}G_{s,1:i-1}^c \zmfuncA[c]{Z_{s,i,c}} (\Gamma_{s,i+1:\nlupdates_s}^c - \Id)}_{R_s^c}
        %\\
        %&= U_s^c + R_s^c 
        \eqsp.
    \end{align}
    First, we bound the term $U_s^c$ as
    \begin{align*}
        \PE[\norm{U_s^c}^2] &\leq \step_s^2 \sum_{i=1}^{\nlupdates_s}(1-\step_s a)^{2(i-1)}{\mathrm{Tr}(\noisecovA[c])} \leq \mathrm{Tr}(\noisecovA[c])\step_s^2 \nlupdates_s \eqsp.%\\
        % \PE[\norm{R_s^c}^2] &\leq \step_s^2 \sum_{i=1}^{\nlupdates_s}(1-\step_s a)^{2(i-1)} \mathrm{Tr}(\noisecovA[c]) \PE[\norm{\Id - \Gamma_{s,i+1:\nlupdates_s}^c}^2] \eqsp.
    \end{align*}
    Then, we write
    \begin{align*}
        % \PE[\norm{U_s^c}^2] &\leq \step_s^2 \sum_{i=1}^{\nlupdates_s}(1-\step_s a)^{2(i-1)}\mathrm{Tr}(\noisecovA[c]) \leq \mathrm{Tr}(\noisecovA[c])\step_s^2 \nlupdates_s \eqsp,\\
        \PE[\norm{R_s^c}^2] &\leq \step_s^2 \sum_{i=1}^{\nlupdates_s}(1-\step_s a)^{2(i-1)} \mathrm{Tr}(\noisecovA[c]) \PE[\norm{\Id - \Gamma_{s,i+1:\nlupdates_s}^c}^2] \eqsp.
    \end{align*}
   % Now, we can bound the term $\PE[\norm{\Id - \Gamma_{s,i+1:\nlupdates_s}^c}^2]$, as
    %\begin{align*}
    Next, we bound $\PE^{1/2}[\norm{\Id - \Gamma_{s,i+1:\nlupdates_s}^c}^2] \leq \sum_{j=i+1}^{\nlupdates_s} \step_s \bConst{A} (1-\step_s a)^{j-i-1} \leq \step_s \nlupdates_s \bConst{A}$.% \eqsp.
   % \end{align*}
    Thus, we obtain
    \begin{align}
        \label{eq:bound_Ri}
        \PE[\norm{R_s^c}^2] \leq \bConst{A}^2 \step_s^4 \nlupdates_s^3 \mathrm{Tr}(\noisecovA[c]) \eqsp.
    \end{align}
    % \maks{Do we really need it here?}

    \paragraph{Leading linear statistics. }
    
    The correct leading linear statistics defined in \eqref{eq:decomposition-theta-leading-term} consists of two terms
    \begin{align}
        M_{t} &=
        -\nagent^{-1} \step_t^{-1/2}\sum_{s=1}^t \step_s \prod_{i=s+1}^t \Gavg_{i}\sum_{c=1}^\nagent \sum_{h=1}^{\nlupdates_s} G_{s,h+1:\nlupdates_s}^c \funcnoise[c]{Z_{s,h,c}} \label{eq:M_t_noise}
        %\eqsp,
        \\
       % M_{t,2}
       & \quad + \nagent^{-1} \step_t^{-1/2}\sum_{s=1}^t \step_s \prod_{i=s+1}^t \Gavg_i \sum_{c=1}^{\nagent}\sum_{h=1}^{\nlupdates_s}G_{s,1:h-1}^c \zmfuncA[c]{Z_{s,h,c}}(\thetalim[c] - \thetalim) \label{eq:M_t_heter}
        %= M_{t,1} + M_{t,2} 
        \eqsp.
        %\\
        %\Sigma_t &= \PE[M_t (M_t)^T] \eqsp.
    \end{align}
    We can decompose $M_t$ in four terms,
    \begin{align*}
        \nonumber M_t &= 
        M_{t,1} + M_{t,2} + 
        M_{t,3} + M_{t,4}
        \eqsp,
    \end{align*}
    where
    \begin{align*}
        M_{t,1} 
        & =
        -\nagent^{-1}\step_t^{-1/2}\sum_{s=1}^t \step_s \prod_{i=s+1}^t \Gavg_{i}\sum_{c=1}^\nagent \sum_{h=1}^{\nlupdates_s} \funcnoise[c]{Z_{s,h,c}} 
        \eqsp,
        \\
        M_{t,2} 
        & = 
        \nagent^{-1} \step_t^{-1/2}\sum_{s=1}^t \step_s \prod_{i=s+1}^t \Gavg_i \sum_{c=1}^{\nagent}\sum_{h=1}^{\nlupdates_s} \zmfuncA[c]{Z_{s,h,c}}(\thetalim[c] - \thetalim) 
        \eqsp,
        \\
        M_{t,3} 
        & = 
        -\nagent^{-1} \step_t^{-1/2}\sum_{s=1}^t \step_s \prod_{i=s+1}^t \Gavg_{i}\sum_{c=1}^\nagent \sum_{h=1}^{\nlupdates_s} (G_{s,h+1:\nlupdates_s}^c - \Id) \funcnoise[c]{Z_{s,h,c}} 
        \eqsp,
        \\
        M_{t,4}
        & = 
        \nagent^{-1} \step_t^{-1/2}\sum_{s=1}^t \step_s \prod_{i=s+1}^t \Gavg_i \sum_{c=1}^{\nagent}\sum_{h=1}^{\nlupdates_s}(G_{s,1:h-1}^c - \Id) \zmfuncA[c]{Z_{s,h,c}}(\thetalim[c] - \thetalim) %= M_{t,1} + M_{t,2} + M_{t,3} + M_{t,4} 
        \eqsp.
    \end{align*}
    From the definition \eqref{eq:def-tilde-Ab-epsilon} of $\varepsilon^c$, we see that the first and second terms can be combined into
    \begin{align*}
        \hat{M}_t = M_{t,1} + M_{t,2} = -\nagent^{-1}\step_t^{-1/2}\sum_{s=1}^t \step_s \prod_{i=s+1}^t \Gavg_i \sum_{c=1}^\nagent \sum_{h=1}^{\nlupdates_s} \funcnoiseth[c]{\thetalim}{Z_{s,h,c}} \eqsp.
    \end{align*}
    Let us denote
    \begin{align}
    \label{eq:Sigma_star_def}
        \noisecovst[c] = \PE[\funcnoiseth[c]{\thetalim}{Z} (\funcnoiseth[c]{\thetalim}{Z})^T] \eqsp, \quad \noisecovavgst = \nagent^{-1}\sum_{c=1}^\nagent \noisecovst[c] \eqsp.
    \end{align}
    Then, we define the covariance matrix, corresponding to the term $M_{t,1} + M_{t,2}$, as
    \begin{align}
    \label{eq:hat_Sigma_t_def}
        \hat{\Sigma}_t = \PE[\hat{M}_t \hat{M}_t^T] = \nagent^{-1}\step_t^{-1}\sum_{s=1}^t \step_s^2 \nlupdates_s \prod_{i=s+1}^t \Gavg_i \Sigma_* \left(\prod_{i=s+1}^t \Gavg_i \right)^T \eqsp,
    \end{align}
    which is the leading component of the FedLSA covariance. According to \Cref{lemma:bound_Mt}, the terms $M_{t,3}$ and $M_{t,4}$ are residual. 

    We may notice that in the case when $\nlupdates = \text{const}$, the leading component of FedLSA covariance $\hat{\Sigma}_t$ coincides with that of~\citet[Theorem 2.2]{bonnerjee2025sharp}.

    \paragraph{Final error decomposition for Gaussian approximation.} We now define
    \begin{align}
    \label{eq:sigma_t_def}
        \Sigma_t &= \PE[M_t M_t^T] \eqsp,
    \end{align}
    which allows us to write the central error decomposition for the normalized statistic $\step_t^{-1/2} \Sigma_t^{-1/2} (\theta_t - \thetalim)$
    \begin{align}
        \label{eq:decomp_clt_last_iter}
        \step_t^{-1/2} \Sigma_t^{-1/2}(\theta_t - \thetalim) = W_t + D_t \eqsp,
    \end{align}
    where we define
    \begin{align}
    \label{eq:last-iterate-error-decomposition}
        W_t &= \Sigma_t^{-1/2} M_t = \sum_{s=1}^t \sum_{h=1}^{\nlupdates_s}\sum_{c=1}^\nagent \upsilon_{s,h,c} \eqsp,\\
        D_t &= D_{t,1} + D_{t,2} + D_{t,3} + D_{t,4} + D_{t,5} + D_{t,6}\eqsp,
    \end{align}
    with
    \begin{align*}
        &\upsilon_{s,h,c} = \nagent^{-1}\step_t^{-1/2} \step_s \Sigma_t^{-1/2}\prod_{i=s+1}^t \Gavg_i \left\{G_{s,1:h-1}^c \zmfuncA[c]{Z_{s,h,c}}(\thetalim[c]-\thetalim) - G_{s,h+1:\nlupdates_s}^c \funcnoise[c]{Z_{s,h,c}} \right\} \eqsp,\\
        &D_{t,1} = \step_t^{-1/2} \Sigma_t^{-1/2}\prod_{i=1}^{t}\Gammaavg_i (\theta_0 -\thetalim) \eqsp, \\
        &D_{t,2} = \nagent^{-1}\step_t^{-1/2}\Sigma_t^{-1/2}\sum_{s=1}^{t} \prod_{i=s+1}^{t}\Gavg_i \sum_{c=1}^\nagent R_s^c(\thetalim[c]-\thetalim) \eqsp, \\
        &D_{t,3} = \step_t^{-1/2}\Sigma_t^{-1/2} \sum_{s=1}^t \Big( \prod_{i=s+1}^t \Gammaavg_i - \prod_{i=s+1}^t \Gavg_i \Big) \Deltaavg_s \eqsp,\\
        &D_{t,4} =  \step_t^{-1/2}\Sigma_t^{-1/2} \sum_{s=1}^t \prod_{i=s+1}^t \Gavg_i \rhoavg_s \eqsp,\\
        &D_{t,5} = \nagent^{-1} \step_t^{-1/2}\Sigma_t^{-1/2}\sum_{s=1}^{t} \step_s \Big( \prod_{i=s+1}^t \Gavg_i - \prod_{i=s+1}^t \Gammaavg_i \Big) \sum_{c=1}^{\nagent}\sum_{h=1}^{\nlupdates_s} \Gamma_{s,h+1:\nlupdates_s}^c \funcnoise[c]{Z_{s,h,c}} \eqsp, \\
        &D_{t,6} = \nagent^{-1} \step_t^{-1/2}\Sigma_t^{-1/2}\sum_{s=1}^{t} \step_s \prod_{i=s+1}^t \Gavg_i \sum_{c=1}^{\nagent}\sum_{h=1}^{\nlupdates_s} (G_{s,h+1:\nlupdates_s}^c - \Gamma_{s,h+1:\nlupdates_s}^c) \funcnoise[c]{Z_{s,h,c}} \eqsp.
    \end{align*}

Theorem~\ref{thm:gauss_approx_last_iter}
Now we are ready to proof \Cref{thm:gauss_approx_last_iter}~-- the main result of this section.

\subsection{Proof of the Theorem~\ref{thm:gauss_approx_last_iter}}

    Applying \citet[Theorem 2.1]{shao2022berry} to the decomposition \eqref{eq:decomp_clt_last_iter}, we have
    \begin{align}
        \label{eq:shao_zheng_decomp_last_iter}
        \kolmogorov (\step_t^{-1/2} \Sigma_t^{-1/2} (\theta_t - \thetalim), Y) \leq \underbrace{259 d^{1/2} \Upsilon_t}_\text{$R_1$} + \underbrace{2 \PE[\norm{W_t} \norm{D_t}]}_\text{$R_2$} + \underbrace{2 \sum \limits_{s = 1}^t \sum \limits_{h = 1}^{H_s} \sum \limits_{c = 1}^N \PE[\norm{\upsilon_{s,h,c}} \norm{D_t - D_t^{(s, h, c)}}]}_\text{$R_3$} \eqsp,
    \end{align}
    where $\Upsilon_t = \sum \limits_{s = 1}^t \sum \limits_{h = 1}^{H_s} \sum \limits_{c = 1}^\nagent \PE[\norm{\upsilon_{s,h,c}}^3]$ and $D_t^{(s, h, c)}$ is a counterpart of $D_t$, in which we have replaced $Z_{s, h, c}$ by its independent copy $Z_{s, h, c}'$. In the same way, we define $D_{t, i}^{(s, h, c)}$ for $i \in \{ 1, 2, 3, 4 \}$.

    Thus, to obtain the desired result we have to bound $R_1$, $R_2$ and $R_3$.

    \textbf{Bound on $R_1$}

    At first, we note that
    \begin{align*}
        \norm{\upsilon_{s,h,c}} &\leq \bConst{\Sigma} \nagent^{-1/2} \step_t^{-1/2} \step_s  \prod_{i = s + 1}^t (1 - \step_i a)^{H_i}( (1 - \step_s a)^{H_s - h} \| \varepsilon \|_\infty + (1-\step_s a)^{h-1}\bConst{A} \norm{\thetalim[c] - \thetalim})\\
        &\leq \bConst{\Sigma} \nagent^{-1/2} \step_t^{-1/2} \step_s  \exp{\left(-a\sum_{\ell=s+1}^t \step_\ell \nlupdates_\ell \right)} ((1 - \step_s a)^{H_s - h} \supconsteps + (1-\step_s a)^{h-1}\bConst{A}\norm{\thetalim[c] - \thetalim})\\
        &\leq \bConst{\Sigma} \nagent^{-1/2} \step_t^{-1/2} \step_s  \exp{\left(-a\sum_{\ell=s+1}^t \step_\ell \nlupdates_\ell \right)} ( \supconsteps + \bConst{A}\norm{\thetalim[c] - \thetalim}) \eqsp.
    \end{align*}
    Thus, applying \Cref{lemma: legendary-sum}, we get
    \begin{align}
        \nonumber \Upsilon_t &= \sum\limits_{s = 1}^t \sum \limits_{h = 1}^{H_s} \sum \limits_{c = 1}^\nagent \PE[\norm{\upsilon_{s,h,c}}^3] \leq \bConst{\Sigma}^3  \nagent^{-1/2} \step_t^{-3/2} \sum_{s=1}^t \step_s^3 \nlupdates_s \exp{\left( -3a\sum_{\ell=s+1}^t \step_\ell \nlupdates_\ell \right)}(\supconsteps + \bConst{A} \zeta_3)^3\\
        &\leq 2(\step^2 \nlupdates + La^{-1}(1 + (3a\step\nlupdates)^{2\gamma_\step + \gamma \over 1-\gamma})) \bConst{\Sigma}^3 \nagent^{-1/2} \step^{1/2} (1+t)^{-\gamma_\step/2} (\supconsteps + \bConst{A} \zeta_3)^3 \eqsp, \label{eq:bound_Upsilon_t}
    \end{align}
    where we have set
    \begin{align*}
        \zeta_3^2 = \nagent^{-1}\sum_{c=1}^\nagent \norm{\thetalim[c] - \thetalim}^2 \eqsp.
    \end{align*}
    % Now we have to bound $S$. Let's rewrite it in the following way,
    % \begin{align*}
    %     S = (1 + t)^{3 \gamma / 2} \left( \underbrace{\sum \limits_{k = 0}^{[t / 2]} (1 + t - k)^{-2 \gamma} \mathrm{e}^{-3 a \step H k}}_\text{$S_1$} + \underbrace{\sum \limits_{k = [t/2] + 1}^{t - 1} (1 + t - k)^{-2 \gamma} \mathrm{e}^{-3 a \step H k}}_\text{$S_2$}\right)
    % \end{align*}
    % For $k \leq t / 2$ $(1 + t - k) \geq \frac{1 + t}{2}$. Therefore,
    % \begin{align*}
    %     S_1 \leq \sum \limits_{k = 0}^{[t / 2]} \Big( \frac{1 + t}{2} \Big)^{-2 \gamma} \mathrm{e}^{-3 a \step H k} \leq 2^{2 \gamma} (1 + t)^{-2 \gamma} \frac{1}{1 - \mathrm{e}^{- 3 a \step H}}
    % \end{align*}
    % For $k > [t / 2]$, we can just bound $(1 + t - k)^{-2 \gamma} \leq 1$. Then,
    % \begin{align*}
    %     S_2 \leq \sum \limits_{k = [t / 2] = 1}^{t - 1} \mathrm{e}^{-3 a \step H k} \leq \frac{\mathrm{e}^{-3 \step H t / 2}}{1 - \mathrm{e}^{-3 a \step H}}
    % \end{align*}
    % Thus, we get the bound on $S = (1 + t)^{3 \gamma / 2}(S_1 + S_2)$
    % \begin{align*}
    %     &S \leq \frac{1}{1 - \mathrm{e}^{-3 a \step H}} \left( 2^{2 \gamma} (1 + t)^{-\gamma / 2} + (1 + t)^{3 \gamma / 2} \mathrm{e}^{- 3 a \step H t / 2}\right)
    % \end{align*}
    % As $1 / (1 - \mathrm{e}^{-3 a \step H}) \leq 1 / (a \step H)$, we get

    \textbf{Bound on $R_2$}

    Firstly, we establish the bound on $\Deltaavg_{j}$ for any $1 \leq j \leq t$. By the definition of $\Deltaavg_{j}$, we have
    \begin{align*}
        \Deltaavg_{j} = N^{-1} \sum \limits_{c = 1}^N (\Id - \Gamma_{j, H_j}^c) (\thetalim[c] - \thetalim) = \tauavg_j + \rhoavg_j \eqsp.
    \end{align*}
    Then, using \Cref{lemma: rho-t-estimation} and \eqref{eq:bound-flbi-tauavg}, we obtain
    \begin{align}
    \label{eq:bound_Deltaavg_j}
    \nonumber \PE[\norm{\Deltaavg_j}^2] \leq 2\PE[\norm{\tauavg_j}^2] +  2 \PE[\norm{\rhoavg_j}^2] &\leq 2 \varhet^2 \nagent^{-1} \step_j^2 \nlupdates_j + \zeta_1^2 \zeta_2^2 \step_j^4 \nlupdates_j^4 \\
    &\eqsp.
    \end{align}
    Now, Applying H\"olders inequality, we obtain
    \begin{align}
    \label{eq:bound_R2_Holder}
        \PE[\norm{W_t} \norm{D_t}] \leq \PE^{1/2}[\norm{W_t}^2] \PE^{1/2}[\norm{D_t}^2] \eqsp.
    \end{align}
    Firstly, we note that $\PE[\norm{W_t}^2] = d$. Therefore, we can proceed with the $D_t$. For $D_{t,1}$, we have
    \begin{align}
        \nonumber \PE^{1/2}[\norm{D_{t,1}}^2] &\leq \rme^{(1-\gamma)^{-1}}\bConst{\Sigma}\step_t^{-1/2}\exp{(-a(1-\gamma)^{-1} \step \nlupdates (1+t)^{1-\gamma})} \norm{\theta_0 - \thetalim}\\
        &\leq \Auxconst_{D,1} \bConst{\Sigma} \exp{\left(-{a \over 2(1-\gamma)}\step\nlupdates (1+t)^{1-\gamma}\right)}\norm{\theta_0 - \thetalim} \eqsp, \label{eq:bound_D_1}
    \end{align}
    where we have used that $x^{1/2}e^{-cx^{1-\gamma}} \leq (c\rme (1-\gamma))^{-{1 \over 2(1-\gamma)}} e^{-cx^{1-\gamma}/2}$ for any $x,c > 0$ and have set $\Auxconst_{D,1} = 2 \rme^{(1-\gamma)^{-1}}(a\rme \step \nlupdates)^{-{1 \over 2(1-\gamma)}}\step^{1/2}$.
    Using the martingale-difference structure of the sequence $\{R_s^c, 1 \leq s \leq t\}$ for any $c$, \eqref{eq:bound_Ri} and applying \Cref{lemma: legendary-sum}, we obtain
    \begin{align}
        \nonumber \PE[\norm{D_{t,2}}^2] &\leq \bConst{A}^2 \bConst{\Sigma}^2\nagent^{-1}\step_t^{-1} \sum_{s=1}^t \step_s^4 \nlupdates_s^3 \exp{\left(-2a \step \nlupdates \sum_{\ell=s+1}^t (1+\ell)^{-\gamma} \right)} \sum_{c=1}^\nagent \mathrm{Tr}(\noisecovA[c])\norm{\thetalim[c] - \thetalim}^2\\
        &\leq \Auxconst_{D,2}^2 \bConst{\Sigma}^2 (1+t)^{-2\gamma} \eqsp, \label{eq:bound_D_2}
    \end{align}
    where
    \begin{equation*}
        \Auxconst_{D,2}^2 = 8 a^{-1} \bConst{A}^2 \zeta_4^2 ((\step \nlupdates)^3 + La^{-1}(\step \nlupdates)^2(1+(2a\step \nlupdates)^{3\gamma+\gamma_\step \over 1-\gamma})) \eqsp
    \end{equation*}
    and we have set
    \begin{align*}
        \zeta_4^2 = {1 \over \nagent}\sum_{c=1}^\nagent \mathrm{Tr}(\noisecovA[c])\norm{\thetalim[c] - \thetalim}^2 \eqsp.
    \end{align*}
    Then, again using the martingale-difference structure of the sequence $\{\Big( \prod_{i=s+1}^t \Gammaavg_i - \prod_{i=s+1}^t \Gavg_i \Big) \tauavg_s, 1 \leq s \leq t \}$, \Cref{lemma: bound_diff_avg_matrices_prod} and \eqref{eq:bound_Deltaavg_j}, we obtain
    \begin{align}
        \nonumber \PE[\norm{D_{t,3}}^2] &\leq \bConst{\Sigma}^2 \nagent \step_t^{-1}  \sum_{s=1}^t \PE[\norm{\prod_{i=s+1}^t \Gammaavg_i - \prod_{i=s+1}^t \Gavg_i}^2] \PE[\norm{\Deltaavg_s}^2] \\
        \nonumber &\leq 2a^{-1}\varboundA^2\bConst{\Sigma}^2 \step_t^{-1} \sum_{s=1}^t \step_s \PE[\norm{\Deltaavg_s}^2] \exp{\left(-2 a\sum_{\ell=s+1}^t \step_\ell \nlupdates_\ell \right)}\\
        \nonumber &\leq 4a^{-1}\varboundA^2 \zeta_3^2\bConst{\Sigma}^2 \nagent^{-1} \step_t^{-1} \sum_{s=1}^t \step_s^3 \nlupdates_s \exp{\left(- 2a\sum_{\ell=s+1}^t \step_\ell \nlupdates_\ell \right)}\\
        \nonumber &+ 2a^{-1}\varboundA^2 \zeta_1^2 \zeta_2^2 \bConst{\Sigma}^2 \step_t^{-1} \sum_{s=1}^t \step_s^5 \nlupdates_s^4 \exp{\left(- 2a\sum_{\ell=s+1}^t \step_\ell \nlupdates_\ell \right)} \\
        &\leq \Auxconst_{D,3,1}^2\bConst{\Sigma}^2 \nagent^{-1}  (1+t)^{-\gamma_\step} + \Auxconst_{D,3,2}^2 \bConst{\Sigma}^2 (1+t)^{-3 \gamma}\eqsp. \label{eq:bound_D_3}
        % &\leq 4 (a \rme)^{-1} \nagent^{-2} \varboundA^2 \varbound^2 \step_t^{-1} \step \nlupdates \bConst{\Sigma}^2 \sum_{s=1}^t \step_s^2 \exp{(- 2a\step \nlupdates (t-s))} \\
        % &\leq 8 (a^2 \rme)^{-1} \nagent^{-2} \varboundA^2 \varbound^2 \step_t \bConst{\Sigma}^2 \eqsp.
    \end{align}
    where
    \begin{align*}
        \Auxconst_{D,3,1}^2 &= 64 a^{-2}\varboundA^2 \zeta_3^2(\step^2\nlupdates + La^{-1}\step(1+ (2a\step\nlupdates)^{\gamma + 2\gamma_\step \over 1-\gamma}))\eqsp, \\
        \Auxconst_{D,3,2}^2 &= 256 a^{-2} \varboundA^2 \zeta_1^2 \zeta_2^2((\step \nlupdates)^4 + La^{-1}(\step \nlupdates)^3(1+ (2a\step \nlupdates)^{4\gamma + \gamma_\step \over 1-\gamma})) \eqsp.
    \end{align*}
    For $D_{t,4}$, we have
    % \begin{align*}
    %     \PE[\norm{D_{t,4}}^2] \leq 4\bConst{\Sigma}^2 \step^2 \nlupdates^2 \step_t^{-1} \zeta_1^2 \zeta_2^2 \eqsp.
    % \end{align*}
    \begin{align}
        \PE^{1/2}[\norm{D_{t,4}}^2] \leq \Auxconst_{D,4} \bConst{\Sigma} \nagent^{1/2}(1+t)^{-(2\gamma - \gamma_\step)/2} \eqsp, \label{eq:bound_D_4}
    \end{align}
    where $\Auxconst_{D,4} = 8a^{-1}\step^{1/2} \nlupdates$.
    For the term $D_{t,5}$, we note that the sequence $\{(\prod\limits_{i=s+1}^t \Gavg_i - \prod\limits_{i=s+1}^t \Gammaavg_i)\sum\limits_{h=1}^{\nlupdates_s} \Gamma_{s,h+1:\nlupdates_s}^c \funcnoise[c]{Z_{s,h,c}}, 1 \leq s \leq t\}$ is martingale-difference w.r.t the filtration $\mathcal{F}_{s}^{-}$. Then, applying \Cref{lemma: bound_diff_avg_matrices_prod} and using \eqref{eq:interm-bound-Delta-rho}, we get
    \begin{align}
        \nonumber
        \PE[\norm{D_{t,5}}^2] &\leq \bConst{\Sigma}^2\nagent^{-1} \step_t^{-1}\sum_{s=1}^t\sum_{c=1}^{\nagent} \step_s^2 \PE[\norm{(\prod\limits_{i=s+1}^t \Gavg_i - \prod\limits_{i=s+1}^t \Gammaavg_i)\sum\limits_{h=1}^{\nlupdates_s} \Gamma_{s,h+1:\nlupdates_s}^c \funcnoise[c]{Z_{s,h,c}}}^2]\\
        \nonumber
        &\leq \bConst{\Sigma}^2\nagent^{-1} \step_t^{-1}\sum_{s=1}^t \sum_{c=1}^{\nagent} \step_s^2 \PE[\norm{\prod\limits_{i=s+1}^t \Gavg_i - \prod\limits_{i=s+1}^t \Gammaavg_i}^2] \PE[\norm{\sum_{h=1}^{\nlupdates_s}\Gamma_{s,h+1:\nlupdates_s}^c \funcnoise[c]{Z_{s,h,c}}}^2] \\
        \label{eq:bound_D_3_mart_diff}
        &\leq 2a^{-1}\varboundA^2\bConst{\Sigma}^2 \nagent^{-2} \step_t^{-1} \sum_{s=1}^t \sum_{c=1}^{\nagent} \step_s^3\PE[\norm{\sum_{h=1}^{\nlupdates_s}\Gamma_{s,h+1:\nlupdates_s}^c \funcnoise[c]{Z_{s,h,c}}}^2] \exp{\left( - 2a \sum_{\ell = s + 1}^t \eta_\ell H_\ell \right)} \eqsp.
    \end{align}
    Now, using the fact that for any $s \in \{1,\dots, t\}$ the sequence $\{\Gamma_{s,h+1:\nlupdates_s}^c \funcnoise[c]{Z_{s,h,c}}, 1 \leq h \leq \nlupdates_s\}$ is martingale-difference w.r.t filtration $\mathcal{F}_{s,h}^{+}$, we get
    \begin{align}
    \label{eq:bound_sum_gamma_eps}
        \PE[\norm{\sum_{h=1}^{\nlupdates_s}\Gamma_{s,h+1:\nlupdates_s}^c \funcnoise[c]{Z_{s,h,c}}}^2] = \sum_{h=1}^{\nlupdates_s} \PE[\norm{\Gamma_{s,h+1:\nlupdates_s}^c \funcnoise[c]{Z_{s,h,c}}}^2] \leq \mathrm{Tr}(\Sigma_\varepsilon^c) \nlupdates_s \eqsp.
    \end{align}
    Combining \eqref{eq:bound_D_3_mart_diff}, \eqref{eq:bound_sum_gamma_eps}, and using \Cref{lemma: legendary-sum}, we obtain
    \begin{align}
        \nonumber \PE[\norm{D_{t,5}}^2] &\leq 2a^{-1} \varboundA^2 \varbound^2\bConst{\Sigma}^2\nagent^{-1} \step_t^{-1} \sum_{s=1}^t \step_s^3 \nlupdates_s \exp{\left( - 2a \sum_{\ell = s + 1}^t \eta_\ell H_\ell \right)}\\
        &\leq \Auxconst_{D,5}^2\bConst{\Sigma}^2\nagent^{-1} (1+t)^{-\gamma_\step} \eqsp, \label{eq:bound_D_5}
    \end{align}
    where
    \begin{align*}
        \Auxconst_{D,5}^2 = 4 a^{-2} \varboundA^2 \varbound^2 (\step^2 \nlupdates + La^{-1}\step(2a\step \nlupdates)^{\gamma + 2\gamma_\step \over 1-\gamma}) \eqsp.
    \end{align*}
    For the last term $D_{t,6}$, we note that again for any $s \in \{1, \dots, t\}$ the sequence $\{(G_{s,h+1:\nlupdates_s}^c - \Gamma_{s,h+1:\nlupdates_s}^c) \funcnoise[c]{Z_{s,h,c}}, 1 \leq h \leq \nlupdates_s\}$ is martingale-difference w.r.t filtration $\mathcal{F}_{s,h}^+$. Thus, applying \Cref{lemma: legendary-sum}, we get
    \begin{align}
        \nonumber \PE[\norm{D_{t,6}}^2] &= \bConst{\Sigma}^2\nagent^{-1} \step_t^{-1} \sum_{s=1}^t \sum_{c=1}^\nagent \step_s^2 \exp{\left(-2a \sum_{\ell=s+1}^t \step_\ell \nlupdates_\ell \right)} \sum_{h=1}^{\nlupdates_s} \PE[\norm{(G_{s,h+1:\nlupdates_s}^c - \Gamma_{s,h+1:\nlupdates_s}^c)\funcnoise[c]{Z_{s,h,c}}}^2]\\
        \nonumber &\leq \varnoise^2\bConst{\Sigma}^2 \step_t^{-1} \sum_{s=1}^t  \step_s^4 \nlupdates_s^2 \exp{\left(-2a \sum_{\ell=s+1}^t \step_\ell \nlupdates_\ell \right)} \\
        &\leq \Auxconst_{D,6}^2\bConst{\Sigma}^2 (1+t)^{-(\gamma+\gamma_\step)}  \eqsp, \label{eq:bound_D_6}
    \end{align}
    where we have set
    \begin{align*}
        \varnoise^2 &= \nagent^{-1}\sum_{c=1}^{\nagent} \mathrm{Tr}(\Sigma_{\varepsilon}^c)\mathrm{Tr}(\noisecovA[c]) \eqsp,\\
        \Auxconst_{D,6}^2 &= 4 a^{-1}\varnoise^2 (\step^3\nlupdates^2 + La^{-1}\step^2\nlupdates(2a\step \nlupdates)^{2(\gamma + \gamma_\step) \over 1 - \gamma}) \eqsp.
    \end{align*}
    and used that
    \begin{align*}
        \PE[\norm{(G_{s,h+1:\nlupdates_s}^c - \Gamma_{s,h+1:\nlupdates_s}^c)u}^2] &= \step_s^2 \sum_{\ell=h+1}^{\nlupdates_s}\PE[\norm{(\Id-\step_s \bA[c])^{\ell-1} \zmfuncA[c]{Z_{s,\ell,c}} \Gamma_{s,\ell+1:\nlupdates_s}^c u}^2]\\
        &\leq \step_s^2 \sum_{\ell=h+1}^{\nlupdates_s} (1-\step_s a)^{2(\ell -1)} \mathrm{Tr}(\noisecovA[c]) (1-\step_s a)^{2(\nlupdates_s - \ell)} \norm{u}^2\\
        &\leq \step_s^2 \nlupdates_s (1-\step_s a)^{2(\nlupdates_s -1)} \mathrm{Tr}(\noisecovA[c])\norm{u}^2 \eqsp.
    \end{align*}

    \textbf{Bound on $R_3$}

    To bound each term in $R_3$, we can again use H\"olders inequality, that is,
    \begin{align*}
        \PE[\norm{\upsilon_{s,h,c}}\norm{D_t - D_t^{(s,h,c)}}] \leq \PE^{1/2}[\norm{\upsilon_{s,h,c}}^2] \PE^{1/2}[\norm{D_{t} - D_{t}^{(s,h,c)}}^2] \eqsp.
    \end{align*}
    Applying Minkowski's inequality, we get
    \begin{align*}
        \PE^{1/2}[\norm{D_{t} - D_{t}^{(s,h,c)}}^2] &\leq \PE^{1/2}[\norm{D_{t,1} - D_{t,1}^{(s,h,c)}}^2] + \PE^{1/2}[\norm{D_{t,2} - D_{t,2}^{(s,h,c)}}^2]\\
        &+ \PE^{1/2}[\norm{D_{t,3} - D_{t,3}^{(s,h,c)}}^2] + \PE^{1/2}[\norm{D_{t,4} - D_{t,4}^{(s,h,c)}}^2]\\
        &+ \PE^{1/2}[\norm{D_{t,5} - D_{t,5}^{(s,h,c)}}^2] + \PE^{1/2}[\norm{D_{t,6} - D_{t,6}^{(s,h,c)}}^2]\eqsp.
    \end{align*}
    Also, we set
    \begin{align*}
        R_{3, i} = \sum \limits_{s = 1}^t \sum \limits_{h = 1}^{H_s} \sum \limits_{c = 1}^N \PE^{1/2}[\norm{\upsilon_{s,h,c}}^2] \, \PE^{1/2}[\norm{D_{t, i} - D_{t, i}^{(s, h, c)}}^2]
    \end{align*}

    \textbf{Perturbations of $D_{t, 1}$.}
    
    Now, we should bound each perturbation term separately. For the first term, applying \Cref{lemma: difference-of-product}, Minkowski's inequality, we obtain
    \begin{align*}
        \PE^{1/2}[\norm{D_{t, 1} - D_{t, 1}^{(s, h, c)}}^2] &= \PE^{1/2}[\norm{\step_t^{-1/2} \Sigma_t^{-1/2} \left(\prod_{i = 1}^t  \Gammaavg_i - \prod_{i = 1}^t \Gammaavgpert{s, h, c}_i \right) (\theta_0 - \thetalim)}^2] \\
        &\leq \bConst{\Sigma} \step_t^{-1/2} \PE^{1/2}[\norm{\left( \prod_{i = 1}^t \Gammaavg_i - \prod_{i = 1}^t \Gammaavgpert{s, h, c}_i \right) (\theta_0 - \thetalim)}^2]\\
        &\leq \bConst{\Sigma} \step_t^{-1/2} \PE^{1/2}[\norm{ \sum \limits_{\ell = 1}^{t} \prod_{i = 1}^{\ell - 1} \Gammaavg_i \cdot (\Gammaavg_\ell - \Gammaavgpert{s, h, c}_\ell) \prod_{i = \ell + 1}^{t} \Gammaavgpert{s, h, c}_i  (\theta_0 - \thetalim)}^2] \\
        &\leq \rme\bConst{\Sigma}\step_t^{-1/2} \exp{\left( -a \sum \limits_{i = 1}^{t} \step_i \nlupdates_i  \right)} \sum \limits_{\ell = 1}^{t} \PE^{1/2}[\norm{ \left( \Gammaavg_{\ell} - \Gammaavgpert{s, h, c}_\ell \right) (\theta_0 - \thetalim) }^2]\\
        &= \rme\bConst{\Sigma}\nagent^{-2}\step_t^{-1/2} \exp{\left( -a \sum \limits_{i = 1}^{t} \step_i \nlupdates_i  \right)} \PE^{1/2}[\norm{ \left( \Gamma_{s,\nlupdates_s}^c - \Gamma_{s,\nlupdates_s}^{c,(s, h, c)} \right) (\theta_0 - \thetalim) }^2] \eqsp.
    \end{align*}

According to \Cref{lemma: gamma-diff-with-copy}, we have
    \begin{align*}
        \PE^{1/2}[\norm{D_{t, 1} - D_{t, 1}^{(s, h, c)}}^2] &\leq 4\rme\bConst{\Sigma} \nagent^{-1/2}\step_t^{-1/2}\step_s \exp{\left( -a \sum_{\ell=1}^t\step_\ell \nlupdates_\ell \right)} \sqrt{\mathrm{Tr}(\noisecovA[c])} \norm{\theta_0 - \thetalim} \\
        &\leq 4\rme\bConst{\Sigma} \nagent^{-1/2} \step_t^{-1/2}\step_s \exp{\left( -{a \step\nlupdates \over 1-\gamma} (1+t)^{1-\gamma} \right)} \sqrt{\mathrm{Tr}(\noisecovA[c])} \norm{\theta_0 - \thetalim} \eqsp.
    \end{align*}
    Thus, for $R_{3,1}$, we obtain
    \begin{align}
        \nonumber R_{3,1} &\leq 4\rme\bConst{\Sigma}^2 \step_t^{-1}\exp{\left( -{a \step\nlupdates \over 1-\gamma} (1+t)^{1-\gamma} \right)} \sum_{s=1}^t \step_s^2\nlupdates_s \exp{\left( -a\sum_{\ell=s+1}^t \step_\ell \nlupdates_\ell \right)}(\varboundA \supconsteps + \zeta_4\bConst{A})\norm{\theta_0 - \thetalim}\\
        \nonumber &\leq 4\rme(2\step \nlupdates + L a^{-1}(1 + (a\step\nlupdates)^{3\gamma_\step \over 2(1 -\gamma)}))\bConst{\Sigma}^2 \exp{\left( -{a \step\nlupdates \over 1-\gamma} (1+t)^{1-\gamma} \right)} (\varboundA \supconsteps + \zeta_4\bConst{A})\norm{\theta_0 - \thetalim}\\
        &= \Auxconst_1\bConst{\Sigma}^2 \exp{\left( -{a \step\nlupdates \over 1-\gamma} (1+t)^{1-\gamma} \right)} \eqsp, \label{eq:bound_R_3_1}
    \end{align}
    where we have set
    \begin{align*}
        \Auxconst_{1} = 8\rme(\step \nlupdates + L a^{-1}(1 + (a\step\nlupdates)^{3\gamma_\step \over 2(1 -\gamma)}))(\varboundA \supconsteps + \zeta_4\bConst{A})\norm{\theta_0 - \thetalim} \eqsp.
    \end{align*}

    \textbf{Perturbations of $D_{t, 2}$.}
    Now, we can proceed with the perturbations. For the perturbations of $D_{t,2}$, we have
    \begin{align*}
        \PE[\norm{D_{t,2} - D_{t,2}^{(s,h,c)}}^2] \leq 2 \bConst{A}^2 \bConst{\Sigma}^2 \nagent^{-1} \step_t^{-1} \step_s^4\nlupdates_s^3 \exp{\left( -2a\sum_{\ell=s+1}^t \step_\ell \nlupdates_\ell \right)} \mathrm{Tr}(\noisecovA[c])\norm{\thetalim[c] -\thetalim}^2 \eqsp,
    \end{align*}
    since $R_{s'}^{c'} - R_{s'}^{c', (s,h,c)} = 0$ for $c'\neq c$ and $s' \neq s$. This leads to the following estimation in the final bound
    \begin{align}
        \nonumber R_{3,2} &\leq 2\bConst{A}\bConst{\Sigma}^2 \step_t^{-1} \sum_{s=1}^t \step_s^3 \nlupdates_s^{5/2} \exp{\left( -2a\sum_{\ell=s+1}^t \step_\ell \nlupdates_\ell \right)} \zeta_4 (\| \varepsilon \|_\infty + \bConst{A} \zeta_3)\\
        &\leq \Auxconst_2 \bConst{\Sigma}^2 (1+t)^{-(3\gamma - \gamma_\step)/2}\eqsp, \label{eq:bound_R_3_2}
    \end{align}
    where we also used the Cauchy-Schwartz inequality to bound $\nagent^{-1}\sum\limits_{c=1}^\nagent \sqrt{\mathrm{Tr}(\noisecovA[c])}\norm{\thetalim[c] - \thetalim}(\| \varepsilon \|_\infty + \bConst{A}\norm{\thetalim[c] -\thetalim})$ and defined 
    \begin{align*}
        \Auxconst_2 = 8\bConst{A}(\step^2 \nlupdates^{5/2} + La^{-1}(1+ (2a\step\nlupdates)^{{5\over 2}+{\gamma_\step \over 2\gamma}}))\zeta_4(\supconsteps + \bConst{A} \zeta_3) \eqsp.
    \end{align*}

    \textbf{Perturbations of $D_{t, 3}$.}
    
    For the perturbations of $D_{t,3}$, we have the decomposition $D_{t,3} - D_{t,3}^{(s,h,c)} = D_{t,3}^1 + D_{t,3}^2$, where we define
    \begin{align*}
        D_{t,3}^1 &= \step_t^{-1/2}\Sigma_t^{-1/2}\sum_{s'=1}^t (\prod_{i=s'+1}^t \Gammaavg_i - \prod_{i=s'+1}^t \Gavg_i) (\Deltaavg_{s'} - \Deltaavgpert{s,h,c}_{s'})\\
        D_{t,3}^2 &= \step_t^{-1/2}\Sigma_t^{-1/2} \sum_{s'=1}^t(\prod_{i=s'+1}^t\Gammaavg_i - \prod_{i=s'+1}^t \Gammaavgpert{s,h,c}_i) \Deltaavgpert{s,h,c}_{s'} \eqsp.
    \end{align*}
    We start with estimating $\PE^{1/2}[\norm{D_{t, 3}^{1}}^2]$. As $\Deltaavg_{s'} = \Deltaavgpert{s, h, c}_{s'}$ for $s' \neq s$, we have
    \begin{align*}
        \PE[\norm{ D_{t, 3}^1}^2] &\leq \bConst{\Sigma}^2\step_t^{-1}  \PE[\norm{\left( \prod_{i = s + 1}^t \Gammaavg_i - \prod_{i = s + 1}^t \Gavg_i \right) (\Deltaavg_s - \Deltaavgpert{s, h, c}_s)}^2] \\
        &\leq \bConst{\Sigma}^2 \step_{t}^{-1} \PE[\norm{\prod_{i=s+1}^t\Gammaavg_i - \prod_{i=s+1}^t \Gavg_i}^2] \PE[\norm{\Deltaavg_s - \Deltaavgpert{s,h,c}_s}^2] \eqsp,
    \end{align*}
    where we have noticed that $\prod\limits_{i=s+1}^t\Gammaavg_\ell - \prod\limits_{i=s+1}^t\Gavg_\ell$ and $\Deltaavg_s - \Deltaavgpert{s, h, c}_s$ are independent. Now, applying \Cref{lemma: bound_diff_avg_matrices_prod} and \Cref{lemma: bound_diff_delta_avg}, we obtain
    \begin{align*}
        \PE[\norm{ D_{t, 3}^1}^2] &\leq 2a^{-1} \varboundA^2 \bConst{\Sigma}^2  \step_t^{-1} \step_s \exp{\left( -a\sum_{\ell=s+1}^t \step_\ell \nlupdates_\ell \right)}\PE[\norm{\Deltaavg_s - \Deltaavgpert{s,h,c}_s}^2] \\
        &\leq \Auxconst_{3,1}^2\bConst{\Sigma}^2\nagent^{-2} \step_t^{-1} \step_s^3 \exp{\left( -a\sum_{\ell=s+1}^t \step_\ell \nlupdates_\ell \right)} \mathrm{Tr}(\noisecovA[c])\norm{\thetalim[c] - \thetalim}^2\eqsp.
    \end{align*}
    where we have set 
    \begin{equation*}
        \Auxconst_{3,1}^2 = 32a^{-1}\varboundA^2  \eqsp.
    \end{equation*}
    For the second term, applying \Cref{lemma: bound_diff_avg_gamma_prod}, using the martingale-difference structure of the sequence $\{ \Big(\prod_{i=s+1}^t \Gammaavg_i - \prod_{i=s+1}^t \Gammaavgpert{s,h,c}_i \Big) \Deltaavg_{s'}, 1 \leq s' \leq t\}$, \eqref{eq:bound_Deltaavg_j} and \Cref{lemma: legendary-sum} together with \eqref{eq:bound_Deltaavg_j}, we get
    \begin{align*}
        \PE[\norm{D_{t,3}^2}^2] &\leq \bConst{\Sigma}^2 \step_t^{-1} \sum_{s'=1}^{s-1}\PE[\norm{\prod_{i=s'+1}^t \Gammaavg_i - \prod_{i=s'+1}^t \Gammaavgpert{s,h,c}_i}^2] \PE[\norm{\Deltaavgpert{s,h,c}_{s'}}^2]\\
        &\leq 4\rme \bConst{\Sigma}^2 \nagent^{-1} \step_t^{-1} \step_{s}^2 \sum_{s'=1}^{s-1} \PE[\norm{\Deltaavg_{s'}}^2] \exp{\left( -2a \sum_{\ell=s'+1}^t \step_\ell \nlupdates_\ell \right)} \mathrm{Tr}(\noisecovA[c])\\
        &\leq 8\rme \varhet^2 \bConst{\Sigma}^2 \nagent^{-2} \step_t^{-1} \step_{s}^2 \sum_{s'=1}^{s-1} \step_{s'}^2\nlupdates_{s'} \exp{\left( -2a \sum_{\ell=s'+1}^t \step_\ell \nlupdates_\ell \right)}\mathrm{Tr}(\noisecovA[c]) \\
        &+ 4\rme \bConst{\Sigma}^2 \zeta_1^2 \zeta_2^2 \nagent^{-1} \step_t^{-1} \step_{s}^2 \sum_{s'=1}^{s-1} \step_{s'}^4 \nlupdates_{s'}^4 \exp{\left( -2a \sum_{\ell=s'+1}^t \step_\ell \nlupdates_\ell \right)} \mathrm{Tr}(\noisecovA[c])\\
        &\leq 8\rme\bConst{A}^2 \bConst{\Sigma}^2 \nagent^{-2} \step_t^{-1} \step_{s}^2 \exp{\left( -2a \sum_{\ell=s+1}^t \step_\ell \nlupdates_\ell \right)} \mathrm{Tr}(\noisecovA[c]) \sum_{s'=1}^{s-1} \step_{s'}^2\nlupdates_{s'} \exp{\left( -2a \sum_{\ell=s'+1}^s \step_\ell \nlupdates_\ell \right)}\\
        &+ 4\rme \zeta_1^2 \zeta_2^2 \bConst{\Sigma}^2 \nagent^{-1} \step_t^{-1} \step_{s}^2 \exp{\left( -2a \sum_{\ell=s+1}^t \step_\ell \nlupdates_\ell \right)} \mathrm{Tr}(\noisecovA[c]) \sum_{s'=1}^{s-1} \step_{s'}^4\nlupdates_{s'}^4 \exp{\left( -2a \sum_{\ell=s'+1}^s \step_\ell \nlupdates_\ell \right)} \\
        &\leq 32 \rme a^{-1}\bConst{A}^2 \bConst{\Sigma}^2 \nagent^{-2} \step_t^{-1} \step_{s}^3 \exp{\left( -2a \sum_{\ell=s+1}^t \step_\ell \nlupdates_\ell \right)} \mathrm{Tr}(\noisecovA[c])\\
        &+ 256\rme a^{-1}\zeta_1^2 \zeta_2^2 \bConst{\Sigma}^2  \nagent^{-1} (\step\nlupdates)^3 \step_t^{-1} \step_{s}^2 (1+s)^{-3\gamma} \exp{\left( -2a \sum_{\ell=s+1}^t \step_\ell \nlupdates_\ell \right)} \mathrm{Tr}(\noisecovA[c]) \eqsp,
        % &\leq \Auxconst_{3,2}^2 \bConst{\Sigma}^2\nagent^{-2} \step^2 \nlupdates^2 \step_t^{-1} \step_s^2 \exp{\left( -2a \sum_{\ell=s+1}^t \step_\ell \nlupdates_\ell \right)} \mathrm{Tr}(\noisecovA[c]) \eqsp,
    \end{align*}
    where we used that $\gamma > 1/2$ for the second term. Therefore, we have
    \begin{align*}
        \PE^{1/2}[\norm{D_{t,3} - D_{t,3}^{(s,h,c)}}^2] &\leq (\Auxconst_{3,2} + \Auxconst_{3,1}\norm{\thetalim[c] - \thetalim})\bConst{\Sigma} \nagent^{-1} \step_t^{-1/2} \step_s^{3/2} \exp{\left( -a\sum_{\ell=s+1}^t \step_\ell \nlupdates_\ell \right)} \sqrt{\mathrm{Tr}(\noisecovA[c])}\\
        &+ \Auxconst_{3,3}\bConst{\Sigma} \nagent^{-1/2} (\step \nlupdates)^{3/2}\step_t^{-1/2} \step_s (1+s)^{-3\gamma/2}\exp{\left( -a \sum_{\ell=s+1}^t \step_\ell \nlupdates_\ell \right)}\sqrt{\mathrm{Tr}(\noisecovA[c])} \eqsp,
    \end{align*}
    where $\Auxconst_{3,2}^2 = 32\rme a^{-1} \bConst{A}^2$ and $\Auxconst_{3,3}^2 = 256\rme a^{-1}\zeta_1^2 \zeta_2^2$.
    Now, using this bound and applying \Cref{lemma: legendary-sum}, we get
    \begin{align}
        \nonumber R_{3,3} &\leq 4(\Auxconst_{3,2}^{1/2}\varbound + \Auxconst_{3,1}^{1/2}\zeta_4) \bConst{\Sigma}^2\nagent^{-1/2}\step_t^{-1} \sum_{s=1}^t \step_s^{5/2}\nlupdates_s \exp{\left( -2a\sum_{\ell=s+1}^t \step_\ell \nlupdates_\ell \right)} (\supconsteps + \bConst{A} \zeta_3)\\
        \nonumber &+ \Auxconst_{3,3}  \bConst{\Sigma}^2 (\step \nlupdates)^{3/2}\step_t^{-1}\sum_{s=1}^t \step_s^2 \nlupdates_s (1+s)^{-3\gamma/2} \exp{\left( -2a\sum_{\ell=s+1}^t \step_\ell \nlupdates_\ell \right)} (\varbound \supconsteps + \bConst{A} \zeta_4) \\
        &\leq \Auxconst_{3,4}\bConst{\Sigma}^2\nagent^{-1/2} (1+t)^{-\gamma_\step/2} + \Auxconst_{3,5} \bConst{\Sigma}^2\nagent^{-1} (1+t)^{-3\gamma/2} \eqsp, \label{eq:bound_R_3_3}
    \end{align}
    where we have set
    \begin{align*}
        \Auxconst_{3,4} &= 2(\Auxconst_{3,2}^{1/2}\varbound + \Auxconst_{3,1}^{1/2}\zeta_4)(\step^{3/2}\nlupdates + L a^{-1}\step^{1/2}(1+(2a\step\nlupdates)^{1+{3\gamma_\step \over 2\gamma}})(\supconsteps + \bConst{A} \zeta_3) \eqsp,\\
        \Auxconst_{3,5} &= 2\Auxconst_{3,3} ((\step\nlupdates)^{5/2} + La^{-1}(\step\nlupdates)^{3/2}(1+(2a\step\nlupdates)^{5/2+\gamma_\step/\gamma})(\varbound\supconsteps + \bConst{A} \zeta_4) \eqsp.
    \end{align*}

    % \textbf{Update}. BUT, if we insead of using \Cref{lemma: legendary-sum}, will bound $\mathrm{e}^{-2 a \step \nlupdates (t - s)} \leq 1$, we get:
    % \begin{align*}
    %     R_{3, 2} \leq 16 \bConst{\Sigma}^2 \bConst{het} (1 + t)^\gamma \step \nlupdates \| \varepsilon \|_\infty \cdot (1 + t)^{1 - \gamma}
    % \end{align*}
    % \begin{remark}
    %     If we have $\step \nlupdates << (1 + t)^{\gamma - 1}$, then $R_{3, 2} << 1$.
    % \end{remark}

    \textbf{Perturbations of $D_{t, 5}$.}

    Note that for $D_{t,5}$, we have the following decomposition
    \begin{align}
        \label{eq:decomp_Dt5}
        \nonumber D_{t,5} &= \nagent^{-1} \step_t^{-1/2}\Sigma_t^{-1/2}\sum_{s',h',c'} \step_{s'} \Big(\prod_{i=s'+1}^t \Gammaavgpert{s,h,c}_i - \prod_{i=s'+1}^t \Gammaavg_i \Big)\Gamma_{s',h'+1:\nlupdates_{s'}}^{c'} \funcnoise[c']{Z_{s',h',c'}}\\
        \nonumber &+ \nagent^{-1} \step_t^{-1/2}\Sigma_t^{-1/2}\sum_{s', h', c'} \step_{s'} \Big(\prod_{i=s'+1}^t \Gavg_i - \prod_{i=s'+1}^t \Gammaavgpert{s,h,c}_i \Big) (\Gamma_{s',h'+1:\nlupdates_{s'}}^{c'} - \Gamma_{s',h'+1:\nlupdates_{s'}}^{c', (s,h,c)}) \funcnoise[c']{Z_{s',h',c'}}\\
        \nonumber &+ \nagent^{-1} \step_t^{-1/2}\step_s \Sigma_t^{-1/2} (\prod_{i=s+1}^t \Gavg_i - \prod_{i=s+1}^t \Gammaavgpert{s,h,c}_i) \Gamma_{s,h+1:\nlupdates_s}^{c, (s,h,c)} (\funcnoise[c]{Z_{s,h,c}} - \funcnoise[c]{Z_{s,h,c}'})\\
        &+ \sum_{s',h',c'} \step_{s'}(\prod_{i=s'+1}^t \Gavg_i - \prod_{i=s'+1}^t \Gammaavgpert{s,h,c}_i) \Gamma_{s',h'+1:\nlupdates_{s'}}^{c',(s,h,c)}\funcnoise[c']{Z_{s',h',c'}^{s,h,c}} = T_{5,1} + T_{5,2} + T_{5,3} + T_{5,4} \eqsp.
    \end{align}
    % \begin{align}
    % % \label{eq:decomp_Dt5}
    %     \nonumber D_{t,5} &= \nagent^{-1} \step_t^{-1/2}\Sigma_t^{-1/2}\sum_{(s',h',c')\neq(s,h,c)} \step_{s'} \Big(\prod_{i=s'+1}^t \Gammaavgpert{s,h,c}_i - \prod_{i=s'+1}^t \Gammaavg_i \Big)\Gamma_{s',h'+1:\nlupdates_{s'}}^{c'} \funcnoise[c']{Z_{s',h',c'}}\\
    %     \nonumber &+ \nagent^{-1} \step_t^{-1/2}\Sigma_t^{-1/2}\sum_{(s',h',c')\neq(s,h,c)} \step_{s'} \Big(\prod_{i=s'+1}^t \Gavg_i - \prod_{i=s'+1}^t \Gammaavgpert{s,h,c}_i \Big) (\Gamma_{s',h'+1:\nlupdates_{s'}}^{c'} - \Gamma_{s',h'+1:\nlupdates_{s'}}^{c', (s,h,c)}) \funcnoise[c']{Z_{s',h',c'}}\\
    %     \nonumber &+ \nagent^{-1} \step_t^{-1/2}\Sigma_t^{-1/2}\sum_{(s',h',c')\neq(s,h,c)} \step_{s'} \Big(\prod_{i=s'+1}^t \Gavg_i - \prod_{i=s'+1}^t \Gammaavgpert{s,h,c}_i \Big) \Gamma_{s',h'+1:\nlupdates_{s'}}^{c', (s,h,c)} \funcnoise[c']{Z_{s',h',c'}}\\
    %     \nonumber &+ \nagent^{-1} \step_t^{-1/2}\step_s \Sigma_t^{-1/2} (\prod_{i=s+1}^t \Gavg_i - \prod_{i=s+1}^t \Gammaavg_i) \Gamma_{s,h+1:\nlupdates_s}^c \{ \funcnoise[c]{Z_{s,h,c}} - \funcnoise[c]{Z_{s,h,c}'} \} \\
    %     \nonumber &+ \nagent^{-1} \step_t^{-1/2}\step_s \Sigma_t^{-1/2} (\prod_{i=s+1}^t \Gavg_i - \prod_{i=s+1}^t \Gammaavg_i) \Gamma_{s,h+1:\nlupdates_s}^c \funcnoise[c]{Z_{s,h,c}'} \\
    %     &= T_{5,1} + T_{5,2} + T_{5,3} + T_{5,4} + T_{5,5} \eqsp.
    % \end{align}
    We can see that in the perturbation $D_{t,5} - D_{t,5}^{(s,h,c)}$ there are only three terms $T_1, T_2$ and $T_3$. Thus, we have
    \begin{align*}
        \PE[\norm{D_{t,5} - D_{t,5}^{(s,h,c)}}^2] \leq 4\PE[\norm{T_{5,1}}^2] + 4\PE[\norm{T_{5,2}}^2] + 4\PE[\norm{T_{5,3}}^2] \eqsp.
    \end{align*}
    Further, we bound these terms separately. For $T_{5,1}$, using the martingale-difference structure of the sequences $\{\Big( \prod_{i=s+1}^t \Gammaavgpert{s,h,c}_i - \prod_{i=s+1}^t \Gammaavg_i \Big)\Gamma_{s,h+1:\nlupdates_s}^c \funcnoise[c]{Z_{s,h,c}}, s \in [t]\}$ and $\{\Gamma_{s,h+1:\nlupdates_s}^c \funcnoise[c]{Z_{s,h,c}}, h \in [\nlupdates_s]\}$, applying \Cref{lemma: bound_diff_avg_gamma_prod} and \Cref{lemma: legendary-sum}, we get
    \begin{align*}
        \PE[&\norm{T_{5,1}}^2] \\
        &\leq \bConst{\Sigma}^2 \nagent^{-1}\step_t^{-1} \sum_{s'=1}^t\sum_{c'=1}^\nagent \step_{s'}^2 \PE[\norm{\sum_{h'=1}^{\nlupdates_{s'}} \Big(\prod_{i=s'+1}^t \Gammaavgpert{s,h,c}_i - \prod_{i=s'+1}^t \Gammaavg_i \Big) \Gamma_{s',h'+1:\nlupdates_{s'}}^{c'} \funcnoise[c']{Z_{s',h',c'}}}^2] \\
        &\leq \bConst{\Sigma}^2 \nagent^{-1}\step_t^{-1} \sum_{s'=1}^t\sum_{c'=1}^\nagent \step_{s'}^2 \PE[\norm{\prod_{i=s'+1}^t \Gammaavgpert{s,h,c}_i - \prod_{i=s'+1}^t \Gammaavg_i}^2] \sum_{h'=1}^{\nlupdates_{s'}} \PE[\norm{\Gamma_{s',h'+1:\nlupdates_{s'}}^{c'}\funcnoise[c']{Z_{s',h',c'}}}^2]\\
        &\leq 4 \rme \varbound^2  \bConst{\Sigma}^2 \nagent^{-2}\step_t^{-1}\step_s^2 \exp{\left( -2a\sum_{\ell=s+1}^t \step_\ell \nlupdates_\ell \right)}\mathrm{Tr}(\noisecovA[c]) \sum_{s'=1}^{s-1}  \step_{s'}^2\nlupdates_{s'} \exp{\left(-2a\sum_{\ell=s'+1}^s \step_\ell \nlupdates_\ell \right)}\\
        &\leq \Auxconst_{5,1}^2 \bConst{\Sigma}^2 \nagent^{-2} \step_t^{-1} \step_s^{3} \exp{\left( -2a\sum_{\ell=s+1}^t \step_\ell \nlupdates_\ell \right)}\mathrm{Tr}(\noisecovA[c]) \eqsp,
    \end{align*}
    where
    \begin{align*}
        \Auxconst_{5,1}^2 = 32 \rme a^{-1} L \varbound^2 (1 + (2a\step \nlupdates)^{\gamma_\step + \gamma \over 1 -\gamma}) \eqsp.
    \end{align*}
    For the term $T_{5,2}$, we note that $\Gamma_{s',h'+1:\nlupdates_{s'}}^{c'} - \Gamma_{s',h'+1:\nlupdates_{s'}}^{c', (s,h,c)} = 0$ for $s' \neq s$. Hence, applying \Cref{lemma: bound_diff_avg_matrices_prod} and \Cref{lemma: gamma-diff-with-copy}, we get
    \begin{align*}
        \PE[\norm{T_{5,2}}^2] &\leq \bConst{\Sigma}^2 \nagent^{-1} \step_t^{-1} \step_s^2 \PE[\norm{\prod_{i=s+1}^t \Gavg_i - \prod_{i=s+1}^t \Gammaavg_i}^2] \sum_{h'=1}^{h-1} \PE[\norm{(\Gamma_{s,h'+1:\nlupdates_{s'}}^{c} - \Gamma_{s,h'+1:\nlupdates_{s}}^{c, (s,h,c)}) \funcnoise[c]{Z_{s,h',c}}}^2]\\
        &\leq 2a^{-1}\varboundA^2 \bConst{\Sigma}^2 \nagent^{-2} \step_t^{-1} \step_s^3 \exp{\left( -2a\sum_{\ell=s+1}^t \step_\ell \nlupdates_\ell \right)}\sum_{h'=1}^{h-1} \PE[\norm{\Gamma_{s,h'+1:\nlupdates_{s'}}^{c} - \Gamma_{s,h'+1:\nlupdates_{s}}^{c, (s,h,c)}}^2] \mathrm{Tr}(\noisecov[c])\\
        &\leq \Auxconst_{5,2}^2 \bConst{\Sigma}^2 \nagent^{-2}\step_t^{-1} \step_s^5 \nlupdates_s \exp{\left( -2a\sum_{\ell=s+1}^t \step_\ell \nlupdates_\ell \right)} \mathrm{Tr}(\noisecovA[c])\mathrm{Tr}(\noisecov[c]) \eqsp,
    \end{align*}
    where we have set $\Auxconst_{5,2}^2 = 32a^{-1}\varboundA^2$.
    % \begin{align*}
    %     \PE[\norm{T_{3,2}}^2] &\leq \nagent^{-2} \step_t^{-1} \bConst{\Sigma}^2 \sum_{(s',h',c')\neq (s,h,c)} \step_{s'}^2 \PE[\norm{\prod_{i=s'+1}^t \Gavg_i - \prod_{i=s'+1}^t \Gammaavgpert{s,h,c}_i}^2]\PE[\norm{(\Gamma_{s',h'+1:\nlupdates_{s'}}^{c'} - \Gamma_{s',h'+1:\nlupdates_{s'}}^{c',(s,h,c)})\funcnoise[c']{Z_{s',h',c'}}}^2] \\
    %     &\leq \nagent^{-3}\step^2 \nlupdates \step_t^{-1} (1+t)^{1-\gamma} \bConst{\Sigma}^2\varboundA^2\sum_{s'=1}^{t} \step_{s'}^2 \exp{(-2a\step \nlupdates (t-s'-1))} \PE[\norm{(\Gamma_{s',h'+1:\nlupdates_{s'}}^{c'} - \Gamma_{s',h'+1:\nlupdates_{s'}}^{c',(s,h,c)})\funcnoise[c']{Z_{s',h',c'}}}^2] \\
    %     &\leq  \eqsp.
    % \end{align*}
    For the term $T_{5,3}$, applying \Cref{lemma: bound_diff_avg_matrices_prod}, we obtain
    \begin{align*}
        \PE[\norm{T_{5,3}}^2] &\leq 2\bConst{\Sigma}^2 \nagent^{-1}\step_t^{-1} \step_s^2  \PE[\norm{\prod_{i=s+1}^t \Gavg_i - \prod_{i=s+1}^t \Gammaavg_i}^2] \mathrm{Tr}(\noisecov[c])\\
        &\leq \Auxconst_{5,3}^2 \bConst{\Sigma}^2 \nagent^{-2} \step_t^{-1}\step_s^3 \exp{\left( -2a \sum_{\ell=s+1}^t \step_\ell \nlupdates_\ell \right)}\mathrm{Tr}(\noisecov[c])\eqsp,
    \end{align*}
    where we define $\Auxconst_{5,3}^2 = 4 a^{-1} \varboundA^2$. Thus, combining the above bounds, we get
    \begin{align*}
        \PE[\norm{D_{t,5} - D_{t,5}^{(s,h,c)}}^2] \leq \Auxconst_{5,4}^2\bConst{\Sigma}^2 \nagent^{-2} \step_t^{-1} \step_s^3 \exp{\left( -2a\sum_{\ell=s+1}^t \step_\ell \nlupdates_\ell \right)}(1+\mathrm{Tr}(\noisecovA[c]))\mathrm{Tr}(\noisecov[c]) \eqsp,
    \end{align*}
    where we have set
    \begin{align*}
        \Auxconst_{5,4}^2 = 4\Auxconst_{5,1}^2 + 4\Auxconst_{5,2}^2 + 4\Auxconst_{5,3}^2 \eqsp.
    \end{align*}
   Thus, again applying \Cref{lemma: legendary-sum}, we get
    \begin{align}
        \nonumber R_{3,5} &\leq 2\Auxconst_{5,4}\bConst{\Sigma}^2 \nagent^{-1/2}\step_t^{-1} \sum_{s=1}^t \step_s^{5/2} \nlupdates_s \exp{\left(-2a\sum_{\ell=s+1}^t \step_\ell \nlupdates_\ell\right)} (\varbound + \varnoise)(\supconsteps + \bConst{A} \zeta_3)\\
        &\leq \Auxconst_{5,5}\bConst{\Sigma}^2\nagent^{-1/2}(1+t)^{-\gamma_\step/2} \eqsp, \label{eq:bound_R_3_5}
        % R_{3,5} &\lesssim \nagent^{-3} \step^{5/2}\nlupdates^{1/2} \step_t^{-1/2} (1 + t)^{1/2}\supconsteps \sum_{s,h,c} \step_s (1 + s)^{-\gamma} \exp{( -a \step H (t - s))} \\
        % &\lesssim \nagent^{-2} \step^{5/2}\nlupdates^{1/2} \step_t^{-1/2} (1 + t)^{1/2-\gamma}\supconsteps \eqsp.
    \end{align}
    where
    \begin{align*}
        \Auxconst_{5,5} = 4\Auxconst_{5,4} (\step^{3/2}\nlupdates + L\step^{1/2}(1+(2a\step\nlupdates)^{2\gamma+3\gamma_\step \over 1-\gamma}))(\varbound + \varnoise)(\supconsteps + \bConst{A} \zeta_3)\eqsp.
    \end{align*}

    \textbf{Perturbations of $D_{t, 6}$.}

    Firstly, we note that
    \begin{align}
        \label{eq:decomp_Dt6}
        \nonumber D_{t,6} - D_{t,6}^{(s,h,c)} &= \nagent^{-1}\step_t^{-1/2}\Sigma_t^{-1/2} \sum_{s',h',c'} \step_{s'}\prod_{i=s'+1}^t \Gavg_i (\Gamma_{s',h'+1:\nlupdates_{s'}}^{c',(s,h,c)} - \Gamma_{s',h'+1:\nlupdates_{s'}}^{c'})\funcnoise[c]{Z_{s',h',c'}}\\
        \nonumber &+ \nagent^{-1}\step_t^{-1/2}\Sigma_t^{-1/2}\step_s \prod_{i=s+1}^t \Gavg_i  (G_{s,h+1:\nlupdates_s}^c - \Gamma_{s,h+1:\nlupdates_s}^{c,(s,h,c)})(\funcnoise[c]{Z_{s,h,c}} - \funcnoise[c]{Z_{s,h,c}'}) \\
        &= T_{6,1} + T_{6,2} \eqsp.
    \end{align}
    % Using the martingale-difference structure of the sequence $\{((\Gamma_{s',h'+1:\nlupdates_{s'}}^{c',(s,h,c)} - \Gamma_{s',h'+1:\nlupdates_{s'}}^{c'})\funcnoise[c]{Z_{s',h',c'}}), h' \in [\nlupdates_{s'}]\}$ for any $s' \in [t]$ w.r.t filtration $\mathcal{F}_{s',h'}^+$, we can bound the first term, as
    % \begin{align*}
    %     \PE[\norm{T_{4,1}}^2] &= \nagent^{-2} \step_t^{-1}\bConst{\Sigma}^2 \sum_{(s',h',c')\neq (s,h,c)} \step_{s'}^2 \PE[\norm{\prod_{i=s'+1}^t \Gavg_i (\Gamma_{s',h'+1:\nlupdates_{s'}}^{c',(s,h,c)} - \Gamma_{s',h'+1:\nlupdates_{s'}}^{c'})\funcnoise[c]{Z_{s',h',c'}}}^2]\\
    %     &\leq \nagent^{-2} \step_t^{-1} \step_s^2 \exp{(-2a\step \nlupdates (t-s))} \bConst{\Sigma}^2 \mathrm{Tr}(\Sigma_{\varepsilon}^c) \sum_{h'=1}^{h-1} \PE[\norm{\Gamma_{s,h'+1:\nlupdates_{s}}^{c,(s,h,c)} - \Gamma_{s,h'+1:\nlupdates_{s}}^{c}}^2]\\
    %     &\leq 8 \nagent^{-2} \step \nlupdates \step_t^{-1} \step_s^3 \exp{(-2a\step \nlupdates (t-s))} \bConst{\Sigma}^2 \mathrm{Tr}(\Sigma_{\varepsilon}^c) \mathrm{Tr}(\noisecovA[c]) \eqsp.
    % \end{align*}
    We can see that the first term is zero for $s' \neq s$, $c' \neq c$ and $h' \geq h$. Thus, applying \Cref{lemma: gamma-diff-with-copy}, we get
    \begin{align*}
        \PE[\norm{T_{6,1}}^2] &\leq \bConst{\Sigma}^2 \nagent^{-1} \step_t^{-1} \step_s^2 \exp{\left( -2a \sum_{\ell=s+1}^t \step_\ell \nlupdates_\ell \right)} \sum_{h'=1}^{h-1}\PE[\norm{\Gamma_{s,h'+1:\nlupdates_{s}}^{c,(s,h,c)} - \Gamma_{s,h'+1:\nlupdates_{s}}^{c}}^2] \mathrm{Tr}(\noisecov[c])\\
        &\leq 16 \bConst{\Sigma}^2 \nagent^{-1} \step_t^{-1}\step_s^4 \nlupdates_s \exp{\left( -2a \sum_{\ell=s+1}^t \step_\ell \nlupdates_\ell \right)} \mathrm{Tr}(\noisecovA[c])\mathrm{Tr}(\noisecov[c]) \eqsp.
    \end{align*}
    The second term $T_{6,2}$ ca be bounded, as
    \begin{align*}
        \PE[\norm{T_{6,2}}^2] &\leq \bConst{\Sigma}^2\nagent^{-1}\step_t^{-1} \step_s^2   \exp{\left(-2a\sum_{\ell=s+1}^t \step_\ell \nlupdates_\ell\right)}\PE[\norm{G_{s,h+1:\nlupdates_s}^c - \Gamma_{s,h+1:\nlupdates_s}^c}^2] \mathrm{Tr}(\noisecov[c]) \\
        &\leq \bConst{\Sigma}^2 \nagent^{-1}\step_t^{-1} \step_s^4 \nlupdates_s  \exp{\left(-2a\sum_{\ell=s+1}^t \step_\ell \nlupdates_\ell\right)}  \mathrm{Tr}(\noisecovA[c]) \mathrm{Tr}(\noisecov[c]) \eqsp.
    \end{align*}
    Combining together the last two inequalities, we obtain
    \begin{align*}
        \PE[\norm{D_{t,6} - D_{t,6}^{(s,h,c)}}^2] \leq 34 \bConst{\Sigma}^2 \nagent^{-1}\step_t^{-1} \step_s^4 \nlupdates_s \exp{\left(-2a\sum_{\ell=s+1}^t \step_\ell \nlupdates_\ell\right)}\mathrm{Tr}(\noisecovA[c])\mathrm{Tr}(\noisecov[c]) \eqsp.
    \end{align*}
    Thus, we have
    \begin{align}
        \nonumber R_{3,6} &\leq 17\varnoise\bConst{\Sigma}^2 \step_t^{-1} \sum_{s=1}^t \step_s^3 \nlupdates_s^{3/2} \exp{\left(-2a\sum_{\ell=s+1}^t \step_\ell \nlupdates_\ell\right)} \\
        &\leq \Auxconst_6 (1+t)^{-(\gamma+\gamma_\step)/2} \eqsp, \label{eq:bound_R_3_6}
        % \nagent^{-3} \step \nlupdates^{1/2} \step_t^{-1} (1 + t)^{\gamma/2}\supconsteps \sum_{s,h,c} \step_s^{3/2}(1 + s)^{-\gamma} \exp{(- a \step H (t - s))} \\
        % &\lesssim \nagent^{-2} \step^{3/2} \nlupdates^{1/2} \step_t^{-1/2} (1+t)^{-2\gamma} \supconsteps \eqsp.
    \end{align}
    where we have set
    \begin{align*}
        \Auxconst_6 = 68\varnoise(\step^2 \nlupdates^{3/2} + La^{-1}\step \nlupdates^{1/2}(1+(2a\step\nlupdates)^{3(\gamma+\gamma_\step) \over 2(1-\gamma)})) \eqsp.
    \end{align*}

    \paragraph{Finalized Berry-Essen bound}
    We return to the decomposition \eqref{eq:shao_zheng_decomp_last_iter}. We get
    \begin{align*}
        R_1 \leq \Auxconst_{R,1} \bConst{\Sigma}^3 \nagent^{-1/2} \step^{1/2}(1+t)^{-\gamma_\step/2} \eqsp.
    \end{align*}
    where $\Auxconst_{R,1} = 259 d^{1/2}8a^{-1} \step^{1/2}$. For $R_2$, we apply \eqref{eq:bound_R2_Holder} and combine together \eqref{eq:bound_D_1}, \eqref{eq:bound_D_2}, \eqref{eq:bound_D_3}, \eqref{eq:bound_D_4}, \eqref{eq:bound_D_5}, \eqref{eq:bound_D_6}, thus we obtain
    \begin{align*}
            R_2 \leq \Auxconst_{D,1}d^{1/2}\nagent^{1/2} \exp{\left(-a\step\nlupdates t/2\right)}\norm{\theta_0 - \thetalim} &+ (\Auxconst_{D,2} + \Auxconst_{D,3,2} + \Auxconst_{D,4})d^{1/2}\nagent^{1/2}(1+t)^{-(2\gamma - \gamma_\step)/2} \\
            &+ (\Auxconst_{D,3,1} + \Auxconst_{D,5} + \Auxconst_{D,6})d^{1/2}\nagent^{-1/2}(1+t)^{-\gamma_\step/2} \eqsp.
    \end{align*}
    For $R_3$, we use the bounds \eqref{eq:bound_R_3_1}, \eqref{eq:bound_R_3_2}, \eqref{eq:bound_R_3_3}, \eqref{eq:bound_R_3_5}, \eqref{eq:bound_R_3_6}, and get
    \begin{align*}
        R_3 &\leq \Auxconst_1 \exp{\left(-{a\step \nlupdates \over 1-\gamma} (1+t)^{1-\gamma}\right)}\norm{\theta_0 - \thetalim} + \Auxconst_2 (1+t)^{-(3\gamma - \gamma_\step)/2} \\
        &+ \Auxconst_{3,5}(1+t)^{-3\gamma/2} + (\Auxconst_{3,4} + \Auxconst_{5,5} + \Auxconst_6)\nagent^{-1/2}(1+t)^{-\gamma_\step/2} \eqsp,
    \end{align*}
    where $\gamma = \gamma_\step - \gamma_\nlupdates$.

\subsection{Lower bound for Gaussian approximation}
\label{appendix: lower-bound-for-gaussian-approximation}

To illustrate the inherent limitations of Gaussian approximation, 
we consider a simple $1$-dimensional linear stochastic approximation with a single agent 
and constant local updates ($\nlupdates_t = 1$ for all $t \ge 1$):
\[
\bA \thetalim = \barb, \qquad \bA = 1, \quad \barb = 0.
\]
Comparing with the general setting, here we have $\gamma_\nlupdates = 0$ and $\gamma_\step = \gamma \in (0, 1)$. The unbiased stochastic estimates are given by
\[
\mathbf{A}(Z_j) = 1 + \xi_j, \qquad \mathbf{b}(Z_j) = \barb,
\]
where $\xi_j \sim \mathcal{N}(0,1)$. Then, \textsc{FedLSA} reduces to
\begin{equation}
\theta_t = \theta_{t-1} - \step_t (\mathbf{A}(Z_t) \theta_{t-1} - \mathbf{b}(Z_t)), \quad \theta_0 = 0.
\end{equation}

Unrolling the recursion gives
\[
\theta_t = -\sum_{j=1}^t \step_j \prod_{\ell=j+1}^t (1-\step_\ell) \xi_j,
\]
and its variance under step-size scaling satisfies
\begin{align}
\label{eq: vt-definition}
\eta_t^{-1/2} \theta_t \sim \mathcal{N}(0,v_t), \quad v_t := \Var[\step_t^{-1/2} \theta_t] = \eta_t^{-1} \sum_{j=1}^t \eta_j^2 \prod_{\ell=j+1}^t (1-\eta_\ell)^2.
\end{align}

Under this setting, we can state the following results
\begin{proposition}
\label{proposition: limiting-variance-lower-bound:appendix}
    Assume that $\step_1 = 0$ and $\step_{\ell} = (1+\ell)^{-\gamma}$ with $0 < \gamma < 1$. For $1$-dimensional single agent \textsc{FedLSA} with $\bA \thetalim = \barb, \bA = 1, \barb = 0$ and $\mathbf{A}(Z_j) = 1 + \xi_j, \mathbf{b}(Z_j) = \barb$ with $\xi_j \sim \mathcal{N}(0,1)$ it holds that
    \begin{align*}
        \lim_{t \to \infty} v_t := \lim_{t \to \infty} \Var[\step_t^{-1/2} \theta_t] = \frac{1}{2} \eqsp,
    \end{align*}
    where $v_t$ is defined in \eqref{eq: vt-definition}. Moreover, for any $t \geq 4$, it can be shown that
    \begin{align*}
        \Big| v_t - \frac{1}{2} \Big| \geq \bConst{v} (t^{\gamma - 1} + t^{-\gamma}) \eqsp,
    \end{align*}
    for some positive constant $\bConst{v} > 0$ that depends only on $\gamma$.
\end{proposition}
From \Cref{proposition: limiting-variance-lower-bound:appendix} it follows that the limiting variance of $v_t$ is $\sigma_\infty^2 = \frac{1}{2}$. Moreover, \Cref{proposition: limiting-variance-lower-bound:appendix} gives us the lower bound for $|v_t - \sigma_\infty^2| = |v_t - \frac{1}{2}| \geq \bConst{v} (t^{\gamma - 1} + t^{-\gamma})$. Consequently, by \citet[Theorem 1.1]{Devroye2018}, the convex distance between $\step_t^{-1/2}(\theta_t - \thetalim)$ and limiting distribution $\mathcal{N}(0, \sigma_\infty^2)$ satisfies
\begin{align*}
    \kolmogorov\big(\step_t^{-1/2} (\theta_t - \thetalim), \mathcal{N}(0, \sigma_\infty^2)\big) \gtrsim t^{\gamma - 1} + t^{-\gamma},
\end{align*}
and the statement follows.

\begin{proof}[Proof of \Cref{proposition: limiting-variance-lower-bound:appendix}]
    We begin by noting that
    \begin{align*}
        \step_j^2 = {1 \over 2}\step_j(1-(1-\step_j)^2) + {1\over 2}\step_j^3 \eqsp.
    \end{align*}
    Denoting $P_{j,t} := \prod_{\ell=j+1}^t (1-\step_\ell)^2$, we can write
    \begin{align*}
        \step_t^{-1} v_t &= \frac{1}{2}\sum_{j=1}^t \step_j (P_{j,t} - P_{j-1,t}) + \frac{1}{2} \sum_{j=1}^t \step_j^3 P_{j,t} \\
        &= \frac{1}{2} S_{t,1} + \frac{1}{2} S_{t,2} \eqsp.
    \end{align*}
    We first consider the term $S_{t,1}$. By applying Abel's summation by parts, we obtain
    \begin{align*}
        \sum_{j=1}^t \step_j (P_{j,t} - P_{j-1,t}) 
        &= \step_t P_{t,t} - \step_1 P_{0,t} + \sum_{j=1}^{t-1} (\step_j - \step_{j+1}) P_{j,t} \\
        &= \step_t + \sum_{j=1}^{t-1} (\step_j - \step_{j+1}) P_{j,t} - \step_1 P_{0,t} \eqsp.
    \end{align*}
    Using the inequalities $(j+1)^{-\gamma} - (j+2)^{-\gamma} \geq \frac{\gamma}{2} (j+1)^{-\gamma-1}$ and $\log(1-x) \geq -4x$ for $x \in [0,1/2]$, we have
    \begin{align*}
        \sum_{j=1}^{t-1} (\step_j - \step_{j+1}) P_{j,t} 
        &\geq \frac{\gamma}{2} \sum_{j=1}^{t-1} (j+1)^{-\gamma-1} \exp\Big(-4 \sum_{\ell=j+1}^t \step_\ell\Big) \\
        &\geq \frac{\gamma}{2} \sum_{j=1}^{t-1} (j+1)^{-\gamma-1} \exp\Big( -\frac{4}{1-\gamma} \big( t^{1-\gamma} - (j+1)^{1-\gamma} \big) \Big) \\
        &= \frac{\gamma}{2} \exp\Big( -\frac{4}{1-\gamma} t^{1-\gamma} \Big) \sum_{j=1}^{t-1} (j+1)^{-\gamma-1} \exp\Big( \frac{4}{1-\gamma} (j+1)^{1-\gamma} \Big) \eqsp.
    \end{align*}
    Since the function $f(x) = (x+1)^{-\gamma-1} \exp\big(\frac{4}{1-\gamma}(x+1)^{1-\gamma}\big)$ is strictly increasing for $x \geq 1$, we can bound the sum by the corresponding integral:
    \begin{align*}
        \sum_{j=1}^{t-1} (\step_j - \step_{j+1})P_{j,t} &\geq \frac{\gamma}{2} \exp{\left( -\frac{4}{1 - \gamma} t^{1 - \gamma} \right)} \int_{1}^{t - 1} (x + 1)^{-\gamma - 1} \exp{\left( \frac{4}{1 - \gamma}(x + 1)^{1 - \gamma} \right)} \rmd x \\
        &= \frac{\gamma}{2} t^{-\gamma} \exp{\left( -\frac{4}{1 - \gamma} t^{1 - \gamma} \right)} \int_{2/t}^{1} u^{-(1 + \gamma)} \exp{\left( \frac{4}{1 - \gamma} u^{1 - \gamma} t^{1 - \gamma} \right)} \rmd u\\
        &\geq \frac{\gamma}{2} t^{-\gamma} \exp{\left( -\frac{4}{1 - \gamma} t^{1 - \gamma} \right)} \int_{1/2}^{1} u^{-(1 + \gamma)} \exp{\left( \frac{4}{1 - \gamma} u^{1 - \gamma} t^{1 - \gamma} \right)} \rmd u \eqsp.
    \end{align*}
    For $t \geq 4$ we may estimate asymptotic behaviour of the integral using Laplace approximation introduced in, for instance, \citet[Theorem 1.2]{fedoryuk1977metod} and \cite{olver1997asymptotics}:
    \begin{align*}
        \frac{\gamma}{2} t^{-\gamma} \int_{1/2}^{1} u^{-(1 + \gamma)} \exp{\left( \frac{4}{1 - \gamma} u^{1 - \gamma} t^{1 - \gamma} \right)} \rmd u = \frac{\gamma}{8} t^{-1} \left( 1 + \mathcal{O}(t^{\gamma - 1}) \right) \eqsp.
    \end{align*}
    Therefore, as $\step_1 P_{0, t} \leq \prod_{\ell = 1}^{t} (1 - \step_\ell)^{2}$ is exponentially small, we may write
    \begin{align*}
        \begin{cases}
            S_{t, 1} = \step_{t} + R_{t, 1} \eqsp, \\
            R_{t, 1} \geq \bConst{v, 1} t^{-1}
        \end{cases}
    \end{align*}
    for some positive constant $\bConst{v,1}$ depending only on $\gamma$.

    Next, we consider $S_{t,2}$. Following a similar approach as for $S_{t,1}$ and approximating the sum with an integral, we have
    \begin{align*}
        S_{t, 2} &= \sum \limits_{j = 1}^{t} \step_{j}^3 P_{j, t} \geq \sum \limits_{j = 1}^{t} (1 + j)^{-3 \gamma} \exp{\left( -4 \sum \limits_{\ell = j + 1}^{t} \step_\ell \right)} \\
        &\geq \exp{\left( -\frac{4}{1 - \gamma} t^{1 - \gamma} \right)} \sum \limits_{j = 1}^{t} (j + 1)^{-3 \gamma} \exp{\left( \frac{4}{1 - \gamma} (j + 1)^{1 - \gamma} \right)} \\
        &\geq \exp{\left( -\frac{4}{1 - \gamma} t^{1 - \gamma} \right)} \int_1^t (x + 1)^{-3 \gamma} \exp{\left( \frac{4}{1 - \gamma} (x + 1)^{1 - \gamma} \right)} \\
        &= \exp{\left( -\frac{4}{1 - \gamma} t^{1 - \gamma} \right)} (t + 1)^{1 - 3 \gamma} \int_{2/(t + 1)}^1 u^{-3 \gamma} \exp{\left( \frac{4}{1 - \gamma} u^{1 - \gamma} (t + 1)^{1 - \gamma} \right)} \rmd u  \\
        &\geq \exp{\left( -\frac{4}{1 - \gamma} t^{1 - \gamma} \right)} (t + 1)^{1 - 3 \gamma} \int_{1/2}^1 u^{-3 \gamma} \exp{\left( \frac{4}{1 - \gamma} u^{1 - \gamma} (t + 1)^{1 - \gamma} \right)} \rmd u \eqsp.
    \end{align*}
    Here we also use the fact that function $f(x) = (x + 1)^{-3 \gamma} \exp{\left( \frac{4}{1 - \gamma} \right)}$ is strictly increasing on $[1, \infty)$. Now we determine the asymptotic behaviour of the following integral using Laplace approximation
    \begin{align*}
        \exp{\left( -\frac{4}{1 - \gamma} t^{1 - \gamma} \right)} (t + 1)^{1 - 3 \gamma} \int_{1/2}^1 u^{-3 \gamma} \exp{\left( \frac{4}{1 - \gamma} u^{1 - \gamma} (t + 1)^{1 - \gamma} \right)} \rmd u = \frac{1}{4} t^{-2 \gamma} (1 + \mathcal{O}(t^{\gamma - 1})) \eqsp.
    \end{align*}
    Finally, taking $\bConst{v} := \min(\bConst{v,1}, \bConst{v,2})$ yields the desired lower bound
    \begin{align*}
        \Big| v_t - \frac{1}{2} \Big| \geq \bConst{v} \big(t^{\gamma-1} + t^{-\gamma}\big) \eqsp.
    \end{align*}
\end{proof}

%% file: appendix/fedlsa_boot.tex
\section{Multiplier bootstrap validity}

We start with the expression for $t_0$ in \Cref{assum:sample_size}.
\setcounter{assumprime}{3}
\begin{assumprime}
\label{assum:sample_size_prime}

{
We assume that $t\geq t_0=\max(\nagent,\rme^3,n_0)$, where $n_0\geq1$ is chosen so that, for all $s\geq n_0$,
\begin{equation*}
    \begin{dcases}
        C_{\infty,1}(1+s)^{\gamma-1}
        +C_{\infty,2}(1+s)^{-\gamma}
        \leq \frac{\lambda_{\min}(\Sigma_\infty)}{2} \eqsp,
        \\[6pt]
        \bConst{\boot}\Big[
        \nagent^{-1}\log^6(s)(1+s)^{-\gamma_\step/2}
        +\nagent^{-1/2}\log^4(s)(1+s)^{-\gamma}
        +\log^4(s)(1+s)^{-2\gamma}\Big]
        \leq \frac{\lambda_{\min}(\Sigma_\infty)}{4} \eqsp.
    \end{dcases}
\end{equation*}
Here $C_{\infty,1}$ and $C_{\infty,2}$ are the constants in \Cref{proposition: covariance-matrix-comparison-main}, and $\bConst{\boot}>0$ is a constant independent of $t$ and $\nagent$ that bounds the implicit constant in \eqref{eq:bound_Sigmat_Sigmatboot_diff}.
\par}
\end{assumprime}

\par 
Fix some $p \geq 2$, we set the events under which we will consider the normal approximation in bootstrap world
\begin{align}
    \nonumber \Omega_1 &= \left\{\begin{aligned}
    &\forall c\in[\nagent],\ i\in[t],\ 1\leq m\leq k\leq\nlupdates_i\eqsp,\quad
    \norm{\Gamma_{i,m:k}^c}\leq\rme(1-\step_i a)^{k-m+1}\eqsp,\\
    &\forall s\in[t],\ h\in[\nlupdates_s]\eqsp,\quad
    \norm{\sum_{c=1}^\nagent\Gamma_{s,h+1:\nlupdates_s}^c\funcnoiseth[c]{\thetalim}{Z_{s,h,c}}}
    \leq \bConst{\varepsilon}\nagent^{1/2}\log^{1/2}(t)(\supconsteps+\bConst{A}\zeta_3)
    \end{aligned}\right\}\eqsp,\\
    \nonumber \Omega_2 &= \{1\leq m\leq k\leq t \eqsp, \quad \norm{\prod_{j=m}^k \Gammaavg_j} \leq \rme\prod_{j=m}^k (1- \step_j a)^{\nlupdates_j}\} \eqsp,\\
    \nonumber \Omega_3(p) &= (\PEb[\norm{(\Sigma_t^\boot)^{1/2}D_t^\boot}^p])^{1/p} \leq 32(1+t)^{1/p}B_{t,p}\}\eqsp, \\
    \nonumber \Omega_4 &= \{\norm{\Sigma_t - \Sigma_t^\boot} \leq \bConst{\boot} (\nagent^{-1}\log^6(t)(1+t)^{-\gamma_\step/2} + \nagent^{-1/2}\log^4(t)(1+t)^{-\gamma} + \log^4(t)(1+t)^{-2\gamma})\} \eqsp,\\
    \Omega_0(p) &= \Omega_1 \cap \Omega_2 \cap \Omega_3(p) \cap \Omega_4 \eqsp. \label{eq:Omega_def}
\end{align}

Following \Cref{prop:moment_bound_Db}, we set
\begin{align*}
    B_{t,p}&=\bar{\Auxconst}_1^\boot\nagent^{-1/2} p^6(1+t)^{-\gamma_\step/2}
    +\bar{\Auxconst}_2^\boot p^4(1+t)^{-\gamma} \eqsp,
\end{align*}
where $\bar{\Auxconst}_1^\boot$ and $\bar{\Auxconst}_2^\boot$ are defined in \eqref{eq:constants_Db_pth_moment}. Also, let $q_t$ be as defined in \Cref{lemma: probability-of-theta-0}, and set
\begin{align*}
    \bConst{\varepsilon} = 2\rme^2 \sqrt{4 + {\log(H+1)\over 3}} \eqsp.
\end{align*}

In addition, throughout this section we use the notation
\begin{align*}
    \mathcal{F}_{s,h}^{\boot,+} = \sigma(w_{t,k,c}, Z_{t,k,c} \mid  t > s \text{ or } t=s \text{ and } k\geq h)\eqsp, \quad \mathcal{F}_{s,h}^{\boot,-} = \sigma(w_{t,k,c}, Z_{t,k,c} \mid t < s \text{ or } t = s \text{ and } k \leq h) \eqsp.
\end{align*}
We start with the definition of the bootstrap procedure. We consider a set of random variables $\mathcal{W}^t = \{w_{s,h}, s \in [t], h \in [\nlupdates_s]\}$, where $w_{s,h}$ are i.i.d and independent of $\mathcal{Z}^t = \{Z_{s,h,c}, s \in [t], h\in [\nlupdates_s], c \in [\nagent]\}$, with $\PE[w_{1,1}] = 1$, $\Var[w_{1,1}] = 1$, and for any $p> 2$, we set $\PE^{1/p}[|w_{s,h} - 1|^p] = m_p < \infty$. Then, we denote $\PPb= \PP(\cdot | \mathcal{Z}^t)$ and $\PEb = \PE(\cdot | \mathcal{Z}^t)$. According to multiplier bootstrap procedure, we slightly modify the iterations \eqref{eq:fedlsa_iter_def_main} introducing multiplicative noise, that is,
\begin{align}
    \label{eq:fedlsa_boot_iter_def}
    \theta_{t,h}^{\boot,c} = \theta_{t,h-1}^{\boot,c} - \step_t w_{t,h}(\zfuncA[c]{Z_{t,h,c}}\theta_{t,h-1}^{\boot,c} - \zfuncb[c]{Z_{t,h,c}}) \eqsp.
\end{align}

\subsection{Proof of the Theorem~\ref{thm:bootstrap_validity}}
\label{appendix: multiplier-bootstrap-analysis}

\begin{proof}
The goal of the \Cref{thm:bootstrap_validity} is to estimate the quantity
\begin{align*}
    \sup_{B \in \text{Conv}(\mathbb{R}^d)} \left| \PPb(\step_t^{-1/2}(\theta_t^\boot - \theta_t) \in B) - \PP(\step_t^{-1/2}(\theta_t - \thetalim) \in B) \right| \eqsp.
\end{align*}
From now, we restrict ourselves to the event $\Omega_0{(\log(t))}$. Restricting to this
event, we obtain that, with triangle inequality
\begin{align*}
    \sup_{B \in \text{Conv}(\mathbb{R}^d)} &\left| \PPb(\step_t^{-1/2}(\theta_t^\boot - \theta_t) \in B) - \PP(\step_t^{-1/2}(\theta_t - \thetalim) \in B) \right|\\
    &\leq \kolmogorovboot \big( \step_t^{-1/2} (\theta_t^\boot - \theta_t), Y^\boot \big)
    +
    \kolmogorov \big( \step_t^{-1/2}(\theta_t - \thetalim), Y \big)+
    \sup_{B \in \text{Conv}(\mathbb{R}^d)}\left| \PP(Y \in B) - \PPb(Y^\boot \in B) \right|
    \eqsp.
\end{align*}
Now we control the first term using \Cref{thm:norm_approx_boot_Yb}, second one using \Cref{thm:gauss_approx_last_iter} and the third one using \Cref{lemma:gauss_comparison} together with
\begin{align*}
        \norm{\Sigma_t^{-1/2}\Sigma_t^\boot \Sigma_t^{-1/2} - \Id} \leq \norm{\Sigma_t^{-1}} \norm{\Sigma_t - \Sigma_t^\boot} \lesssim_{\log_t} \bConst{\Sigma}^2 ( t^{-\gamma_\step/2} + \nagent^{1/2}t^{-2\gamma}) \eqsp.
    \end{align*}

We show that $\PP(\Omega_0{(\log(t))}) \geq 1 - \frac{4}{t}$ in \Cref{lemma: probability-of-theta-0}.
\end{proof}

\begin{lemma}
\label{lemma: probability-of-theta-0}
Fix $t\geq 1$ and $p\geq 2$, and let $\Omega_0{(p)}$ be defined in \eqref{eq:Omega_def}. Let
\begin{align*}
    \delta_t = \big[4\nagent(\nlupdates+1)^2(1+t)^{3+2\gamma_\nlupdates}\big]^{-1} \eqsp, \quad\quad q_t = \log(d/\delta_t) \eqsp.
\end{align*}
Assume $p\leq q_t$, independence of the observation sequences of distinct clients, \Cref{ass: linear-decay}($q_t$), \Cref{assum:noise-level-flsa}, \Cref{ass: lr-polynomial-decay}, \Cref{ass: boot_weights}{($q_t$)} and \Cref{assum:sample_size}, and suppose that $0<\step\leq \step_{\infty,q_t}\wedge \frac{a}{q_t\bConst{A}^2}$. Then, we have
\begin{align*}
    \PP(\Omega_0{(p)})\geq 1-\frac{4}{t} \eqsp.
\end{align*}
\end{lemma}

\begin{proof}
Write $r_i=1-a\step_i\in(0,1)$ and $q=q_t$.

Applying McDiarmid's inequality to $\norm{\sum_{c=1}^{\nagent}\Gamma_{s,h+1:\nlupdates_s}^c\funcnoiseth[c]{\thetalim}{Z_{s,h,c}}}$ conditionally on the future observations, with probability at least $1-{1 \over 2t^2 \nlupdates_t}$, it holds
\begin{align*}
    \norm{\sum_{c=1}^\nagent \Gamma_{s,h+1:\nlupdates_s}^c \funcnoiseth[c]{\thetalim}{Z_{s,h,c}}} &\leq \rme \nagent^{1/2} (1 + \sqrt{2\log(2t^2 \nlupdates_t)}) (\supconsteps + \bConst{A}\zeta_3)\\
    &\leq \bConst{\varepsilon}\nagent^{1/2}\log^{1/2}(t) (\supconsteps + \bConst{A}\zeta_3) \eqsp.
\end{align*}
Combining a union bound with \Cref{lemma:matrix-product-moments} and Markov's inequality, we obtain
\begin{align*}
    \PP\big((\Omega_1\cap\Omega_2)^c\big)
    \leq (\nagent t\nlupdates_t^2+t^2)\delta_t+\frac{1}{2t}
    \leq\frac{1}{t} \eqsp.
\end{align*}
Applying Markov's inequality and \Cref{prop:moment_bound_Db}, we get
\begin{align*}
    \PP(\Omega_3{(p)}^c)
    &\leq \frac{\PE[\PEb[\norm{(\Sigma_t^\boot)^{1/2}D_t^\boot}^p]]}{(1+t)B_{t,p}^p}
    \leq \frac{1}{1+t} \eqsp.
\end{align*}

Using the covariance estimate and lower bound from \Cref{lemma:sigma_t_boot_norm_bound}, we have $\PP(\Omega_4^c)\leq 2/t$ and conclude that
\begin{align*}
    \PP(\Omega_0{(p)}^c)
    &\leq \PP\big((\Omega_1\cap\Omega_2)^c\big)
        +\PP(\Omega_4^c)+\PP(\Omega_3{(p)}^c)\\
    &\leq \frac{1}{t}+\frac{2}{t}+\frac{1}{1+t}
    \leq \frac{4}{t} \eqsp.
\end{align*}
\end{proof}

\subsection{Rate of Gaussian approximation in the bootstrap world}

Now, our aim is to quantify the convergence rate of the convex distance in the bootstrap world
\begin{align*}
    \kolmogorovboot(\step_t^{-1/2}(\theta_t^\boot - \theta_t), Y^\boot) = \sup_{B\in\Conv(\rset^d)}\left| \PPb(\step_t^{-1/2}(\theta_t^\boot-\theta_t) \in B) - \PPb(Y^\boot \in B) \right| \eqsp,
\end{align*}
where $Y^\boot \sim \gauss(0, \Sigma_t^\boot)$ and $\Sigma_t^\boot$ is defined in \eqref{eq:Sigma_t_boot_def}. We start with the decomposition
\begin{align*}
    \theta_{t,h}^{\boot, c} - \theta_{t,h}^c &= (\Id - \step_t \zfuncA[c]{Z_{t,h,c}})(\theta_{t,h-1}^{\boot, c} - \theta_{t,h-1}^c) \\
    &+ \step_t(\funcnoiseth[c]{\theta_{t,h-1}^{\boot,c}}{Z_{t,h,c}} + \bA[c]\theta_{t,h-1}^{\boot,c} - \barb[c]) - \step_t w_{t,h}(\funcnoiseth[c]{\theta_{t,h-1}^{\boot,c}}{Z_{t,h,c}} + \bA[c]\theta_{t,h-1}^{\boot,c} - \barb[c]) \\
    &= (\Id - \step_t \zfuncA[c]{Z_{t,h,c}})(\theta_{t,h-1}^{\boot,c} - \theta_{t,h-1}^c) - \step_t (w_{t,h} - 1)(\funcnoise[c]{Z_{t,h,c}} + \zfuncA[c]{Z_{t,h,c}}(\theta_{t,h-1}^{\boot, c} - \thetalim[c])) \\
    &= (\Id - \step_t w_{t,h} \zfuncA[c]{Z_{t,h,c}})(\theta_{t,h-1}^{\boot,c} - \theta_{t,h-1}^c) - \step_t (w_{t,h} - 1)\funcnoiset[c]_{t,h} \eqsp.
\end{align*}
Here, we define
\begin{equation*}
    \funcnoiset[c]_{t,h} = \funcnoise[c]{Z_{t,h,c}} + \zfuncA[c]{Z_{t,h,c}}(\theta_{t,h-1}^c - \thetalim[c]) \eqsp.
\end{equation*}
We also define the bootstrap counterparts of random matrices product
\begin{align*}
    \Gamma_{t,\ell:r}^{\boot,c} = \prod_{h=l}^r (\Id - \step_t w_{t,h}\zfuncA[c]{Z_{t,h,c}})\eqsp, \quad \Gamma_t^{\boot,c} = \Gamma_{t,1:\nlupdates_t}^{\boot,c} \eqsp, \quad \Gammaavgboot_t = \nagent^{-1}\sum_{c=1}^\nagent \Gamma_t^{\boot, c} \eqsp.
\end{align*}
Setting $h = \nlupdates_t$ and unrolling the recursion, we get
\begin{align*}
    \theta_{t,\nlupdates_t}^{\boot, c} - \theta_{t,\nlupdates_t}^c = \Gamma_{t,\nlupdates_t}^{\boot,c} (\theta_{t,0}^{\boot,c} - \theta_{t,0}^c) - \step_t\sum_{h=1}^{\nlupdates_t} (w_{t,h}-1)\Gamma_{t,h+1:\nlupdates_t}^{\boot,c} \funcnoiset[c]_{t,h} \eqsp,
\end{align*}
Averaging over clients, we obtain
\begin{align*}
    \theta_t^\boot - \theta_t &= \nagent^{-1}\sum_{c=1}^\nagent \{\theta_{t,\nlupdates_t}^{\boot,c} - \theta_{t,\nlupdates_t}^c\}\\
    &= \Gammaavgboot_t (\theta_{t-1}^{\boot} - \theta_{t-1}) - \step_t \underbrace{\nagent^{-1} \sum_{h=1}^{\nlupdates_t} (w_{t,h}-1)\sum_{c=1}^\nagent\Gamma_{t,h+1:\nlupdates_t}^{\boot,c}\funcnoiset[c]_{t,h}}_{\Epsavgboot_t}\\
    &= \prod_{i=1}^t \Gammaavgboot_i(\theta_0^\boot - \theta_0) - \sum_{s=1}^t \step_s \prod_{i=s+1}^t \Gammaavgboot_i \Epsavgboot_s \eqsp.
\end{align*}
Again, in this decomposition we isolate the leading term with $\funcnoise[c]{Z_{t,h,c}}$, that is,
\begin{align*}
    \theta_t^\boot - \theta_t &= \underbrace{-\nagent^{-1}\sum_{s=1}^t \step_s \prod_{i=s+1}^t \Gammaavg_i\sum_{h=1}^{\nlupdates_s} (w_{s,h}-1)\sum_{c=1}^\nagent\Gamma_{s,h+1:\nlupdates_s}^c \funcnoise[c]{Z_{s,h,c}}}_{\text{Leading term}} + \prod_{i=1}^t \Gammaavgboot_i (\theta_0^\boot - \theta_0)\\
    &+ \nagent^{-1}\sum_{s=1}^t \step_s (\prod_{i=s+1}^t \Gammaavg_i - \prod_{i=s+1}^t \Gammaavgboot_i)\sum_{h=1}^{\nlupdates_s} (w_{s,h}-1)\sum_{c=1}^\nagent\Gamma_{s,h+1:\nlupdates_s}^c \funcnoise[c]{Z_{s,h,c}}\\
    &+ \nagent^{-1}\sum_{s=1}^t \step_s \prod_{i=s+1}^t \Gammaavgboot_i \sum_{h=1}^{\nlupdates_s} (w_{s,h}-1)\sum_{c=1}^\nagent(\Gamma_{s,h+1:\nlupdates_s}^c - \Gamma_{s,h+1:\nlupdates_s}^{\boot,c}) \funcnoise[c]{Z_{s,h,c}}\\
    &- \nagent^{-1}\sum_{s=1}^t \step_s \prod_{i=s+1}^t\Gammaavgboot_i \sum_{h=1}^{\nlupdates_s} (w_{s,h}-1)\sum_{c=1}^\nagent\Gamma_{s,h+1:\nlupdates_s}^{\boot,c} \zfuncA[c]{Z_{s,h,c}}(\theta_{s,h-1}^c - \thetalim[c])  \eqsp.
\end{align*}
The second term here vanish since $\theta_0^\boot = \theta_0$. Further, we can extract the heterogeneous bias from the last term, and get the completed leading term
\begin{align}
    \nonumber M_t^\boot &= -\nagent^{-1}\step_t^{-1/2}\sum_{s=1}^t \step_s \prod_{i=s+1}^t \Gammaavg_i\sum_{h=1}^{\nlupdates_s} (w_{s,h}-1)\sum_{c=1}^\nagent\Gamma_{s,h+1:\nlupdates_s}^c \funcnoise[c]{Z_{s,h,c}}\\
    \nonumber &- \nagent^{-1}\step_t^{-1/2}\sum_{s=1}^t \step_s \prod_{i=s+1}^t\Gammaavg_i \sum_{h=1}^{\nlupdates_s} (w_{s,h}-1)\sum_{c=1}^\nagent\Gamma_{s,h+1:\nlupdates_s}^{c} \zfuncA[c]{Z_{s,h,c}}(\thetalim - \thetalim[c]) \\
    \nonumber &= - \nagent^{-1}\step_t^{-1/2}\sum_{s=1}^t \step_s \prod_{i=s+1}^t\Gammaavg_i \sum_{h=1}^{\nlupdates_s} (w_{s,h}-1)\sum_{c=1}^\nagent\Gamma_{s,h+1:\nlupdates_s}^{c} \{\funcnoiseth[c]{\thetalim}{Z_{s,h,c}} + \bA[c](\thetalim - \thetalim[c])\} \eqsp,\\
    \Sigma_t^\boot &= \PEb[M_t^\boot (M_t^\boot)^T] \eqsp. \label{eq:Sigma_t_boot_def}
\end{align}
Analogously to the case of $M_t$, we can further specify the leading term of $M_t^\boot$, as,
\begin{align*}
    M_t^\boot &= - \nagent^{-1}\step_t^{-1/2}\sum_{s=1}^t \step_s \prod_{i=s+1}^t\Gavg_i \sum_{h=1}^{\nlupdates_s} (w_{s,h}-1)\sum_{c=1}^\nagent\{\funcnoiseth[c]{\thetalim}{Z_{s,h,c}}\}\\
    &+ \nagent^{-1}\step_t^{-1/2}\sum_{s=1}^t \step_s \left(\prod_{i=s+1}^t\Gavg_i  - \prod_{i=s+1}^t\Gammaavg_i \right) \sum_{h=1}^{\nlupdates_s} (w_{s,h}-1)\sum_{c=1}^\nagent\{\funcnoiseth[c]{\thetalim}{Z_{s,h,c}}\}\\
    &+ \nagent^{-1}\step_t^{-1/2}\sum_{s=1}^t \step_s \prod_{i=s+1}^t\Gammaavg_i \sum_{h=1}^{\nlupdates_s} (w_{s,h}-1)\sum_{c=1}^\nagent(\Id - \Gamma_{s,h+1:\nlupdates_s}^{c})\{\funcnoiseth[c]{\thetalim}{Z_{s,h,c}} + \bA[c](\thetalim - \thetalim[c])\}\\
    &= M_{t,1}^\boot + M_{t,2}^\boot + M_{t,3}^\boot \eqsp.
\end{align*}
Since the weights are shared across clients and
\begin{align*}
    \sum_{c=1}^\nagent \bA[c](\thetalim - \thetalim[c])
    = \sum_{c=1}^\nagent (\bA[c]\thetalim - \barb[c])
    = \nagent(\bA\thetalim - \barb) = 0 \eqsp,
\end{align*}
the contribution of $\bA[c](\thetalim - \thetalim[c])$ cancels in $M_{t,1}^\boot$ and $M_{t,2}^\boot$. In $M_{t,3}^\boot$, a heterogeneity contribution remains through the client-dependent factor $\Id - \Gamma_{s,h+1:\nlupdates_s}^{c}$.
Therefore, we can state our final decomposition of self-normalized statistic in the boostrap world $\step_t^{-1/2}(\Sigma_t^\boot)^{-1/2}(\theta_t^\boot - \theta_t)$ in the following form
\begin{align}
\label{eq:error_decomp_boot}
    \step_t^{-1/2}(\Sigma_t^\boot)^{-1/2}(\theta_t^\boot - \theta_t) = W_t^\boot + D_t^\boot \eqsp,
\end{align}
where we have set 
\begin{align*}
    W_t^\boot &= (\Sigma_t^\boot)^{-1/2}M_t^\boot = \sum_{s=1}^t\sum_{h=1}^{\nlupdates_s}\sum_{c=1}^\nagent \upsilon_{s,h,c}^\boot \eqsp,\\
    \upsilon_{s,h,c}^\boot &= -\nagent^{-1}\step_t^{-1/2}\step_s (\Sigma_t^\boot)^{-1/2}\prod_{i=s+1}^t \Gammaavg_i (w_{s,h} - 1)\Gamma_{s,h+1:\nlupdates_s}^c \{\funcnoiseth[c]{\thetalim}{Z_{s,h,c}} + \bA[c](\thetalim - \thetalim[c])\} \eqsp,
\end{align*}
and $D_t^\boot = D_{t,1}^\boot + D_{t,2}^\boot + D_{t,3}^\boot +D_{t,4}^\boot + D_{t,5}^\boot$, with
\begin{align*}
    D_{t,1}^\boot &= \step_t^{-1/2}\nagent^{-1} (\Sigma_t^\boot)^{-1/2}\sum_{s=1}^t \step_s (\prod_{i=s+1}^t \Gammaavg_i - \prod_{i=s+1}^t \Gammaavgboot_i)\sum_{h=1}^{\nlupdates_s} (w_{s,h}-1)\sum_{c=1}^\nagent\Gamma_{s,h+1:\nlupdates_s}^c \funcnoise[c]{Z_{s,h,c}} \eqsp,\\
    D_{t,2}^\boot &= \step_t^{-1/2}\nagent^{-1}(\Sigma_t^\boot)^{-1/2}\sum_{s=1}^t \step_s \prod_{i=s+1}^t \Gammaavgboot_i \sum_{h=1}^{\nlupdates_s} (w_{s,h}-1)\sum_{c=1}^\nagent(\Gamma_{s,h+1:\nlupdates_s}^c - \Gamma_{s,h+1:\nlupdates_s}^{\boot,c}) \funcnoise[c]{Z_{s,h,c}} \eqsp, \\
    D_{t,3}^\boot &= -\step_t^{-1/2}\nagent^{-1}(\Sigma_t^\boot)^{-1/2}\sum_{s=1}^t \step_s \prod_{i=s+1}^t\Gammaavgboot_i \sum_{h=1}^{\nlupdates_s} (w_{s,h}-1)\sum_{c=1}^\nagent\Gamma_{s,h+1:\nlupdates_s}^{\boot,c} \zfuncA[c]{Z_{s,h,c}}(\theta_{s,h-1}^c - \thetalim) \eqsp,\\
    D_{t,4}^\boot &= \step_t^{-1/2}\nagent^{-1}(\Sigma_t^\boot)^{-1/2}\sum_{s=1}^t \step_s (\prod_{i=s+1}^t\Gammaavg_i - \prod_{i=s+1}^t\Gammaavgboot_i) \sum_{h=1}^{\nlupdates_s} (w_{s,h}-1)\sum_{c=1}^\nagent\Gamma_{s,h+1:\nlupdates_s}^{\boot,c} \zfuncA[c]{Z_{s,h,c}}(\thetalim - \thetalim[c]) \eqsp,\\
    D_{t,5}^\boot &= \step_t^{-1/2}\nagent^{-1}(\Sigma_t^\boot)^{-1/2}\sum_{s=1}^t \step_s \prod_{i=s+1}^t\Gammaavg_i \sum_{h=1}^{\nlupdates_s} (w_{s,h}-1)\sum_{c=1}^\nagent(\Gamma_{s,h+1:\nlupdates_s}^{c} - \Gamma_{s,h+1:\nlupdates_s}^{\boot,c}) \zfuncA[c]{Z_{s,h,c}}(\thetalim - \thetalim[c]) \eqsp.
\end{align*}
Using the decomposition \eqref{eq:error_decomp_boot} together with 
\citet[Theorem 2.1]{shao2022berry}, we are now ready to state the main result of this section.

\begin{theorem} (Gaussian approximation in the "boostrap world")
\label{thm:norm_approx_boot_Yb}
Let $Y^\boot \big| \mathcal{Z}^t \sim \gauss(0, \Sigma_t^\boot)$ and $\step \nlupdates < \beta_\infty$. Let $q_t$ be as in \Cref{lemma: probability-of-theta-0}. Suppose that $0<\step\leq \step_{\infty,q_t}\wedge \frac{a}{q_t\bConst{A}^2}$. Under the assumptions \Cref{ass: linear-decay}($q_t$), \Cref{assum:noise-level-flsa}, \Cref{ass: lr-polynomial-decay}, \Cref{ass: boot_weights} and \Cref{assum:sample_size}, on the event $\Omega_0(\log(t))$, we have
    \begin{align}
    \label{eq:norm_approx_boot_Yb}
        \kolmogorovboot(\step_t^{-1/2}(\theta_t^\boot - \theta_t), Y^\boot) \lesssim \log^6(t) (1+t)^{-\gamma_\step/2} + \nagent^{1/2} \log^4(t)(1+t)^{-\gamma} \eqsp,
    \end{align}
    where the event $\Omega_0{(\log(t))}$ is defined in \eqref{eq:Omega_def} and $\PP(\Omega_0{(\log(t))}) \geq 1 - {4\over t}$.
\end{theorem}
\begin{proof}
{Set $p=\log(t)\geq2$, since $t\geq t_0\geq\rme^3$.}
We start by applying \Cref{prop:nonlinearapprox} to the statistics $X = W_t^\boot$ and $Y = D_t^\boot$, that is,
\begin{align}
\label{eq:decomp_kolmog_linear_pth_moment}
    \kolmogorovboot(\step_t^{-1/2}(\Sigma_t^\boot)^{-1/2}(\theta_t^\boot - \theta_t), Z^\boot) \leq \kolmogorovboot(W_t^\boot, Z^\boot) + 2 c_d^{p/(p+1)}(\PEb[\norm{D_t^\boot}^p])^{1/(1+p)} \eqsp.
\end{align}
% Using the decomposition \eqref{eq:error_decomp_boot} and applying \citet[Theorem 2.1]{shao2022berry}, we get
% \begin{align*}
%     \kolmogorovboot(\step_t^{-1/2}(\Sigma_t^\boot)^{-1/2}(\theta_t^\boot - \theta_t), Y^\boot) \leq \underbrace{259 d^{1/2} \Upsilon_t^\boot}_\text{$R_1^\boot$} + \underbrace{2 \PEb[\norm{M_t^\boot} \norm{D_t^\boot}]}_\text{$R_2^\boot$} + \underbrace{2 \sum \limits_{s = 1}^t \sum \limits_{h = 1}^{H_s} \sum \limits_{c = 1}^N \PEb[\norm{\upsilon_{s,h,c}^\boot} \norm{D_t^\boot - D_t^{(s, h, c)}}]}_\text{$R_3^\boot$} \eqsp,
% \end{align*}

Conditionally on $\mathcal{Z}^t$, the sums $\sum_{c=1}^\nagent\upsilon_{s,h,c}^\boot$ are independent and centered across $(s,h)$. Thus, applying \citet[Theorem 2.1]{shao2022berry}, with $Z^\boot\sim\gauss(0,\Id)$, we obtain
\begin{align*}
    \kolmogorovboot(W_t^\boot,Z^\boot) \leq 259d^{1/2}\Upsilon_t^\boot \eqsp, \quad\quad \Upsilon_t^\boot = \sum_{s=1}^t\sum_{h=1}^{\nlupdates_s}
    \PEb\left[\norm{\sum_{c=1}^\nagent\upsilon_{s,h,c}^\boot}^3\right] \eqsp.
\end{align*}
Since $\funcnoiseth[c]{\thetalim}{Z}=\funcnoise[c]{Z}+\zmfuncA[c]{Z}(\thetalim-\thetalim[c])$ and $\sum_{c=1}^\nagent\bA[c](\thetalim-\thetalim[c])=0$, we write
\begin{align*}
    &\sum_{c=1}^\nagent\Gamma_{s,h+1:\nlupdates_s}^c
    \{\funcnoiseth[c]{\thetalim}{Z_{s,h,c}}+\bA[c](\thetalim-\thetalim[c])\} = \sum_{c=1}^\nagent\Gamma_{s,h+1:\nlupdates_s}^c\funcnoiseth[c]{\thetalim}{Z_{s,h,c}}
    +\sum_{c=1}^\nagent(\Gamma_{s,h+1:\nlupdates_s}^c-\Id)\bA[c](\thetalim-\thetalim[c]) \eqsp.
\end{align*}
On $\Omega_1$, we have
\begin{align*}
    \norm{\sum_{c=1}^\nagent(\Gamma_{s,h+1:\nlupdates_s}^c-\Id)\bA[c](\thetalim-\thetalim[c])}
    \lesssim \nagent\step_s\nlupdates_s \eqsp.
\end{align*}
Thus, on $\Omega_0{(\log(t))}$, using the centered-sum bound in $\Omega_1$, we obtain
\begin{align*}
    \norm{\sum_{c=1}^\nagent\upsilon_{s,h,c}^\boot}
    &\lesssim \step_t^{-1/2}\step_s
    \exp{\left(-a\sum_{i=s+1}^t\step_i\nlupdates_i\right)} \left[\log^{1/2}(t)
    +\nagent^{1/2}\step_s\nlupdates_s\right] \eqsp.
\end{align*}
By conditional independence and $\PEb[W_t^\boot(W_t^\boot)^T]=\Id$, we have
\begin{align*}
    \sum_{s=1}^t\sum_{h=1}^{\nlupdates_s}
    \PEb\left[\norm{\sum_{c=1}^\nagent\upsilon_{s,h,c}^\boot}^2\right]=d \eqsp.
\end{align*}
Taking the maximum over $(s,h)$ in the preceding bound and using this identity yields
\begin{align*}
    \Upsilon_t^\boot
    &\lesssim d \left[\log^{1/2}(t)(1+t)^{-\gamma_\step/2}+ \nagent^{1/2}(1+t)^{-\gamma-\gamma_\step/2}\right] \eqsp,
\end{align*}
Indeed, under \Cref{ass: lr-polynomial-decay} and $t\geq\max(\nagent,\rme^3)$, we have and
\begin{align*}
    \max_{1\leq s\leq t}\left\{\step_s\exp{\left(-a\sum_{i=s+1}^t\step_i\nlupdates_i\right)}\right\}
    &\lesssim\step_t \eqsp,\\
    \max_{1\leq s\leq t}\left\{\step_s^2\nlupdates_s\exp{\left(-a\sum_{i=s+1}^t\step_i\nlupdates_i\right)}\right\}
    &\lesssim (1+t)^{-\gamma_\step-\gamma} \eqsp.
\end{align*}

{Then, on $\Omega_3(p)\cap\Omega_4$, we have}
\begin{align*}
    (\PEb[\norm{D_t^\boot}^p])^{1/(p+1)}
    &\leq \left[\norm{(\Sigma_t^\boot)^{-1/2}}
    (\PEb[\norm{(\Sigma_t^\boot)^{1/2}D_t^\boot}^p])^{1/p} \right]^{p/(p+1)} \\
    &\lesssim \left[ \nagent^{1/2}  \left(\nagent^{-1/2} p^6  (1+t)^{1/p-\gamma_\step/2} +  p^4 (1+t)^{1/p - \gamma} \right)\right]^{p/(p+1)} \\
    &\lesssim \log^6(t) (1+t)^{-\gamma_\step/2 + (1+\gamma_\step/2)/(p+1)} + \nagent^{1/2}\log^4(t) (1+t)^{-\gamma + (1+\gamma)/(p+1)}\\
    &\lesssim \log^6(t) (1+t)^{-\gamma_\step/2} + \nagent^{1/2}\log^4(t) (1+t)^{-\gamma} \eqsp.
\end{align*}
\end{proof}

% \begin{corollary}
%     Under the assumptions of \Cref{thm:gauss_approx_last_iter}, on the event $\Omega_0$, we have
%     \begin{align*}
%         \kolmogorovboot(\step_t^{-1/2}(\theta_t^\boot - \theta_t), \step_t^{-1/2}(\theta_t - \thetalim)) \lesssim_{\log_t} \nagent^{-1/2}t^{-\gamma_\step/2} + \nagent^{1/2}t^{-\gamma} + t^{-(3\gamma-\gamma_\step)/2} \eqsp.
%     \end{align*}
% \end{corollary}
% \begin{proof}
%     Using the decomposition and combining together \Cref{thm:norm_approx_boot_Yb}, \Cref{thm:norm_approx_last_iterate_appendix} and \Cref{lemma:gauss_comparison} with
    
%     we complete the proof.
% \end{proof}

In the following proposition, we assume that the expectation is taken over the probability space generated by $\mathcal{W}^t$ and $\mathcal{Z}^t$.

\begin{proposition}
\label{prop:moment_bound_Db}

Assume that the observation sequences of distinct clients are independent. Let $p\geq2$ and $\step\leq[4a(p-1)]^{-1}$. Under Assumptions \Cref{ass: linear-decay}($p$), ~\Cref{assum:noise-level-flsa}, ~\Cref{ass: lr-polynomial-decay}, ~\Cref{ass: boot_weights}{($p$)}, and ~\Cref{assum:sample_size}, we get

\begin{align*}
        \PE^{1/p}[\norm{(\Sigma_t^\boot)^{1/2} D_t^\boot}^p] \leq \bar{\Auxconst}_1^\boot\nagent^{-1/2} p^6 (1+t)^{-\gamma_\step/2}  + \bar{\Auxconst}_2^\boot p^4  (1+t)^{-\gamma} \eqsp,
    \end{align*}

where the constants $\bar{\Auxconst}_1$, $\bar{\Auxconst}_2$ are defined in \eqref{eq:constants_Db_pth_moment}.
\end{proposition}
\begin{proof}

We interpret $(\Sigma_t^\boot)^{1/2}D_t^\boot$ as the unnormalized remainder
\begin{align*}
    (\Sigma_t^\boot)^{1/2}D_t^\boot
    =\step_t^{-1/2}(\theta_t^\boot-\theta_t)-M_t^\boot \eqsp.
\end{align*}
Likewise, $(\Sigma_t^\boot)^{1/2}D_{t,i}^\boot$ denotes the expression for $D_{t,i}^\boot$ after cancellation of $(\Sigma_t^\boot)^{-1/2}$.
{In this proof, $L$ is the maximum of the constants in \eqref{eq:L_def_legendary_sum} with $\beta=\gamma$ and $\alpha\in\{2\gamma_\step+\gamma,\,2\gamma_\step+2\gamma,\,\gamma_\step+3\gamma,\,2\gamma_\step+3\gamma\}$.}

Firstly, we apply the Minkowski's inequality, that is,
\begin{align*}
        \PE^{1/p}[\norm{(\Sigma_t^\boot)^{1/2} D_t^\boot}^p]
        &\leq \PE^{1/p}[\norm{(\Sigma_t^\boot)^{1/2} D_{t,1}^\boot}^p] + \PE^{1/p}[\norm{(\Sigma_t^\boot)^{1/2} D_{t,2}^\boot}^p] \\
        &+ \PE^{1/p}[\norm{(\Sigma_t^\boot)^{1/2} D_{t,3}^\boot}^p] + \PE^{1/p}[\norm{(\Sigma_t^\boot)^{1/2} D_{t,4}^\boot}^p] \\
        &+ \PE^{1/p}[\norm{(\Sigma_t^\boot)^{1/2} D_{t,5}^\boot}^p] \eqsp.
    \end{align*}

To bound $(\Sigma_t^\boot)^{1/2} D_{t,1}^\boot$, using that the sequence $\{\left( \prod_{i=s+1}^t \Gammaavg_i - \prod_{i=s+1}^t \Gammaavgboot_i \right)(w_{s,h} - 1)\Gamma_{s,h+1:\nlupdates_s}^c \funcnoise[c]{Z_{s,h,c}}\}$ is martingale-difference w.r.t filtration $\mathcal{F}_{s,h}^{\boot,+}$ and applying Burkholder's inequality \cite[Theorem 8.6]{osekowski:2012}, together with \Cref{lemma: legendary-sum} and \Cref{lemma:prod_gammaavg_gammaavgboot_diff}, we get

\begin{align*}
        \PE^{1/p}&[\norm{(\Sigma_t^\boot)^{1/2} D_{t,1}^\boot}^p] \\
        &\leq p \nagent^{-1}\step_t^{-1/2} \left( \sum_{s=1}^t \step_s^2 \PE^{2/p}[\norm{\prod_{i=s+1}^t\Gammaavg_i - \prod_{i=s+1}^t\Gammaavgboot_i}^p]\PE^{2/p}[\norm{\sum_{c=1}^\nagent \sum_{h=1}^{\nlupdates_s}(w_{s,h} - 1)\Gamma_{s,h+1:\nlupdates}^c \funcnoise[c]{Z_{s,h,c}}}^p] \right)^{1/2}\\
        &\leq W_{\max} p^3 \nagent^{-1/2} \step_t^{-1/2} \supconsteps\left(\sum_{s=1}^t \step_s^2 \nlupdates_s \PE^{2/p}[\norm{\prod_{i=s+1}^t\Gammaavg_i - \prod_{i=s+1}^t\Gammaavgboot_i}^p]\right)^{1/2} \\
        &\leq \rme d^{1/2}a^{-1/2} \bConst{A} W_{\max}^2 p^5 \nagent^{-1/2}\step_t^{-1/2} \supconsteps \left(\sum_{s=1}^t \step_s^3 \nlupdates_s \exp{\left(-a\sum_{\ell=s+1}^t \step_\ell \nlupdates_\ell\right)}\right)^{1/2} \\
        &\leq \Auxconst_{1}^\boot \nagent^{-1/2} p^6 (1+t)^{-\gamma_\step/2} \eqsp,
    \end{align*}

where, again applying Burkholder's inqequality, we used the bound

\begin{align*}
        \PE^{1/p}[\norm{\sum_{c=1}^\nagent \sum_{h=1}^{\nlupdates_s} (w_{s,h} - 1)\Gamma_{s,h+1:\nlupdates}^c \funcnoise[c]{Z_{s,h,c}}}^p] &\leq p\left(\sum_{c=1}^\nagent \PE^{2/p}[\norm{\sum_{h=1}^{\nlupdates_s}(w_{s,h} - 1)\Gamma_{s,h+1:\nlupdates_s}^c \funcnoise[c]{Z_{s,h,c}}}^p]\right)^{1/2}\\
        &\leq W_{\max} p^2 \nagent^{1/2} \nlupdates_s^{1/2} \supconsteps \eqsp,
    \end{align*}

    and have set

\begin{align*}
        \Auxconst_{1}^\boot = 2\rme d^{1/2}a^{-1/2} \bConst{A}W_{\max}^2  (\step(\nlupdates+1)^{1/2} + a^{-1/2}L^{1/2}(\step(1+\nlupdates^{-1}))^{1/2}(1+(a\step\nlupdates)^{-{\gamma + 2\gamma_\step \over 2(1-\gamma)}}))\supconsteps\, \eqsp.
    \end{align*}

Further, using the martingale-difference structure of the sequence $\{\prod_{i=s+1}^t \Gammaavgboot_i (w_{s,h}-1)(\Gamma_{s,h+1:\nlupdates_s}^c - \Gamma_{s,h+1:\nlupdates_s}^{\boot, c})\funcnoise[c]{Z_{s,h,c}}, 1 \leq s \leq t\}$ w.r.t filtration $\mathcal{F}_{s,h}^{\boot,+}$, applying Burkholder's inequality and \Cref{lemma: legendary-sum}, we bound $(\Sigma_t^\boot)^{1/2} D_{t,2}^\boot$, as

\begin{align*}
        \PE^{1/p}&[\norm{(\Sigma_t^\boot)^{1/2} D_{t,2}^\boot}^p] \\
        &\leq p\nagent^{-1} \step_t^{-1/2}\left( \sum_{s=1}^t \step_s^2 \PE^{2/p}[\norm{\prod_{i=s+1}^t \Gammaavgboot_i}^p]\PE^{2/p}[\norm{\sum_{c=1}^\nagent\sum_{h=1}^{\nlupdates_s} (w_{s,h}-1)(\Gamma_{s,h+1:\nlupdates_s}^c - \Gamma_{s,h+1:\nlupdates_s}^{\boot,c})\funcnoise[c]{Z_{s,h,c}}}^p] \right)^{1/2} \\
        &\leq d^{1/2}\bConst{A} W_{\max}^2 p^4 \nagent^{-1/2} \step_t^{-1/2}\left( \sum_{s=1}^t \step_s^4 \nlupdates_s^2 \exp{\left(-2a \sum_{\ell=s+1}^t \step_\ell \nlupdates_\ell \right)} \right)^{1/2} \supconsteps\\
        &\leq \Auxconst_{2}^\boot \nagent^{-1/2} p^4 (1+t)^{-(\gamma_\step+\gamma)/2} \eqsp,
    \end{align*}

where, applying Burkholder's inequality and the fact that the sequence $\{(w_{s,h} - 1)(\Gamma_{s,h+1:\nlupdates_s}^c - \Gamma_{s,h+1:\nlupdates_s}^{\boot,c})\funcnoise[c]{Z_{s,h,c}}, 1 \leq h \leq \nlupdates_s\}$is martingale-difference w.r.t filtration $\mathcal{F}_{s,h}^{\boot,+}$, we used the bound

\begin{align*}
        \PE^{1/p}[\norm{\sum_{c=1}^\nagent\sum_{h=1}^{\nlupdates_s} &(w_{s,h}-1)(\Gamma_{s,h+1:\nlupdates_s}^c - \Gamma_{s,h+1:\nlupdates_s}^{\boot,c})\funcnoise[c]{Z_{s,h,c}}}^p] \\
        &\leq p \left(\sum_{c=1}^\nagent \PE^{2/p}[\norm{\sum_{h=1}^{\nlupdates_s} (w_{s,h} - 1)(\Gamma_{s,h+1:\nlupdates_s}^c - \Gamma_{s,h+1:\nlupdates_s}^{\boot,c})\funcnoise[c]{Z_{s,h,c}}}^p]\right)^{1/2}\\
        &\leq d^{1/2}\bConst{A} W_{\max}^2 p^3 \nagent^{1/2} \step_s \nlupdates_s \supconsteps \eqsp,
    \end{align*}

    and

\begin{align*}
        \PE^{1/p}[\norm{\sum_{h=1}^{\nlupdates_s} &(w_{s,h} - 1)(\Gamma_{s,h+1:\nlupdates_s}^c - \Gamma_{s,h+1:\nlupdates_s}^{\boot,c})\funcnoise[c]{Z_{s,h,c}}}^p]\\
        &\leq p \supconsteps\left( \sum_{h=1}^{\nlupdates_s} \PE^{2/p}[|w_{s,h}-1|^p]\PE^{2/p}[\norm{\Gamma_{s,h+1:\nlupdates_s}^c - \Gamma_{s,h+1:\nlupdates_s}^{\boot,c}}^p]\right)^{1/2}\\
        &\leq d^{1/2}\bConst{A}W_{\max}^2 p^2 \step_s \nlupdates_s \supconsteps\eqsp.
    \end{align*}

Here we used \Cref{lemma:bound_gamma_gammaboot_diff} for the local product differences and have set

{
\begin{align*}
        \Auxconst_{2}^\boot = 4d^{1/2}\bConst{A} W_{\max}^2(\step \nlupdates + a^{-1/2}L^{1/2}(\step\nlupdates)^{1/2}(1+(2a\step\nlupdates)^{-{\gamma+\gamma_\step \over 1-\gamma}}))\supconsteps\, \eqsp.
    \end{align*}
}

Now, to bound $(\Sigma_t^\boot)^{1/2}D_{t,4}^\boot$, we use $\sum_{c=1}^\nagent\bA[c](\thetalim-\thetalim[c])=0$ and write
\begin{align*}
    &\sum_{c=1}^\nagent\Gamma_{s,h+1:\nlupdates_s}^{\boot,c}\zfuncA[c]{Z_{s,h,c}}(\thetalim-\thetalim[c]) = \sum_{c=1}^\nagent\Gamma_{s,h+1:\nlupdates_s}^{\boot,c}\zmfuncA[c]{Z_{s,h,c}}(\thetalim-\thetalim[c])
    +\sum_{c=1}^\nagent(\Gamma_{s,h+1:\nlupdates_s}^{\boot,c}-\Id)\bA[c](\thetalim-\thetalim[c]) \eqsp.
\end{align*}
Define
\begin{align*}
    S_{s,4}^\boot=\PE^{2/p}\left[\norm{\sum_{c=1}^\nagent\sum_{h=1}^{\nlupdates_s}(w_{s,h}-1)
    \Gamma_{s,h+1:\nlupdates_s}^{\boot,c}\zfuncA[c]{Z_{s,h,c}}(\thetalim-\thetalim[c])}^p\right] \eqsp.
\end{align*}
For the first term, conditionally on the weights, the client sums are independent and centered. Applying Burkholder's inequality, we have
\begin{align*}
    \PE^{1/p}\left[\norm{(\Gamma_{s,h+1:\nlupdates_s}^{\boot,c}-\Id)\bA[c](\thetalim-\thetalim[c])}^p\right]
    \leq\step_s W_{\max}\bConst{A}^2(\nlupdates_s-h)\norm{\thetalim-\thetalim[c]} \eqsp.
\end{align*}
Minkowski's inequality over clients and Burkholder's inequality in reverse weight order therefore yield
\begin{align*}
    (S_{s,4}^\boot)^{1/2}
    &\leq\bConst{A}\zeta_3 W_{\max}p^2\nagent^{1/2}\nlupdates_s^{1/2}
    +\bConst{A}^2\zeta_3 W_{\max}^2p\nagent\step_s\nlupdates_s^{3/2} \eqsp.
\end{align*}
Applying Burkholder's inequality over rounds, followed by \Cref{lemma:prod_gammaavg_gammaavgboot_diff} and \Cref{lemma: legendary-sum}, we obtain
\begin{align*}
    \PE^{1/p}&[\norm{(\Sigma_t^\boot)^{1/2}D_{t,4}^\boot}^p]\\
    &\leq p\nagent^{-1}\step_t^{-1/2}
    \left(\sum_{s=1}^t\step_s^2\PE^{2/p}\left[\norm{\prod_{i=s+1}^t\Gammaavg_i-\prod_{i=s+1}^t\Gammaavgboot_i}^p\right]S_{s,4}^\boot\right)^{1/2}\\
    &\leq\rme d^{1/2}a^{-1/2}\bConst{A}^2\zeta_3 W_{\max}^2p^5\nagent^{-1/2}\step_t^{-1/2}
    \left(\sum_{s=1}^t\step_s^3\nlupdates_s\exp{\left(-a\sum_{\ell=s+1}^t\step_\ell\nlupdates_\ell\right)}\right)^{1/2}\\
    &\quad+\rme d^{1/2}a^{-1/2}\bConst{A}^3\zeta_3 W_{\max}^3p^4\step_t^{-1/2}
    \left(\sum_{s=1}^t\step_s^5\nlupdates_s^3\exp{\left(-a\sum_{\ell=s+1}^t\step_\ell\nlupdates_\ell\right)}\right)^{1/2}\\
    &\leq\Auxconst_4^\boot\nagent^{-1/2}p^6(1+t)^{-\gamma_\step/2}
    +\Auxconst_{4,2}^\boot p^4(1+t)^{-\gamma-\gamma_\step/2} \eqsp,
\end{align*}
where
\begin{align*}
    \Auxconst_4^\boot
    &=2\rme d^{1/2}a^{-1/2}\zeta_3\bConst{A}^2W_{\max}^2
    \left(\step(\nlupdates+1)^{1/2}+a^{-1/2}L^{1/2}(\step(1+\nlupdates^{-1}))^{1/2}
    (1+(a\step\nlupdates)^{-{\gamma+2\gamma_\step\over2(1-\gamma)}})\right) \eqsp,\\
    \Auxconst_{4,2}^\boot
    &=\rme d^{1/2}a^{-1/2}\zeta_3\bConst{A}^3W_{\max}^3\step^2(\nlupdates+1)^{3/2} \left(1+2^{\gamma_\step+\gamma}L^{1/2}(a\step\nlupdates)^{-1/2}
    (1+(a\step\nlupdates)^{-{2\gamma_\step+3\gamma\over2(1-\gamma)}})\right) \eqsp.
\end{align*}
Finally, to bound $(\Sigma_t^\boot)^{1/2} D_{t,5}^\boot$, we again use the martingale-difference structure of the underlying sequence $\{\prod_{i=s+1}^t \Gammaavg_i (w_{s,h}-1)(\Gamma_{s,h+1:\nlupdates_s}^c - \Gamma_{s,h+1:\nlupdates_s}^{\boot, c})\zfuncA[c]{Z_{s,h,c}}(\thetalim-\thetalim[c]), 1 \leq s \leq t\}$ w.r.t the filtration $\mathcal{F}_{s,h}^{\boot, +}$ and apply Burkholder's inequality together with \Cref{lemma: legendary-sum}, thus, denoting

\begin{align*}
        S_{s,5}^\boot = \PE^{2/p}[\norm{\sum_{c=1}^\nagent\sum_{h=1}^{\nlupdates_s} (w_{s,h}-1)(\Gamma_{s,h+1:\nlupdates_s}^c - \Gamma_{s,h+1:\nlupdates_s}^{\boot,c})\zfuncA[c]{Z_{s,h,c}}(\thetalim-\thetalim[c])}^p] \eqsp,
    \end{align*}

    we obtain

\begin{align*}
        \PE^{1/p}&[\norm{(\Sigma_t^\boot)^{1/2} D_{t,5}^\boot}^p] \leq p \nagent^{-1} \step_t^{-1/2}\left( \sum_{s=1}^t \step_s^2 \PE^{2/p}[\norm{\prod_{i=s+1}^t \Gammaavg_i}^p] S_{s,5}^\boot \right)^{1/2}\\
        &\leq d^{1/2}\bConst{A}^2 \zeta_3 W_{\max}^2 p^4 \step_t^{-1/2} \left( \sum_{s=1}^t \step_s^4 \nlupdates_s^2 \exp{\left(-2a \sum_{\ell=s+1}^t \step_\ell \nlupdates_\ell \right)} \right)^{1/2} \\
        &\leq \Auxconst_{5}^\boot  p^4(1+t)^{-(\gamma_\step + \gamma)/2} \eqsp,
    \end{align*}

where, using \Cref{lemma:bound_gamma_gammaboot_diff}, we have used the bound

\begin{align*}
        \PE^{1/p}[\norm{\sum_{c=1}^\nagent\sum_{h=1}^{\nlupdates_s} &(w_{s,h}-1)(\Gamma_{s,h+1:\nlupdates_s}^c - \Gamma_{s,h+1:\nlupdates_s}^{\boot,c})\zfuncA[c]{Z_{s,h,c}}(\thetalim-\thetalim[c])}^p]\\
        &\leq \sum_{c=1}^\nagent \PE^{1/p}[\norm{\sum_{h=1}^{\nlupdates_s} (w_{s,h} - 1)(\Gamma_{s,h+1:\nlupdates_s}^c - \Gamma_{s,h+1:\nlupdates_s}^{\boot,c})\zfuncA[c]{Z_{s,h,c}}(\thetalim-\thetalim[c])}^p]\\
        &\leq d^{1/2}W_{\max}^2\zeta_3\bConst{A}^2  p^3 \nagent\step_s \nlupdates_s \eqsp,
    \end{align*}

and have set

{
\begin{align*}
        \Auxconst_{5}^\boot = 4d^{1/2}W_{\max}^2 \zeta_3\bConst{A}^2 (\step \nlupdates + a^{-1/2}L^{1/2}(\step\nlupdates)^{1/2}(1+(2a\step\nlupdates)^{-{\gamma+\gamma_\step \over 1-\gamma}}))\,\eqsp.
    \end{align*}
}

For the term $(\Sigma_t^\boot)^{1/2} D_{t,3}^\boot$, we first note that

\begin{align*}
        \theta_{s,h-1}^c - \thetalim = \Gamma_{s,h-1}^c (\theta_{s-1} - \thetalim) - \step_s \sum_{j=1}^{h-1} \Gamma_{s,j+1:h-1}^c \funcnoise[c]{Z_{s,j,c}} + (\Id - \Gamma_{s,h-1}^c)(\thetalim[c]-\thetalim) \eqsp.
    \end{align*}

This decomposition separates $D_{t,3}^\boot$ into three terms, that is, $D_{t,3}^\boot = D_{t,3,1}^\boot + D_{t,3,2}^\boot + D_{t,3,3}^\boot$. For the term $(\Sigma_t^\boot)^{1/2}D_{t,3,1}^\boot$, we apply Minkowski's inequality over clients and, conditionally on all observations, Burkholder's inequality in reverse weight order, together with \Cref{corr:last_iter_pth_moment}, setting

\begin{align*}
        S_{s,3,1}^\boot = \sum_{h=1}^{\nlupdates_s}\left(\sum_{c=1}^\nagent \PE^{1/p}[\norm{ (w_{s,h}-1)\Gamma_{s,h+1:\nlupdates_s}^{\boot,c} \zfuncA[c]{Z_{s,h,c}}\Gamma_{s,h-1}^c (\theta_{s-1} - \thetalim) }^p]\right)^2\eqsp,
    \end{align*}

we get

\begin{align*}
        \PE^{1/p}&[\norm{(\Sigma_t^\boot)^{1/2} D_{t,3,1}^\boot}^p]\\
        &\leq p^3\nagent^{-1}\step_t^{-1/2} \left( \sum_{s=1}^t \step_s^2 S_{s,3,1}^\boot \exp{\left(-2a\sum_{\ell=s+1}^t \step_\ell \nlupdates_\ell\right)} \right)^{1/2}\\
        &\leq \bConst{A} W_{\max} p^3\step_t^{-1/2} \left( \sum_{s=1}^t \step_s^2 \nlupdates_s \PE^{2/p}[\norm{\theta_{s-1} - \thetalim }^p]\exp{\left(-2a\sum_{\ell=s+1}^t \step_\ell \nlupdates_\ell\right)} \right)^{1/2}\\
        &\leq 4\rme\bConst{A} W_{\max} p^3 \norm{\theta_0-\thetalim}\step_t^{-1/2} \left( \exp{\left(-2a\sum_{\ell=1}^t \step_\ell \nlupdates_\ell\right)} \sum_{s=1}^t \step_s^2 \nlupdates_s \right)^{1/2} \\
        &+ 4\Auxconst_{\text{last}, 1}\bConst{A} W_{\max} p^6\nagent^{-1/2}\step_t^{-1/2} \left( \sum_{s=1}^t \step_s^2 \nlupdates_s (1+s)^{-\gamma_\step}\exp{\left(-2a\sum_{\ell=s+1}^t \step_\ell \nlupdates_\ell\right)} \right)^{1/2}\\
        &+ 4\Auxconst_{\text{last}, 2}\bConst{A} W_{\max} p^3 \step_t^{-1/2} \left( \sum_{s=1}^t \step_s^2 \nlupdates_s (1+s)^{-2\gamma}\exp{\left(-2a\sum_{\ell=s+1}^t \step_\ell \nlupdates_\ell\right)} \right)^{1/2}\\
        &\leq \Auxconst_{3,1,2}^\boot p^6 \nagent^{-1/2} (1+t)^{-\gamma_\step/2} + (\Auxconst_{3,1,1}^\boot+\Auxconst_{3,1,3}^\boot) p^3  (1+t)^{-\gamma}\eqsp.
    \end{align*}

where we have used that $x^{1/2}e^{-cx} \leq (c\rme)^{-1/2}e^{-cx/2}$ together with the bound
    
{
\begin{align*}
        \step_t^{-1/2}\exp{\left(-{a\over 2}\sum_{\ell=1}^t \step_\ell\nlupdates_\ell\right)}
        &\leq \exp{\left({a\step\nlupdates\,2^{-\gamma}\over 1-\gamma}\right)}
        \left({\gamma_\step+2\gamma\over a\rme\step\nlupdates}\right)^{\gamma_\step+2\gamma\over 2(1-\gamma)}
        \step^{-1/2}(1+t)^{-\gamma} \eqsp,
    \end{align*}
}
    and have set

{
\begin{align*}
        \Auxconst_{3,1,1}^\boot &= 4a^{-1/2}\exp{\left({1\over2}+{a\step\nlupdates\,2^{-\gamma}\over 1-\gamma}\right)}\bConst{A}\, W_{\max}
        \left({\gamma_\step+2\gamma\over a\rme\step\nlupdates}\right)^{\gamma_\step+2\gamma\over 2(1-\gamma)}
        \step^{-1/2}\norm{\theta_0-\thetalim} \eqsp,\\
        \Auxconst_{3,1,2}^\boot &= 8\Auxconst_{\text{last}, 1}\bConst{A} W_{\max} ((\step\nlupdates)^{1/2} + a^{-1/2}L^{1/2}(1 + (2a\step \nlupdates)^{-{\gamma + 2\gamma_\step \over 2(1-\gamma)}}))\, \eqsp,\\
        \Auxconst_{3,1,3}^\boot &= 16\Auxconst_{\text{last}, 2}\bConst{A}W_{\max} ((\step\nlupdates)^{1/2} + a^{-1/2}L^{1/2}(1 + (2a\step \nlupdates)^{-{3\gamma + \gamma_\step \over 2(1-\gamma)}}))\,\eqsp.
    \end{align*}
}

Similarly, for the term $(\Sigma_t^\boot)^{1/2}D_{t,3,2}^\boot$, applying Burkholder's inequality together with \Cref{lemma: legendary-sum}, setting

\begin{align*}
        S_{s,3,2}^\boot = \sum_{c=1}^\nagent \sum_{h=1}^{\nlupdates_s} \PE^{2/p}[\norm{(w_{s,h}-1)\Gamma_{s,h+1:\nlupdates_s}^{\boot,c} \zfuncA[c]{Z_{s,h,c}}\sum_{j=1}^{h-1}\Gamma_{s,j+1:h-1}^c\funcnoise[c]{Z_{s,j,c}}}^p] \eqsp,
    \end{align*}

we obtain

\begin{align*}
        \PE^{1/p}&[\norm{(\Sigma_t^\boot)^{1/2} D_{t,3,2}^\boot}^p] \leq p^3 \nagent^{-1}\step_t^{-1/2} \left(\sum_{s=1}^t\step_s^4 \PE^{2/p}[\norm{\prod_{i=s+1}^t \Gammaavgboot_i}^p] S_{s,3,2}^\boot\right)^{1/2}\\
        &\leq \bConst{A} W_{\max} p^4 \nagent^{-1/2}\step_t^{-1/2} \left( \sum_{s=1}^t \step_s^4 \nlupdates_s^2 \exp{\left(-2a\sum_{\ell=s+1}^t \step_\ell \nlupdates_\ell\right)}\right)^{1/2} \supconsteps \\
        &\leq \Auxconst_{3,2}^\boot p^4 \nagent^{-1/2} (1+t)^{-(\gamma_\step + \gamma)/2} \eqsp,
    \end{align*}

    where we have used that

\begin{align*}
        \PE^{1/p}[\norm{\sum_{j=1}^{h-1}\Gamma_{s,j+1:h-1}^c \funcnoise[c]{Z_{s,j,c}}}^p] \leq p\left( \sum_{j=1}^{h-1} \PE^{2/p}[\norm{\Gamma_{s,j+1:h-1}^c \funcnoise[c]{Z_{s,j,c}}}^p]\right)^{1/2} \leq p \nlupdates_s^{1/2} \supconsteps \eqsp,
    \end{align*}

    and have set

{
\begin{align*}
        \Auxconst_{3,2}^\boot = 4 W_{\max}\bConst{A} (\step\nlupdates + a^{-1/2}L^{1/2}(\step\nlupdates)^{1/2}(1+(2a\step\nlupdates)^{-{\gamma+\gamma_\step \over 1-\gamma}})) \supconsteps\, \eqsp.
    \end{align*}
}

Finally, to bound $(\Sigma_t^\boot)^{1/2} D_{t,3,3}^\boot$, we apply Minkowski's inequality over clients and Burkholder's inequality conditionally on all observations and \Cref{lemma: legendary-sum}, thus, setting

\begin{align*}
        S_{s,3,3}^\boot = \sum_{h=1}^{\nlupdates_s}\left(\sum_{c=1}^\nagent \PE^{1/p}[\norm{(w_{s,h}-1)\Gamma_{s,h+1:\nlupdates_s}^{\boot,c} \zfuncA[c]{Z_{s,h,c}}(\Id - \Gamma_{s,h-1}^c)(\thetalim[c] - \thetalim)}^p]\right)^2\eqsp,
    \end{align*}

we get

\begin{align*}
        \PE^{1/p}&[\norm{(\Sigma_t^\boot)^{1/2} D_{t,3,3}^\boot}^p] \leq p^3\nagent^{-1}\step_t^{-1/2} \left( \sum_{s=1}^t \step_s^2 \PE^{2/p}[\norm{\prod_{i=s+1}^t \Gammaavgboot_i}^p]S_{s,3,3}^\boot\right)^{1/2}\\
        &\leq W_{\max}\zeta_3\bConst{A}^2  p^3 \step_t^{-1/2} \left(\sum_{s=1}^t \step_s^4 \nlupdates_s^3\exp{\left(-2a \sum_{\ell=s+1}^t \step_\ell \nlupdates_\ell\right)} \right)^{1/2}\\
        &\leq \Auxconst_{3,3}^\boot p^3  (1+t)^{-\gamma}  \eqsp,
    \end{align*}

where we defined

{
\begin{align*}
        \Auxconst_{3,3}^\boot = 8W_{\max} \zeta_3\bConst{A}^2  ((\step \nlupdates)^{3/2} + a^{-1/2}L^{1/2} \step\nlupdates(1 + (2a\step\nlupdates)^{-{3\gamma+\gamma_\step \over 2(1-\gamma)}}))\, \eqsp.
    \end{align*}
}

    Finally, setting

\begin{align}
        \label{eq:constants_Db_pth_moment}
        \bar{\Auxconst}_1^\boot &= \Auxconst_{1}^\boot + \Auxconst_{3,1,2}^\boot + \Auxconst_{4}^\boot\\
        \nonumber\bar{\Auxconst}_2^\boot &= \Auxconst_{3,1,1}^\boot +\Auxconst_{2}^\boot + \Auxconst_{3,1,3}^\boot + \Auxconst_{3,2}^\boot + \Auxconst_{3,3}^\boot + \Auxconst_{4,2}^\boot + \Auxconst_{5}^\boot \eqsp,
    \end{align}

    we complete the proof.
\end{proof}

%% file: appendix/experiments.tex
 
In this section, we provide the experimental results for the normal approximation of FedLSA. Code is available in \url{https://github.com/levensons/FedLSA_NormApprox}. We consider the Garnet environment \cite{archibald1995generation,geist2014off} with $n = 30$ states, dimension $d=5$, number of actions $a=2$, and branching number $b=2$. We then generate $\nagent = 5$ heterogeneous Garnet environments with small perturbations. The process for generating heterogeneous environments follows that described in \cite{mangold2024scafflsa}. In our experiment, we first sample a unit vector $u \in \mathbb{S}_{d-1}$ onto which we project the obtained confidence sets. Then, we sample $R=1024$ trajectories $\{Z_{s,h,c}^r, s \in [t], h \in [\nlupdates_s], c \in [\nagent]\}_{r=1}^{1024}$ and, for each of them, additionally sample $N_\boot=256$ bootstrap trajectories. The bootstrap weights $\{w_{s,h,c}, s \in [t], h\in [\nlupdates_s], c\in [\nagent]\}$ are i.i.d.\ and defined as
\begin{align*}
    w_{s,h,c} = {\tilde{w}_{s,h,c} - \PE[\tilde{w}_{s,h,c}] \over \sqrt{\Var[\tilde{w}_{s,h,c}]}} \eqsp, \quad \tilde{w}_{s,h,c} \sim \text{Beta}(0.5, 2) \eqsp.
\end{align*}
We therefore consider several algorithms for constructing confidence sets for the last iterate $\theta_t$. Then, based on $R$ samples, we estimate the probability that $\thetalim$ belongs to the confidence sets and report it as coverage.

\paragraph{Plug-in Estimate (PE).}
As a baseline, we consider the plug-in estimate in a way similar to \cite{li2022statistical}. First, we approximate the noise covariance and design matrix at step $t$ as
\begin{align*}
    \widehat{\Sigma}_{*,t}^{\sf avg} = \nagent^{-1}\bar{\nlupdates}_t^{-1}\sum_{s=1}^{t} \sum_{h=1}^{\nlupdates_s}\sum_{c=1}^\nagent \funcnoiseth[c]{\theta_s}{Z_{s,h,c}}(\funcnoiseth[c]{\theta_s}{Z_{s,h,c}})^T \eqsp, \quad \bA_t = \nagent^{-1}\bar{\nlupdates}_t^{-1} \sum_{s=1}^t \sum_{h=1}^{\nlupdates_s} \sum_{c=1}^\nagent \zfuncA[c]{Z_{s,h,c}} \eqsp,
\end{align*}
where $\bar{\nlupdates}_t = \sum_{s=1}^{\nlupdates_t} \nlupdates_s$. Afterwards, we solve the Poisson equation with $\noisecovavgst$ replaced by $\widehat{\Sigma}_{*,t}^{\sf avg}$; that is, we find $\widehat{\Sigma}_{\infty,t}$ such that
\begin{align*}
    \bA_t \widehat{\Sigma}_\infty + \widehat{\Sigma}_\infty \bA[T]_t = \nagent^{-1} \widehat{\Sigma}_{*,t}^{\sf avg} \eqsp.
\end{align*}
Finally, we consider the confidence interval $u^T \theta_t \pm z_{\alpha/2}\step_t^{1/2} \sqrt{u^T \hat{\Sigma}_{\infty,t} u}$, where $z_{\alpha/2}$ is the $\alpha/2$-quantile of $\mathcal{N}(0,1)$.

\paragraph{Empirical Quantiles (EQ).}
As described in \Cref{sec:bootstrap}, we construct the confidence interval using the empirical quantiles of the bootstrap samples.

\paragraph{Standard Deviation-Based Confidence Intervals (SDB).}
We estimate the standard deviation $\hat{\sigma}_t^\boot(u)$ of the bootstrap samples $\{\theta_t^b, b\in [N_\boot]\}$ and consider the confidence interval $u^T \theta_t \pm z_{\alpha/2}\step_t^{1/2}\hat{\sigma}_t^\boot(u)$.

\begin{table}[t]
\centering
    \centering
    \caption{Coverage comparison for different methods with confidence levels $\alpha \in \{0.95, 0.9, 0.8\}$ and $\gamma_\nlupdates = 0$. Standard deviations are shown as subscripts.}
    \label{tab:coverage_gammaH=0}
\footnotesize
\setlength{\tabcolsep}{3pt}
\begin{tabularx}{\linewidth}{c|YYY|YYY|YYY}
\toprule
& \multicolumn{3}{c|}{$\alpha = 0.95$} & \multicolumn{3}{c|}{$\alpha = 0.9$} & \multicolumn{3}{c}{$\alpha = 0.8$} \\
\midrule
$T$ & SDB & EQ & PE & SDB & EQ & PE & SDB & EQ & PE \\
\midrule
2000  & $0.773_{\scriptscriptstyle \pm 0.013}$ & $0.780_{\scriptscriptstyle \pm 0.013}$ & $0.754_{\scriptscriptstyle \pm 0.013}$
      & $0.676_{\scriptscriptstyle \pm 0.015}$ & $0.683_{\scriptscriptstyle \pm 0.015}$ & $0.652_{\scriptscriptstyle \pm 0.015}$
      & $0.547_{\scriptscriptstyle \pm 0.016}$ & $0.544_{\scriptscriptstyle \pm 0.016}$ & $0.517_{\scriptscriptstyle \pm 0.016}$ \\
6000  & $0.936_{\scriptscriptstyle \pm 0.008}$ & $0.941_{\scriptscriptstyle \pm 0.007}$ & $0.931_{\scriptscriptstyle \pm 0.008}$
      & $0.870_{\scriptscriptstyle \pm 0.011}$ & $0.880_{\scriptscriptstyle \pm 0.010}$ & $0.866_{\scriptscriptstyle \pm 0.011}$
      & $0.778_{\scriptscriptstyle \pm 0.013}$ & $0.778_{\scriptscriptstyle \pm 0.013}$ & $0.768_{\scriptscriptstyle \pm 0.013}$ \\
10000 & $0.938_{\scriptscriptstyle \pm 0.008}$ & $0.947_{\scriptscriptstyle \pm 0.007}$ & $0.938_{\scriptscriptstyle \pm 0.008}$
      & $0.888_{\scriptscriptstyle \pm 0.010}$ & $0.890_{\scriptscriptstyle \pm 0.010}$ & $0.883_{\scriptscriptstyle \pm 0.010}$
      & $0.786_{\scriptscriptstyle \pm 0.013}$ & $0.792_{\scriptscriptstyle \pm 0.013}$ & $0.775_{\scriptscriptstyle \pm 0.013}$ \\
14000 & $0.941_{\scriptscriptstyle \pm 0.007}$ & $0.946_{\scriptscriptstyle \pm 0.007}$ & $0.941_{\scriptscriptstyle \pm 0.007}$
      & $0.895_{\scriptscriptstyle \pm 0.010}$ & $0.900_{\scriptscriptstyle \pm 0.009}$ & $0.888_{\scriptscriptstyle \pm 0.010}$
      & $0.798_{\scriptscriptstyle \pm 0.013}$ & $0.804_{\scriptscriptstyle \pm 0.012}$ & $0.789_{\scriptscriptstyle \pm 0.013}$ \\
\bottomrule
\end{tabularx}
\end{table}

\begin{table}[t]
\centering
\footnotesize
\setlength{\tabcolsep}{3pt}
\caption{Coverage comparison for different methods with confidence levels $\alpha \in \{0.95, 0.9, 0.8\}$ and $\gamma_\nlupdates = 0.2$. Standard deviations are shown as subscripts.}
\label{tab:coverage_gammaH=0.2}
\begin{tabularx}{\linewidth}{c|YYY|YYY|YYY}
\toprule
& \multicolumn{3}{c|}{$\alpha = 0.95$} & \multicolumn{3}{c|}{$\alpha = 0.9$} & \multicolumn{3}{c}{$\alpha = 0.8$} \\
\midrule
$T$ & SDB & EQ & PE & SDB & EQ & PE & SDB & EQ & PE \\
\midrule
2000  & $0.398_{\scriptscriptstyle \pm 0.015}$ & $0.404_{\scriptscriptstyle \pm 0.015}$ & $0.376_{\scriptscriptstyle \pm 0.015}$
      & $0.291_{\scriptscriptstyle \pm 0.014}$ & $0.291_{\scriptscriptstyle \pm 0.014}$ & $0.268_{\scriptscriptstyle \pm 0.014}$
      & $0.173_{\scriptscriptstyle \pm 0.012}$ & $0.177_{\scriptscriptstyle \pm 0.012}$ & $0.161_{\scriptscriptstyle \pm 0.011}$ \\
6000  & $0.959_{\scriptscriptstyle \pm 0.006}$ & $0.961_{\scriptscriptstyle \pm 0.006}$ & $0.959_{\scriptscriptstyle \pm 0.006}$
      & $0.914_{\scriptscriptstyle \pm 0.009}$ & $0.926_{\scriptscriptstyle \pm 0.008}$ & $0.917_{\scriptscriptstyle \pm 0.009}$
      & $0.817_{\scriptscriptstyle \pm 0.012}$ & $0.820_{\scriptscriptstyle \pm 0.012}$ & $0.814_{\scriptscriptstyle \pm 0.012}$ \\
10000 & $0.946_{\scriptscriptstyle \pm 0.007}$ & $0.949_{\scriptscriptstyle \pm 0.007}$ & $0.946_{\scriptscriptstyle \pm 0.007}$
      & $0.903_{\scriptscriptstyle \pm 0.009}$ & $0.907_{\scriptscriptstyle \pm 0.009}$ & $0.907_{\scriptscriptstyle \pm 0.009}$
      & $0.815_{\scriptscriptstyle \pm 0.012}$ & $0.825_{\scriptscriptstyle \pm 0.012}$ & $0.818_{\scriptscriptstyle \pm 0.012}$ \\
14000 & $0.950_{\scriptscriptstyle \pm 0.007}$ & $0.950_{\scriptscriptstyle \pm 0.007}$ & $0.952_{\scriptscriptstyle \pm 0.007}$
      & $0.892_{\scriptscriptstyle \pm 0.010}$ & $0.894_{\scriptscriptstyle \pm 0.010}$ & $0.892_{\scriptscriptstyle \pm 0.010}$
      & $0.788_{\scriptscriptstyle \pm 0.013}$ & $0.789_{\scriptscriptstyle \pm 0.013}$ & $0.788_{\scriptscriptstyle \pm 0.013}$ \\
\bottomrule
\end{tabularx}
\end{table}

\paragraph{Discussion.}
In \Cref{tab:coverage_gammaH=0} and \Cref{tab:coverage_gammaH=0.2}, we compare coverage across different confidence levels and iteration lengths. In our experiments, we consider a decreasing step size with $\gamma_\eta=0.6$. We also consider two regimes: one with a constant number of local updates, $\nlupdates=20$, and another in which the number of local updates increases at the rate $\gamma_\nlupdates=0.2$ with $H=5$. In the case $\gamma_\nlupdates=0$, we can see that the bootstrap-based methods outperform the baseline even for a small number of iterations and, for $T=6000$, produce decent coverage. On the other hand, when we consider an increasing number of local updates, for large $T$ it becomes comparable to, or slightly better than, the bootstrap-based methods. This can be explained by the fact that, as the number of samples increases, the PE estimate becomes sufficiently accurate to provide good coverage.

%% file: appendix/add_lemmas.tex
In this section we collect several auxiliary results that will be used in the proofs of the main theorems.

\begin{lemma}
\label{lemma:matrix-product-moments}
Let $q\geq 2$. Assume \Cref{ass: linear-decay}($q$) and \Cref{assum:noise-level-flsa}, and let $0<\step_i\leq \step_{\infty,q}\wedge \frac{a}{q\bConst{A}^2}$ for $i\in[t]$. Then, for every $c\in[\nagent]$ and $i\in[t]$,
\begin{align*}
    \PE^{1/q}[\norm{\Gamma_{i,m:k}^c}^q]
    &\leq d^{1/q}(1-a\step_i)^{k-m+1} \eqsp,
    \qquad 1\leq m\leq k\leq\nlupdates_i \eqsp,
\end{align*}
and
\begin{align*}
    \PE^{1/q}\!\left[\norm{\prod_{j=m}^k\Gammaavg_j}^q\right]
    &\leq d^{1/q}\prod_{j=m}^k(1-a\step_j)^{\nlupdates_j} \eqsp,
    \qquad 1\leq m\leq k\leq t \eqsp.
\end{align*}
\end{lemma}
\begin{proof}
Under \Cref{assum:noise-level-flsa}, the moment inequality in \Cref{ass: linear-decay}($q$) extends to $\step_i=\step_{\infty,q}$ by continuity.
Triangle inequality and \Cref{ass: linear-decay}($q$) give $\norm{\Id-\step_i\bA[c]}\leq 1-2a\step_i$, hence $2a\leq\norm{\bA[c]}\leq\bConst{A}$. Using $\step_i\leq a/(q\bConst{A}^2)$, we obtain
\begin{align*}
    (1-2a\step_i)^2+(q-1)\bConst{A}^2\step_i^2
    \leq 1-3a\step_i\leq(1-a\step_i)^2 \eqsp.
\end{align*}
The first bound follows from \citet[Proposition~15]{durmus2022finite}. Applying Minkowski's inequality over clients, we obtain the second result.
\end{proof}

\begin{proposition}
\label{prop:nonlinearapprox}
Let $\nu$ be a standard Gaussian measure in $\rset^d$. Then for any random vectors $X, Y$ taking values in $\rset^d$, and any $p \geq 1$,
  \begin{equation}
\label{eq:shao_zhang_bound}
\sup_{B \in \Conv(\rset^d)}|\PP(X + Y \in B) - \nu(B)| \le \sup_{B \in \Conv(\rset^d)}|\PP(X
\in B) - \nu(B)| + 2 c_d^{p/(p+1)} \PE^{1/(p+1)}[\|Y\|^p]\eqsp, 
\end{equation}
where $c_d$ is the isoperimetric constant of class $\Conv(\rset^d)$.
\end{proposition}
\begin{proof}
    Proof and definition for the isoperimetric constant $c_d$ can be found in \citet[Proposition 1]{sheshukova2025gaussian}.
\end{proof}

\begin{lemma}
\label{lemma:high_prob_bound}
    Fix $\delta \in (0, 1/\rme^2)$ and let $Y$ be a positive random variable, such that $\PE^{1/p}[Y^p] \leq C_1 + C_2 p$ for any $2 \leq p \leq \log(1/\delta)$. Then it holds with probability at least $1-\delta$, that
    \begin{align*}
        Y \leq \rme C_1 + \rme C_2 \log(1/\delta) \eqsp.
    \end{align*}
\end{lemma}
\begin{proof}
    For the complete proof, see \citet[Lemma 1]{samsonov2024gaussian}.
\end{proof}

\begin{lemma}
    Under Assumptions \Cref{ass: linear-decay}, ~\Cref{assum:noise-level-flsa}(2), 
~\Cref{ass: lr-polynomial-decay} and
~\Cref{ass: boot_weights}, choosing $\step\nlupdates \leq \beta_\infty$ for some $\beta_\infty > 0$, we have that 
    \begin{align*}
        \norm{\Sigma_t - \Sigma_\infty} \leq C_{\infty, 1} (1+t)^{\gamma -1} + C_{\infty,2}(1+t)^{-\gamma} \eqsp,
    \end{align*}
    where the matrix $\Sigma_\infty$ is the solution of the Lyapunov's equation
    \begin{align}
    \label{eq:Lyapunov_Sigmast}
        \bA \Sigma_\infty + \Sigma_{\infty}\bA[T] = \nagent^{-1}\noisecovavgst \eqsp,
    \end{align}
    where $\noisecovavgst$ is defined in \eqref{eq:Sigma_star_def}.
\end{lemma}
\begin{proof}
    Due to \citet[Lemma 9.1]{poznyak:control}, the solution of Lyapunov's equation \eqref{eq:Lyapunov_Sigmast} exists since the matrix $\bA$ is Hurwitz. Let us consider only the leading part of the term $M_t$, that is, $\hat{M}_t = M_{t,1} + M_{t,2}$. Using the notation $P_s^t = \prod_{i=s+1}^t \Gavg_i$, we define
    \begin{align*}
            \hat{\Sigma}_t = \PE[\hat{M}_t \hat{M}_t^T] = \nagent^{-1} \step_t^{-1} \sum_{s=1}^t \step_s^2 \nlupdates_s P_s^t  \noisecovavgst (P_s^t)^T
    \end{align*}
    To estimate the difference between $\Sigma_t$ and $\Sigma_\infty$, we write the recursion
    \begin{align*}
        \hat{\Sigma}_{t+1} &= \step_t \step_{t+1}^{-1} \Gavg_{t+1}\hat{\Sigma}_t (\Gavg_{t+1})^T + \nagent^{-1}\step_{t+1}\nlupdates_{t+1}\noisecovavgst\\
        &= \step_t \step_{t+1}^{-1} \{\hat{\Sigma}_t + \hat{\Sigma}_t (\Gavg_{t+1} - \Id)^T + (\Gavg_{t+1} - \Id)\hat{\Sigma}_t \} \\
        &+ \step_t \step_{t+1}^{-1}(\Gavg_{t+1} - \Id)\hat{\Sigma}_t (\Gavg_{t+1} - \Id)^T + \nagent^{-1}\step_{t+1}\nlupdates_{t+1}\noisecovavgst \\
        &= \hat{\Sigma}_t - \step_{t+1}\nlupdates_{t+1}\{\bA \hat{\Sigma}_t + \hat{\Sigma}_t \bA[T] - \nagent^{-1} \noisecovavgst\} \\
        &+ (\step_t \step_{t+1}^{-1} - 1)(\hat{\Sigma}_t - \hat{\Sigma}_t \bar{R}_{t+1}^T - \bar{R}_{t+1} \hat{\Sigma}_t) \\
        &+ \step_t \step_{t+1}^{-1}(\Gavg_{t+1} - \Id)\hat{\Sigma}_t (\Gavg_{t+1} - \Id)^T \eqsp,
    \end{align*}
    where we have used the identity 
    \begin{align*}
        (\Id - \step_{t+1}\bA[c])^{\nlupdates_{t+1}} - \Id &= -\step_{t+1}\bA[c] \sum\limits_{h=1}^{\nlupdates_{t+1}}(\Id - \step_{t+1}\bA[c])^{h-1} \\
        &= -\step_{t+1}\nlupdates_{t+1} \bA[c] - \underbrace{\step_{t+1}\bA[c]\sum_{h=1}^{\nlupdates_{t+1}} \{(\Id - \step_{t+1}\bA[c])^{h-1} - \Id\}}_{R_{t+1}^c} \eqsp.
    \end{align*}
    Now, the error $E_t = \hat{\Sigma}_t - \Sigma_\infty$ evolves as
    \begin{align*}
        E_{t+1} &= E_t - \step_{t+1}\nlupdates_{t+1}\{\bA E_t + E_t \bA[T]\}\\
        &+ (\step_t \step_{t+1}^{-1} - 1)(\hat{\Sigma}_t - \hat{\Sigma}_t \bar{R}_{t+1}^T - \bar{R}_{t+1} \hat{\Sigma}_t) + \step_t \step_{t+1}^{-1}(\Gavg_{t+1} - \Id)\hat{\Sigma}_t (\Gavg_{t+1} - \Id)^T\\
        &= E_t - \step_{t+1}\nlupdates_{t+1}\{\bA E_t + E_t \bA[T]\} + r_{t+1} \eqsp.
    \end{align*}
    To quantify the convergence, we transition the recursion into the vectorized form. Therefore, we get
    \begin{align*}
        \text{vec}(E_{t+1}) = (\Id - \step_{t+1}\nlupdates_{t+1}((\Id \otimes \bA) + (\bA \otimes \Id)))\text{vec}(E_t) + \text{vec}(r_{t+1}) = F_{t+1} \text{vec}(E_t) + \text{vec}(r_{t+1}) \eqsp.
    \end{align*}
    Clearly, the matrix $(\Id \otimes \bA) + (\bA \otimes \Id)$ is Hurwitz. Therefore, due to the \citet[Proposition 1]{durmus2021tight}, we have 
    \begin{equation*}
        \norm{F_t} \leq \qcond^{1/2} (1 - \step \nlupdates b)^{1/2}\eqsp, \quad \step_{t}\nlupdates_{t} \leq \beta_\infty \eqsp.
    \end{equation*}
    Now, to bound $R_t$, we apply \Cref{lemma: difference-of-product}, thus
    \begin{align*}
        \norm{R_t} \leq \bConst{A} \step_{t}^2 \nlupdates_t^2 \eqsp.
    \end{align*}
    Therefore, using \Cref{lemma:bound_Mt} and the bound $\norm{\Gavg_t - \Id} \leq \bConst{A}\step_t \nlupdates_t$, the remainder term $r_t$ can be bounded, as
    \begin{align*}
        \norm{r_{t+1}} \leq 8\Auxconst_{\hat{\Sigma}}\gamma_\step (1 + \bConst{A}(\step\nlupdates)^2) (1+t)^{-1} + 2\bConst{A}^2\Auxconst_{\hat{\Sigma}} (\step_{t+1} \nlupdates_{t+1})^2 \eqsp.
    \end{align*}
    Using that $\norm{\text{vec}(r_{t+1})} \leq d^{1/2}\norm{r_{t+1}}$, we get
    \begin{align*}
        \norm{\text{vec}(E_{t+1})} &\leq \qcond^{1/2}\prod_{i=1}^{t+1}(1-\step_i \nlupdates_i b)^{1/2} \norm{\text{vec}(\hat{\Sigma}_1 - \Sigma_\infty)} + \qcond^{1/2}d^{1/2}\sum_{j=1}^{t+1} \prod_{i=j+1}^{t+1}(1-\step_i \nlupdates_i b)^{1/2} \norm{r_j}\\
        &\leq \qcond^{1/2} \exp{\left( -{b\over 2} \sum_{\ell=1}^{t+1} \step_\ell \nlupdates_\ell\right)}\norm{\text{vec}(\hat{\Sigma}_1 - \Sigma_\infty)} \\
        &+ 16\Auxconst_{\hat{\Sigma}}\qcond^{1/2}\gamma_\step (1 + \bConst{A}(\step\nlupdates)^2) d^{1/2}\sum_{j=1}^{t+1} (1+j)^{-1} \exp{\left( -{b\over 2} \sum_{\ell=j+1}^{t+1} \step_\ell \nlupdates_\ell\right)}\\
        &+2\bConst{A}^2\Auxconst_{\hat{\Sigma}}\qcond^{1/2}d^{1/2}\sum_{j=1}^{t+1} \step_j^2 \nlupdates_j^2 \exp{\left( -{b\over 2} \sum_{\ell=j+1}^{t+1} \step_\ell \nlupdates_\ell\right)}\\
        &\leq \qcond^{1/2} \exp{\left( -{b\over 2} \sum_{\ell=1}^{t+1} \step_\ell \nlupdates_\ell\right)}\norm{\text{vec}(\hat{\Sigma}_1 - \Sigma_\infty)} + M_{\Sigma,1} (2+t)^{\gamma - 1} + M_{\Sigma,2} (2+t)^{-\gamma} \eqsp,
    \end{align*}
    where
    \begin{align*}
        M_{\Sigma, 1} &= 16\Auxconst_{\hat{\Sigma}}\qcond^{1/2}\gamma_\step (1 + \bConst{A}(\step\nlupdates)^2)(1+2b^{-1}L(1+((b/2)\step \nlupdates)^{1+\gamma \over 1-\gamma})) d^{1/2} \eqsp,\\
        M_{\Sigma, 2} &= 8\bConst{A}^2\Auxconst_{\hat{\Sigma}}\qcond^{1/2}((\step \nlupdates)^2 + b^{-1}L\step \nlupdates(1+ ((b/2)\step\nlupdates)^{3\gamma \over 1 - \gamma}))d^{1/2} \eqsp.
    \end{align*}
    Finally, using \eqref{eq:bound_Mt3}, \eqref{eq:bound_Mt4}, the bound $x e^{-cx} \leq (c\rme)^{-1}$ and that $\norm{E_t} \leq \|E_t\|_F$, we get
    \begin{align*}
        \norm{\Sigma_t - \Sigma_\infty} \leq \norm{\Sigma_t - \hat{\Sigma}_t} + \norm{\hat{\Sigma}_t - \Sigma_\infty}  \leq C_{\infty,1} (1+t)^{\gamma-1}  + C_{\infty,2}(1+t)^{-\gamma} \eqsp,
    \end{align*}
    where we have used that $\norm{\text{vec}(\Sigma_\infty)} \leq (2\lambda_{\min}(\bA))^{-1} \nagent^{-1}\norm{\text{vec}(\noisecovavgst)}$, and have set
    \begin{align}
    \label{eq:constant_sigma_infty_1}
        C_{\infty,1} &=\qcond^{1/2} \left({b\rme \over 2(1-\gamma)} \step \nlupdates\right)^{-1}e^{{b\over 1-\gamma}\step \nlupdates} d^{1/2}(2\step \nlupdates + (2\lambda_{\min}(\bA))^{-1})\nagent^{-1}\norm{\noisecovavgst} + M_{\Sigma,1} \eqsp,\\
    \label{eq:constant_sigma_infty_2}
        C_{\infty,2} &= M_{\Sigma, 2}+ 2(\Auxconst_{M,3} + \Auxconst_{M,4})(\Auxconst_{\hat{\Sigma}}^{1/2} + \Auxconst_{M,3} + \Auxconst_{M,4})\nagent^{-1}  \eqsp.
    \end{align}
    % \begin{align*}
    %     S_t^1 = \nagent^{-2}\step_t^{-1}\sum_{s=1}^t \step_s^2 \sum_{h=1}^{\nlupdates_s}\sum_{c=1}^\nagent \prod_{i=s+1}^t \Gavg_i G_{s,h+1:\nlupdates_s}^c \noisecov[c] \Big(\prod_{i=s+1}^t \Gavg_i G_{s,h+1:\nlupdates_s}^c \Big)^T \eqsp.
    % \end{align*}
    % Then, we have the following recursive relation
    % \begin{align*}
    %     S_{t+1}^1 &= \step_t \step_{t+1}^{-1} \Gavg_{t+1} S_t^1 (\Gavg_{t+1})^T + \nagent^{-2}\step_{t+1} \sum_{h=1}^{\nlupdates_{t+1}} \sum_{c=1}^\nagent G_{t+1,h+1:\nlupdates_{t+1}}^c \noisecov[c] (G_{t+1,h+1:\nlupdates_{t+1}}^c)^T\\
    %     &= S_t^1 + \step \nlupdates (S_t^1 \bA[T] + \bA S_t^1 + \nagent^{-1}  \noisecovavg) + B_{t+1} \eqsp,
    % \end{align*}
    % where $\norm{B_{t+1}} \leq \mathcal{O}((\step \nlupdates)^2)$. Thus, in the limit $\step \nlupdates \to 0$, these iterations converge to the solution of the following equation
    % \begin{align*}
    %     f_1(S) = S \bA[T] + \bA S + \nagent^{-1}\noisecovavg = 0 \eqsp.
    % \end{align*}
\end{proof}

\begin{lemma}
\label{lemma: sigma-n-norm-estimation}
    Under Assumptions \Cref{ass: linear-decay}(1), ~\Cref{assum:noise-level-flsa}, 
~\Cref{ass: lr-polynomial-decay}, 
~\Cref{ass: boot_weights}, and ~\Cref{assum:sample_size}, for $\step\nlupdates \leq \beta_\infty$, we have
    \begin{align}
    \label{eq:bound_lambda_min_Sigma_t}
        \norm{\Sigma_t^{-1/2}} \leq \nagent^{1/2}\bConst{\Sigma} \eqsp,
    \end{align}
    where $\Sigma_t$ is defined in \eqref{eq:sigma_t_def} and
    \begin{align*}
        \bConst{\Sigma} =  2\bConst{A}^{1/2}(\lambda_{\min}(\noisecovavgst))^{-1/2} \eqsp.
    \end{align*}
\end{lemma}
\begin{proof}
    Applying Lidski's inequality, we obtain
    \begin{align*}
        \lambda_{\min}(\Sigma_t) \geq \lambda_{\min}(\Sigma_\infty) - \norm{\Sigma_t - \Sigma_\infty}
    \end{align*}
    Thus, under \Cref{assum:sample_size}, we get that 
    \begin{align*}
        \norm{\Sigma_t - \Sigma_\infty} \leq {\lambda_{\min}(\Sigma_\infty) \over 2}   \eqsp. 
    \end{align*}
    Finally, using the relation $\lambda_{\min}(\Sigma_\infty) \geq {\lambda_{\min}(\noisecovavgst) \over 2\nagent \norm{\bA}}$, the result \eqref{eq:bound_lambda_min_Sigma_t} follows.
\end{proof}

\begin{proposition}
\label{lemma:sigma_t_boot_norm_bound}
{
Let $q_t$ be as in \Cref{lemma: probability-of-theta-0}. Assume \Cref{ass: linear-decay}($q_t$), \Cref{assum:noise-level-flsa}, \Cref{ass: lr-polynomial-decay} and \Cref{ass: boot_weights}, and suppose that the observation sequences of distinct clients are independent. Let $t\geq\max(\nagent,\rme^3)$, $\gamma>0$, and $0<\step\leq\step_{\infty,q_t}\wedge a/(q_t\bConst{A}^2)$. Then, with probability at least $1-2/t$,
\begin{align}
\label{eq:bound_Sigmat_Sigmatboot_diff}
    \norm{\Sigma_t-\Sigma_t^\boot}
    &\lesssim \nagent^{-1}\log^6(t)(1+t)^{-\gamma_\step/2} + \nagent^{-1/2}\log^4(t)(1+t)^{-\gamma}
    +\log^4(t)(1+t)^{-2\gamma} \eqsp,
\end{align}
where $\Sigma_t^\boot$ is defined in \eqref{eq:Sigma_t_boot_def}. 
Moreover, under \Cref{assum:sample_size}, on the same event,
\begin{align*}
    \norm{(\Sigma_t^\boot)^{-1/2}}\leq\nagent^{1/2}\bConst{\Sigma^\boot}
    \eqsp,\qquad
    \bConst{\Sigma^\boot}=2\big(\nagent\lambda_{\min}(\Sigma_\infty)\big)^{-1/2} \eqsp.
\end{align*}
}
\end{proposition}
\begin{proof}
% Now, we set
% \begin{align}
% \label{eq:def_n0}
%     n_0 = \max\left(\left({8\Auxconst_U \step \log(2d) \over \nagent \gamma_\step \lambda_{\min}(\noisecovavgst)^2}\right)^{1/\gamma_\step}, \left(\frac{8\,\bar{\Auxconst}_{M,3}^\boot}{\lambda_{\min}(\noisecovavgst)}\right)^{1/\gamma}  \right)
% \end{align}
{
Firstly, we note that
    \begin{align*}
        \norm{\Sigma_t - \Sigma_t^\boot} \leq \norm{\Sigma_t - \hat{\Sigma}_t} + \norm{\hat{\Sigma}_t - \hat{\Sigma}_t^\boot} + \norm{\hat{\Sigma}_t^\boot - \Sigma_t^\boot} \eqsp,
    \end{align*}
where $\hat{\Sigma}_t = \PE[\hat{M}_{t} \hat{M}_t^T]$ and $\hat{\Sigma}_t^\boot = \PEb[M_{t,1}^\boot (M_{t,1}^\boot)^T]$. Set $K_{s,t}=\prod_{i=s+1}^t\Gavg_i$ and $S_{s,h}=\sum_{c=1}^{\nagent}\funcnoiseth[c]{\thetalim}{Z_{s,h,c}}$. Since the weights are shared across clients,
\begin{align*}
    \hat\Sigma_t^\boot&=\PEb[M_{t,1}^\boot(M_{t,1}^\boot)^T]
    =\nagent^{-2}\step_t^{-1}\sum_{s=1}^t\sum_{h=1}^{\nlupdates_s}
    \step_s^2K_{s,t}S_{s,h}S_{s,h}^TK_{s,t}^T \eqsp,\\
    \hat\Sigma_t-\hat\Sigma_t^\boot
    &=\nagent^{-2}\step_t^{-1}\sum_{s=1}^t\sum_{h=1}^{\nlupdates_s}
    \underbrace{\step_s^2K_{s,t}(\nagent\noisecovavgst-S_{s,h}S_{s,h}^T)K_{s,t}^T}_{U_{s,h}} \eqsp.
\end{align*}
The matrices $U_{s,h}$ are independent and centered. Independence across clients gives $\PE^{1/(2p)}[\norm{S_{s,h}}^{2p}]\lesssim p\nagent^{1/2}(\supconsteps+\bConst{A}\zeta_3)$. Applying Burkholder's inequality and \Cref{lemma: legendary-sum}, we obtain, for $p\geq2$,
\begin{align}
\label{eq:bound_hSigmat_hSigmatboot_diff}
    \PE^{1/p}[\norm{\hat\Sigma_t-\hat\Sigma_t^\boot}^p]
    &\lesssim \nagent^{-1}p^3 \step_t^{-1}
    \left(\sum_{s=1}^t\step_s^4\nlupdates_s
    \exp\left\{-4a\sum_{i=s+1}^t\step_i\nlupdates_i\right\}\right)^{1/2}
    \lesssim p^3\nagent^{-1}(1+t)^{-\gamma_\step/2} \eqsp.
\end{align}
By \Cref{lemma:bound_Mtb}, with $\delta=1/t$ and $p=\log t\leq q_t/2$, the remaining random difference $\norm{\Sigma_t^\boot-\hat\Sigma_t^\boot}$ satisfies \eqref{eq:bound_Sigmatboot_hSigmatboot_diff} with probability at least $1-1/t$. For this choice of $p$, the third term in \eqref{eq:bound_Sigmatboot_hSigmatboot_diff} is of smaller order than the first and can be absorbed into it.
By \Cref{lemma:bound_Mt}, we also have $\norm{\Sigma_t-\hat\Sigma_t}\lesssim\nagent^{-1}(1+t)^{-\gamma}$.
Applying Markov's inequality to \eqref{eq:bound_hSigmat_hSigmatboot_diff} with $p=\log t$, followed by a union bound and the triangle inequality, proves \eqref{eq:bound_Sigmat_Sigmatboot_diff}. 
Finally, under \Cref{assum:sample_size} and the stated bound on the right-hand side of \eqref{eq:bound_Sigmat_Sigmatboot_diff}, Weyl's inequality gives
\begin{align*}
    \lambda_{\min}(\Sigma_t^\boot)
    \geq\lambda_{\min}(\Sigma_\infty)-\norm{\Sigma_t-\Sigma_\infty}
    -\norm{\Sigma_t^\boot-\Sigma_t}
    \geq\lambda_{\min}(\Sigma_\infty)/4 \eqsp,
\end{align*}
which proves the inverse bound.
}
\end{proof}

\begin{lemma}
\label{lemma:bound_Mt}
    {
    Under Assumptions \Cref{ass: linear-decay}(2), \Cref{assum:noise-level-flsa} and \Cref{ass: lr-polynomial-decay}, we have that
    }
    \begin{align*}
        \PE[\norm{M_t}^2] \leq \Auxconst_M \nagent^{-1} \eqsp.
    \end{align*}
    Particularly, we get that
    \begin{align*}
        \norm{\hat{\Sigma}_t} = \norm{\PE[\hat{M}_t \hat{M}_t^T]} \leq \Auxconst_{\hat{\Sigma}}\nagent^{-1} \eqsp,
    \end{align*}
    where $\hat{M}_t = M_{t,1} + M_{t,2}$, and
    {
    \begin{align}
    \label{eq:bound_Sigmat_hSigmatboot_diff}
        \norm{\Sigma_t-\hat\Sigma_t}
        \leq (\Auxconst_{M,3}+\Auxconst_{M,4})
        \big(2\Auxconst_{\hat\Sigma}^{1/2}+\Auxconst_{M,3}+\Auxconst_{M,4}\big)
        \nagent^{-1}(1+t)^{-\gamma} \eqsp.
    \end{align}
    }
\end{lemma}
\begin{proof}
    Applying again \Cref{lemma: difference-of-product} and \Cref{lemma: legendary-sum}, we get
    \begin{align}
        \nonumber\PE[\norm{M_{t,3}}^2] &\leq \nagent^{-2}\step_t^{-1}\sum_{s=1}^t \step_s^2 \sum_{c=1}^\nagent \sum_{h=1}^{\nlupdates_s} \PE[\norm{(G_{s,h+1:\nlupdates_s}^c - \Id)\funcnoise[c]{Z_{s,h,c}}}^2] \exp{\Big(-2a\sum_{\ell=s+1}^t \step_\ell \nlupdates_\ell \Big)} \\
        \nonumber&\leq \bConst{A}^2\varbound^2\nagent^{-1}\step_t^{-1} \sum_{s=1}^t \step_s^4 \nlupdates_s^3\exp{\Big(-2a\sum_{\ell=s+1}^t \step_\ell \nlupdates_\ell \Big)} \\
        \nonumber&\leq 4\bConst{A}^2 \varbound^2 ((\step\nlupdates)^3 + a^{-1}L(\step \nlupdates)^2(1 + (2a\step\nlupdates)^{-{3\gamma + \gamma_\step\over 1-\gamma}})) \nagent^{-1} (1+t)^{-2\gamma} \\
        &= \Auxconst_{M,3}^2 \nagent^{-1} (1+t)^{-2\gamma} \eqsp,
        \label{eq:bound_Mt3}
    \end{align}
    and
    \begin{align}
        \nonumber \PE[\norm{M_{t,4}}^2] &\leq \nagent^{-2}\step_t^{-1}\sum_{s=1}^t \step_s^2 \sum_{c=1}^\nagent \sum_{h=1}^{\nlupdates_s} \PE[\norm{(G_{s,1:h-1}^c - \Id)\zmfuncA[c]{Z_{s,h,c}}(\thetalim-\thetalim[c])}^2] \exp{\Big(-2a\sum_{\ell=s+1}^t \step_\ell \nlupdates_\ell \Big)} \\
        \nonumber &\leq 4 \bConst{A}^2\varhet^2 ((\step\nlupdates)^3 +a^{-1}L(\step\nlupdates)^2(1+ (2a\step \nlupdates)^{-{3\gamma + \gamma_\step \over 1-\gamma}})) \nagent^{-1} (1+t)^{-2\gamma}\\
        &= \Auxconst_{M,4}^2 \nagent^{-1} (1+t)^{-2\gamma}\eqsp.
        \label{eq:bound_Mt4}
    \end{align}
    Similarly, we get that
    \begin{align*}
        \PE[\norm{M_{t,1}}^2] &\leq \varbound^2 \nagent^{-1}\step_t^{-1}\sum_{s=1}^t \step_s^2 \nlupdates_s \exp{\Big(-2a\sum_{\ell=s+1}^t \step_\ell \nlupdates_\ell \Big)}\\
        &\leq 2 \varbound^2 (\step \nlupdates + a^{-1}L(1+(2a\step \nlupdates)^{\gamma+\gamma_\step \over 1-\gamma}))\nagent^{-1} = \Auxconst_{M,1}\nagent^{-1} \eqsp,
    \end{align*}
    and
    \begin{align*}
        \PE[\norm{M_{t,2}}^2] &\leq \varhet^2 \nagent^{-1}\step_t^{-1}\sum_{s=1}^t \step_s^2 \nlupdates_s \exp{\Big(-2a\sum_{\ell=s+1}^t \step_\ell \nlupdates_\ell \Big)}\\
        &\leq 2\varhet^2 (\step \nlupdates + a^{-1}L(1+(2a\step \nlupdates)^{\gamma+\gamma_\step \over 1-\gamma}))\nagent^{-1} = \Auxconst_{M,2} \nagent^{-1} \eqsp.
    \end{align*}
    {
    By Minkowski's inequality, the first two claims follow with
    \begin{align*}
        \Auxconst_{\hat\Sigma}
        &=\big(\Auxconst_{M,1}^{1/2}+\Auxconst_{M,2}^{1/2}\big)^2 \eqsp,\\
        \Auxconst_M
        &=\big(\Auxconst_{M,1}^{1/2}+\Auxconst_{M,2}^{1/2}+\Auxconst_{M,3}+\Auxconst_{M,4}\big)^2 \eqsp.
    \end{align*}
    In particular,
    \begin{align*}
        \PE[\norm{\hat M_t}^2]&\leq\Auxconst_{\hat\Sigma}\nagent^{-1} \eqsp,\\
        \PE[\norm{M_{t,3}+M_{t,4}}^2]
        &\leq(\Auxconst_{M,3}+\Auxconst_{M,4})^2\nagent^{-1}(1+t)^{-2\gamma} \eqsp.
    \end{align*}
    Finally, since $M_t=\hat M_t+M_{t,3}+M_{t,4}$, Cauchy--Schwarz gives
    \begin{align*}
        \norm{\Sigma_t-\hat\Sigma_t}
        &\leq2\PE^{1/2}[\norm{\hat M_t}^2]\PE^{1/2}[\norm{M_{t,3}+M_{t,4}}^2]
        +\PE[\norm{M_{t,3}+M_{t,4}}^2]\\
        &\leq2\Auxconst_{\hat\Sigma}^{1/2}(\Auxconst_{M,3}+\Auxconst_{M,4})\nagent^{-1}(1+t)^{-\gamma}
        +(\Auxconst_{M,3}+\Auxconst_{M,4})^2\nagent^{-1}(1+t)^{-2\gamma} \eqsp.
    \end{align*}
    Using $\gamma\geq0$, we obtain \eqref{eq:bound_Sigmat_hSigmatboot_diff}.
    }
\end{proof}

\begin{lemma}
\label{lemma:bound_Mtb}
{
    Let $p\geq2$. Under Assumptions \Cref{ass: linear-decay}($2p$), \Cref{assum:noise-level-flsa}, \Cref{ass: lr-polynomial-decay} and \Cref{ass: boot_weights}, assuming independence of the observation sequences across clients, $\gamma>0$ and $0<\step\leq\step_{\infty,2p}\wedge a/(2p\bConst{A}^2)$, we have
    \begin{align*}
        \PE^{1/p}[\norm{\Sigma_t^\boot-\hat\Sigma_t^\boot}^p]
        &\lesssim p^6\nagent^{-3/2}(1+t)^{-\gamma_\step/2}
        +p^4\nagent^{-1/2}(1+t)^{-\gamma}\\
        &\quad+p^8\nagent^{-2}(1+t)^{-\gamma_\step}+p^4(1+t)^{-2\gamma} \eqsp,
    \end{align*}
    where $\hat\Sigma_t^\boot=\PEb[M_{t,1}^\boot(M_{t,1}^\boot)^T]$. Moreover, for $\delta\in(0,\rme^{-2}]$, under the same assumptions with $p=\log(1/\delta)$, with probability at least $1-\delta$, we have
    \begin{align}
    \label{eq:bound_Sigmatboot_hSigmatboot_diff}
        \nonumber\norm{\Sigma_t^\boot-\hat\Sigma_t^\boot}
        &\lesssim\nagent^{-3/2}\log^6(1/\delta)(1+t)^{-\gamma_\step/2}
        +\nagent^{-1/2}\log^4(1/\delta)(1+t)^{-\gamma}\\
        &\quad+\nagent^{-2}\log^8(1/\delta)(1+t)^{-\gamma_\step}+\log^4(1/\delta)(1+t)^{-2\gamma} \eqsp.
    \end{align}
}
\end{lemma}
\begin{proof}
{
    We note that, by conditional Jensen's inequality and H\"older's inequality, for $p\geq2$,
    \begin{align*}
        \PE^{1/p}[\norm{\Sigma_t^\boot-\hat\Sigma_t^\boot}^p]
        &\leq 2\PE^{1/(2p)}[\norm{M_{t,1}^\boot}^{2p}]
        \PE^{1/(2p)}[\norm{M_{t,2}^\boot+M_{t,3}^\boot}^{2p}] + \PE^{1/p}[\norm{M_{t,2}^\boot+M_{t,3}^\boot}^{2p}] \eqsp.
    \end{align*}
    Now, applying Burkholder's inequality conditionally on all observations to the independent weights $w_{s,h}-1$, and then using independence across clients and \Cref{lemma: legendary-sum}, we get
    \begin{align*}
        \PE^{1/(2p)}[\norm{M_{t,1}^\boot}^{2p}]
        &\lesssim p^2\nagent^{-1/2}\step_t^{-1/2}
        \left(\sum_{s=1}^t\step_s^2\nlupdates_s
        \exp\left\{-2a\sum_{\ell=s+1}^t\step_\ell\nlupdates_\ell\right\}\right)^{1/2}
        \lesssim p^2\nagent^{-1/2} \eqsp.
    \end{align*}
    Similarly, we obtain
    \begin{align*}
        \PE^{1/(2p)}[\norm{M_{t,2}^\boot}^{2p}]
        &\lesssim p^4\nagent^{-1}\step_t^{-1/2}
        \left(\sum_{s=1}^t\step_s^3\nlupdates_s
        \exp\left\{-2a\sum_{\ell=s+1}^t\step_\ell\nlupdates_\ell\right\}\right)^{1/2}\\
        &\lesssim p^4\nagent^{-1}(1+t)^{-\gamma_\step/2} \eqsp,
    \end{align*}
    where, by \Cref{lemma:prod_gavg_gammaavg_diff} with moment order $2p$ and $\step_{s+1}\leq\step_s$, we used that
    \begin{align*}
        \PE^{1/(2p)}\left[\norm{\prod_{i=s+1}^t\Gavg_i-\prod_{i=s+1}^t\Gammaavg_i}^{2p}\right]
        \lesssim p^2\nagent^{-1/2}\step_s^{1/2}
        \exp\left\{-a\sum_{\ell=s+1}^t\step_\ell\nlupdates_\ell\right\} \eqsp.
    \end{align*}
    Again, applying Burkholder's inequality conditionally on all observations, Minkowski's inequality over clients, and \Cref{lemma: legendary-sum}, we obtain
    \begin{align*}
        \PE^{1/(2p)}[\norm{M_{t,3}^\boot}^{2p}]
        &\lesssim p^2\step_t^{-1/2}
        \left(\sum_{s=1}^t\step_s^4\nlupdates_s^3
        \exp\left\{-2a\sum_{\ell=s+1}^t\step_\ell\nlupdates_\ell\right\}\right)^{1/2}\\
        &\lesssim p^2(1+t)^{-\gamma} \eqsp,
    \end{align*}
    where we used $\PE^{1/(2p)}[\norm{\Id-\Gamma_{s,h+1:\nlupdates_s}^c}^{2p}]\lesssim\step_s\nlupdates_s$. Here the term involving $\bA[c](\thetalim-\thetalim[c])$ is not centered after multiplication by $\Id-\Gamma_{s,h+1:\nlupdates_s}^c$, so its sum over clients is bounded by Minkowski's inequality.
    Combining the above bounds and using Minkowski's inequality, we get
    \begin{align*}
        \PE^{1/p}[\norm{\Sigma_t^\boot-\hat\Sigma_t^\boot}^p]
        &\lesssim p^6\nagent^{-3/2}(1+t)^{-\gamma_\step/2}
        +p^4\nagent^{-1/2}(1+t)^{-\gamma}\\
        &\quad+p^8\nagent^{-2}(1+t)^{-\gamma_\step}+p^4(1+t)^{-2\gamma} \eqsp.
    \end{align*}
    To prove \eqref{eq:bound_Sigmatboot_hSigmatboot_diff}, we apply Markov's inequality with $p=\log(1/\delta)$.
}
\end{proof}

\begin{lemma}
\label{lemma: gamma-diff-with-copy}
    Assume \Cref{ass: linear-decay}(2), \Cref{assum:noise-level-flsa} and \Cref{ass: lr-polynomial-decay}. For any $k \neq c$ and $n \neq t$, we have $\Gamma_{n, \alpha:\beta}^k - \Gamma_{n, \alpha:\beta}^{k, (t, j, c)} = 0$. For the case $k=c$ and $n =t$, we obtain the bound
    \begin{align*}
        \PE^{1/2}[\norm{\Gamma_{t, \alpha:\beta}^c - \Gamma_{t, \alpha:\beta}^{c, (t, j, c)}}^2] \leq  4 \eta_t \mathrm{e}^{-a \eta_t (\beta - \alpha + 1)} \sqrt{\mathsf{Tr}( \noisecovA[c])},
    \end{align*}
    where $1 \le \alpha \le \beta \le H_n$.
\end{lemma}
\begin{proof}
    Decompose $\Gamma_{t, \alpha:\beta}^c - \Gamma_{t, \alpha:\beta}^{c, (t, j, c)}$ using \Cref{lemma: difference-of-product}
    \begin{align*}
        &\Gamma_{t, \alpha:\beta}^c - \Gamma_{t, \alpha:\beta}^{c, (t, j, c)} = \\
        &=\sum \limits_{i = \alpha}^{\beta} \left \{ \prod_{\ell = \alpha}^{i - 1} (\Id - \step_{t} \zfuncA[c]{Z_{t, \ell}^c}) \right \} \step_t \left( \zfuncA[c]{Z_{t, i, c}^{(t, j, c)}} - \zfuncA[c]{Z_{t, i, c}} \right) \left \{ \prod_{\ell = i + 1}^{\beta} (\Id - \eta_t \zfuncA[c]{Z_{t, \ell, c}^{(t,j,c)}}) \right \}
    \end{align*}
    For any $i \neq j$ the corresponding summand is zero. Therefore
    \begin{align*}
        \Gamma_{t, \alpha:\beta}^c - \Gamma_{t, \alpha:\beta}^{c, (t, j, c)} = \left \{ \prod_{\ell = \alpha}^{j - 1} (\Id - \step_{t} \zfuncA[c]{Z_{t, \ell}^c}) \right \} \step_t \left( \zfuncA[c]{Z_{t, j, c}^{(t, j, c)}} - \zfuncA[c]{Z_{t, j, c}} \right) \left \{ \prod_{\ell = j + 1}^{\beta} (\Id - \eta_t \zfuncA[c]{Z_{t, \ell, c}^{(t,j,c)}}) \right \}
    \end{align*}
    Using assumption \Cref{ass: linear-decay}, we get
    \begin{align*}
        \PE^{1/2}[\norm{ \Gamma_{t, \alpha:\beta}^c - \Gamma_{t, \alpha:\beta}^{c, (t, j, c)} }^2] \leq \step_t (1 - \step_t a)^{\beta - \alpha} \PE^{1/2}[\norm{ \zfuncA[c]{Z_{t, j, c}^{(t, j, c)}} - \zfuncA[c]{Z_{t, j, c}} }^2]
    \end{align*}
    As $\zfuncA[c]{Z_{t, j, c}^{(t, j, c)}} - \zfuncA[c]{Z_{t, j, c}} = \zmfuncA[c]{Z_{t, j, c}^{(t, j, c)}} - \zmfuncA[c]{Z_{t, j, c}}$ then, using Minkowski's inequality, we get
    \begin{align*}
        \PE^{1/2}[\norm{ \Gamma_{t, \alpha:\beta}^c - \Gamma_{t, \alpha:\beta}^{c, (t, j, c)} }^2] &\leq \frac{2}{1 - \eta_t a} \eta_t (1 - \step_t a)^{\beta - \alpha + 1} \PE^{1/2} [\norm{ \zmfuncA[c]{Z_{t, j, c}} }^2] \\
        &\leq 4 \eta_t \mathrm{e}^{-a\eta_t(\beta - \alpha + 1)} \sqrt{\mathrm{Tr}( \noisecovA[c] )}
    \end{align*}
    
 \end{proof}

 \begin{lemma}
    \label{lemma: bound_diff_avg_matrices_prod}
     Under assumptions \Cref{ass: linear-decay}(2), \Cref{assum:noise-level-flsa} and \Cref{ass: lr-polynomial-decay} with $\gamma_\step \geq \gamma_\nlupdates$, for any $s \in \{0, \dots, t-1\}$, we have
     \begin{align*}
         \PE[\norm{\prod_{i=s+1}^t \Gavg_i - \prod_{i=s+1}^t \Gammaavg_i}^2] \leq 2a^{-1} \varboundA^2 \nagent^{-1} \step_s  \exp{\left(-2a \sum_{\ell=s+1}^t \step_\ell \nlupdates_\ell\right)} \eqsp.
     \end{align*}
 \end{lemma}
 \begin{proof}
    Using \Cref{lemma: difference-of-product}, we also note that the sequence $\{\prod\limits_{j=s+1}^{i-1}\Gavg_j (\Gavg_i - \Gammaavg_i)\prod\limits_{j=i}^t \Gammaavg_j, s+1 \leq i \leq t\}$ is martingale-difference, thus
    \begin{align*}
        \PE[\norm{\prod_{i=s+1}^t \Gavg_i - \prod_{i=s+1}^t \Gammaavg_i}^2] &= \sum_{i=s+1}^t \PE[\norm{\prod_{j=s+1}^{i-1}\Gavg_j (\Gavg_i - \Gammaavg_i) \prod_{j=i}^t \Gammaavg_j}^2]\\
        &\leq \exp{\left( -2a\sum_{\ell=s+1}^t \step_\ell \nlupdates_\ell \right)} \sum_{i=s+1}^t \exp{(2a\step_i \nlupdates_i)} \PE[\norm{\Gavg_i - \Gammaavg_i}^2] \eqsp.
    \end{align*}
    Again, applying \Cref{lemma: difference-of-product} to the last term, we obtain
    \begin{align*}
        \PE[\norm{\Gavg_i - \Gammaavg_i}^2] &\leq \nagent^{-2}\step_i^2 \sum_{c=1}^\nagent \sum_{h=1}^{\nlupdates_i} \exp{(-2a \step_i (\nlupdates_i - h- 1))}\mathrm{Tr}(\noisecovA[c])\\
        &\leq \varboundA^2 \nagent^{-1}\step_i^2 \nlupdates_i \eqsp.
    \end{align*}
    Therefore, we get
    \begin{align*}
         \PE[\norm{\prod_{i=s+1}^t \Gavg_i - \prod_{i=s+1}^t \Gammaavg_i}^2] &\leq \rme \varboundA^2 \nagent^{-1} \exp{\left(-2a \sum_{\ell=s+1}^t \step_\ell \nlupdates_\ell \right)} \sum_{i=s+1}^t \step_i^2 \nlupdates_i\\
         &\leq 2a^{-1} \varboundA^2 \nagent^{-1} \step_s  \exp{\left(-2a \sum_{\ell=s+1}^t \step_\ell \nlupdates_\ell\right)} \eqsp,
    \end{align*}
    where we used the inequality $k\exp{(-2ck)} \leq (c\rme)^{-1} \exp{(-ck)}$.
 \end{proof}

 \begin{lemma}
    \label{lemma: bound_diff_delta_avg}
      Under assumptions \Cref{ass: linear-decay}(2), \Cref{assum:noise-level-flsa} and \Cref{ass: lr-polynomial-decay} with $\gamma_\step \geq \gamma_\nlupdates$, for any $s \in \{1, \dots, t\}, h \in \{1, \dots, \nlupdates_s\}, c \in \{1,\dots, \nagent\}$, we have
      \begin{align*}
          \PE[\norm{\Deltaavg_s - \Deltaavgpert{s,h,c}_s}^2] \leq  16 \nagent^{-2} \step_s^2 \mathrm{Tr}(\noisecovA[c])\norm{\thetalim[c] - \thetalim}^2 \eqsp.
      \end{align*}
 \end{lemma}
 \begin{proof}
     Note that
     \begin{align*}
         \PE[\norm{\Deltaavg_s - \Deltaavgpert{s,h,c}_s}^2]  &= \nagent^{-2}\PE[\norm{(\Gamma_s^{c,(s,h,c)} - \Gamma_s^c)(\thetalim[c] - \thetalim)}^2] \\
         &\leq 16 \nagent^{-2} \step_s^2 \mathrm{Tr}(\noisecovA[c])\norm{\thetalim[c] - \thetalim}^2 \eqsp.
     \end{align*}
     where we have applied \Cref{lemma: gamma-diff-with-copy}.
 \end{proof}

\begin{lemma}
    \label{lemma: bound_diff_avg_gamma_prod}
    Under assumptions \Cref{ass: linear-decay}(2), \Cref{assum:noise-level-flsa} and \Cref{ass: lr-polynomial-decay} with $\gamma_\step \geq \gamma_\nlupdates$, for any $s \in \{1, \dots, t\}, h \in \{1, \dots, \nlupdates_s\}, c \in \{1,\dots, \nagent\}$ and $s' < s$, we have
    \begin{align*}
        \PE[\norm{\prod_{i=s'+1}^t \Gammaavg_i - \prod_{i=s'+1}^t \Gammaavgpert{s,h,c}_i}^2] \leq 4 \rme \nagent^{-2}\step_s^2 \exp{\left( -2a\sum_{\ell=s'+1}^t \step_\ell \nlupdates_\ell \right)} \mathrm{Tr}(\noisecovA[c]) \eqsp.
    \end{align*}
\end{lemma}
\begin{proof}
    Using Minkowski's inequality and  \Cref{lemma: difference-of-product}, we note that
    \begin{align*}
        \PE[\norm{\prod_{i=s'+1}^t \Gammaavg_i - \prod_{i=s'+1}^t \Gammaavgpert{s,h,c}_i}^2] \leq \PE[\norm{\prod_{i=s'+1}^{s-1}\Gammaavg_i (\Gammaavg_s - \Gammaavgpert{s,h,c}_s)\prod_{i=s'+1}^t \Gammaavgpert{s,h,c}_i}^2] \eqsp,
    \end{align*}
    since $\Gammaavg_i - \Gammaavgpert{s,h,c}_i = 0$ for $i \neq s$. Thus, we get
    \begin{align*}
        \PE[\norm{\prod_{i=s'+1}^t \Gammaavg_i - \prod_{i=s'+1}^t \Gammaavgpert{s,h,c}_i}^2] &\leq \rme  \PE[\norm{\Gammaavg_s - \Gammaavgpert{s,h,c}_s}^2] \exp{\left( -2a\sum_{\ell=s'+1}^t \step_\ell \nlupdates_\ell \right)} \\
        &\leq 4 \rme \nagent^{-2}\step_s^2 \exp{\left( -2a\sum_{\ell=s'+1}^t \step_\ell \nlupdates_\ell \right)} \mathrm{Tr}(\noisecovA[c]) \eqsp.
    \end{align*}
    where we have used that
    \begin{align*}
        \PE[\norm{\Gammaavg_s - \Gammaavgpert{s,h,c}_s}^2] = \nagent^{-2}\PE[\norm{\Gamma_{s,\nlupdates_s}^c - \Gamma_{s,\nlupdates_s}^{c,(s,h,c)}}^2] \leq 4\nagent^{-2}\step_s^2 \mathrm{Tr}(\noisecovA[c]) \eqsp.
    \end{align*}
\end{proof}

{
\begin{lemma}
\label{lemma:prod_gavg_gammaavg_diff}
Let $p\geq2$. Under \Cref{ass: linear-decay}($p$), \Cref{assum:noise-level-flsa} and \Cref{ass: lr-polynomial-decay}, assume that the observation sequences of distinct clients are independent and $0<\step\leq\step_{\infty,p}\wedge(4a)^{-1}$. Then, for $0\leq s<t$,
\begin{align*}
    \PE^{1/p}\left[\norm{\prod_{i=s+1}^t\Gavg_i-\prod_{i=s+1}^t\Gammaavg_i}^p\right]
    \leq 2 \bConst{A}(a\rme)^{-1/2}d^{1/2}p^2\nagent^{-1/2}\step_{s+1}^{1/2}\exp{\left(-a \sum_{\ell=s+1}^t \step_\ell \nlupdates_\ell \right)} \eqsp.
\end{align*}
\end{lemma}
\begin{proof}
{
Applying \Cref{lemma: difference-of-product}, we have
\begin{align*}
    \Gavg_i-\Gammaavg_i
    =\frac{\step_i}{\nagent}\sum_{c=1}^{\nagent}\sum_{h=1}^{\nlupdates_i}
    G_{i,1:h-1}^c\zmfuncA[c]{Z_{i,h,c}}\Gamma_{i,h+1:\nlupdates_i}^c \eqsp.
\end{align*}
The summands are martingale differences when the observations are revealed client by client, in reverse local order. Applying Burkholder's inequality to the sum over $(c,h)$, followed by Minkowski's inequality, we obtain
\begin{align*}
    \PE^{1/p}[\norm{\Gavg_i-\Gammaavg_i}^p]
    \leq \bConst{A}d^{1/2}p\nagent^{-1/2}\step_i\nlupdates_i^{1/2} \eqsp.
\end{align*}
Again, applying \Cref{lemma: difference-of-product}, we get
\begin{align*}
    \prod_{i=s+1}^t\Gavg_i-\prod_{i=s+1}^t\Gammaavg_i
    =\sum_{i=s+1}^t\left(\prod_{j=s+1}^{i-1}\Gavg_j\right)
    (\Gavg_i-\Gammaavg_i)\left(\prod_{j=i+1}^t\Gammaavg_j\right) \eqsp.
\end{align*}
Since $\PE[\Gammaavg_i]=\Gavg_i$, these summands are martingale differences in reverse round order. Thus, applying Burkholder's inequality, we obtain
\begin{align*}
    \PE^{1/p}\left[\norm{\prod_{i=s+1}^t\Gavg_i-\prod_{i=s+1}^t\Gammaavg_i}^p\right]
    &\leq 2\bConst{A}d^{1/2}p^2\nagent^{-1/2} \left(\sum_{i=s+1}^t\step_i^2\nlupdates_i\right)^{1/2} \exp{\left(-2a \sum_{\ell=s+1}^t \step_\ell \nlupdates_\ell \right)}\\
    &\leq 2 \bConst{A}(a\rme)^{-1/2}d^{1/2}p^2\nagent^{-1/2}\step_{s+1}^{1/2}\exp{\left(-a \sum_{\ell=s+1}^t \step_\ell \nlupdates_\ell \right)} \eqsp,
\end{align*}
where we used $\step_i\leq\step_{s+1}$ and $x^{1/2} \exp(-ax)\leq(2a\rme)^{-1/2}$ for $x \geq 0$.
}
\end{proof}
}

\begin{lemma}
\label{lemma:bound_gamma_gammaboot_diff}
Let $p\geq2$. Under \Cref{ass: linear-decay}($p$), \Cref{assum:noise-level-flsa}, \Cref{ass: lr-polynomial-decay} and \Cref{ass: boot_weights}, for every $c\in[\nagent]$ and $1\leq i\leq t$, we have
\begin{align*}
    \PE^{1/p}[\norm{\Gamma_i^c-\Gamma_i^{\boot,c}}^p]
    \leq d^{1/2}\bConst{A}W_{\text{max}}p\step_i\nlupdates_i^{1/2} \eqsp.
\end{align*}
Here the expectation is taken over both observations and bootstrap weights.
\end{lemma}
\begin{proof}
By \Cref{lemma: difference-of-product},
\begin{align*}
    \Gamma_i^c-\Gamma_i^{\boot,c}
    =\step_i\sum_{h=1}^{\nlupdates_i}(w_{i,h}-1)
    \Gamma_{i,1:h-1}^c\zfuncA[c]{Z_{i,h,c}}
    \Gamma_{i,h+1:\nlupdates_i}^{\boot,c} \eqsp.
\end{align*}
Conditionally on all observations, the summands are martingale differences when the weights are revealed in the order $w_{i,\nlupdates_i},\ldots,w_{i,1}$. Applying Burkholder's inequality, we obtain
\begin{align*}
    \PE^{1/p}\left[\norm{\Gamma_{i,1:h-1}^c
    \zfuncA[c]{Z_{i,h,c}}\Gamma_{i,h+1:\nlupdates_i}^{\boot,c}}^p\right]
    \leq d^{1/2}\bConst{A} \eqsp.
\end{align*}
Using $|w_{i,h}-1|\leq W_{\text{max}}$, we obtain
\begin{align*}
    \PE^{1/p}[\norm{\Gamma_i^c-\Gamma_i^{\boot,c}}^p]
    &\leq \PE^{1/p}[\norm{\Gamma_i^c-\Gamma_i^{\boot,c}}^p] \\
    &\leq p\step_i\left(\sum_{h=1}^{\nlupdates_i}
    \PE^{2/p}\left[\norm{(w_{i,h}-1)\Gamma_{i,1:h-1}^c
    \zfuncA[c]{Z_{i,h,c}}\Gamma_{i,h+1:\nlupdates_i}^{\boot,c}}^p\right]\right)^{1/2} \\
    &\leq d^{1/2}\bConst{A}W_{\text{max}}p\step_i\nlupdates_i^{1/2} \eqsp.
\end{align*}
\end{proof}

\begin{lemma}
\label{lemma:bound_gammaavg_gammaavgboot_diff}

Let $p\geq2$. Under \Cref{ass: linear-decay}($p$), \Cref{assum:noise-level-flsa}, \Cref{ass: lr-polynomial-decay} and \Cref{ass: boot_weights}, for any $1\leq i\leq t$, we have
\begin{align*}
    \PE^{1/p}[\norm{\Gammaavg_i-\Gammaavgboot_i}^p]
    \leq d^{1/2}\bConst{A}W_{\text{max}}p\step_i\nlupdates_i^{1/2} \eqsp.
\end{align*}
\end{lemma}
\begin{proof}

Applying Minkowski's inequality over clients and \Cref{lemma:bound_gamma_gammaboot_diff}, we get
\begin{align*}
    \PE^{1/p}[\norm{\Gammaavg_i-\Gammaavgboot_i}^p]
    &\leq \nagent^{-1}\sum_{c=1}^\nagent
    \PE^{1/p}[\norm{\Gamma_i^c-\Gamma_i^{\boot,c}}^p] \\
    &\leq d^{1/2}\bConst{A}W_{\text{max}}p\step_i\nlupdates_i^{1/2} \eqsp.
\end{align*}
\end{proof}

\begin{lemma}
    \label{lemma:prod_gammaavg_gammaavgboot_diff}
Let $p\geq2$ and $\step\leq[4a(p-1)]^{-1}$. Under \Cref{ass: linear-decay}($p$), \Cref{assum:noise-level-flsa}, \Cref{ass: lr-polynomial-decay} and \Cref{ass: boot_weights}, for $0\leq s<t$, we have
\begin{align*}
    \PE^{1/p}\left[\norm{\prod_{i=s+1}^t\Gammaavg_i-\prod_{i=s+1}^t\Gammaavgboot_i}^p\right]
    \leq \rme d^{1/2}a^{-1/2}\bConst{A}W_{\max}p^2
    \step_{s+1}^{1/2}
    \exp{\left(-\frac{a}{2}\sum_{\ell=s+1}^t\step_\ell\nlupdates_\ell\right)} \eqsp.
\end{align*}
\end{lemma}
\begin{proof}
By \Cref{lemma: difference-of-product},
\begin{align*}
    \prod_{i=s+1}^t\Gammaavg_i-\prod_{i=s+1}^t\Gammaavgboot_i
    =\sum_{i=s+1}^t\left(\prod_{j=s+1}^{i-1}\Gammaavg_j\right)
    (\Gammaavg_i-\Gammaavgboot_i)\left(\prod_{j=i+1}^t\Gammaavgboot_j\right) \eqsp.
\end{align*}
Conditionally on all observations, $\PEb[\Gammaavgboot_i]=\Gammaavg_i$. Hence the summands are martingale differences when the weights are revealed in reverse round order. Applying Burkholder's inequality together with \Cref{lemma:bound_gammaavg_gammaavgboot_diff}, we obtain
\begin{align*}
    \PE^{1/p}\left[\norm{\prod_{i=s+1}^t\Gammaavg_i-\prod_{i=s+1}^t\Gammaavgboot_i}^p\right]
    &\leq p\left(\sum_{i=s+1}^t
    \PE^{2/p}\left[\left\|\left(\prod_{j=s+1}^{i-1}\Gammaavg_j\right)
    (\Gammaavg_i-\Gammaavgboot_i)\left(\prod_{j=i+1}^t\Gammaavgboot_j\right)\right\|^p\right]\right)^{1/2} \\
    &\leq d^{1/2}\bConst{A}W_{\max}p^2
    \exp{\left(-a\sum_{\ell=s+1}^t\step_\ell\nlupdates_\ell\right)} \left(\sum_{i=s+1}^t\step_i^2\nlupdates_i
    \exp{(2a\step_i\nlupdates_i)}\right)^{1/2} \\
    &\leq \rme^{3/2}d^{1/2}\bConst{A}W_{\max}p^2
    \step_{s+1}^{1/2}
    \exp{\left(-a\sum_{\ell=s+1}^t\step_\ell\nlupdates_\ell\right)}
    \left(\sum_{i=s+1}^t\step_i\nlupdates_i\right)^{1/2} \\
    &\leq \rme d^{1/2}a^{-1/2}\bConst{A}W_{\max}p^2
    \step_{s+1}^{1/2}
    \exp{\left(-\frac{a}{2}\sum_{\ell=s+1}^t\step_\ell\nlupdates_\ell\right)} \eqsp.
\end{align*}
Here $W_{\max}>1$, $a\step\nlupdates<1$ and $a\step W_{\max}<1/2$ imply $a\step_i\nlupdates_i\leq a\step(\nlupdates+1)<3/2$; we also used $\step_i\leq\step_{s+1}$ for $i\geq s+1$ and $x\exp(-cx)\leq(c\rme)^{-1}$ for $x\geq0$, $c>0$.
\end{proof}

Below we state the Gaussian comparison inequality due to \cite{Devroye2018}, see also \cite{BarUly86}.
\begin{lemma}
\label{lemma:gauss_comparison}
    Let $\Sigma_1$ and $\Sigma_2$ be positive definite covariance matrices in $\rset^{d\times d}$. Let $X \sim \gauss(0,\Sigma_1)$ and $Y \sim \gauss(0, \Sigma_2)$. Then
    \begin{align*}
        \mathsf{d}_{\mathsf{TV}}(X, Y) \leq {3 \over 2}\| \Sigma_2^{-1/2} \Sigma_1 \Sigma_2^{-1/2} - \Id \|_{\mathsf{F}} \eqsp.
    \end{align*}
\end{lemma}

%% file: appendix/technical_lemmas.tex
Here we present some useful technical lemmas we use during the proofs of our main results.

\begin{lemma}
\label{lemma: difference-of-product}
    For any matrix-valued sequences $(A_i)_{i = 1}^N$ and $(B_i)_{i = 1}^N$ it holds that:
    \begin{align*}
        \prod_{i = 1}^N A_i - \prod_{i = 1}^N B_i = \sum \limits_{i = 1}^N \Big \{ \prod_{j = 1}^{i - 1} A_j  \Big \} (A_i - B_i) \Big \{ \prod_{j = i + 1}^N B_j \Big \}
    \end{align*}
\end{lemma}

\begin{lemma}
\label{lemma: lr_sum_estimation}
Assume $\eta_t = c_0 (t_0 + t)^{-\gamma}$ for some $\gamma \in [1/2, 1)$. Then:
\begin{align*}
    \sum \limits_{i = s + 1}^t \eta_i \leq \frac{\step}{1 - \gamma} \left( (1 + t)^{1 - \gamma} - (1 + s)^{1 - \gamma}  \right)
\end{align*}
\end{lemma}
\begin{proof}
    By definition of $\eta_t$:
    \begin{align*}
        \sum \limits_{i = s + 1}^t \eta_i = \sum \limits_{i = s + 1}^t \step (1 + i)^{-\gamma} = \step \sum \limits_{k = s + 2}^{1 + t} k^{-\gamma}
    \end{align*}
    As $f(x) = x^{-\gamma}$ is a decreasing function for $x \geq 0$, then:
    \begin{align*}
        c_0 \sum \limits_{k = s + 2}^{1 + t} k^{-\gamma} \leq \step \int_{1 + s}^{1 + t} x^{-\gamma} \mathrm{d} x = \frac{\step}{1 - \gamma} ((1 + t)^{1 - \gamma} - (1 + s)^{1 - \gamma}) \eqsp.
    \end{align*}
\end{proof}

\begin{lemma}
\label{lemma: tech-sum-1}
    Assume $\gamma \in (0, 1)$ and $1 \leq \ell \leq r \in \mathbb{N}$. Then, 
    \begin{align*}
        \frac{(r + 2)^{1 - \gamma} - (\ell + 1)^{1 - \gamma}}{1 - \gamma} \leq \sum \limits_{k = \ell}^{r} (1 + k)^{-\gamma} \leq \frac{(r + 1)^{1 - \gamma} - \ell^{1 - \gamma}}{1 - \gamma} \eqsp.
    \end{align*}
\end{lemma}
\begin{proof}
    As $f(x) = (1 + x)^{-\gamma}$ is a decreasing function, then
    \begin{align*}
        \sum \limits_{k = \ell}^{r} (1 + k)^{-\gamma} \leq \int_{\ell - 1}^{r} (1 + x)^{-\gamma} \mathrm{d}x = \int_{\ell}^{r + 1} x^{-\gamma} \mathrm{d} x = \frac{(r + 1)^{1 - \gamma} - \ell^{1 - \gamma}}{1 - \gamma} \eqsp .
    \end{align*}
    Similarly, we can give a lower bound of our function as,
    \begin{align*}
        \sum \limits_{k = \ell}^{r} (1 + k)^{-\gamma} \geq \int_{\ell}^{r + 1} (1 + x)^{-\gamma} \mathrm{d} x = \int_{\ell + 1}^{r + 2} (1 + x)^{-\gamma} \mathrm{d} x = \frac{(r + 2)^{1 - \gamma} - (\ell + 1)^{1 - \gamma}}{1 - \gamma}
    \end{align*}
\end{proof}

\begin{lemma}
\label{lemma: legendary-sum}
    For any $t \geq 1$ and $\beta < 1$, it holds that
    \begin{align*}
        \sum_{s=1}^{t-1} s^{-\alpha}\exp{\{-u\sum_{\ell=s+1}^t (1+\ell)^{-\beta}\}} \leq L u^{-1}(1+u^{-{\alpha \over 1-\beta}}) t^{\beta - \alpha} \eqsp,
    \end{align*}
    where $L$ is defined in \eqref{eq:L_def_legendary_sum}.
\end{lemma}
\begin{proof}
    We set $m = \lceil {t\over 2} \rceil$. To estimate
    \begin{align*}
        S_t := \sum \limits_{s = 1}^{t-1} s^{-\alpha} \exp{\left( -u \sum \limits_{\ell = s + 1}^t (1+\ell)^{-\beta} \right)} \eqsp,
    \end{align*}
    we split this sum into two parts:
    \begin{align*}
        S_t = \underbrace{\sum \limits_{s = 1}^{m} s^{-\alpha} \exp{\left( -u \sum \limits_{\ell = s + 1}^t (1 + \ell)^{-\beta} \right)}}_\text{$S_{t, 1}$} + \underbrace{\sum \limits_{s = m + 1}^{t} s^{-\alpha} \exp{ \left( -u \sum \limits_{\ell = s + 1}^t (1 + \ell)^{-\beta} \right)}}_\text{$S_{t, 2}$} \eqsp.
    \end{align*}
    \textbf{Estimating $S_{t, 1}$}. Here, we note that for $s \leq m$:
    \begin{align*}
        \sum \limits_{\ell = s + 1}^t (1 + \ell)^{-\beta} \geq \sum \limits_{\ell = m + 1}^{t} (1 + \ell)^{-\beta} \geq \frac{t}{2}t^{-\beta} = \frac{1}{2}t^{1 - \beta} \eqsp.
    \end{align*}
    Then
    \begin{align*}
        S_{t, 1} \leq \exp{\left(-\frac{u}{2} t^{1 - \beta} \right)} \sum \limits_{s = 1}^{m} s^{-\alpha} \leq t \exp{\left(-\frac{u}{2} t^{1 - \beta} \right)}\eqsp.
    \end{align*}
    We can see that the function $g(x) = x^{1+\alpha-\beta}\exp{(-v x^{1-\beta})}$ for $x \geq 1$ can be uniformly bounded by the value
    \begin{align*}
        L_0(v) = \left( {1+\alpha -\beta \over v \rme (1-\beta)} \right)^{1+\alpha-\beta \over 1-\beta} &\leq \rme^{-1} \Big(1+{\alpha \over 1-\beta} \Big)^{1+{\alpha \over 1-\beta}} v^{-\left(1+{\alpha \over 1-\beta} \right)}\\
        &\leq L_1 v^{-\left(1+{\alpha \over 1-\beta}\right)} \eqsp,
    \end{align*}
    where we have set $L_1 = \rme^{-1} \Big(1+{\alpha \over 1-\beta} \Big)^{1+{\alpha \over 1-\beta}}$. Therefore, we obtain the bound
    \begin{align*}
        S_{t, 1} \leq L_0(u/2) t^{\beta-\alpha}  \eqsp.
    \end{align*}

    \textbf{Estimating $S_{t, 2}$.} Here, we know that $s^{-\alpha} \leq (t/2)^{-\alpha}$. That is, we get
    \begin{align*}
        S_{t, 2} \leq 2^\alpha t^{-\alpha} \sum_{s=m+1}^{t-1} \exp{(-u(t-s)(1+t)^{-\beta})} &\leq 2^\alpha t^{-\alpha} e^{-u(1+t)^{-\beta}}\left(1-e^{-u(1+t)^{-\beta}} \right)^{-1} \\
        &\leq 2^{\alpha + \beta}u^{-1} t^{\beta-\alpha} \eqsp,
    \end{align*}
    where we have used that ${1 \over 1-e^{-x}} \leq {a \over 1-e^{-a}}x^{-1}$ for any $x \leq a$, that is, in our case $a = 2^{-\beta}u$. Finally, setting
    \begin{align}
    \label{eq:L_def_legendary_sum}
        L = 2^{\alpha+\beta} + 2^{1+2\alpha}L_1 \eqsp,
    \end{align}
    we obtain the result.
\end{proof}